\documentclass[]{bytedance_seed}

\usepackage[utf8]{inputenc} 
\usepackage[T1]{fontenc}    
\usepackage{hyperref}       
\usepackage{url}            
\usepackage{booktabs}       %

\setcitestyle{numbers,square}

\usepackage{blindtext}

\usepackage{algorithm}
\usepackage{algpseudocode}
\usepackage{longtable}

\usepackage{tabularray}

\usepackage{pgfplots}
\pgfplotsset{compat=newest}

\usepackage{subcaption}
\usepackage{pstricks}

\allowdisplaybreaks

\usepackage{tcolorbox}

\newtcolorbox{takeawaybox}{
    colback=blue!8,
    colframe=black,
    boxrule=0.6pt,
    arc=4pt,
    left=6pt,
    right=6pt,
    top=6pt,
    bottom=6pt
}

\usepackage{amssymb}

\usepackage{mathtools}
\usepackage{commath}
\usepackage{relsize}
\usepackage[font={small}]{caption}
\usepackage{comment,soul}

\usepackage[thinlines]{easytable}

\usepackage{multirow}
\usepackage[title]{appendix}
\usepackage{cleveref}

\usepackage{amsmath,amsfonts,bm}
\usepackage{amsthm}

\theoremstyle{plain}
\newtheorem{theorem}{Theorem}[section]

\newtheorem{condition}{Condition}[section]
\newtheorem{claim}{Claim}[section]
\newtheorem{principle}{Principle}[section]

\theoremstyle{definition}

\newtheorem{assumption}[theorem]{Assumption}

\theoremstyle{remark}

\def\eqref#1{equation~\ref{#1}}

\def\1{\bm{1}}

\def\vzero{{\bm{0}}}

\def\vepsilon{{\bm{\epsilon}}}
\def\va{{\bm{a}}}
\def\vb{{\bm{b}}}

\def\vh{{\bm{h}}}

\def\vm{{\bm{m}}}

\def\vu{{\bm{u}}}
\def\vv{{\bm{v}}}

\def\vx{{\bm{x}}}
\def\vy{{\bm{y}}}
\def\vz{{\bm{z}}}

\def\mA{{\bm{A}}}
\def\mB{{\bm{B}}}

\def\mG{{\bm{G}}}

\def\mI{{\bm{I}}}

\def\mL{{\bm{L}}}
\def\mM{{\bm{M}}}

\def\mQ{{\bm{Q}}}
\def\mR{{\bm{R}}}

\def\mU{{\bm{U}}}
\def\mV{{\bm{V}}}
\def\mW{{\bm{W}}}

\def\mSigma{{\bm{\Sigma}}}

\DeclareMathAlphabet{\mathsfit}{\encodingdefault}{\sfdefault}{m}{sl}
\SetMathAlphabet{\mathsfit}{bold}{\encodingdefault}{\sfdefault}{bx}{n}

\newcommand{\R}{\mathbb{R}}

\newcommand{\normrms}{\mathrm{R}}

\newcommand{\nin}{n_{\mathrm{in}}}
\newcommand{\nout}{n_{\mathrm{out}}}

\title{Spectral Condition for $\mu$P under Width–Depth Scaling}

\author[1\ddagger]{Chenyu Zheng}
\author[1]{Rongzhen Wang}
\author[2]{Xinyu Zhang}
\author[1\dagger]{Chongxuan Li}

\affiliation[1]{Gaoling School of AI, Renmin University of China}
\affiliation[2]{ByteDance Seed}

\contribution[\ddagger]{Work done during an internship at ByteDance Seed}
\contribution[\dagger]{Correspondence to Chongxuan Li.}

\abstract{
Generative foundation models are increasingly scaled in both width and depth, posing significant challenges for stable feature learning and reliable hyperparameter (HP) transfer across model sizes.
While maximal update parameterization~($\mu$P) has provided a principled solution to both problems for width scaling, existing extensions to the joint width–depth scaling regime remain fragmented, architecture- and optimizer-specific, and often rely on technically involved theories.
In this work, we develop a simple and unified spectral framework for $\mu$P under joint width–depth scaling. 
For deep residual networks whose residual blocks contain $k$ transformations, the framework specifies how the norms of weights and their per-step updates should scale with width and depth.
It reveals a fundamental transition from $k=1$ to $k\geq 2$, {unifying previously disparate $\mu$P formulations and identifying the CompleteP-style parameterization ($k\geq 2$)} as more appropriate for practical architectures with multi-transformation branches such as Transformers.
Building on this framework, we derive a general recipe for implementing $\mu$P across a broad class of optimizers by mapping spectral constraints to {concrete, optimizer-specific HP parameterizations}, recovering existing results and extending them to additional optimizers. {Finally, experiments on GPT-2 style language models show that the CompleteP-style parameterization ($k\geq 2$) achieves stable feature learning and robust HP transfer under width–depth scaling, whereas standard parameterization and the Depth-$\mu$P-style parameterization ($k=1$) often fail to do so.} These results support the practical effectiveness of the proposed spectral framework.
}

\date{Feb 28, 2026 (v1); May 11, 2026 (v2); Oct 30, 2026 (v3)}

\checkdata[Project Page]{\url{https://github.com/ML-GSAI/Width-Depth-muP}}

\begin{document}

\maketitle

\setcounter{tocdepth}{-1}

\section{Introduction}
\label{sec: intro}

Generative foundation models have been rapidly scaling in \emph{both width and depth}~\citep{scalinglaw-kaplan,scalinglaw-Chinchilla,liu2024deepseek,gpt-5,yang2025qwen3}, and this trend is expected to continue in the foreseeable future as datasets grow and task complexity increases. However, when model sizes become sufficiently large (e.g., billions of parameters), feature learning dynamics often become unstable or degenerate~\citep{DBLP:conf/iclr/SchoenholzGGS17-deepinfo,DBLP:conf/nips/JacotHG18-NTK}, and the hyperparameter (HP) tuning becomes prohibitively expensive~\citep{TP5}. These issues pose fundamental obstacles to efficient scaling, underscoring the need for principled methods enabling stable feature learning and reliable HP transfer from small models to larger ones.

Maximal update parameterization ($\mu$P)~\citep{TP4,TP5} was originally proposed to address both challenges for width scaling, and {has recently been preliminarily extended to settings that jointly scale width and depth~\citep{TP-6,DBLP:conf/iclr/BordelonNLHP24-dmft-depth,DBLP:conf/nips/BordelonCP24-transformer,chizat2025hidden,completep,muon-cp}}. By appropriately reparameterizing HPs with model size, $\mu$P aims to preserve scale-invariant feature learning, while maximizing the feature change induced by parameter updates, leading to stable and efficient training dynamics~\citep{TP4,completep}. Moreover, $\mu$P empirically stabilizes optimal HPs across different scales, enabling direct transfer of HPs tuned on small models to much larger ones~\citep{TP5,dit-mup}.

However, in the joint width–depth scaling regime, existing $\mu$P formulations remain preliminary. They are often tightly coupled to specific architectural choices, such as the internal depth of residual blocks~\citep{TP-6,DBLP:conf/iclr/BordelonNLHP24-dmft-depth,DBLP:conf/nips/BordelonCP24-transformer,completep} and 
{ particular optimization algorithms~\citep{TP-6,DBLP:conf/iclr/BordelonNLHP24-dmft-depth,chizat2025hidden,completep,muon-cp}}. Moreover, their derivations typically rely on technically involved tools such as Tensor Programs~\citep{TP4b,TP-6} or dynamical mean-field theory~\citep{DBLP:conf/nips/BordelonCP24-transformer,DBLP:conf/iclr/BordelonNLHP24-dmft-depth}.
Consequently, it remains difficult for the community to both systematically understand existing results and extend the $\mu$P principle to new architectures and optimizers, highlighting the need for a simple and unified theoretical framework.

To address the challenges outlined above, we draw inspiration from the unified spectral perspective developed for width-scaling $\mu$P~\citep{mup-spectral}.
{We extend this spectral perspective to the joint width–depth scaling regime, yielding a simple, unified, optimizer-independent spectral framework for realizing the $\mu$P principle in deep residual networks and systematically deriving optimizer-specific width-depth $\mu$P parameterizations across a broad class of optimizers.}
Our main contributions are summarized as follows.

First, we introduce a unified spectral scaling framework for realizing the $\mu$P principle in deep residual networks with any fixed $k$-layer residual blocks under width-depth scaling. {It specifies, independently of the optimizer, how the RMS operator norms of weights and per-step updates should scale with model size.}
Across architectures with different fixed residual-block depth $k$s, the spectral condition changes fundamentally from $k=1$ (Condition~\ref{condition: one-layer}) to $k=2$ (Condition~\ref{condition: scale-invariant fl}) due to the emergence of higher-order update terms, but yields no essential further change in the resulting $\mu$P formulation for $k\geq 2$.
{{Besides, Condition~\ref{condition: one-layer} recovers Depth-$\mu$P-style~\citep{TP-6,DBLP:conf/iclr/BordelonNLHP24-dmft-depth} and MLU-style results~\citep{chizat2025hidden}}}, while Condition~\ref{condition: scale-invariant fl} recovers CompleteP-style results~\citep{DBLP:conf/nips/BordelonCP24-transformer,completep,muon-cp}, which our theory suggests are more appropriate for practical architectures with multi-layer residual branches such as Transformers.
Notably, our analysis uses only elementary linear algebra and probability, making it easier to follow than previous works.

{Second, building on the proposed spectral condition, we present a unified recipe for implementing $\mu$P across a broad class of optimizers by mapping the spectral constraints to concrete, optimizer-specific HP parameterizations.}
Concretely, we systematically derive {width-depth $\mu$P parameterizations} for Muon-Kimi~\citep{muon-kimi}, Muon~\citep{jordan6muon}, Shampoo~\citep{gupta2018shampoo}, SOAP~\citep{vyas2024soap}, AdamW~\citep{adamw}, Sophia~\citep{DBLP:conf/iclr/Liu0HL024-sophia}, Lion~\citep{DBLP:conf/nips/ChenLHRW0DLHLL23-lion}, SGD, and Spectral Sphere Optimizer (SSO)~\citep{xie2026controlled-sso}.
These parameterizations are derived from the optimizers' update rules rather than ad hoc tuning heuristics; they also recover important existing {width-depth $\mu$P parameterizations}~\citep{TP-6,DBLP:conf/iclr/BordelonNLHP24-dmft-depth,DBLP:conf/nips/BordelonCP24-transformer,completep,muon-cp} as special cases.
{Together, the spectral framework and implementation recipe provide a common, elementary route for organizing previously fragmented results and extending the analysis to additional optimizers.}

Finally, through controlled experiments on GPT-2 style Transformer language models~\citep{gpt-2,nanogpt}, we empirically demonstrate that the {CompleteP-style parameterization from Condition~\ref{condition: scale-invariant fl} ($k\geq2$)} achieves stable feature learning and robust HP transfer under joint width-depth scaling. In contrast, standard parameterization (SP) and the {Depth-$\mu$P-style parameterization from Condition~\ref{condition: one-layer} ($k=1$)} are often unstable or fail to transfer HPs reliably. Together, these results support the practical effectiveness of the proposed spectral framework in realistic pretraining settings.

\section{Preliminaries}
\label{sec: prelimi}

We begin by establishing the necessary background for mathematical techniques and $\mu$P.
Additional related work is discussed in Appendix~\ref{app: related work}.

\subsection{Mathematical Notations and Properties}

{
\paragraph{Scaling notation.}
Let $s$ index the asymptotic scaling regime and let $a_s>0$ be a deterministic scale.
For a nonnegative random quantity $X_s$, we write $X_s=\mathcal{O}_p(a_s)$ if, for every $\epsilon>0$, there exists $C_\epsilon<\infty$ such that
$\Pr[X_s\le C_\epsilon a_s]\ge 1-\epsilon$ eventually;
$X_s=\Omega_p(a_s)$ if, for every $\epsilon>0$, there exists $c_\epsilon>0$ such that
$\Pr[X_s\ge c_\epsilon a_s]\ge 1-\epsilon$ eventually; and
$X_s=\Theta_p(a_s)$ if both relations hold.
Throughout the paper, the unadorned $\mathcal{O}$, $\Omega$, and $\Theta$ notation has this probabilistic meaning when applied to nonnegative random quantities associated with random weights, features, or updates, and the ordinary deterministic meaning otherwise.
In width-scaling settings, $s\to\infty$ means that width tends to infinity; in joint width-depth settings, it means that the width and depth along the scaling sequence both tend to infinity.
Fixed quantities, such as data dimensions, are absorbed into the constants.
}

\paragraph{Vector and matrix norms.}
We define $[n] = \{1, 2, \dots, n\}$.
For a vector $\va \in \R^n$, we use $\Vert \va \Vert_2$ and $\Vert \va \Vert_{\normrms}$ to denote its $\ell_2$ norm and Root Mean Square (RMS) norm, respectively. By definition, we have $\Vert \va \Vert_{\normrms} = \Vert \va \Vert_2 / \sqrt{n}$. 
For a matrix $\mA \in \R^{m \times n}$, we use $\Vert \mA \Vert_2$, and $\Vert \mA \Vert_{\normrms}$ to denote its spectral norm and RMS operator norm, respectively. The RMS operator norm is defined as $\Vert \mA \Vert_{\normrms}
:= \max_{\vv \neq \vzero} \frac{\Vert \mA \vv \Vert_{\normrms}}{\Vert \vv \Vert_{\normrms}}
= \sqrt{\frac{n}{m}} \, \Vert \mA \Vert_2$. Since spectral norm conditions can be equivalently expressed using the RMS operator norm, we adopt the latter to write spectral conditions for notational simplicity throughout this paper.
Finally, in the main text, we primarily rely on the following elementary properties of vector and matrix norms.

\begin{itemize}
    \item \textbf{Subadditivity:} $\Vert \mA + \mB \Vert_{\normrms} \leq \Vert \mA \Vert_{\normrms} + \Vert \mB \Vert_{\normrms}$ and $\Vert \va + \vb \Vert_{\normrms} \leq \Vert \va \Vert_{\normrms} + \Vert \vb \Vert_{\normrms}$.
    \item \textbf{Submultiplicativity:} $\Vert \mA \mB \Vert_{\normrms} \leq \Vert \mA \Vert_{\normrms} \Vert \mB \Vert_{\normrms}$ and $\Vert \mA \vv \Vert_{\normrms} \leq \Vert \mA \Vert_{\normrms} \Vert \vv \Vert_{\normrms}$.
    \item \textbf{Spectral norm of random matrices}~\citep{hdp}: for a matrix $\mA \in \R^{m \times n}$ with i.i.d. entries sampled from Gaussian distribution $\mathcal{N}(0, \sigma^2)$, its spectral norm satisfies $\Vert \mA \Vert_2 = \Theta\big(\sigma(\sqrt{m}+\sqrt{n})\big)$ with high probability.
\end{itemize}

\subsection{Spectral Condition for $\mu$P under Width Scaling}
\label{sec: width setup}

We briefly review $\mu$P and its spectral condition under width scaling~\citep{mup-spectral}, which serves as the conceptual foundation of our extension to joint width–depth scaling. 

\paragraph{Theoretical setup.}
A canonical setting~\citep{mup-spectral} for analyzing $\mu$P under width scaling is the deep linear multi-layer perceptron (MLP) trained with one step on a single data point $(\vx,\vy)$.
Specifically, we set $\vh_0(\vx)=\mW_0\vx$ and denote by $\mW_l$ the matrix weight at layer $l$.
The network is then defined as
\begin{align*}
\vh_l(\vx) = \mW_l \vh_{l-1}(\vx), \quad l\in[L+1],
\end{align*}
where the depth $L=\Theta(1)$ is fixed, while the model widths scale to infinity.
Although highly simplified, this setup captures the core width-scaling behavior of feature learning~\citep{mup-spectral}.
Moreover, $\mu$P formulations motivated by this setup have been successfully used in practical pretraining~\citep{TP5,hu2024minicpm,dit-mup}, including Transformers trained with AdamW, enabling stable feature learning and reliable HP transfer.

\paragraph{$\mu$P principle and its spectral condition under width scaling.}
As network size increases, standard parameterization (SP) typically leads to either exploding or vanishing feature updates.
$\mu$P resolves this issue by reparameterizing HPs with size to realize the following stable and efficient principle~\citep{TP4}.
\begin{principle}[$\mu$P principle]
\label{principle: mup}
$\mu$P aims to realize scale-invariant feature learning while maximizing the feature change induced by parameter updates.
Formally, it requires
\begin{align}
&\|\vh_l(\vx)\|_{\normrms}=\Theta(1), \ \|\Delta\vh_l(\vx)\|_{\normrms}=\Theta(1),\ l\in[L]. \tag{P1} \label{principle: stable} \\
& \text{maximize $\Delta\mW_l$'s contribution to $\Delta\vh_L(\vx),\ l\in[L]$}. \tag{P2} \label{principle: max}
\end{align}
\end{principle}
Under the width-scaling regime,~\citet{mup-spectral} showed that Principle~\ref{principle: mup} is ensured by the following simple spectral scaling condition on the weights and their per-step updates:
\begin{align}
\|\mW_l\|_{\normrms}=\Theta(1),\ 
\|\Delta \mW_l\|_{\normrms}=\Theta(1), \ l\in [L+1].
\label{eqn:spec-width}
\end{align}
This spectral condition provides a concise and unified perspective on $\mu$P under width scaling, from which the HP parameterization of a broad class of optimization algorithms can be derived in a unified and transparent manner~\citep{mup-spectral,DBLP:conf/nips/SAM-mup,extending-mup}.

\paragraph{Limitation of the width-scaling condition.}
The spectral condition~(\ref{eqn:spec-width}), however, applies only when depth is fixed.
In contrast, modern foundation models scale both width and depth, and existing $\mu$P results in this regime rely on complex analyses,
{with conclusions that depend on specific architectures and optimizers~\citep{TP-6,DBLP:conf/nips/BordelonCP24-transformer,DBLP:conf/iclr/BordelonNLHP24-dmft-depth,chizat2025hidden,completep,muon-cp}}.
This motivates our central question: \emph{Can we establish a simple and unified spectral perspective in the joint width–depth scaling regime?}

\section{Spectral Condition for $\mu$P under Width-Depth Scaling}
\label{sec: spec condition}

In this section, we establish the spectral condition for $\mu$P under width-depth scaling. We first introduce our problem setup, then derive the corresponding spectral $\mu$P condition and discuss its implications.

\subsection{Problem Setup}
\label{sec: two-layer setup}

Our setup extends the width-scaling setting in Section~\ref{sec: width setup} by introducing residual connections, which are essential for stabilizing training of deep networks~\citep{resnet}.
Motivated by practical residual branches that often contain multiple transformations (e.g., attention or FFN modules in Transformers), we consider residual networks whose residual branches contain $k=\Theta(1)$ linear transformations.
For clarity, the main text focuses on the two-layer residual block ($k=2$), \emph{which is the minimal setting that captures the essential scaling behavior of any fixed-depth branches with $k\ge 2$ while keeping the analysis minimal}.
The $k=1$ case and the general $k\ge2$ case are deferred to Appendices~\ref{app: spec one-layer} and~\ref{app: spectral multi-layer}, respectively.
Formally, the $k=2$ network studied in the main text is defined as:
\begin{align}
& \vh_0(\vx) = \alpha_0\mW_0\vx, \nonumber \\
& \vh_l(\vx) = \vh_{l-1}(\vx) + \alpha_l\mW_l^{(2)}\mW_l^{(1)}\vh_{l-1}(\vx), \ l \in [L] \label{eqn: resnet} \\
&\vh_{L+1}(\vx) = \alpha_{L+1}\mW_{L+1}\vh_L(\vx), \nonumber
\end{align}
where the weights
$\mW_0 \in \R^{n\times d_0}, \mW_l^{(1)} \in \R^{n_l\times {n}}, \mW_l^{(2)} \in \R^{n \times {n}_{l}}, \mW_{L+1} \in \R^{d_{L+1}\times n}$ are all initialized with Gaussian distribution $\left({\mW}_l\right)_{ij} \overset{\mathrm{i.i.d.}}{\sim} \mathcal{N}(0,\sigma_l^2)$\footnote{For notation simplicity, when quantities associated with $\mW_l^{(1)}$ and $\mW_l^{(2)}$ take the same form, we omit the superscript.}
and trained with layerwise learning rates $\eta_l$. Furthermore, $\{\alpha_l\}_{l=0}^{L+1}$ are block multipliers that control the effective strength of each transformation. 

Following existing $\mu$P literature~\citep{TP-6,DBLP:conf/nips/BordelonCP24-transformer,DBLP:conf/iclr/BordelonNLHP24-dmft-depth,completep}, we fix the input and output data dimensions and scale the width and depth to infinity, that is
\begin{equation}
d_0, d_{L+1} = \Theta(1), \quad n_l = \Theta(n),\quad n,L \to \infty .
\label{eq:dimensions}
\end{equation}
This setting is standard in Transformer-based large models~\citep{DBLP:conf/nips/transformer,gpt-5}, where $n$ denotes the model width and is typically of the same order as $n_l$ (e.g., the feed-forward width).
Moreover, we assume $\|\vx\|_{\normrms} = \Theta(1)$,
which holds for common data modalities such as natural images ($\|\vx\|_\normrms=\Theta(1)$) and one-hot encoded language inputs ($\|\vx\|_\normrms=\sqrt{1/d_0}=\Theta(1)$).

In the following subsections, we derive a sufficient and unified spectral condition under this $k=2$ setup for realizing the $\mu$P Principle~\ref{principle: mup} under joint width-depth scaling.

\subsection{Spectral Scaling Condition}
\label{sec: subsec Spectral Scaling Condition}

Analogous to the width-scaling condition in Equation~(\ref{eqn:spec-width}), the width-depth condition has two components.
{The initial condition on $\mW_l$ controls forward hidden-feature propagation, yielding $\|\vh_l(\vx)\|_{\normrms}=\Theta(1)$ for $l\in[L]$.}
The update condition on $\Delta\mW_l$ controls the one-step feature change, ensuring $\|\Delta\vh_l(\vx)\|_{\normrms}=\Theta(1)$ while maximizing the direct contribution of weight updates as required by Principle~(\ref{principle: max}).
We now state the resulting sufficient spectral condition.

\begin{condition}[Spectral condition for $\mu$P under joint width-depth scaling, two-layer residual block]
\label{condition: scale-invariant fl}
To realize $\mu$P Principle~\ref{principle: mup}, the initial weights and their per-step updates should satisfy:
\begin{itemize}
    \item \textbf{Initial condition.}
    
    Input and output weights:
    \begin{align}
        \alpha_0 \|\mW_0\|_{\normrms}  =  \Theta(1), \ \alpha_{L+1}\|\mW_{L+1}\|_{\normrms}  =  \Theta(1). \tag{C1.1} \label{eq:init_inandout}
    \end{align}
    Hidden weights:
    \begin{align}
        \alpha_l \|\mW_l^{(2)}\|_{\normrms}\,\|\mW_l^{(1)}\|_{\normrms}=\Theta(1/L), \ l\in[L]. \tag{C1.2} \label{eq:init_hidd}
    \end{align}
    
    \item \textbf{Update condition.}

    Input and output weights:
    \begin{align}
        &\alpha_0\|\Delta \mW_0\|_{\normrms} = \Theta(1),\  \alpha_{L+1}\|\Delta \mW_{L+1}\|_{\normrms} = \Theta(1). \tag{C2.1} \label{eq:update_inandout}
    \end{align}
    Hidden weights (first-order weight update):
    \begin{equation}
    \begin{aligned}
        &\alpha_l\|\Delta \mW_l^{(2)}\|_{\normrms}\,\|\mW_l^{(1)}\|_{\normrms}=\Theta(1/L), \ l\in[L] \\ &\alpha_l\|\mW_l^{(2)}\|_{\normrms}\,\|\Delta\mW_l^{(1)}\|_{\normrms}
        =\Theta(1/L), \ l\in[L]. 
    \end{aligned}
    \tag{C2.2} \label{eq:update_hidd_1}
    \end{equation}
    Hidden weights (second-order weight update):
    \begin{align}
        \alpha_l\|\Delta \mW_l^{(2)}\|_{\normrms}\,\|\Delta \mW_l^{(1)}\|_{\normrms}
        =\Theta(1/L), \ l\in[L]. \tag{C2.3} \label{eq:update_hidd_2}
    \end{align}
\end{itemize}
\end{condition}

Compared with the width-only condition in Equation~(\ref{eqn:spec-width}), Condition~\ref{condition: scale-invariant fl} introduces explicit depth factors for hidden residual blocks.
In particular, the products of the residual multiplier and the relevant weight or update norms must scale as $\Theta(1/L)$, reflecting the accumulation of residual contributions over $L$ blocks.

The hidden-layer update constraints in Condition~\ref{condition: scale-invariant fl} arise from expanding the one-step feature update $\Delta\vh_l(\vx)$.
For a two-layer residual block, this expansion contains first-order terms like $\Delta \mW_l^{(2)}\mW_l^{(1)}$, where exactly one branch weight is updated, and a second-order term including $\Delta\mW_l^{(2)} \Delta\mW_l^{(1)}$, where both branch weights are updated; these give rise to~(\ref{eq:update_hidd_1}) and~(\ref{eq:update_hidd_2}), respectively.
By contrast, a one-layer residual block produces only first-order direct update terms, so the second-order constraint is absent.

{
This update-order viewpoint helps unify prior disparate $\mu$P results under joint width-depth scaling~\citep{TP-6,completep,DBLP:conf/iclr/BordelonNLHP24-dmft-depth,DBLP:conf/nips/BordelonCP24-transformer,muon-cp,chizat2025hidden} by varying the residual block depth $k$.
When $k=1$, the update expansion contains no second-order term, and the
same analysis gives the looser Condition~\ref{condition: one-layer} in Appendix~\ref{app: spec one-layer}, whose initialization condition only requires $\alpha_l\|\mW_l\|_{\normrms}=\mathcal{O}(1/\sqrt{L})$.
Under the width-$\mu$P initialization $\|\mW_l\|_{\normrms}=\Theta(1)$ in Equation~(\ref{eqn:spec-width}), it therefore permits the range $\alpha_l=\mathcal{O}(1/\sqrt{L})$, including both $\alpha_l=\Theta(1/\sqrt{L})$ and $\alpha_l=\Theta(1/L)$, which recovers Depth-$\mu$P-style results~\citep{TP-6,DBLP:conf/iclr/BordelonNLHP24-dmft-depth} and MLU-style results~\citep{chizat2025hidden}, respectively.
We note that selecting between them requires additional criteria beyond these spectral constraints.
For example, the Depth-$\mu$P-style boundary choice further maximizes the zero-order feature-update contribution, whereas the MLU-style choice maintains non-vanishing local parameter updates; these distinctions are discussed in Appendix~\ref{app: spec one-layer}.
Our framework thus places the Depth-$\mu$P and MLU choices within a common set of spectral conditions, while leaving open which scaling is preferable for a given training objective.

In contrast, the $k=2$ case introduces the second-order constraint~(\ref{eq:update_hidd_2}), which requires the stricter product scale $\Theta(1/L)$ and leads to residual multipliers $\alpha_l=\Theta(1/L)$ under the same initialization convention. This recovers
CompleteP-style results~\citep{DBLP:conf/nips/BordelonCP24-transformer,completep,muon-cp}, which are more appropriate for practical architectures with multi-transformation residual branches (e.g., Transformers).
Detailed HP parameterizations in this case are given in Section~\ref{sec: implementation} and Appendix~\ref{app: implementation opts HPs}.
We note that, despite sharing $\alpha_l=\Theta(1/L)$, the MLU-style choice is not equivalent to the CompleteP-style formulation because the derived rule for other HPs (e.g., learning rate) differs.
}

Moreover, our analysis naturally extends to residual blocks with any fixed depth $k$ (Condition~\ref{condition: multi-layer} in Appendix~\ref{app: spectral multi-layer}) and to architectures with biases (Condition~\ref{condition: bias} in Appendix~\ref{app: spectral condition bias}).
For residual blocks of depth $k$, the update condition constrains all first- through $k$-th order update terms to scale as $\Theta(1/L)$.
Analogous conditions arise in the presence of biases, accounting for their interactions with weight updates.
As shown in Appendix~\ref{app: spectral multi-layer} and~\ref{app: spectral condition bias}, {these additional constraints do not lead to a different $\mu$P formulation compared to the CompleteP-style parameterization} in the $k=2$ case.
Therefore, \emph{the residual network with two-layer blocks in Equation~(\ref{eqn: resnet}) is the minimal setting that captures the core scaling behavior of practical architectures with multi-layer residual blocks.}

Although the spectral results are derived for a linear residual MLP with a one-step update, they generalize to more general and practical training regimes.
Theoretically, we introduce and verify some natural assumptions from~\citet{mup-spectral} in Appendix~\ref{app: Extension to General Training Settings}, under which the spectral results generalize to finite multiple gradient steps, nonlinearities, and finite multiple training examples.
Empirically, experiments in Section~\ref{sec: experiment} demonstrate that the resulting {CompleteP-style parameterizations from Condition~\ref{condition: scale-invariant fl}} achieve stable feature learning and reliable HP transfer on GPT-2 style models.
Together with prior empirical $\mu$P studies~\citep{completep,muon-cp}, these results suggest that the simplified setup captures the core scaling behavior relevant in practice.

\begin{takeawaybox}
\textbf{Takeaway 1.}
{\textbf{Residual block depth $k$ determines the depth $\mu$P scaling constraint}:
under standard width-$\mu$P initialization, $k=1$ permits $\alpha_l=\mathcal{O}(1/\sqrt{L})$ including Depth-$\mu$P-style scaling.
For fixed $k\ge2$, higher-order update terms require the stricter CompleteP-style scaling $\alpha_l=\Theta(1/L)$.
The latter is better suited to practical architectures such as Transformers.}
\end{takeawaybox}

\subsection{Theoretical Derivation}
\label{sec: Theoretical Derivation}

In this section, we derive Condition~\ref{condition: scale-invariant fl} for the residual network in Equation~(\ref{eqn: resnet}).
We first obtain a preliminary initialization condition from forward feature stability, which controls the accumulated residual branch at initialization.
We then expand the one-step feature update into zero-, first-, and second-order terms.
The first- and second-order terms should satisfy the maximal-update requirement in Principle~(\ref{principle: max}) while keeping the total feature update stable. 
In particular, the second-order term, absent in one-layer residual blocks, yields the additional constraint~(\ref{eq:update_hidd_2}).
Finally, combining these update constraints refines the preliminary initialization condition to Condition~\ref{condition: scale-invariant fl}.

{
Throughout the derivation, subadditivity and submultiplicativity in Section~\ref{sec: prelimi} are used to provide $\mathcal{O}(\cdot)$ upper bounds on the typical size of $\|\vh_l(\vx)\|_{\normrms}$ and $\|\Delta\vh_l(\vx)\|_{\normrms}$.
Therefore, whenever we write a $\Theta(\cdot)$ estimate, we additionally require a matching $\Omega(\cdot)$ lower bound.
Such estimates are typically tight in standard neural-network initialization and training settings~\citep{mup-spectral}, but the source of the matching lower bound differs between initialization and the update regime.
At random initialization, the tightness of $\|\vh_l(\vx)\|_{\normrms}$ bounds for the input and hidden layers follows from independence and concentration of the randomly initialized Gaussian matrices (Claim~\ref{claim:init}).
For the one-step update $\|\Delta\vh_l(\vx)\|_{\normrms}$, the matching lower bounds use coherent accumulation of the feature update terms (Assumption~\ref{ass:update_coherent_accumulation}) and update-feature alignment (Claim~\ref{claim:update}).
When no such lower bound is available, we retain only the $\mathcal{O}(\cdot)$ estimate, as for the output layer $\|\vh_{L+1}(\vx)\|_{\normrms}$ at initialization.
Appendix~\ref{app: lower bound} gives the complete assumptions and derivations, following the width-scaling treatment of~\citet{mup-spectral}.
}

\subsubsection{Preliminary Initial Condition}
\label{sec: Preliminary Initial Condition}

We first derive a preliminary initialization condition that ensures stability of feature magnitudes during forward propagation.
We consider each layer sequentially.

\paragraph{Input layer.} By submultiplicativity of the RMS operator norm, we can estimate the norm of $\vh_0(\vx) = \alpha_0 \mW_0\vx$ as
$
\Vert\vh_0(\vx)\Vert_\normrms
=\Theta(\alpha_0\Vert\mW_0\Vert_\normrms \Vert\vx\Vert_\normrms) = \Theta(\alpha_0\Vert\mW_0\Vert_\normrms),
$
where we have assumed $\|\vx\|_{\normrms}=\Theta(1)$.
{The matching lower bound here follows from the independence of $\mW_0$ and the input $\vx$, together with Gaussian concentration, as detailed in Appendix~\ref{app: Submultiplicativity Inequalities}.}
Thus, requiring $\alpha_0\|\mW_0\|_{\normrms}=\Theta(1)$ ensures $\|\vh_0(\vx)\|_{\normrms}=\Theta(1)$.

\paragraph{Hidden layers.} 
To estimate the scale of hidden features, we expand the residual recursion in Equation~(\ref{eqn: resnet}), which yields
\begin{align}
\vh_s(\vx)
= \vh_{0}(\vx) + \sum_{l=1}^s \alpha_l \mW_l^{(2)}\mW_l^{(1)}\vh_{l-1}(\vx),\ s\in[L].
\label{eqn: hl(x)}
\end{align}

Applying subadditivity together with the scale-estimation convention described above, we estimate
\begin{align*}
\|\vh_s(\vx)\|_{\normrms}
= \Theta\bigg(\|\vh_0(\vx)\|_{\normrms}
+ \big\|\sum_{l=1}^s \alpha_l \mW_l^{(2)} \mW_l^{(1)} \vh_{l-1}(\vx)\big\|_{\normrms}\bigg).
\end{align*}
Since we have $\|\vh_0(\vx)\|_{\normrms}=\Theta(1)$, it suffices to ensure that
$\|\sum_{l=1}^s \alpha_l \mW_l^{(2)} \mW_l^{(1)} \vh_{l-1}(\vx)\|_{\normrms}=\mathcal{O}(1)$  for any $s\in[L]$ to preserve $\|\vh_s(\vx)\|_{\normrms}=\Theta(1)$.
Under i.i.d.\ zero-mean Gaussian initialization, the summands are approximately independent zero-mean random vectors~\citep{TP4,TP-6,completep}, so the typical squared RMS norm of their sum scales as the sum of the squared RMS norms (see Theorem~3.3.1 in~\citet{hdp}), yielding that
$
\|\sum_{l=1}^s \alpha_l \mW_l^{(2)} \mW_l^{(1)} \vh_{l-1}(\vx)\|_{\normrms} = \Theta(
\sqrt{\sum_{l=1}^s \|\alpha_l \mW_l^{(2)} \mW_l^{(1)} \vh_{l-1}(\vx)\|_{\normrms}^2}).
$
By submultiplicativity, we can further estimate $\|\alpha_l \mW_l^{(2)} \mW_l^{(1)} \vh_{l-1}(\vx)\|_{\normrms} =  \Theta( \alpha_l
\|\mW_l^{(2)}\|_{\normrms}
\|\mW_l^{(1)}\|_{\normrms}
\|\vh_{l-1}(\vx)\|_{\normrms})$.
{For the latter equality, the matching lower bound follows by conditioning on $\vh_{l-1}(\vx)$ and applying Gaussian concentration successively to the independent matrices $\mW_l^{(1)}$ and $\mW_l^{(2)}$.
We defer the details to Claim~\ref{claim:init} in Appendix~\ref{app: Submultiplicativity Inequalities}.}
Therefore, starting from $\|\vh_0(\vx)\|_{\normrms}=\Theta(1)$, imposing
$$\alpha_l\|\mW_l^{(2)}\|_{\normrms}\|\mW_l^{(1)}\|_{\normrms}=\mathcal{O}(1/\sqrt{L}),\ l\in[L],$$ recursively ensures
$\|\sum_{l=1}^s \alpha_l\mW_l^{(2)}\mW_l^{(1)}\vh_{l-1}(\vx)\|_{\normrms}
=\mathcal{O}(1)$ for any $s\in[L]$.
This provides a preliminary initial condition on the hidden weights, which will be further refined by incorporating update constraints established in the next subsection.

\paragraph{Output layer.} {Submultiplicativity gives
\[
\|\vh_{L+1}(\vx)\|_{\normrms}
= \alpha_{L+1} \|\mW_{L+1} \vh_{L}(\vx)\|_{\normrms}
= \mathcal{O}(\alpha_{L+1} \|\mW_{L+1}\|_{\normrms} \|\vh_{L}(\vx)\|_{\normrms})
= \mathcal{O}(\alpha_{L+1} \|\mW_{L+1}\|_{\normrms}),
\]
where $\|\vh_L(\vx)\|_{\normrms}=\Theta(1)$ follows from the hidden-layer argument above. We do not invoke a matching lower bound here because the comparable-dimension condition in Claim~\ref{claim:init} fails for $\mW_{L+1}\in\R^{d_{L+1}\times n}$, where $d_{L+1}=\Theta(1)$ while $n\to\infty$ by Equation~(\ref{eq:dimensions}). Thus, $\alpha_{L+1}\|\mW_{L+1}\|_{\normrms}=\Theta(1)$ is a sufficient readout condition ensuring $\|\vh_{L+1}(\vx)\|_{\normrms}=\mathcal{O}(1)$, rather than $\Theta(1)$. Nevertheless, this suffices for Principle~(\ref{principle: stable}), which constrains $\vh_l$ only for $l\in[L]$ and does not require $\|\vh_{L+1}(\vx)\|_{\normrms}=\Theta(1)$ at initialization.}

\subsubsection{Update Condition} 

We next derive the update condition required to ensure stable feature evolution $\|\Delta \vh_l(\vx)\|_{\normrms}=\Theta(1)$ in Principle~(\ref{principle: stable}),
while maximally updating parameters as prescribed by Principle~(\ref{principle: max}).

\paragraph{Input layer.}
Since $\Delta\vh_0(\vx)=\alpha_0\Delta\mW_0\vx$, submultiplicativity of matrix norms yields
$$
\|\Delta\vh_0(\vx)\|_{\normrms}
=
\Theta(\alpha_0\|\Delta\mW_0\|_{\normrms}\|\vx\|_{\normrms})
= \Theta(\alpha_0\|\Delta\mW_0\|_{\normrms}),
$$
{The matching lower bound is supplied by update-feature alignment between $\Delta\mW_0$ and $\vx$, which is justified for the input layer in Appendix~\ref{app: Submultiplicativity Inequalities}. Hence,} we set $\alpha_0\|\Delta\mW_0\|_{\normrms}=\Theta(1)$ to ensure $\Vert\Delta\vh_0(\vx)\Vert_\normrms=\Theta(1)$.

\paragraph{Hidden layers.}
To analyze the hidden feature updates $\Delta\vh_s(\vx)$, we expand the residual representation in Equation~(\ref{eqn: hl(x)}) after a single gradient step: $\vh_s(\vx) + \Delta \vh_s(\vx)
= \vh_{0}(\vx) + \Delta\vh_{0}(\vx) + \sum_{l=1}^s \alpha_l (\mW_l^{(2)}+\Delta\mW_l^{(2)})(\mW_l^{(1)}+\Delta\mW_l^{(1)})(\vh_{l-1}(\vx)+\Delta\vh_{l-1}(\vx))$, leading to
\begin{align}
\Delta\vh_s(\vx) 
&= \Delta\vh_0(\vx) + \underbrace{\sum_{l=1}^s \alpha_l \mW_l^{(2)}\mW_l^{(1)}\Delta\vh_{l-1}(\vx)}_{\vepsilon_0(s)} + \underbrace{\sum_{l=1}^s \alpha_l \mW_l^{(2)}\Delta\mW_l^{(1)}(\vh_{l-1}(\vx)+\Delta\vh_{l-1}(\vx))}_{\vepsilon_1^{(1)}(s)} \nonumber \\
&\quad+ \underbrace{\sum_{l=1}^s \alpha_l \Delta\mW_l^{(2)}\mW_l^{(1)}(\vh_{l-1}(\vx)+\Delta\vh_{l-1}(\vx))}_{\vepsilon_1^{(2)}(s)} + \underbrace{\sum_{l=1}^s \alpha_l \Delta\mW_l^{(2)}\Delta\mW_l^{(1)}(\vh_{l-1}(\vx)+\Delta\vh_{l-1}(\vx))}_{\vepsilon_2(s)}.
\label{eqn: feature update k=2}
\end{align}
According to the degree of weight updates in the current residual block, we refer to these contributions as the zero-, first-, and second-order update terms,
denoted respectively by
$\vepsilon_0(s)$, $\vepsilon_1^{(1)}(s),\vepsilon_1^{(2)}(s)$, and $\vepsilon_2(s)$.
{Under Assumption~\ref{ass:update_coherent_accumulation} for coherent accumulation of the feature update terms, the subadditive upper bound below has a matching lower bound, so we have}
\begin{align}
\Vert\Delta\vh_s(\vx)\Vert_\normrms = \Theta\big(\Vert\Delta\vh_0(\vx)\Vert_\normrms + \Vert\vepsilon_0(s)\Vert_\normrms + \Vert\vepsilon_1^{(1)}(s)\Vert_\normrms + \Vert\vepsilon_1^{(2)}(s)\Vert_\normrms + \Vert\vepsilon_2(s)\Vert_\normrms\big).
\label{eqn: last hidden update}
\end{align}
Since $\|\Delta\vh_0(\vx)\|_{\normrms}=\Theta(1)$ by the input-layer update, we have
$\|\Delta\vh_s(\vx)\|_{\normrms}=\Omega(1)$ for all $s\in[L]$.
Moreover, by subadditivity estimation, the order of remaining terms do not decay with depth, implying
$\|\Delta\vh_s(\vx)\|_{\normrms}=\mathcal{O}(\|\Delta\vh_L(\vx)\|_{\normrms})$ for any $s\in[L]$.
Therefore, to enforce Principle~\ref{principle: mup},
it suffices to require $\|\Delta\vh_L(\vx)\|_{\normrms}=\Theta(1)$ while satisfying
Principle~(\ref{principle: max}).
We next control terms on the right-hand side of Equation~(\ref{eqn: last hidden update}).

\textbf{Zero-order term.}
The term $\vepsilon_0(L)$ propagates feature updates from earlier layers and does not depend on the weight update $\Delta\mW_l$ at the current layer, so it does not need to be maximized from Principle~(\ref{principle: max}).
Therefore, it suffices to verify that $\vepsilon_0(L)$ remains $\mathcal{O}(1)$ under the preliminary initial condition. In fact, the same argument used for deriving $\Vert\vh_L(\vx)\Vert_\normrms$ in Section~\ref{sec: Preliminary Initial Condition} directly implies $\|\vepsilon_0(L)\|_{\normrms} =\Theta\left(
\sqrt{\sum_{l=1}^L \alpha_l^2
\|\mW_l^{(2)}\|_{\normrms}^2
\|\mW_l^{(1)}\|_{\normrms}^2
\|\Delta\vh_{l-1}(\vx)\|_{\normrms}^2}
\right)
= \mathcal{O}(1)$, where we use the self-consistent fact that $\Vert \Delta\vh_{l-1}(\vx)\Vert_\normrms = \Theta(1)$ for $l\in[L]$ if we finally ensure $\Vert \Delta\vh_L(\vx)\Vert_\normrms= \Theta(1)$.

\textbf{First-order terms.}
{Using subadditivity and submultiplicativity for the upper bound, Assumption~\ref{ass:update_coherent_accumulation} for coherent accumulation, and update-feature alignment for the constituent products (Claim~\ref{claim:update}),} we estimate the order of $\|\vepsilon_1^{(1)}(L)\|_{\normrms}$ as
\begin{align*}
\|\vepsilon_1^{(1)}(L)\|_{\normrms}
= \Theta\bigg(\sum_{l=1}^L \alpha_l
\|\mW_l^{(2)}\|_{\normrms}
\|\Delta\mW_l^{(1)}\|_{\normrms}
\|\vh_{l-1}(\vx)\|_\normrms\bigg) + \Theta\bigg(\sum_{l=1}^L \alpha_l
\|\mW_l^{(2)}\|_{\normrms}
\|\Delta\mW_l^{(1)}\|_{\normrms}
\|\Delta\vh_{l-1}(\vx)\|_{\normrms}\bigg).
\end{align*}
For $l\in[L]$, using $\|\vh_{l-1}(\vx)\|_{\normrms}=\Theta(1)$ by the preliminary initial condition and $\Vert \Delta\vh_{l-1}(\vx)\Vert_\normrms = \Theta(1)$ if we finally set $\Vert \Delta\vh_L(\vx)\Vert_\normrms= \Theta(1)$, we can obtain
$
\|\vepsilon_1^{(1)}(L)\|_{\normrms}
= \Theta\big(\sum_{l=1}^L \alpha_l
\|\mW_l^{(2)}\|_{\normrms}
\|\Delta\mW_l^{(1)}\|_{\normrms}\big).
$
To satisfy Principle~(\ref{principle: max}), we need to maximize the contribution from each $\Delta\mW_l^{(1)}$ and ensure $\|\vepsilon_1^{(1)}(L)\|_{\normrms}=\Theta(1)$ at the same time, which naturally requires
\[
\alpha_l \|\mW_l^{(2)}\|_{\normrms}\|\Delta\mW_l^{(1)}\|_{\normrms}
= \Theta(1/L), \qquad \forall\, l\in[L].
\]
To control the scale of $\vepsilon_1^{(2)}(L)$, an identical argument gives
$\alpha_l\|\Delta\mW_l^{(2)}\|_{\normrms}\|\mW_l^{(1)}\|_{\normrms}=\Theta(1/L)$ for every $l \in [L]$, which completes the first-order update condition~(\ref{eq:update_hidd_1}).

\textbf{Second-order term.} 
\emph{This is the key term that distinguishes two-layer residual branches from one-layer residual branches because it is absent when $k=1$}. As a direct update term, it is also subject to Principle~(\ref{principle: max}).
{Using subadditivity and submultiplicativity for the upper bound, Assumption~\ref{ass:update_coherent_accumulation} for coherent accumulation, and update-feature alignment for the constituent products (Claim~\ref{claim:update}),} we can estimate its scale as
\begin{align*}
\|\vepsilon_2(L)\|_{\normrms}
=
\Theta\bigg(
\sum_{l=1}^L \alpha_l
\|\Delta\mW_l^{(2)}\|_{\normrms}
\|\Delta\mW_l^{(1)}\|_{\normrms}
\bigg).
\end{align*}
Principle~(\ref{principle: max}) requires maximizing each summand and ensuring $\|\vepsilon_2(L)\|_{\normrms}=\Theta(1)$ in the meanwhile, leading to
$
\alpha_l\|\Delta\mW_l^{(2)}\|_{\normrms}\|\Delta\mW_l^{(1)}\|_{\normrms}
= \Theta(1/L)
$ for all $l\in[L]$, which completes the derivation for the second-order update condition on hidden weights~(\ref{eq:update_hidd_2}).

\paragraph{Output layer.} For the output layer $\vh_{L+1}(\vx) = \alpha_{L+1}\mW_{L+1}\vh_L(\vx)$, its one-step feature update is 
\begin{align*}
\Delta\vh_{L+1}(\vx) &= \alpha_{L+1}\mW_{L+1}\Delta\vh_L(\vx) +
\alpha_{L+1}\Delta\mW_{L+1}\bigl(\vh_L(\vx)+\Delta\vh_L(\vx)\bigr).
\end{align*}
{Using subadditivity and submultiplicativity for the upper bound, Assumption~\ref{ass:update_coherent_accumulation} for coherent accumulation of output-update terms, and update-feature alignment (Claim~\ref{claim:update}),} we estimate
\begin{align*}
\|\Delta\vh_{L+1}(\vx)\|_{\normrms} 
&=
\Theta(\alpha_{L+1} \|\mW_{L+1}\|_{\normrms}
\|\Delta\vh_L(\vx)\|_{\normrms}) + \Theta(\alpha_{L+1} \|\Delta\mW_{L+1}\|_{\normrms}
\|\vh_L(\vx)+\Delta\vh_L(\vx)\|_{\normrms}) \\
&=
\Theta(1) + \Theta\left(
\alpha_{L+1}\|\Delta\mW_{L+1}\|_{\normrms}
\right),
\end{align*}
where we used
$\alpha_{L+1}\|\mW_{L+1}\|_{\normrms},\  \|\vh_L(\vx)\|_{\normrms}=\Theta(1)$ by the preliminary initial condition,
and
$\|\Delta\vh_L(\vx)\|_{\normrms}=\Theta(1)$ by the update condition on the hidden weights.
Therefore, requiring Principle~\ref{principle: mup} yields the update condition $\alpha_{L+1}\|\Delta\mW_{L+1}\|_{\normrms}=\Theta(1)$.

\subsubsection{Final Initial Condition}
\label{sec: Final Initial Condition}
We now derive the final initialization condition for the hidden weights~(\ref{eq:init_hidd}) by incorporating the constraints
imposed by the update conditions.
Multiplying the two first-order update conditions on hidden weights yields
\begin{align*}
\alpha_l^2\|\mW_l^{(1)}\|_{\normrms}
\|\mW_l^{(2)}\|_{\normrms}
\|\Delta\mW_l^{(1)}\|_{\normrms}
\|\Delta\mW_l^{(2)}\|_{\normrms}
= \Theta(1/L^2), \ \forall l \in [L].   
\end{align*}
On the other hand, the second-order update condition is
\(
\alpha_l\|\Delta\mW_l^{(1)}\|_{\normrms}
\|\Delta\mW_l^{(2)}\|_{\normrms}
= \Theta(1/L)
\)
for all $l\in[L]$.
Dividing the product of the first-order conditions by the second-order condition immediately gives
$\alpha_l\|\mW_l^{(1)}\|_{\normrms}
\|\mW_l^{(2)}\|_{\normrms}
= \Theta(1/L)$ for any $l\in[L]$,
which completes the derivation of Condition~\ref{condition: scale-invariant fl}.

Given the first-order update condition~(\ref{eq:update_hidd_1}), the refined initialization condition~(\ref{eq:init_hidd}) and the second-order update condition~(\ref{eq:update_hidd_2}) are algebraically equivalent up to constant factors.
Thus, one of them could be derived from the other under~(\ref{eq:update_hidd_1}).
We nevertheless keep both in Condition~\ref{condition: scale-invariant fl} because~(\ref{eq:init_hidd}) states the final initial forward-scale requirement, while~(\ref{eq:update_hidd_2}) makes explicit the second-order maximal-update constraint responsible for the $k=1$ versus $k\ge2$ distinction.

{
\subsubsection{Extension to Any Fixed $k > 2$}
As we have mentioned in Section~\ref{sec: subsec Spectral Scaling Condition}, any fixed residual branch depth $k > 2$ does not give a different formulation compared to the {CompleteP-style parameterization} in the $k = 2$ case.
Here, we give a derivation sketch and defer the complete derivation to Appendix~\ref{app: spectral multi-layer}.
For any fixed $k\ge2$, expanding its one-step update $\Delta\vh_{L}(\vx)$ like Equation~(\ref{eqn: feature update k=2}) yields higher-order update terms indexed by nonempty subsets $S\subseteq[k]$, where the factors indexed by $S$ are replaced by their updates.
The norm product associated with each term factorizes as
\begin{align*}
\alpha_l\prod_{i\in S}\|\Delta\mW_l^{(i)}\|_{\normrms}
\prod_{i\notin S}\|\mW_l^{(i)}\|_{\normrms}
=\left(\alpha_l\prod_{i=1}^k\|\mW_l^{(i)}\|_{\normrms}\right)
\prod_{i\in S}\frac{\|\Delta\mW_l^{(i)}\|_{\normrms}}
{\|\mW_l^{(i)}\|_{\normrms}}.
\end{align*}
The initial condition in Condition~\ref{condition: multi-layer} makes the first factor $\Theta(1/L)$, while dividing each first-order update condition by the initial condition gives
$\|\Delta\mW_l^{(i)}\|_{\normrms}/\|\mW_l^{(i)}\|_{\normrms}=\Theta(1)$.
Because $k$ is fixed, the second factor is $\Theta(1)$ for every $S$.
Thus, under initial and first-order update conditions, the norm product of every higher-order update term is $\Theta(1/L)$, satisfying the higher-order update conditions in Condition~\ref{condition: multi-layer}.
As a result, any fixed $k>2$ introduces no new HP scaling beyond the $k=2$ case.
}

\section{Implementation of Spectral Condition}
\label{sec: implementation}

In this section, we map the CompleteP-style formulation from Condition~\ref{condition: scale-invariant fl} ($k = 2$) to concrete HP parameterizations.
The initialization parameters $\sigma_l$ and $\alpha_l$ are optimizer-agnostic, while the learning-rate scaling of $\eta_l$ depends on the optimizer. In the main text, we instantiate this recipe for Muon-Kimi~\citep{muon-kimi}. Table~\ref{tab: optimizer-family-overview} in
Appendix~\ref{app: implementation opts HPs} extends the similar recipe to additional optimizers and optimizer-dependent HPs, including weight decay and stability term $\varepsilon$.
The corresponding results for Condition~\ref{condition: one-layer} ($k = 1$) can be found in Table~\ref{tab: optimizer-family-overview-one-layer} of Appendix~\ref{app: implementation one-layer}.

\subsection{Initial Condition}
\label{sec:initial_condition}

Since these HPs interact to satisfy the spectral condition, multiple equivalent parameterization solutions exist~\citep{TP4,TP4b}. 
To facilitate practical adoption, we choose to align the initial variance to the standard width-scaling $\mu$P implementation~\citep{TP5}.
Specifically, for any weight matrix
${\mW}_l \in \R^{\nout \times \nin}$, we set:
\begin{align*}
    &\sigma_l=
    \left\{
    \begin{array}{ll}
     \Theta\big(\frac{1}{\sqrt{\nin}}\min\{1,\sqrt{\frac{\nout}{\nin}}\}\big),   & 0 \leq l \leq L,\\
    \Theta(1),  & l = L+1.
    \end{array} \right.  
\end{align*}

Under this variance parameterization, the RMS operator norms of weight matrices at initialization satisfy:
\begin{align}
\|{\mW}_l\|_{\normrms}
=\sqrt{\frac{\nin}{\nout}}\|{\mW}_l\|_2
=\sqrt{\frac{\nin}{\nout}}\cdot
\Theta\left(\sigma_l(\sqrt{\nin}+\sqrt{\nout})\right)
=
\left\{
\begin{array}{ll}
\Theta(1),   & 0 \leq l \leq L,\\
\Theta(\nin),  & l = L+1,
\end{array} \right.
\label{eq:w_norm}
\end{align}
where we used the spectral norm property of random matrices reviewed in Section~\ref{sec: prelimi}.
Based on Equation~(\ref{eq:w_norm}), we are ready to determine the parameterization of $\alpha_l$ to satisfy initial conditions in Condition~\ref{condition: scale-invariant fl}.

\textbf{For the input and output layers}, given
\begin{align*}
   \alpha_l\|\mW_l\|_{\normrms} =\left\{
    \begin{array}{ll}
    \Theta(\alpha_0),   & l = 0,\\
    \Theta(\alpha_{L+1}\nin),  & l = L+1,
    \end{array} \right. 
\end{align*}
to satisfy (\ref{eq:init_inandout}), we need to set 
\begin{align}
    \alpha_0=\Theta(1), \quad \alpha_{L+1}=\Theta(1/\nin). \label{eq:alpha_inandout}
\end{align}
\textbf{For the hidden layers}, given
\begin{align*}
&\alpha_l\|\mW_l^{(1)}\|_{\normrms}
\|\mW_l^{(2)}\|_{\normrms} =\Theta(\alpha_l), \ l\in[L],
\end{align*}
to satisfy (\ref{eq:init_hidd}), we need to set 
\begin{align}
    \alpha_l=\Theta(1/L), \ l\in[L]. \label{eq:alpha_l}
\end{align}

\subsection{Update Condition for Muon-Kimi}
\label{sec:update_condition_muonkimi}

\begin{table*}[t]

\renewcommand{\arraystretch}{1.3}
\renewcommand{\hl}[1]{\textcolor{purple}{#1}}
\renewcommand{\ll}[1]{\textcolor{gray}{#1}}
\centering

\caption{\textbf{{CompleteP-style $\mu$P parameterization from Condition~\ref{condition: scale-invariant fl} ($k=2$)} for Muon-Kimi~\citep{muon-kimi} under width-depth scaling.}
Entries in \hl{purple} indicate differences between $\mu$P and SP, while \ll{gray} shows the corresponding SP choices.
Here, $r_n$ and $r_L$ denote the width and depth scaling ratios relative to the base model. The variance of input weights is $\sigma^2_{\mathrm{base}}$ for language and $\sigma^2_{\mathrm{base}}/d_0$ for image.}
\label{tab: muon-kimi mup}
\begin{tabular}{cccc}
\toprule
   & Input weights & Hidden weights & Output weights \\
\midrule
Block Multiplier
& $\alpha_{\mathrm{base}}$
& \hl{$\alpha_{\mathrm{base}}/r_L$} \ \ll{($\alpha_{\mathrm{base}}$)}
& \hl{$\alpha_{\mathrm{base}}/r_n$} \ \ll{($\alpha_{\mathrm{base}}$)} \\

Initial Variance
& $\sigma^2_{\mathrm{base}}/d_0$ or $\sigma^2_{\mathrm{base}}$
& \hl{$\sigma^2_{\mathrm{base}}/r_n$} \ \ll{($\sigma^2_{\mathrm{base}}$)}
& $\sigma^2_{\mathrm{base}}$ \\

Learning Rate
& $\eta_{\mathrm{base}}$
& \hl{$\eta_{\mathrm{base}}/\sqrt{r_n}$} \ \ll{($\eta_{\mathrm{base}}$)}
& $\eta_{\mathrm{base}}$ \\

\bottomrule
\end{tabular}
\end{table*}

Since different optimizers take different scales of $\|\Delta\mW_l\|_{\normrms}$, the implementation of the update condition depends on the choice of optimizer.
In the main text, we focus on Muon-Kimi~\citep{muon-kimi}.
The derivations for Muon~\citep{jordan6muon}, Shampoo~\citep{gupta2018shampoo}, SOAP~\citep{vyas2024soap}, AdamW~\citep{adamw}, Sophia~\citep{DBLP:conf/iclr/Liu0HL024-sophia}, Lion~\citep{DBLP:conf/nips/ChenLHRW0DLHLL23-lion}, SGD, and SSO~\citep{xie2026controlled-sso} are deferred to Appendix~\ref{app: implementation opts HPs}, where we recover several existing $\mu$P results (e.g., for SGD, AdamW, and some matrix-preconditioned optimizers) in the width–depth scaling setting~\citep{TP-6,DBLP:conf/nips/BordelonCP24-transformer,DBLP:conf/iclr/BordelonNLHP24-dmft-depth,completep,muon-cp}.

Muon-Kimi~\citep{muon-kimi} is a widely used variant of Muon~\citep{jordan6muon} designed to align its update scales of matrix parameters with those of AdamW-optimized vector parameters by applying RMS normalization, which facilitates the reuse of HPs well-tuned for AdamW. 
It has been successfully applied to pretraining models with up to 1T parameters~\citep{team2025kimik2}.
Specifically, for a weight matrix ${\mW}_l \in \R^{\nout\times\nin}$, the update rule (without weight decay) is:
\begin{equation*}
\Delta{\mW}_l
=
-\,\eta_l \cdot 0.2 \sqrt{\max\{\nin,\nout\}}
\cdot \mU_l \mV_l^\top,
\end{equation*}
where $\mU_l,\mV_l$ are the left and right singular vector matrices of the gradient that
$\nabla_{{\mW}_l}\mathcal{L}=\mU_l \mSigma_l \mV_l^\top$.
The resulting update norm satisfies
\begin{align}
&\Vert\Delta{\mW}_l\Vert_\normrms = \sqrt{\frac{\nin}{\nout}} \Vert\Delta{\mW}_l\Vert_2 
= \Theta\left(\eta_l \sqrt{\nin}\max\left\{1,\sqrt{\frac{\nin}{\nout}}\right\}\right)
=\left\{
    \begin{array}{ll}
    \Theta(\eta_l),   & l = 0,\\
    \Theta(\eta_l\sqrt{\nin}),   & l \in [L],\\
    \Theta(\eta_l\nin),  & l = L+1.
    \end{array} \right. \label{eq:dw_norm}
\end{align}
Based on Equation~(\ref{eq:dw_norm}), we are now ready to determine the parameterization of $\eta_l$ to satisfy the update condition.

\textbf{For the input and output layers}, given the dimension magnitude assumption in Equation~(\ref{eq:dimensions}) and the $\alpha_l$ parameterization in Equation~(\ref{eq:alpha_inandout}), we have:
\begin{align*}
    \alpha_l\|\Delta\mW_l\|_{\normrms}
    &=\left\{
    \begin{array}{ll}
    \Theta(1)\Theta(\eta_l),   & l = 0, \notag\\
    \Theta(1/\nin)\Theta(\eta_l\nin),  & l = L+1,
    \end{array} \right. \\
    &=\Theta(\eta_l).
\end{align*}
Thus, to satisfy (\ref{eq:update_inandout}), we need to set:
\begin{align*}
    \eta_0 = \Theta(1), \quad \eta_{L+1} = \Theta(1).
\end{align*}

\textbf{For the hidden layers}, let us first consider the first-order update condition. 
Given the dimension magnitude in Equation~(\ref{eq:dimensions}), the weight norm at initialization in Equation~(\ref{eq:w_norm}), and the $\alpha_l$ parameterization in Equation~(\ref{eq:alpha_l}), we have:
\begin{align*}
    &\alpha_l\Vert \Delta\mW_l^{(2)}\Vert_\normrms \Vert \mW_l^{(1)}\Vert_\normrms 
    =\Theta(1/L) \cdot \|\Delta\mW_l^{(2)}\|_{\normrms} 
    =\Theta\left(\frac{1}{L}\eta_l^{(2)} \sqrt{\nin}\right). 
\end{align*}
Thus, to satisfy (\ref{eq:update_hidd_1}), we need to set:
\begin{align}
    \eta_l^{(2)} = \Theta(\frac{1}{\sqrt{\nin}}), \ l\in[L].
\label{eq:maintext_muonkimi_hid_lr}
\end{align}
Symmetrically, we have the same choice for $\mW_l^{(1)}$ that $\eta_l^{(1)} = \Theta(\frac{1}{\sqrt{\nin}})$ to ensure the first-order condition.

By the algebraic equivalence discussed in Section~\ref{sec: Final Initial Condition}, the second-order condition~(\ref{eq:update_hidd_2}) is then automatically satisfied by the initial condition~(\ref{eq:init_hidd}) and the first-order condition~(\ref{eq:update_hidd_1}), so no further constraint is needed for implementing the second-order condition~(\ref{eq:update_hidd_2}).

This completes the {CompleteP-style parameterization} of Muon-Kimi, which is summarized in Table~\ref{tab: muon-kimi mup}.

\subsection{Simplified Implementation Rule for Modern Optimizers}
\label{sec: simplified implementation rule}

The Muon-Kimi derivation above, together with derivations for other optimizers in Appendix~\ref{app: implementation opts HPs}, reveals a useful implementation-level simplification.
For the modern optimizers considered in this paper, except SGD, implementing Condition~\ref{condition: scale-invariant fl} amounts to taking the corresponding width-scaling $\mu$P implementation and adding the hidden residual multiplier $\alpha_l=\Theta(1/L)$.
The reason is that preconditioned updates remove the depth factor $\alpha_l$ inherited by the raw hidden-layer gradients.
As a result, their update norms do not depend on $\alpha_l$, and the optimizer-specific HP rule remains the same as in width-scaling $\mu$P.
SGD does not share this simplification because its update is proportional to the raw gradient.
After setting $\alpha_l=\Theta(1/L)$, the spectral update condition therefore still requires an additional $\alpha_l$-dependent learning rate rescaling.

\begin{takeawaybox}
\textbf{Takeaway 2.}
For most modern optimizers, normalization or preconditioning removes the depth factor $\alpha_l$ of hidden gradients, so the {CompleteP-style $\mu$P formulation from Condition~\ref{condition: scale-invariant fl} ($k = 2$)} is just width-scaling $\mu$P plus a hidden residual multiplier $\alpha_l = \Theta(1/L)$.
\end{takeawaybox}

\subsection{Practical HP Parameterization and Transfer}

In practice, $\mu$P is often implemented using a ratio-based approach~\citep{TP5,completep,dit-mup}. 
We define width and depth scaling ratios as $r_n = n/n_{\mathrm{base}}$ and $r_L = L/L_{\mathrm{base}}$, where $n_{\mathrm{base}}$ and $L_{\mathrm{base}}$ are some fixed base model constants. 
The target model's HPs are then set by scaling the corresponding base HPs, denoted as $\alpha_{\mathrm{base}}$, $\sigma_{\mathrm{base}}^2$, and $\eta_{\mathrm{base}}$, according to these ratios. 

For instance, for Muon-Kimi, the hidden layer learning rate is set to $\eta_l = \eta_{\mathrm{base}}/\sqrt{r_n}$, which satisfies the theoretical requirement $\eta_l = \eta_{\mathrm{base}}/\sqrt{n/n_{\mathrm{base}}} = \Theta(\eta_{\mathrm{base}}/\sqrt{n})$ in Equation~(\ref{eq:maintext_muonkimi_hid_lr}). 
Table~\ref{tab: muon-kimi mup} summarizes the complete HP parameterization for Muon-Kimi under width-depth scaling derived above.

As illustrated in Section~\ref{sec: intro}, a critical utility of $\mu$P is enabling HP transfer, which effectively reduces the cost of HP search for training large models.
In practice, the transfer follows this procedure: optimal base HPs (e.g., $\eta_{\mathrm{base}}$) are first identified on a small model; these optimal values are then transferred to a larger target model to obtain true HPs (e.g., $\eta_{\mathrm{base}}/\sqrt{r_n}$). Consequently, we only need to search the base HPs on the computationally inexpensive small model.

Note that although the $\mu$P parameterization is derived from a simplified setup, we apply it to standard language models pretraining and verify its practical utility in the next section.

\section{Experiments}
\label{sec: experiment}

In this section, we empirically evaluate the $\mu$P formulations derived from our spectral conditions on GPT-2 style language models.
We first show that the {CompleteP-style parameterization from Condition~\ref{condition: scale-invariant fl} ($k\ge2$)}\footnote{In fact, Condition~\ref{condition: scale-invariant fl} is derived when $k=2$. Since Condition~\ref{condition: multi-layer} ($k\ge2$) leads to the same {CompleteP-style parameterization}, we use $k\ge2$ to refer to this parameterization throughout the experiments.} enables stable feature learning and robust HP transfer under width-depth scaling.
We then compare it with {Depth-$\mu$P-style parameterization from Condition~\ref{condition: one-layer} ($k=1$)} to validate the role of residual block depth $k$ on depth scaling.
The complete details and results are deferred to Appendix~\ref{app: Additional Experimental Details}.

\subsection{Experimental Settings}

Following standard empirical $\mu$P studies~\citep{completep,muon-cp}, we train GPT-2 style Transformer language models~\citep{gpt-2,nanogpt} on the OpenWebText dataset~\citep{Gokaslan2019OpenWeb}, using the GPT-2 tokenizer with a maximum sequence length of 1024.
All models fix the attention head dimension to $64$ and use a feedforward (FFN) dimension of $4n$.
We vary model width and depth around a base model $(n_{\mathrm{base}}, L_{\mathrm{base}})=(256,4)$, with widths up to $4096$ and depths up to $256$.

Since SGD and AdamW have been extensively studied in prior $\mu$P studies under width-depth scaling~\citep{TP-6,DBLP:conf/nips/BordelonCP24-transformer,DBLP:conf/iclr/BordelonNLHP24-dmft-depth,completep}, we focus on more recent optimizers: Muon-Kimi-AdamW, Muon-AdamW, Shampoo-AdamW, and Sophia.
Following common practice~\citep{muon-kimi}, Muon-Kimi-AdamW uses Muon-Kimi for hidden matrix parameters and AdamW for all other parameters, such as embeddings, the LM head, and biases.
Muon-AdamW and Shampoo-AdamW are defined analogously.
In the main text, we present Muon-Kimi-AdamW as the primary example; additional optimizer results are reported in Appendix~\ref{app: Additional Experimental Details}.

We implement the {CompleteP-style formulation ($k\ge2$)} and the {Depth-$\mu$P-style formulation ($k=1$)}; the corresponding HP parameterization overviews are given in Table~\ref{tab: optimizer-family-overview} at Appendix~\ref{app: implementation opts HPs} and Table~\ref{tab: optimizer-family-overview-one-layer} at Appendix~\ref{app: implementation one-layer}, respectively.
For example, for the hybrid Muon-Kimi-AdamW optimizer under Condition~\ref{condition: scale-invariant fl}, the Muon-Kimi part follows Table~\ref{tab: muon-kimi-wd mup}, while the AdamW part follows Table~\ref{tab: adamw-mup}. Muon-AdamW and Shampoo-AdamW are implemented analogously, with details deferred to Appendix~\ref{app: Additional Experimental Details}.
The main text focuses on learning rate transfer without weight decay, and weight decay transfer results are deferred to Figure~\ref{figures: mup vs sp wd} in Appendix~\ref{app: Additional Details of HP Transfer Experiments}.

\subsection{Feature Learning and HP Transfer}

\begin{figure*}[t]
\centering

\subfloat{
\includegraphics[height=0.185\textwidth]{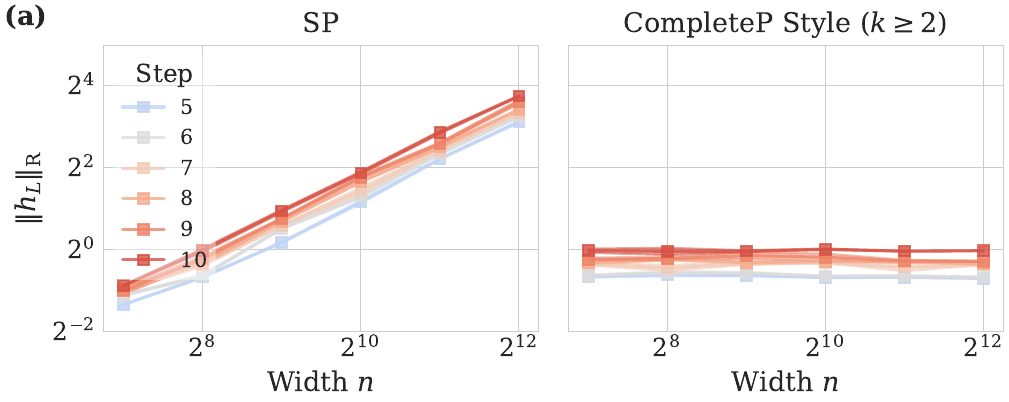}
\label{fig: sp_width_fl}
}%
\hskip 1ex
\subfloat{
\includegraphics[height=0.185\textwidth]{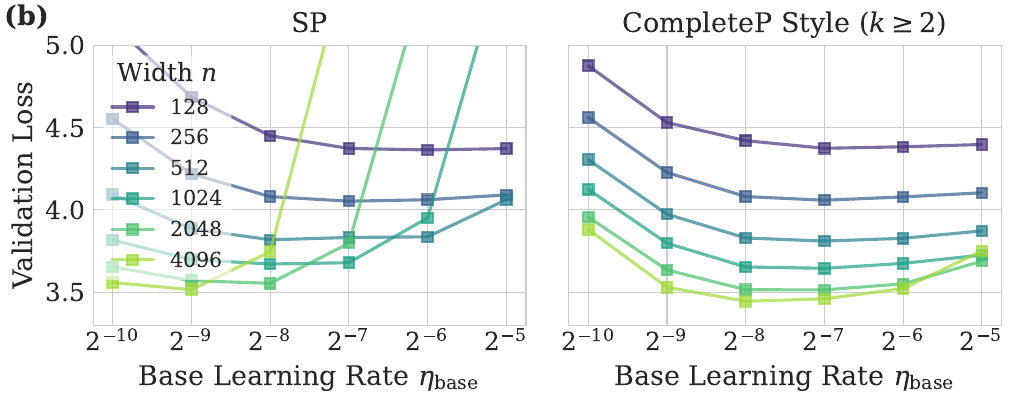}
\label{fig: sp_width_transfer}
}%

\vskip 1ex

\subfloat{
\includegraphics[height=0.185\textwidth]{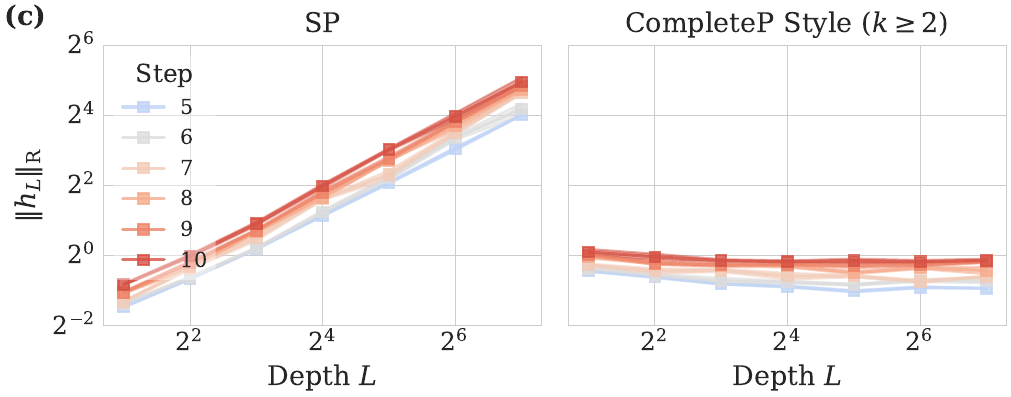}
\label{fig: sp_depth_fl}
}%
\hskip 1ex
\subfloat{
\includegraphics[height=0.185\textwidth]{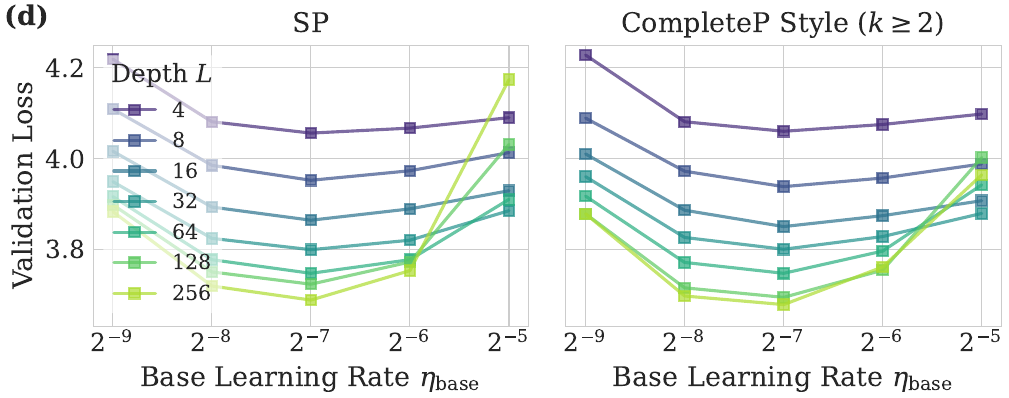}
\label{fig: sp_depth_transfer}
}%

\vskip 1ex

\subfloat{
\includegraphics[height=0.185\textwidth]{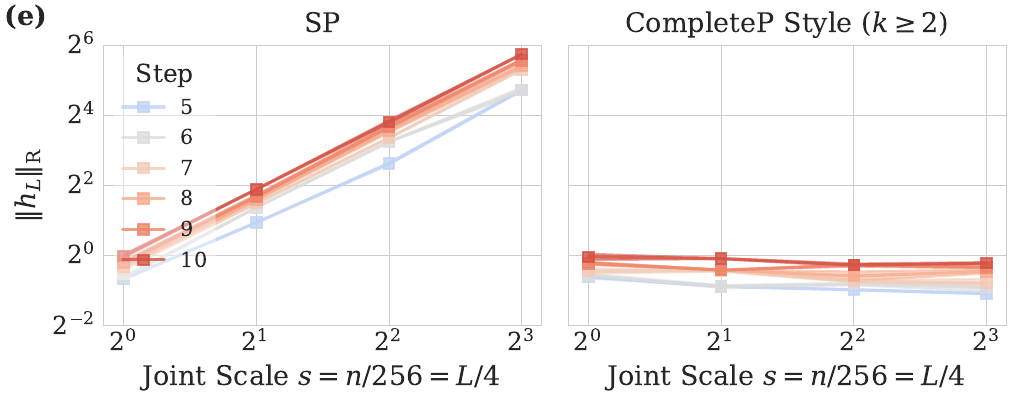}
\label{fig: sp_wdf_fl}
}%
\hskip 1ex
\subfloat{
\includegraphics[height=0.185\textwidth]{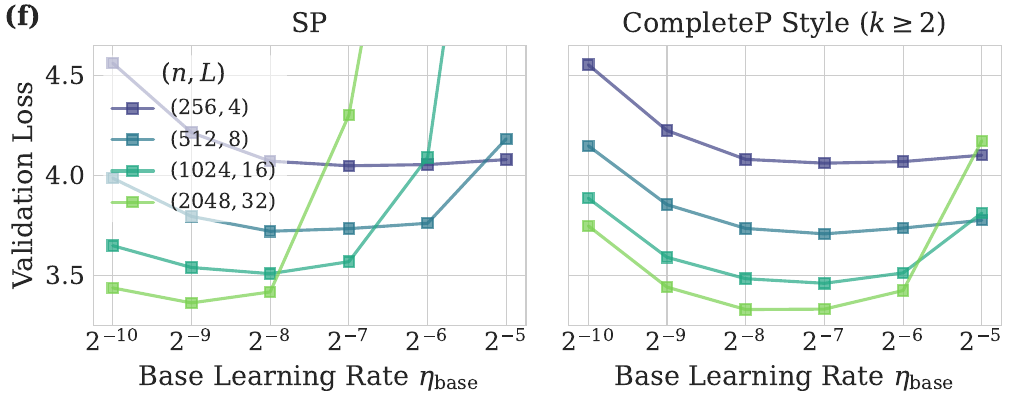}
\label{fig: sp_wdf_transfer}
}%

\vskip 0.1in
\caption{
\textbf{Feature learning and HP transfer of Muon-Kimi-AdamW under SP and the {CompleteP-style $\mu$P parameterization from Condition~\ref{condition: scale-invariant fl} ($k\ge2$)}.}
We train GPT-2 style models with Muon-Kimi-AdamW using SP and the {CompleteP-style parameterization} (see Tables~\ref{tab: muon-kimi mup} and~\ref{tab: adamw-mup}). {The CompleteP-style parameterization maintains stable feature norms and enables robust HP transfer across width, depth, and simultaneous width-depth scaling, while generally achieving lower loss than SP as the model size increases. The experimental settings and additional results are provided in Appendix~\ref{app: Additional Experimental Details}.}
}

\label{figures: mup vs sp}
\end{figure*}

In this section, we compare the feature learning stability and HP transferability of SP and the {CompleteP-style parameterization from Condition~\ref{condition: scale-invariant fl} ($k\ge 2$)}. The main results are shown in Figure~\ref{figures: mup vs sp}, with {experimental details and additional numerical results deferred to Appendix~\ref{app: Additional Details of Feature Learning Experiments} and~\ref{app: Additional Details of HP Transfer Experiments}}.

\paragraph{Feature learning.}
We first examine feature-scale stability using standard coordinate-check tests~\citep{TP5,completep,extending-mup}.
{Models are trained for $10$ steps while scaling width $n$, depth $L$, or both simultaneously, and we measure the RMS norm at the output of the final Transformer block $\|\vh_L\|_{\normrms}$.}
{As shown in Figure~\ref{figures: mup vs sp}(a,c,e), under SP, the feature scale grows rapidly when the model size is increased.}
In contrast, the {CompleteP-style parameterization} maintains stable and scale-invariant feature scales, consistent with Principle~\ref{principle: mup}.
{
Besides, Figures~\ref{figures: mup vs sp muon},~\ref{figures: mup vs sp shampoo}, and~\ref{figures: mup vs sp sophia} in Appendix~\ref{app: Additional Experimental Details} report a similar advantage of CompleteP-style parameterization over SP for Muon-AdamW, Shampoo-AdamW, and Sophia, respectively.}
These results support that the $\mu$P
formulation derived from Condition~\ref{condition: scale-invariant fl} preserves stable feature learning under width-depth scaling.

\paragraph{HP transfer.}
We next evaluate HP transferability by training all models for 300M tokens with a batch size of 240, using a learning rate schedule with linear warmup followed by cosine decay.  
{As shown in Figure~\ref{figures: mup vs sp}(b), SP exhibits substantial shifts in the optimal base learning rate when width is scaled, whereas the CompleteP-style parameterization keeps the optimal base learning rate nearly invariant.}
Under depth scaling, the {CompleteP-style parameterization} also preserves HP transferability and achieves lower loss than SP as depth increases (Figure~\ref{figures: mup vs sp}(d)).
{When width and depth are scaled simultaneously, Figure~\ref{figures: mup vs sp}(f) shows that the optimal base learning rate under SP obviously shifts as the joint model scale increases, whereas the CompleteP-style parameterization keeps it nearly invariant and generally achieves lower validation loss at the larger joint scales.
In addition to Muon-Kimi-AdamW, Figures~\ref{figures: mup vs sp muon},~\ref{figures: mup vs sp shampoo}, and~\ref{figures: mup vs sp sophia} in Appendix~\ref{app: Additional Experimental Details} report analogous HP transfer results for Muon-AdamW, Shampoo-AdamW, and Sophia, respectively. They show a similar advantage of the CompleteP-style parameterization over SP.
These results support that Condition~\ref{condition: scale-invariant fl} provides a practical HP-transfer rule, which can significantly reduce tuning cost when scaling model size, particularly for pretraining large models~\citep{gpt-5,team2025kimik2}.
}

{
\paragraph{Additional Diagnostics.}
One may find that SP can appear to transfer the optimal learning rate reasonably well across the tested depths in Figure~\ref{figures: mup vs sp}(d), but this should not be interpreted as reliable transfer in the asymptotic depth regime.
We attribute this to two factors.
First, the tested depths are still moderate; as depth increases further, Figure~\ref{figures: mup vs sp}(c) suggests that hidden features eventually diverge, making stable depth scaling under SP infeasible.
Second, modern architectural components such as LayerNorm~\citep{DBLP:journals/corr/BaKH16-LN} and QKNorm~\citep{DBLP:conf/emnlp/HenryDPC20-qknorm} substantially enhance training stability, partially masking the underlying scaling pathology of SP at practical depths.
To isolate this effect, we remove LayerNorm layers and repeat the depth-scaling experiments in Appendix~\ref{app: Additional Details of Muon-Kimi-AdamW}.
The results in Figure~\ref{figures: mup vs sp-hp-nonLN} show that SP training becomes unstable and depth-wise HP transfer breaks down, while the {CompleteP-style parameterization} remains stable even at large depths ($L=256$) and continues to exhibit robust HP transfer.

We emphasize that this apparent SP depth transfer is not universal.
Figure~\ref{figures: mup vs sp shampoo} and~\ref{figures: mup vs sp sophia} in Appendix~\ref{app: Additional Experimental Details} show that for Shampoo-AdamW and Sophia, the {CompleteP-style parameterization} gives substantially more reliable depth-wise HP transfer than SP.
Thus, the Muon-Kimi-AdamW behavior in Figure~\ref{figures: mup vs sp}(d) appears to be a partially masked case, likely aided by normalization and the tested depth range.
}

\subsection{Role of Residual Block Depth}

\begin{figure*}[t]
\centering

\subfloat{
\includegraphics[height=0.185\textwidth]{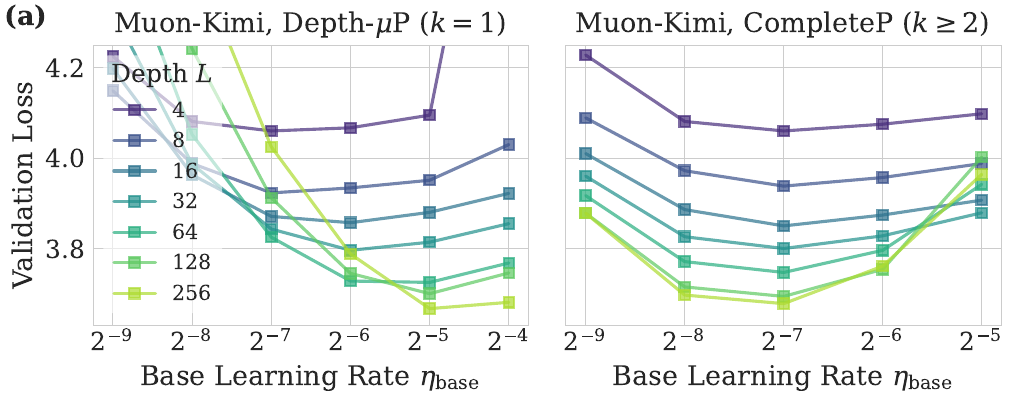}
\label{fig: cp vs dmup muon-kimi}
}%
\hskip 1ex
\subfloat{
\includegraphics[height=0.185\textwidth]{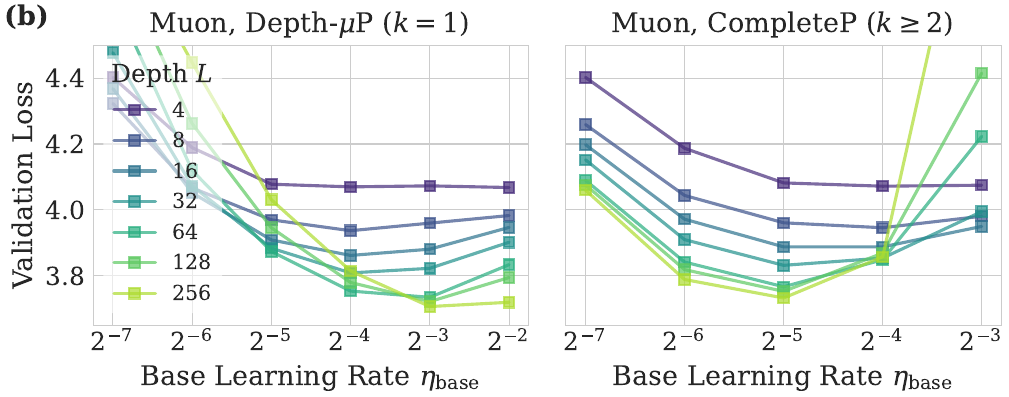}
\label{fig: cp vs dmup muon}
}%

\vskip 0.05in
\caption{
\textbf{Validating the role of residual block depth $k$.}
We compare two $\mu$P parameterizations for GPT-2-style models trained with Muon-Kimi-AdamW and Muon-AdamW: {the Depth-$\mu$P-style  formulation from Condition~\ref{condition: one-layer} ($k=1$) and the CompleteP-style  formulation from Condition~\ref{condition: scale-invariant fl} ($k\ge 2$)}.
Condition~\ref{condition: scale-invariant fl} yields more stable HP transfer, empirically supporting it as the appropriate {depth-scaling condition} for architectures with multi-transformation residual branches.
The numerical values are provided in Appendix~\ref{app: Additional Details of HP Transfer Experiments}.
}

\label{figures: cp vs dmup muon style}
\end{figure*}

We further examine the central prediction of our theory: residual-branch depth $k$ determines the appropriate {depth-scaling parameterization}.
We compare CompleteP-style scaling from Condition~\ref{condition: scale-invariant fl} ($k\ge 2$) with Depth-$\mu$P-style scaling from Condition~\ref{condition: one-layer} ($k=1$). {We use Depth-$\mu$P style rather than the alternative MLU-style choice as the representative $k=1$ case because it is commonly adopted in width-depth $\mu$P practice~\citep{TP-6,completep,hu2024minicpm}.}
Since the GPT-2 style Transformer contains multiple transformations in each residual branch, {our theory predicts that the CompleteP-style formulation should provide a better depth scaling}.

To test this prediction, we compare the two {depth-scaling parameterizations}.
As shown in Figure~\ref{figures: cp vs dmup muon style}, CompleteP-style scaling yields stable learning-rate transfer for both Muon-Kimi-AdamW and Muon-AdamW.
In contrast, Depth-$\mu$P-style formulation shifts the optimal learning rate as depth increases, indicating a failure of HP transfer in practical multi-transformation architectures.
These results {support} the claim in Section~\ref{sec: spec condition}: the high-order update term appearing at $k\ge2$ imposes the stricter scaling needed for Transformer-like residual blocks.
The same comparison for Shampoo-AdamW and Sophia in Figures~\ref{figures: cp vs dmup shampoo} and~\ref{figures: cp vs dmup sophia} at Appendix~\ref{app: Additional Details of HP Transfer Experiments} shows similar trends, further supporting the effectiveness of Condition~\ref{condition: scale-invariant fl} ($k\ge 2$).

{
We additionally conduct a controlled $k=3$ experiment by adding a linear layer at the end of both the attention and FFN modules while otherwise retaining the Muon-Kimi-AdamW depth-scaling setup.
Figure~\ref{figures: sp, cp, dmup Muon-kimi k3} in Appendix~\ref{app: Additional Details of Muon-Kimi-AdamW} shows that CompleteP-style parameterization transfers the optimal base learning rate and achieves comparable validation losses across the tested depths, whereas the Depth-$\mu$P-style parameterization exhibits a clear rightward shift in the optimal base learning rate.
The complete setting and numerical results are reported in Appendix~\ref{app: Additional Details of Muon-Kimi-AdamW}.
}

\begin{takeawaybox}
\textbf{Takeaway 3.}
CompleteP-style scaling from Condition~\ref{condition: scale-invariant fl} ($k\ge 2$) achieves stable feature learning and robust HP transfer under width-depth scaling, while SP and Depth-$\mu$P-style scaling from Condition~\ref{condition: one-layer} ($k=1$) often fail to do so.
\end{takeawaybox}

\section{Conclusion}
\label{sec: conclusion}

In this paper, we present a simple and unified spectral framework for $\mu$P under joint width-depth scaling.
The framework gives an operational condition on the norm of weights and their per-step updates, and explains why residual blocks with one transformation and those with multiple transformations lead to different depth-scaling rules.
By mapping this condition to concrete HP choices, we obtain a general recipe for implementing $\mu$P across a broad class of optimizers.
{Beyond placing previously fragmented joint width-depth results within one elementary framework, this recipe yields parameterizations for additional optimizers and provides a direct template for analyzing future optimizers.}
Experiments on GPT-2 style language models show that the resulting $k\ge2$ formulation preserves scale-invariant feature learning and supports robust HP transfer, whereas the $k=1$ formulation and SP often fail to do so.
These results suggest that the proposed spectral perspective provides a practical and interpretable route to width-depth $\mu$P for modern architectures.

{
\paragraph{Limitations and Discussion.}
Our theory characterizes scale-invariant feature learning in infinite-width/depth linear residual MLPs, whereas our experiments evaluate HP transfer and validation loss in finite-size Transformers.
Accordingly, it guarantees neither exact finite-size HP transfer nor the lowest validation loss after separately retuning each model size.
For example, the CompleteP-style parameterization still exhibits small HP shifts in Figures~\ref{figures: mup vs sp} and~\ref{figures: cp vs dmup muon style}.
We also find that SP achieves a lower tuned loss for Sophia in Figure~\ref{figures: mup vs sp sophia}(c).
These observations do not contradict the theoretical prediction but highlight architecture, finite-size, and target-metric gaps whose respective contributions remain to be quantified.
}

\subsubsection*{Broader Impact Statement}

This is mainly a theoretical work, and the proposed $\mu$P formulations have the potential to accelerate progress in scaling generative foundation models, including language modeling, text-to-image and video generation. However, improvements in scaling foundation models could also facilitate the creation of deepfakes for disinformation.

\subsubsection*{Acknowledgments}
This work was supported by the National Natural Science Foundation of China (Nos. 625B2177, 62522609, 92470118), the Beijing Major Science and Technology Project under Contract no. Z251100008425002, the Beijing Natural Science Foundation (No. L247030); the ByteDance Seed Fund; and the fund for building world-class universities (disciplines) of Renmin University of China. Chenyu Zheng was also supported by the Outstanding Innovative Talents Cultivation Funded Programs 2025 of Renmin University of China.

\bibliography{main}  

\begin{thebibliography}{50}
\providecommand{\natexlab}[1]{#1}
\providecommand{\url}[1]{\texttt{#1}}
\expandafter\ifx\csname urlstyle\endcsname\relax
  \providecommand{\doi}[1]{doi: #1}\else
  \providecommand{\doi}{doi: \begingroup \urlstyle{rm}\Url}\fi

\bibitem[Ba et~al.(2016)Ba, Kiros, and Hinton]{DBLP:journals/corr/BaKH16-LN}
Lei~Jimmy Ba, Jamie~Ryan Kiros, and Geoffrey~E. Hinton.
\newblock Layer normalization.
\newblock \emph{CoRR}, abs/1607.06450, 2016.

\bibitem[Balzano et~al.(2025)Balzano, Ding, Haeffele, Kwon, Qu, Wang, Wang, and Yaras]{DBLP:journals/corr/low-rank}
Laura Balzano, Tianjiao Ding, Benjamin~D. Haeffele, Soo~Min Kwon, Qing Qu, Peng Wang, Zhangyang Wang, and Can Yaras.
\newblock An overview of low-rank structures in the training and adaptation of large models.
\newblock \emph{CoRR}, abs/2503.19859, 2025.

\bibitem[Blake et~al.(2025)Blake, Eichenberg, Dean, Balles, Prince, Deiseroth, Cruz{-}Salinas, Luschi, Weinbach, and Orr]{DBLP:conf/iclr/u-muP}
Charlie Blake, Constantin Eichenberg, Josef Dean, Lukas Balles, Luke~Yuri Prince, Bj{\"{o}}rn Deiseroth, Andr{\'{e}}s~Felipe Cruz{-}Salinas, Carlo Luschi, Samuel Weinbach, and Douglas Orr.
\newblock u-{\(\mu\)}p: The unit-scaled maximal update parametrization.
\newblock In \emph{{ICLR}}, 2025.

\bibitem[Bordelon and Pehlevan(2022)]{DBLP:conf/nips/BordelonP22-dmft}
Blake Bordelon and Cengiz Pehlevan.
\newblock Self-consistent dynamical field theory of kernel evolution in wide neural networks.
\newblock In \emph{NeurIPS}, 2022.

\bibitem[Bordelon et~al.(2024{\natexlab{a}})Bordelon, Chaudhry, and Pehlevan]{DBLP:conf/nips/BordelonCP24-transformer}
Blake Bordelon, Hamza~Tahir Chaudhry, and Cengiz Pehlevan.
\newblock Infinite limits of multi-head transformer dynamics.
\newblock In \emph{NeurIPS}, 2024{\natexlab{a}}.

\bibitem[Bordelon et~al.(2024{\natexlab{b}})Bordelon, Noci, Li, Hanin, and Pehlevan]{DBLP:conf/iclr/BordelonNLHP24-dmft-depth}
Blake Bordelon, Lorenzo Noci, Mufan~Bill Li, Boris Hanin, and Cengiz Pehlevan.
\newblock Depthwise hyperparameter transfer in residual networks: Dynamics and scaling limit.
\newblock In \emph{{ICLR}}, 2024{\natexlab{b}}.

\bibitem[Chen et~al.(2023)Chen, Liang, Huang, Real, Wang, Pham, Dong, Luong, Hsieh, Lu, and Le]{DBLP:conf/nips/ChenLHRW0DLHLL23-lion}
Xiangning Chen, Chen Liang, Da~Huang, Esteban Real, Kaiyuan Wang, Hieu Pham, Xuanyi Dong, Thang Luong, Cho{-}Jui Hsieh, Yifeng Lu, and Quoc~V. Le.
\newblock Symbolic discovery of optimization algorithms.
\newblock In \emph{NeurIPS}, 2023.

\bibitem[Chickering et~al.(2026)Chickering, Wang, Wu, Moreno, Chen, Ma, Soboleva, Hestness, Liu, and Xing]{gqa-mup}
Kyle~R. Chickering, Huijuan Wang, Mengxi Wu, Alexander Moreno, Muhao Chen, Xuezhe Ma, Daria Soboleva, Joel Hestness, Zhengzhong Liu, and Eric~P. Xing.
\newblock Gqa-{\(\mu\)}p: The maximal parameterization update for grouped query attention.
\newblock \emph{CoRR}, abs/2605.15290, 2026.

\bibitem[Chizat(2025)]{chizat2025hidden}
L{\'{e}}na{\"{\i}}c Chizat.
\newblock The hidden width of deep resnets: Tight error bounds and phase diagram.
\newblock \emph{CoRR}, abs/2509.10167, 2025.

\bibitem[Dey et~al.(2023)Dey, Gosal, Chen, Khachane, Marshall, Pathria, Tom, and Hestness]{DBLP:journals/corr/cerebras}
Nolan Dey, Gurpreet Gosal, Zhiming Chen, Hemant Khachane, William Marshall, Ribhu Pathria, Marvin Tom, and Joel Hestness.
\newblock Cerebras-gpt: Open compute-optimal language models trained on the cerebras wafer-scale cluster.
\newblock \emph{CoRR}, abs/2304.03208, 2023.

\bibitem[Dey et~al.(2024)Dey, Bergsma, and Hestness]{DBLP:conf/nips/Dey-sparse}
Nolan Dey, Shane Bergsma, and Joel Hestness.
\newblock Sparse maximal update parameterization: {A} holistic approach to sparse training dynamics.
\newblock In \emph{NeurIPS}, 2024.

\bibitem[Dey et~al.(2025)Dey, Zhang, Noci, Li, Bordelon, Bergsma, Pehlevan, Hanin, and Hestness]{completep}
Nolan Dey, Bin~Claire Zhang, Lorenzo Noci, Mufan~Bill Li, Blake Bordelon, Shane Bergsma, Cengiz Pehlevan, Boris Hanin, and Joel Hestness.
\newblock Don't be lazy: Completep enables compute-efficient deep transformers.
\newblock \emph{CoRR}, abs/2505.01618, 2025.

\bibitem[Gokaslan and Cohen(2019)]{Gokaslan2019OpenWeb}
Aaron Gokaslan and Vanya Cohen.
\newblock Openwebtext corpus.
\newblock \url{http://Skylion007.github.io/OpenWebTextCorpus}, 2019.

\bibitem[Gupta et~al.(2018)Gupta, Koren, and Singer]{gupta2018shampoo}
Vineet Gupta, Tomer Koren, and Yoram Singer.
\newblock Shampoo: Preconditioned stochastic tensor optimization.
\newblock In \emph{ICML}, pages 1842--1850. PMLR, 2018.

\bibitem[Haas et~al.(2024)Haas, Xu, Cevher, and Vankadara]{DBLP:conf/nips/SAM-mup}
Moritz Haas, Jin Xu, Volkan Cevher, and Leena~Chennuru Vankadara.
\newblock {\(\mu\)}p\({}^{\mbox{2}}\): Effective sharpness aware minimization requires layerwise perturbation scaling.
\newblock In \emph{NeurIPS}, 2024.

\bibitem[He et~al.(2016)He, Zhang, Ren, and Sun]{resnet}
Kaiming He, Xiangyu Zhang, Shaoqing Ren, and Jian Sun.
\newblock Deep residual learning for image recognition.
\newblock In \emph{{CVPR}}, pages 770--778, 2016.

\bibitem[Henry et~al.(2020)Henry, Dachapally, Pawar, and Chen]{DBLP:conf/emnlp/HenryDPC20-qknorm}
Alex Henry, Prudhvi~Raj Dachapally, Shubham~Shantaram Pawar, and Yuxuan Chen.
\newblock Query-key normalization for transformers.
\newblock In Trevor Cohn, Yulan He, and Yang Liu, editors, \emph{Findings of the Association for Computational Linguistics: {EMNLP} 2020}, volume {EMNLP} 2020, pages 4246--4253, 2020.

\bibitem[Hoffmann et~al.(2022)Hoffmann, Borgeaud, Mensch, Buchatskaya, Cai, Rutherford, Casas, Hendricks, Welbl, Clark, et~al.]{scalinglaw-Chinchilla}
Jordan Hoffmann, Sebastian Borgeaud, Arthur Mensch, Elena Buchatskaya, Trevor Cai, Eliza Rutherford, Diego de~Las Casas, Lisa~Anne Hendricks, Johannes Welbl, Aidan Clark, et~al.
\newblock Training compute-optimal large language models.
\newblock \emph{CoRR}, abs/2203.15556, 2022.

\bibitem[Hu et~al.(2024)Hu, Tu, Han, He, Cui, Long, Zheng, Fang, Huang, Zhao, et~al.]{hu2024minicpm}
Shengding Hu, Yuge Tu, Xu~Han, Chaoqun He, Ganqu Cui, Xiang Long, Zhi Zheng, Yewei Fang, Yuxiang Huang, Weilin Zhao, et~al.
\newblock Minicpm: Unveiling the potential of small language models with scalable training strategies.
\newblock \emph{arXiv preprint arXiv:2404.06395}, 2024.

\bibitem[Ishikawa and Karakida(2024)]{DBLP:conf/iclr/second-mup}
Satoki Ishikawa and Ryo Karakida.
\newblock On the parameterization of second-order optimization effective towards the infinite width.
\newblock In \emph{{ICLR}}, 2024.

\bibitem[Jacot et~al.(2018)Jacot, Hongler, and Gabriel]{DBLP:conf/nips/JacotHG18-NTK}
Arthur Jacot, Cl{\'{e}}ment Hongler, and Franck Gabriel.
\newblock Neural tangent kernel: Convergence and generalization in neural networks.
\newblock In \emph{NeurIPS}, pages 8580--8589, 2018.

\bibitem[Jordan et~al.(2024{\natexlab{a}})Jordan, Bernstein, Rappazzo, @fernbear.bsky.social, Vlado, Jiacheng, Cesista, Koszarsky, and @Grad62304977]{modded_nanogpt_2024}
Keller Jordan, Jeremy Bernstein, Brendan Rappazzo, @fernbear.bsky.social, Boza Vlado, You Jiacheng, Franz Cesista, Braden Koszarsky, and @Grad62304977.
\newblock modded-nanogpt: Speedrunning the nanogpt baseline, 2024{\natexlab{a}}.
\newblock URL \url{https://github.com/KellerJordan/modded-nanogpt}.

\bibitem[Jordan et~al.(2024{\natexlab{b}})Jordan, Jin, Boza, Jiacheng, Cecista, Newhouse, and Bernstein]{jordan6muon}
Keller Jordan, Yuchen Jin, Vlado Boza, You Jiacheng, Franz Cecista, Laker Newhouse, and Jeremy Bernstein.
\newblock Muon: An optimizer for hidden layers in neural networks.
\newblock \emph{URL https://kellerjordan. github. io/posts/muon}, 6, 2024{\natexlab{b}}.

\bibitem[Kaplan et~al.(2020)Kaplan, McCandlish, Henighan, Brown, Chess, Child, Gray, Radford, Wu, and Amodei]{scalinglaw-kaplan}
Jared Kaplan, Sam McCandlish, Tom Henighan, Tom~B. Brown, Benjamin Chess, Rewon Child, Scott Gray, Alec Radford, Jeffrey Wu, and Dario Amodei.
\newblock Scaling laws for neural language models.
\newblock \emph{CoRR}, abs/2001.08361, 2020.

\bibitem[Karpathy(2022)]{nanogpt}
Andrej Karpathy.
\newblock nanogpt.
\newblock \url{https://github.com/karpathy/nanoGPT}, 2022.

\bibitem[Liu et~al.(2024{\natexlab{a}})Liu, Feng, Xue, Wang, Wu, Lu, Zhao, Deng, Zhang, Ruan, et~al.]{liu2024deepseek}
Aixin Liu, Bei Feng, Bing Xue, Bingxuan Wang, Bochao Wu, Chengda Lu, Chenggang Zhao, Chengqi Deng, Chenyu Zhang, Chong Ruan, et~al.
\newblock Deepseek-v3 technical report.
\newblock \emph{arXiv preprint arXiv:2412.19437}, 2024{\natexlab{a}}.

\bibitem[Liu et~al.(2024{\natexlab{b}})Liu, Li, Hall, Liang, and Ma]{DBLP:conf/iclr/Liu0HL024-sophia}
Hong Liu, Zhiyuan Li, David Leo~Wright Hall, Percy Liang, and Tengyu Ma.
\newblock Sophia: {A} scalable stochastic second-order optimizer for language model pre-training.
\newblock In \emph{{ICLR}}, 2024{\natexlab{b}}.

\bibitem[Liu et~al.(2025)Liu, Su, Yao, Jiang, Lai, Du, Qin, Xu, Lu, Yan, et~al.]{muon-kimi}
Jingyuan Liu, Jianlin Su, Xingcheng Yao, Zhejun Jiang, Guokun Lai, Yulun Du, Yidao Qin, Weixin Xu, Enzhe Lu, Junjie Yan, et~al.
\newblock Muon is scalable for llm training.
\newblock \emph{arXiv preprint arXiv:2502.16982}, 2025.

\bibitem[Loshchilov and Hutter(2019)]{adamw}
Ilya Loshchilov and Frank Hutter.
\newblock Decoupled weight decay regularization.
\newblock In \emph{{ICLR}}, 2019.

\bibitem[Nesterov(1983)]{nesterov1983method}
Yurii Nesterov.
\newblock A method for solving the convex programming problem with convergence rate o(1/k2).
\newblock In \emph{Dokl akad nauk Sssr}, volume 269, page 543, 1983.

\bibitem[Ngom et~al.(2025)Ngom, Foreman, Vishwanath, et~al.]{extending-mup}
Marieme Ngom, Sam Foreman, Venkatram Vishwanath, et~al.
\newblock Extending $\mu$p: Spectral conditions for feature learning across optimizers.
\newblock In \emph{OPT 2025: Optimization for Machine Learning}, 2025.

\bibitem[Qiu et~al.(2025)Qiu, Chen, Phan, Lei, and Wilson]{muon-cp}
Shikai Qiu, Zixi Chen, Hoang Phan, Qi~Lei, and Andrew~Gordon Wilson.
\newblock Hyperparameter transfer enables consistent gains of matrix-preconditioned optimizers across scales.
\newblock \emph{arXiv preprint arXiv:2512.05620}, 2025.

\bibitem[Radford et~al.(2019)Radford, Wu, Child, Luan, Amodei, Sutskever, et~al.]{gpt-2}
Alec Radford, Jeffrey Wu, Rewon Child, David Luan, Dario Amodei, Ilya Sutskever, et~al.
\newblock Language models are unsupervised multitask learners.
\newblock \emph{OpenAI blog}, 1\penalty0 (8):\penalty0 9, 2019.

\bibitem[Schoenholz et~al.(2017)Schoenholz, Gilmer, Ganguli, and Sohl{-}Dickstein]{DBLP:conf/iclr/SchoenholzGGS17-deepinfo}
Samuel~S. Schoenholz, Justin Gilmer, Surya Ganguli, and Jascha Sohl{-}Dickstein.
\newblock Deep information propagation.
\newblock In \emph{{ICLR}}, 2017.

\bibitem[Singh et~al.(2025)Singh, Fry, Perelman, Tart, Ganesh, El-Kishky, McLaughlin, Low, Ostrow, Ananthram, et~al.]{gpt-5}
Aaditya Singh, Adam Fry, Adam Perelman, Adam Tart, Adi Ganesh, Ahmed El-Kishky, Aidan McLaughlin, Aiden Low, AJ~Ostrow, Akhila Ananthram, et~al.
\newblock Openai gpt-5 system card.
\newblock \emph{arXiv preprint arXiv:2601.03267}, 2025.

\bibitem[Team et~al.(2025)Team, Bai, Bao, Chen, Chen, Chen, Chen, Chen, Chen, Chen, et~al.]{team2025kimik2}
Kimi Team, Yifan Bai, Yiping Bao, Guanduo Chen, Jiahao Chen, Ningxin Chen, Ruijue Chen, Yanru Chen, Yuankun Chen, Yutian Chen, et~al.
\newblock Kimi k2: Open agentic intelligence.
\newblock \emph{arXiv preprint arXiv:2507.20534}, 2025.

\bibitem[Vankadara et~al.(2024)Vankadara, Xu, Haas, and Cevher]{mamba-mup}
Leena~Chennuru Vankadara, Jin Xu, Moritz Haas, and Volkan Cevher.
\newblock On feature learning in structured state space models.
\newblock In \emph{NeurIPS}, 2024.

\bibitem[Vaswani et~al.(2017)Vaswani, Shazeer, Parmar, Uszkoreit, Jones, Gomez, Kaiser, and Polosukhin]{DBLP:conf/nips/transformer}
Ashish Vaswani, Noam Shazeer, Niki Parmar, Jakob Uszkoreit, Llion Jones, Aidan~N. Gomez, Lukasz Kaiser, and Illia Polosukhin.
\newblock Attention is all you need.
\newblock In \emph{NIPS}, pages 5998--6008, 2017.

\bibitem[Vershynin(2018)]{hdp}
Roman Vershynin.
\newblock \emph{High-dimensional probability: An introduction with applications in data science}, volume~47.
\newblock Cambridge university press, 2018.

\bibitem[Vyas et~al.(2024)Vyas, Morwani, Zhao, Kwun, Shapira, Brandfonbrener, Janson, and Kakade]{vyas2024soap}
Nikhil Vyas, Depen Morwani, Rosie Zhao, Mujin Kwun, Itai Shapira, David Brandfonbrener, Lucas Janson, and Sham Kakade.
\newblock Soap: Improving and stabilizing shampoo using adam.
\newblock \emph{arXiv preprint arXiv:2409.11321}, 2024.

\bibitem[Xie et~al.(2026)Xie, Luo, Tang, Hu, Liu, Ren, Wang, Zhao, Yan, Su, et~al.]{xie2026controlled-sso}
Tian Xie, Haoming Luo, Haoyu Tang, Yiwen Hu, Jason~Klein Liu, Qingnan Ren, Yang Wang, Wayne~Xin Zhao, Rui Yan, Bing Su, et~al.
\newblock Controlled llm training on spectral sphere.
\newblock \emph{arXiv preprint arXiv:2601.08393}, 2026.

\bibitem[Yang et~al.(2025)Yang, Li, Yang, Zhang, Hui, Zheng, Yu, Gao, Huang, Lv, et~al.]{yang2025qwen3}
An~Yang, Anfeng Li, Baosong Yang, Beichen Zhang, Binyuan Hui, Bo~Zheng, Bowen Yu, Chang Gao, Chengen Huang, Chenxu Lv, et~al.
\newblock Qwen3 technical report.
\newblock \emph{arXiv preprint arXiv:2505.09388}, 2025.

\bibitem[Yang(2020)]{TP3}
Greg Yang.
\newblock Tensor programs {III:} neural matrix laws.
\newblock \emph{CoRR}, abs/2009.10685, 2020.

\bibitem[Yang and Hu(2021)]{TP4}
Greg Yang and Edward~J. Hu.
\newblock Tensor programs {IV:} feature learning in infinite-width neural networks.
\newblock In \emph{{ICML}}, volume 139, pages 11727--11737. {PMLR}, 2021.

\bibitem[Yang and Littwin(2023)]{TP4b}
Greg Yang and Etai Littwin.
\newblock Tensor programs ivb: Adaptive optimization in the infinite-width limit.
\newblock \emph{CoRR}, abs/2308.01814, 2023.

\bibitem[Yang et~al.(2022)Yang, Hu, Babuschkin, Sidor, Liu, Farhi, Ryder, Pachocki, Chen, and Gao]{TP5}
Greg Yang, Edward~J. Hu, Igor Babuschkin, Szymon Sidor, Xiaodong Liu, David Farhi, Nick Ryder, Jakub Pachocki, Weizhu Chen, and Jianfeng Gao.
\newblock Tensor programs {V:} tuning large neural networks via zero-shot hyperparameter transfer.
\newblock \emph{CoRR}, abs/2203.03466, 2022.

\bibitem[Yang et~al.(2023)Yang, Simon, and Bernstein]{mup-spectral}
Greg Yang, James~B. Simon, and Jeremy Bernstein.
\newblock A spectral condition for feature learning.
\newblock \emph{CoRR}, abs/2310.17813, 2023.

\bibitem[Yang et~al.(2024)Yang, Yu, Zhu, and Hayou]{TP-6}
Greg Yang, Dingli Yu, Chen Zhu, and Soufiane Hayou.
\newblock Tensor programs {VI:} feature learning in infinite depth neural networks.
\newblock In \emph{{ICLR}}, 2024.

\bibitem[Zhao et~al.(2024)Zhao, Zhang, Chen, Wang, Anandkumar, and Tian]{DBLP:conf/icml/Zhao0CWAT24-galore}
Jiawei Zhao, Zhenyu Zhang, Beidi Chen, Zhangyang Wang, Anima Anandkumar, and Yuandong Tian.
\newblock Galore: Memory-efficient {LLM} training by gradient low-rank projection.
\newblock In \emph{{ICML}}, 2024.

\bibitem[Zheng et~al.(2025)Zheng, Zhang, Wang, Huang, Tian, Huang, Zhu, and Li]{dit-mup}
Chenyu Zheng, Xinyu Zhang, Rongzhen Wang, Wei Huang, Zhi Tian, Weilin Huang, Jun Zhu, and Chongxuan Li.
\newblock Scaling diffusion transformers efficiently via {\(\mu\)}p.
\newblock In \emph{NeurIPS}, 2025.

\end{thebibliography}
\bibliographystyle{plainnat}

\newpage

\begin{appendices}

\renewcommand{\contentsname}{Contents of Appendix}
\tableofcontents

\addtocontents{toc}{\protect\setcounter{tocdepth}{2}} 

\newpage

\section{Additional Related Work}
\label{app: related work}

\subsection{$\mu$P under Width Scaling}
\label{app: muP under Width Scaling}

$\mu$P was originally introduced to characterize and control training dynamics in the infinite-width limit of neural networks, to enable stable feature learning through appropriate HPs adjustment~\citep{TP4}.
Early theoretical work formalized $\mu$P for MLP trained with SGD using Tensor Programs~\citep{TP3,TP4} and dynamical mean-field theory~\citep{DBLP:conf/nips/BordelonP22-dmft}.
Empirically, \citet{TP5} showed that $\mu$P stabilizes optimal HPs across model widths, thereby substantially reducing the tuning cost when scaling up model size.

Motivated by these advantages, the $\mu$P principle has been successfully extended to a wide range of modern architectures, including convolutional neural networks~\citep{TP4b}, Transformers~\citep{TP4b}, diffusion Transformers~\citep{dit-mup}, and state-space models~\citep{mamba-mup}.
In parallel, $\mu$P has been developed for a broad class of optimization algorithms, such as AdamW~\citep{adamw}, Muon~\citep{extending-mup}, sharpness-aware optimizer~\citep{DBLP:conf/nips/SAM-mup}, second-order optimizers~\citep{DBLP:conf/iclr/second-mup}, low-precision training~\citep{DBLP:conf/iclr/u-muP}, and sparse training~\citep{DBLP:conf/nips/Dey-sparse}.
These $\mu$P-based methods have also been successfully applied to the pretraining of large-scale foundation models in industrial settings~\citep{TP5,DBLP:journals/corr/cerebras,hu2024minicpm,dit-mup}.

Despite substantial progress, $\mu$P formulations are often tightly coupled to specific architectures~\citep{TP4b,dit-mup,mamba-mup} or particular optimization algorithms~\citep{TP4b,DBLP:conf/nips/SAM-mup,DBLP:conf/iclr/second-mup,extending-mup}, and their derivations typically rely on technically involved tools such as Tensor Programs or dynamical mean-field theory~\citep{TP3,TP4,TP4b,DBLP:conf/nips/BordelonP22-dmft}. As a result, it remains difficult to systematically analyze new architectures or optimizers and derive the corresponding $\mu$P formulations. To alleviate this limitation, \citet{mup-spectral} proposed a simple and general spectral condition for realizing $\mu$P in the width-scaling regime, enabling transparent derivations for a broad class of optimization algorithms~\citep{mup-spectral,extending-mup,DBLP:conf/nips/SAM-mup}. However, this spectral perspective focuses solely on width scaling and does not account for depth scaling, which is crucial for modern deep architectures.

\subsection{$\mu$P under Width-Depth Scaling}
\label{app: muP under Width-Depth Scaling}

Recent work has begun to extend the $\mu$P principle beyond pure width scaling to regimes where network depth grows jointly with model size.
Early theoretical analyses~\citep{TP-6,DBLP:conf/iclr/BordelonNLHP24-dmft-depth} of residual networks with one-layer residual blocks trained by SGD or Adam showed that a hidden residual multiplier $\alpha_l$ of order $\Theta(1/\sqrt{L})$ suffices to preserve stable feature learning, but observed that this scaling fails to maintain HP transferability in practical architectures such as Transformers~\citep{TP-6,completep}.
{ Recently,~\citet{chizat2025hidden} instead use $\alpha_l=\Theta(1/L)$ and hidden-layer learning rate $\eta_l = \Theta(L)$ to realize stable feature learning and maximal local feature update.
}

Subsequent studies~\citep{DBLP:conf/nips/BordelonCP24-transformer} of Transformers with two-layer residual blocks trained by SGD using dynamical mean-field theory argued that a stronger hidden residual scaling of $\Theta(1/L)$ is preferable, as it enables both nontrivial feature learning and non-negligible updates in attention layers.
More recently,~\citet{completep} shows that for residual networks with two-layer blocks trained by AdamW, the residual multiplier $\alpha_l$ of $\Theta(1/L)$ is in fact necessary to simultaneously maintain stable feature learning and maximize parameter updates. Moreover,~\citet{completep} empirically find that this parameterization enables HP transfer in GPT-2-style Transformer.
This hidden residual multiplier~\citep{DBLP:conf/nips/BordelonCP24-transformer,completep} is further applied to some matrix-preconditioned optimizers~\citep{muon-cp}.

Overall, existing $\mu$P extensions to the joint width-depth scaling regime remain fragmented, architecture- and optimizer-specific, and often rely on technically involved analyses, motivating the need for a simple and unified framework.

{
\paragraph{Concurrent and independent work.}
\citet{gqa-mup} shares our motivation of extending the spectral framework of~\citet{mup-spectral} to derive $\mu$P parameterizations.
Its distinctive focus is grouped-query attention (GQA) trained by AdamW: it introduces an expected operator norm to account for rank degeneracy induced by key/value sharing and derives the corresponding learning-rate and weight-decay scalings.
Their framework also recovers CompleteP-style depth scaling for GQA trained by AdamW.
Our work emphasizes joint width-depth spectral conditions for general residual networks and a systematic recipe for translating these conditions into HP parameterizations across a broad class of optimizers.
}

\section{Spectral Condition for General Residual Networks}
\label{app: Spectral Condition for General Residual Networks}

In this section, we provide derivations that complement the main-text analysis of the two-layer residual block ($k=2$) in Section~\ref{sec: spec condition}.
The goal is to clarify how the spectral condition changes with the internal depth $k$ of the residual branch and to justify why the two-layer case is the minimal representative of fixed-depth branches with $k\ge2$.
We first analyze one-layer residual blocks ($k=1$), where the absence of second-order update terms leads to {the looser spectral condition in Condition~\ref{condition: one-layer}, which accommodates both Depth-$\mu$P-style and MLU-style choices}.
We then extend the analysis to residual blocks with an arbitrary fixed number $k\ge2$ of transformations, showing that the additional higher-order update terms do not change the resulting $\mu$P implementation beyond the two-layer case.
Finally, we discuss bias parameters and show that they can be incorporated as a lightweight extension without changing the main HP parameterization for the matrix weights.

As in the main text, we assume $\|\vx\|_{\normrms}=\Theta(1)$ for simplicity, which holds for natural image data and one-hot language data ($\Theta(1/\sqrt{d_0})=\Theta(1)$). We also assume the network dimensions satisfy Equation~(\ref{eq:dimensions}). Furthermore, we also use norm estimates based on subadditivity and submultiplicativity to track the typical scales of $\|\vh_l(\vx)\|_{\normrms}$ and $\|\Delta\vh_l(\vx)\|_{\normrms}$.
The tightness justification under width-depth scaling is in Appendix~\ref{app: lower bound}, following the {treatment in Section~\ref{sec: Theoretical Derivation}.}

\subsection{One-layer Residual Block}
\label{app: spec one-layer}

\subsubsection{Problem Setup}

We consider a residual network with one-layer residual blocks ($k=1$), defined as
\begin{align*}
\vh_0(\vx) &= \alpha_0 \mW_0 \vx,\\
\vh_l(\vx) &= \vh_{l-1}(\vx) + \alpha_l \mW_l \vh_{l-1}(\vx), \quad \forall\, l \in [L], \\
\vh_{L+1}(\vx) &= \alpha_{L+1} \mW_{L+1} \vh_L(\vx),
\end{align*}
where $\mW_0 \in \R^{n \times d_0}$, $\mW_l \in \R^{n \times n}$ for $l \in [L]$, and $\mW_{L+1} \in \R^{d_{L+1} \times n}$.
The network output $\vh_{L+1}(\vx) \in \R^{d_{L+1}}$ is used to compute the loss
$\mathcal{L}(\vh_{L+1}(\vx), \vy)$.
This case serves as a reference point for understanding the transition from $k=1$ to $k\ge2$: because each residual branch contains only one transformation, its update expansion has no second-order direct update term.

\subsubsection{Spectral Scaling Condition}

We now state the spectral scaling condition for the above residual network with one-layer blocks for realizing the $\mu$P Principle~\ref{principle: mup} under joint width–depth scaling.

\begin{condition}[Spectral condition for $\mu$P under joint width-depth scaling, one-layer residual block]
\label{condition: one-layer}
To realize $\mu$P Principle~\ref{principle: mup}, the initial weights and their per-step updates should satisfy:
\begin{itemize}
    \item \textbf{Initial condition.}
    \begin{itemize}
        \item Input and output weights:
        $\alpha_0 \|\mW_0\|_{\normrms},\ \alpha_{L+1}\|\mW_{L+1}\|_{\normrms}=\Theta(1)$.
        \item Hidden weights: $\alpha_l\|\mW_l\|_{\normrms}=\mathcal{O}(1/\sqrt{L})$,
        $\forall l\in[L]$.
    \end{itemize}
    \item \textbf{Update condition.}
    \begin{itemize}
        \item Input and output weights:
        $\alpha_0\|\Delta\mW_0\|_{\normrms},\ \alpha_{L+1}\|\Delta\mW_{L+1}\|_{\normrms}=\Theta(1)$.
        \item Hidden weights (first-order):
        $\alpha_l\|\Delta\mW_l\|_{\normrms}
        =\Theta(1/L),\ \forall l\in[L].$
    \end{itemize}
\end{itemize}
\end{condition}

{
The essential distinction between one-layer and two-layer residual blocks lies in the
\emph{order of the weight-update terms} that directly affect feature evolution.
For a one-layer residual block, the feature update expansion contains only zero-order ($\vepsilon_0(L)$) and first-order ($\vepsilon_1(L)$) terms in the weight updates (see details in Appendix~\ref{app: one-layer Derivation for Update Condition}).
Principle~(\ref{principle: max}) applies only to the first-order direct-update term, leading to $\alpha_l\|\Delta \mW_l\|_{\normrms}=\Theta(1/L)$; the zero-order propagation term contains no current-block weight update and is therefore outside Principle~(\ref{principle: max}),
while forward stability only requires the preliminary condition
$\alpha_l\|\mW_l\|_{\normrms}=\mathcal{O}(1/\sqrt{L})$.
Given the standard width-$\mu$P initialization $\|\mW_l\|_{\normrms}=\Theta(1)$ in Equation~(\ref{eqn:spec-width}), Condition~\ref{condition: one-layer} therefore permits any $\alpha_l=\mathcal{O}(1/\sqrt{L})$, including both $\Theta(1/\sqrt{L})$ and $\Theta(1/L)$.

Depth-$\mu$P-style formulation~\citep{TP-6,DBLP:conf/iclr/BordelonNLHP24-dmft-depth} selects the boundary $\alpha_l=\Theta(1/\sqrt{L})$, which they interpret as further
promoting feature diversity.
Within our spectral framework, this choice is unified as a natural case corresponding
to additionally maximizing the magnitude of the zero-order feature update $\|\vepsilon_0(L)\|_{\normrms}$ (see derivations
in Appendix~\ref{app: one-layer Derivation for Update Condition}).
Given $\alpha_l\|\Delta \mW_l\|_{\normrms}=\Theta(1/L)$, this boundary choice gives $\|\Delta\mW_l\|_{\normrms}=\Theta(1/\sqrt{L})$, so an individual block's local parameter change vanishes even though the accumulated feature update remains $\Theta(1)$.
By contrast, the feasible choice $\alpha_l=\Theta(1/L)$ gives $\|\Delta\mW_l\|_{\normrms}=\Theta(1)$, corresponding to the maximal local update (MLU) regime of~\citet{chizat2025hidden}.
Thus, the two parameterizations select different cases within Condition~\ref{condition: one-layer}, rather than imposing conflicting spectral conditions.
We defer the complete derivation of MLU for SGD to Appendix~\ref{app:sgd dmup}.
}

In contrast, two-layer residual blocks introduce \emph{second-order} update terms
arising from products of weight updates across the two sublayers.
To satisfy the $\mu$P principle~(\ref{principle: max}), these second-order contributions should be maximized to $\Theta(1)$ as in (\ref{eq:update_hidd_2}).
This requirement imposes an additional constraint on the scaling of weight updates,
which in turn tightens the initialization condition to $
\alpha_l\|\mW_l^{(1)}\|_{\normrms}\,\|\mW_l^{(2)}\|_{\normrms}=\Theta(1/L)
$ in~(\ref{eq:init_hidd}) and thus the residual multiplier to $\alpha_l = \Theta(1/L)$.
This important difference explains why the Depth-$\mu$P-style scaling ($k=1$) does not directly capture residual branches with two or more transformations, and helps account for its poor depth-wise HP transfer behavior in Transformer experiments in Section~\ref{sec: experiment} and in~\citet{TP-6,completep}.
{
We also note that, although the MLU choice shares $\alpha_l=\Theta(1/L)$ with CompleteP, the two are not equivalent because the derived rule for other HPs (e.g., learning rate) differs.}

\subsubsection{Derivation for Initial Condition}

We first derive the initialization condition that ensures stability
of feature magnitudes during forward propagation for single-layer residual
blocks. We consider each layer sequentially.

\paragraph{Input layer.}
The argument is identical to the two-layer case. By the submultiplicativity of the RMS operator norm, we have
\begin{align*}
\Vert\vh_0(\vx)\Vert_\normrms
= \alpha_0\Vert\mW_0\vx\Vert_\normrms
=  \Theta(\alpha_0\Vert\mW_0\Vert_\normrms \Vert\vx\Vert_\normrms)
= \Theta(\alpha_0\Vert\mW_0\Vert_\normrms),
\end{align*}
where we assume $\|\vx\|_{\normrms}=\Theta(1)$. Thus, choosing
$\alpha_0\|\mW_0\|_{\normrms}=\Theta(1)$ ensures
$\|\vh_0(\vx)\|_{\normrms}=\Theta(1)$.

\paragraph{Hidden layers.}
For a single-layer residual block, the forward recursion is
\[
\vh_l(\vx) = \vh_{l-1}(\vx) + \alpha_l\mW_l \vh_{l-1}(\vx).
\]
Expanding the recursion yields
\begin{align}
\vh_s(\vx)
= \vh_0(\vx) + \sum_{l=1}^s \alpha_l\mW_l \vh_{l-1}(\vx).
\label{eqn: hl_single}
\end{align}
Applying subadditivity, we can estimate their order as
\begin{align*}
\|\vh_s(\vx)\|_\normrms
= \Theta\left(\|\vh_0(\vx)\|_\normrms + \|\sum_{l=1}^s \alpha_l\mW_l \vh_{l-1}(\vx)\|_\normrms\right).
\end{align*}
Since we have $\|\vh_0(\vx)\|_{\normrms}=\Theta(1)$, it suffices to ensure that
$\|\sum_{l=1}^s \alpha_l \mW_l \vh_{l-1}(\vx)\|_{\normrms}=\mathcal{O}(1)$  for any $s\in[L]$ to preserve $\|\vh_s(\vx)\|_{\normrms}=\Theta(1)$.
Under i.i.d.\ zero-mean Gaussian initialization, the summands are approximately independent zero-mean random vectors~\citep{TP4,TP-6,completep}, so the typical squared RMS norm of their sum scales as the sum of the squared RMS norms (see Theorem~3.3.1 in~\citet{hdp}), yielding that
$$
\|\sum_{l=1}^s \alpha_l \mW_l \vh_{l-1}(\vx)\|_{\normrms} = \Theta\left(
\sqrt{\sum_{l=1}^s \|\alpha_l \mW_l \vh_{l-1}(\vx)\|_{\normrms}^2}\right).
$$
By submultiplicativity, we can further estimate $\|\alpha_l \mW_l \vh_{l-1}(\vx)\|_{\normrms} =  \Theta( \alpha_l
\|\mW_l\|_{\normrms}
\|\vh_{l-1}(\vx)\|_{\normrms})$.
Therefore, starting from $\|\vh_0(\vx)\|_{\normrms}=\Theta(1)$, imposing
$$\alpha_l\|\mW_l\|_{\normrms}=\mathcal{O}(1/\sqrt{L}),\quad l\in[L]$$ recursively ensures
$\|\sum_{l=1}^s \alpha_l\mW_l\vh_{l-1}(\vx)\|_{\normrms}
=\mathcal{O}(1)$ for any $s\in[L]$.
This provides the initial condition on the hidden weights.

\paragraph{Output layer.}
{The same submultiplicativity argument as for the two-layer block case gives
\begin{align*}
\|\vh_{L+1}(\vx)\|_{\normrms}
&=
\|\alpha_{L+1} \mW_{L+1}\vh_L(\vx)\|_{\normrms}
=
\mathcal{O}(\alpha_{L+1} \|\mW_{L+1}\|_{\normrms}\|\vh_L(\vx)\|_{\normrms})
=
\mathcal{O}(\alpha_{L+1}\|\mW_{L+1}\|_{\normrms}).
\end{align*}
As in the analysis for the two-layer block case, no matching lower bound is assumed at initialization. Therefore, choosing $\alpha_{L+1}\|\mW_{L+1}\|_{\normrms}=\Theta(1)$ is only sufficient to ensure $\|\vh_{L+1}(\vx)\|_{\normrms}=\mathcal{O}(1)$. Nevertheless, this suffices for Principle~(\ref{principle: stable}), which constrains $\vh_l$ only for $l\in[L]$ and does not require $\|\vh_{L+1}(\vx)\|_{\normrms}=\Theta(1)$ at initialization.}
It completes the initialization analysis.

\subsubsection{Derivation for Update Condition}
\label{app: one-layer Derivation for Update Condition}

We next derive the update condition required to ensure stable feature evolution,
i.e., $\|\Delta \vh_l(\vx)\|_{\normrms}=\Theta(1)$, while maximally updating parameters as prescribed by $\mu$P Principle~(\ref{principle: max}).

\paragraph{Input layer.}
Since $\Delta\vh_0(\vx)=\alpha_0 \Delta\mW_0\vx$, submultiplicativity yields
\[
\|\Delta\vh_0(\vx)\|_{\normrms}
=
\Theta(\alpha_0 \|\Delta\mW_0\|_{\normrms}\|\vx\|_{\normrms})
=
\Theta(\alpha_0 \|\Delta\mW_0\|_{\normrms}),
\]
and thus we set $\alpha_0\|\Delta\mW_0\|_{\normrms}=\Theta(1)$.

\paragraph{Hidden layers.}
Expanding Equation~(\ref{eqn: hl_single}) after a single gradient step gives
\begin{align*}
\Delta\vh_s(\vx)
&=
\Delta\vh_0(\vx)
+
\underbrace{\sum_{l=1}^s \alpha_l \mW_l \Delta\vh_{l-1}(\vx)}_{\vepsilon_0(s)}
+
\underbrace{\sum_{l=1}^s \alpha_l \Delta\mW_l (\vh_{l-1}(\vx)+\Delta\vh_{l-1}(\vx))}_{\vepsilon_1(s)}.
\end{align*}
\emph{Unlike the two-layer case, there is no second-order update term}, since each residual block contains only a single weight matrix.
By the subadditivity of vector norms, we have
\begin{align*}
\Vert\Delta\vh_s(\vx)\Vert_\normrms = \Theta(\Vert\Delta\vh_0(\vx)\Vert_\normrms + \Vert\vepsilon_0(s)\Vert_\normrms + \Vert\vepsilon_1(s)\Vert_\normrms).
\end{align*}
Since $\|\Delta\vh_0(\vx)\|_{\normrms}=\Theta(1)$ by the input-layer update, we have
$\|\Delta\vh_s(\vx)\|_{\normrms}=\Omega(1)$ for all $s\in[L]$.
Moreover, by subadditivity, the remaining terms do not decay with depth, implying
$\|\Delta\vh_s(\vx)\|_{\normrms}=\mathcal{O}(\|\Delta\vh_L(\vx)\|_{\normrms})$ for any $s\in[L]$.
Therefore, to enforce Principle~\ref{principle: mup},
it suffices to require $\|\Delta\vh_L(\vx)\|_{\normrms}=\Theta(1)$ while satisfying
Principle~(\ref{principle: max}).

\textbf{Zero-order term.}
The term $\vepsilon_0(L)$ propagates feature updates from earlier layers and does not depend on the weight update $\Delta\mW_l$ at the current layer, so it does not need to be maximized from Principle~(\ref{principle: max}).
Therefore, it suffices to verify that $\vepsilon_0(L)$ remains $\mathcal{O}(1)$ under the initial condition. In fact, the same argument used for deriving $\Vert\vh_L(\vx)\Vert_\normrms$ directly implies $$\|\vepsilon_0(L)\|_{\normrms}
=
\Theta\left(
\sqrt{\sum_{l=1}^L \alpha_l^2
\|\mW_l\|_{\normrms}^2
\|\Delta\vh_{l-1}(\vx)\|_{\normrms}^2
}
\right)
= \mathcal{O}(1),$$
where we use the self-consistent fact that $\Vert \Delta\vh_{l-1}(\vx)\Vert_\normrms = \Theta(1)$ for $l\in[L]$ if we finally set $\Vert \Delta\vh_L(\vx)\Vert_\normrms= \Theta(1)$. 
{\emph{The additional feature-diversity criterion in the original Depth-$\mu$P formulation~\citep{TP-6} selects the boundary $\alpha_l\|\mW_l\|_{\normrms}=\Theta(1/\sqrt{L})$ for all $l\in[L]$; equivalently, this makes $\|\vepsilon_0(L)\|_{\normrms}=\Theta(1)$ maximized. We note that this additional criterion is not part of Principle~(\ref{principle: max}).}}

\textbf{First-order terms.}
The first-order update terms reflect the direct effect of weight updates $\Delta\mW_l$ on features
and must be maximized ($\Theta(1)$) to satisfy the $\mu$P Principle~(\ref{principle: max}).
Using subadditivity and submultiplicativity, we estimate the order of $\|\vepsilon_1(L)\|_{\normrms}$ as
\begin{align*}
\|\vepsilon_1(L)\|_{\normrms}
= \Theta\left(\sum_{l=1}^L \alpha_l
\|\Delta\mW_l\|_{\normrms}
\|\vh_{l-1}(\vx)\|_\normrms\right) + \Theta\left(\sum_{l=1}^L \alpha_l
\|\Delta\mW_l\|_{\normrms}
\|\Delta\vh_{l-1}(\vx)\|_{\normrms}\right).
\end{align*}
For $l\in[L]$, using $\|\vh_{l-1}(\vx)\|_{\normrms}=\Theta(1)$ by the preliminary initial condition and the self-consistent fact that $\Vert \Delta\vh_{l-1}(\vx)\Vert_\normrms = \Theta(1)$ if we finally set $\Vert \Delta\vh_L(\vx)\Vert_\normrms= \Theta(1)$, we can obtain
$
\|\vepsilon_1(L)\|_{\normrms}
= \Theta\big(\sum_{l=1}^L \alpha_l
\|\Delta\mW_l\|_{\normrms}\big).
$
To satisfy Principle~(\ref{principle: max}), we need to maximize the contribution from each $\Delta\mW_l$ and ensure $\|\vepsilon_1(L)\|_{\normrms}=\Theta(1)$ at the same time, which naturally requires
\[
\alpha_l \|\Delta\mW_l\|_{\normrms}
= \Theta(1/L), \qquad \forall\, l\in[L],
\]
which completes the first-order update condition on hidden weights.

\paragraph{Output layer.}
The output layer has the same form as in the two-layer block case ($k=2$), since it depends only on $\vh_L(\vx)$ and not on the internal structure of the residual blocks.
Given $\|\vh_L(\vx)\|_{\normrms}=\Theta(1)$ and
$\|\Delta\vh_L(\vx)\|_{\normrms}=\Theta(1)$ from the hidden-layer analysis, the same argument as in Section~\ref{sec: Theoretical Derivation} yields
\[
\alpha_{L+1}\|\Delta\mW_{L+1}\|_{\normrms}=\Theta(1).
\]

\subsection{Multi-layer Residual Block}
\label{app: spectral multi-layer}

This subsection justifies the main-text choice of analyzing the two-layer block ($k=2$) as the representative case for all fixed-depth residual branches with $k\ge2$.

\subsubsection{Problem Setup}

We now extend the spectral analysis from one- and two-layer residual blocks
to the general case of $k$-layer residual blocks, where $k \ge 2$ is a fixed $\Theta(1)$ constant.
The fixed-$k$ assumption is important: we do not address regimes where the internal block depth itself scales with width or network depth.
Specifically, we consider a residual network of depth $L$ whose forward propagation is given by
\begin{align*}
\vh_0(\vx) &= \alpha_0 \mW_0 \vx,\\
\vh_l(\vx) &= \vh_{l-1}(\vx)
+ \alpha_l \mW_l^{(k)} \mW_l^{(k-1)} \cdots \mW_l^{(1)} \vh_{l-1}(\vx)
= \vh_{l-1}(\vx) + \alpha_l \prod_{i=1}^k \mW_l^{(i)} \vh_{l-1}(\vx),
\quad \forall\, l \in [L], \\
\vh_{L+1}(\vx) &= \alpha_{L+1} \mW_{L+1} \vh_L(\vx).
\end{align*}
Here, each residual block consists of a depth-$k$ linear transformation,
with $\{\mW_l^{(i)}\}_{i=1}^k$ denoting the weight matrices within the $l$-th block.
As in the previous sections, $\vh_{L+1}(\vx)$ denotes the network output used to compute the loss.
As in the two-layer block case, our goal is to derive a spectral condition for realizing $\mu$P Principle~\ref{principle: mup} in this setting.

In the following, we show that although increasing the internal block depth $k$
introduces higher-order interactions between weight updates, the resulting spectral conditions admit a simple and systematic form, and do not fundamentally alter the algorithmic implementation of $\mu$P.

\subsubsection{Spectral Scaling Condition}

We now state the spectral scaling condition for the above residual network with $k$-layer residual blocks that is sufficient for the $\mu$P principle under joint width–depth scaling.

\begin{condition}[Spectral condition for $\mu$P under joint width-depth scaling, $k$-layer residual block]
\label{condition: multi-layer}
To realize $\mu$P Principle~\ref{principle: mup}, the initial weights and their per-step updates should satisfy:
\begin{itemize}
    \item \textbf{Initial condition.}
    \begin{itemize}
        \item Input and output weights:
        $\alpha_0 \|\mW_0\|_{\normrms},\ \alpha_{L+1}\|\mW_{L+1}\|_{\normrms}=\Theta(1)$.
        \item Hidden weights: $\alpha_l\prod_{i=1}^k \|\mW_l^{(i)}\|_{\normrms}=\Theta(1/L)$,
        $\forall l\in[L]$.
    \end{itemize}
    \item \textbf{Update condition.}
    \begin{itemize}
        \item Input and output weights:
        $\alpha_0\|\Delta\mW_0\|_{\normrms},\ \alpha_{L+1}\|\Delta\mW_{L+1}\|_{\normrms}=\Theta(1)$.
        \item Hidden weights (first-order):
        $\alpha_l\|\Delta\mW_l^{(i)}\|_{\normrms}
\prod_{m\neq i}\|\mW_l^{(m)}\|_{\normrms}
=\Theta(1/L),\ \forall l\in[L], i\in[k].$
        \item Hidden weights ($j$-order, $j\ge2$), automatically satisfied by combining the initial condition and the first-order update condition: $\alpha_l \prod_{i\in S}\|\Delta\mW_l^{(i)}\|_{\normrms}
\prod_{i\notin S}\|\mW_l^{(i)}\|_{\normrms}
=
\Theta(1/L),\ \forall S \subseteq [k],\ |S|=j,\ j\in [k],\ l\in[L]$.
    \end{itemize}
\end{itemize}
\end{condition}

Condition~\ref{condition: multi-layer} reveals that extending residual blocks from two layers to a general fixed depth $k\ge2$ does not change the algorithmic realization of the $\mu$P principle.
Compared to the two-layer case, the new elements introduced by a deeper block are higher-order interaction terms among weight updates within the same block.
However, we show these higher-order terms do not impose additional constraints beyond those already enforced by the initial condition and the first-order update condition.

Concretely, once the product of spectral norms at initialization satisfies
$\alpha_l \prod_{i=1}^k \|\mW_l^{(i)}\|_{\normrms}=\Theta(1/L)$
and each update obeys the first-order scaling
$\alpha_l\|\Delta\mW_l^{(i)}\|_{\normrms}\prod_{m\neq i}\|\mW_l^{(m)}\|_{\normrms}=\Theta(1/L)$, all higher-order update contributions of order $j\ge2$ are automatically controlled as $\Theta(1/L)$.
As a result, increasing the internal block depth $k$ only increases the number of such higher-order contributions, but does not alter their scaling behavior.

Following the same steps as derivations for implementations in Section~\ref{sec: implementation} and Appendix~\ref{app: implementation opts HPs}, we can find that \emph{implementing $\mu$P for a $k$-layer residual block requires no additional parameterization beyond those already needed for the two-layer case}.
In particular, when the initialization variance is aligned with the standard width-scaling $\mu$P formulation~\citep{TP4,TP5} as in Section~\ref{sec: implementation} ($\|\mW_l\|_\normrms=\Theta(1),\ \forall l\in[L]$),
the initial condition still induces the residual multiplier $\alpha_l=\Theta(1/L)$ for $l\in[L]$, which is the same as the two-layer case. Built upon the initial condition, the first-order update condition is reduced to
\begin{align*}
\alpha_l\|\Delta\mW_l^{(i)}\|_{\normrms}\prod_{m\neq i}\|\mW_l^{(m)}\|_{\normrms} 
=
\Theta\left(\frac{1}{L} \|\Delta{\mW}_l^{(i)}\|_{\normrms} \right).
\end{align*}
Therefore, requiring the first-order update condition yields $\|\Delta{\mW}_l^{(i)}\|_{\normrms} = \Theta(1)$ for $\forall l\in[L],i\in[k]$. This is also in the same way as the two-layer case ($\|\Delta{\mW}_l^{(i)}\|_{\normrms} = \Theta(1)$ for $\forall l\in[L],i\in[2]$), thus leading to the same optimizer-related HPs adjustment.
The multi-layer analysis, therefore, serves to justify the robustness and generality of the two-layer $\mu$P prescription, rather than to introduce a distinct algorithm dependent on block depth.

\subsubsection{Derivation for Preliminary Initial Condition}

We first derive a preliminary initialization condition that guarantees
stability of feature magnitudes during forward propagation for $k$-layer
($k \ge 2$) residual blocks.
As in the two-layer case, we analyze each layer sequentially.

\paragraph{Input layer.}
By the submultiplicativity of the RMS operator norm, we have
\begin{align*}
\|\vh_0(\vx)\|_{\normrms}
=
\alpha_0 \|\mW_0 \vx\|_{\normrms}
=
\Theta(\alpha_0 \|\mW_0\|_{\normrms}\,\|\vx\|_{\normrms})
=
\Theta(\alpha_0 \|\mW_0\|_{\normrms}),
\end{align*}
where we have assumed $\|\vx\|_{\normrms}=\Theta(1)$.
Thus, choosing $\alpha_0\|\mW_0\|_{\normrms}=\Theta(1)$ ensures
$\|\vh_0(\vx)\|_{\normrms}=\Theta(1)$.

\paragraph{Hidden layers.}
Expanding the residual recursion yields
\begin{align}
\vh_s(\vx)
&=
\vh_{s-1}(\vx)
+
\alpha_s \prod_{i=1}^k \mW_s^{(i)} \vh_{s-1}(\vx)
= \cdots
= \vh_0(\vx)
+ \sum_{l=1}^s \alpha_l
\prod_{i=1}^k \mW_l^{(i)} \vh_{l-1}(\vx).
\label{eqn: hl-k}
\end{align}
Applying subadditivity, we can estimate their order as
\begin{align*}
\|\vh_s(\vx)\|_{\normrms}
=
\Theta\left(\|\vh_0(\vx)\|_{\normrms}
+
\Big\|
\sum_{l=1}^s \alpha_l
\prod_{i=1}^k \mW_l^{(i)} \vh_{l-1}(\vx)
\Big\|_{\normrms}\right).
\end{align*}
Since we have $\|\vh_0(\vx)\|_{\normrms}=\Theta(1)$, it suffices to ensure that
$\|\sum_{l=1}^s \alpha_l
\prod_{i=1}^k \mW_l^{(i)} \vh_{l-1}(\vx)
\|_{\normrms}=\mathcal{O}(1)$  for any $s\in[L]$ to preserve $\|\vh_s(\vx)\|_{\normrms}=\Theta(1)$.
Under i.i.d.\ zero-mean Gaussian initialization, the summands are approximately independent zero-mean random vectors~\citep{TP4,TP-6,completep}, so the typical squared RMS norm of their sum scales as the sum of the squared RMS norms (see Theorem~3.3.1 in~\citet{hdp}), yielding that
$$
\left\|\sum_{l=1}^s \alpha_l
\prod_{i=1}^k \mW_l^{(i)} \vh_{l-1}(\vx)
\right\|_{\normrms} = \Theta\left(
\sqrt{\sum_{l=1}^s \|\alpha_l
\prod_{i=1}^k \mW_l^{(i)} \vh_{l-1}(\vx)\|_{\normrms}^2}\right).
$$
By submultiplicativity, we can further estimate $\|\alpha_l
\prod_{i=1}^k \mW_l^{(i)} \vh_{l-1}(\vx)\|_{\normrms} =  \Theta( \alpha_l
\prod_{i=1}^k \|\mW_l^{(i)}\|_\normrms
\|\vh_{l-1}(\vx)\|_{\normrms})$.
Therefore, starting from $\|\vh_0(\vx)\|_{\normrms}=\Theta(1)$, imposing
$$\alpha_l
\prod_{i=1}^k \|\mW_l^{(i)}\|_\normrms=\mathcal{O}(1/\sqrt{L}),\quad l\in[L]$$ recursively ensures
$\|\sum_{l=1}^s \alpha_l
\prod_{i=1}^k \mW_l^{(i)} \vh_{l-1}(\vx)
\|_{\normrms}=\mathcal{O}(1)$ for any $s\in[L]$.
This provides a preliminary initial condition on the hidden weights, which will be further refined once update constraints are incorporated.

\paragraph{Output layer.}
{The output layer has the same form as in the two-layer case ($k=2$), and submultiplicativity gives
\[
\|\vh_{L+1}(\vx)\|_{\normrms}
= \mathcal{O}(\alpha_{L+1}\|\mW_{L+1}\|_{\normrms}\|\vh_L(\vx)\|_{\normrms})
=\mathcal{O}(\alpha_{L+1}\|\mW_{L+1}\|_{\normrms}).
\]
As in the analysis for the two-layer block case, no matching lower bound is assumed at initialization. Thus, $\alpha_{L+1}\|\mW_{L+1}\|_{\normrms}=\Theta(1)$ remains only a sufficient readout condition ensuring $\|\vh_{L+1}(\vx)\|_{\normrms}=\mathcal{O}(1)$. Nevertheless, this suffices for Principle~(\ref{principle: stable}), which constrains $\vh_l$ only for $l\in[L]$ and does not require $\|\vh_{L+1}(\vx)\|_{\normrms}=\Theta(1)$ at initialization.}
This completes the preliminary initialization analysis.

\subsubsection{Derivation for Update Condition}

We next derive the update conditions required to ensure stable feature evolution by Principle~(\ref{principle: stable}),
i.e., $\|\Delta \vh_l(\vx)\|_{\normrms}=\Theta(1)$,
while maximally updating parameters as prescribed by Principle~(\ref{principle: max}).

\paragraph{Input layer.}
Since $\Delta\vh_0(\vx)=\alpha_0\Delta\mW_0\vx$, the submultiplicativity yields
\[
\|\Delta\vh_0(\vx)\|_{\normrms}
=
\Theta(\alpha_0\|\Delta\mW_0\|_{\normrms}\|\vx\|_{\normrms})
=
\Theta(\alpha_0\|\Delta\mW_0\|_{\normrms}),
\]
and thus we set $\alpha_0\|\Delta\mW_0\|_{\normrms}=\Theta(1)$.

\paragraph{Hidden layers.}
Expanding the residual recursion in Equation~(\ref{eqn: hl-k}) after one update step gives
\begin{align*}
\Delta\vh_s(\vx)
&=
\Delta\vh_0(\vx)
+
\underbrace{\sum_{l=1}^s \alpha_l
\prod_{i=1}^k \mW_l^{(i)} \Delta\vh_{l-1}(\vx)}_{\vepsilon_0(s)}
+
\sum_{j=1}^k \vepsilon_j(s),
\end{align*}
where $\vepsilon_j(s)$ collects all terms that are \emph{$j$-th order} in
$\{\Delta \mW_l^{(i)}\}_{i=1}^k$.
By the subadditivity of vector norms, we have
\begin{align*}
\Vert\Delta\vh_s(\vx)\Vert_\normrms = \Theta\left(\Vert\Delta\vh_0(\vx)\Vert_\normrms + \Vert\vepsilon_0(s)\Vert_\normrms + \sum_{j=1}^k\Vert\vepsilon_j(s)\Vert_\normrms\right).
\end{align*}
Since $\|\Delta\vh_0(\vx)\|_{\normrms}=\Theta(1)$ by the input-layer update, we have
$\|\Delta\vh_s(\vx)\|_{\normrms}=\Omega(1)$ for all $s\in[L]$.
Moreover, by subadditivity, the remaining terms do not decay with depth, implying
$\|\Delta\vh_s(\vx)\|_{\normrms}=\mathcal{O}(\|\Delta\vh_L(\vx)\|_{\normrms})$ for any $s\in[L]$.
Therefore, to enforce Principle~\ref{principle: mup},
it suffices to require $\|\Delta\vh_L(\vx)\|_{\normrms}=\Theta(1)$ while satisfying
Principle~(\ref{principle: max}).

\textbf{Zero-order term.}
The term $\vepsilon_0(L)$ propagates feature updates from earlier layers and does not depend on the weight update $\Delta\mW_l$ at the current layer, so it does not need to be maximized from Principle~(\ref{principle: max}).
Therefore, it suffices to verify that $\vepsilon_0(L)$ remains $\mathcal{O}(1)$ under the preliminary initial condition. In fact, the same argument used for deriving $\Vert\vh_L(\vx)\Vert_\normrms$ directly implies 
$$\|\vepsilon_0(L)\|_{\normrms}
=
\Theta\left(
\sqrt{\sum_{l=1}^L \alpha_l^2
\prod_{i=1}^k \|\mW_l^{(i)}\|_{\normrms}^2
\|\Delta\vh_{l-1}(\vx)\|_{\normrms}^2 }
\right)
= \mathcal{O}(1),$$
where we use the self-consistent fact that $\|\Delta\vh_{l-1}(\vx)\|_{\normrms}=\Theta(1)$ for $l \in [L]$ if we finally enforce $\|\Delta\vh_{L}(\vx)\|_{\normrms}=\Theta(1)$.

\paragraph{First-order terms.}
The first-order contributions take the form
\[
\vepsilon_1(L)
=
\sum_{l=1}^L \alpha_l
\sum_{i=1}^k
\Big(
\mW_l^{(k)}\cdots \Delta\mW_l^{(i)} \cdots \mW_l^{(1)}
\Big)
(\vh_{l-1}(\vx)+\Delta\vh_{l-1}(\vx)).
\]
Using subadditivity and submultiplicativity,
\[
\|\vepsilon_1(L)\|_{\normrms}
=
\Theta\left(
\sum_{l=1}^L \alpha_l
\sum_{i=1}^k
\|\Delta\mW_l^{(i)}\|_{\normrms}
\prod_{m\neq i}\|\mW_l^{(m)}\|_{\normrms}
\right)
=
\sum_{i=1}^k \Theta\left(
\sum_{l=1}^L \alpha_l
\|\Delta\mW_l^{(i)}\|_{\normrms}
\prod_{m\neq i}\|\mW_l^{(m)}\|_{\normrms}
\right),
\]
where we used
$\|\vh_{l-1}(\vx)\|_{\normrms}=\Theta(1)$ for $l\in [L]$ by the preliminary initial condition and the self-consistent fact that $\|\Delta\vh_{l-1}(\vx)\|_{\normrms}=\Theta(1)$ for $l \in [L]$ if we finally enforce $\|\Delta\vh_{L}(\vx)\|_{\normrms}=\Theta(1)$. To satisfy Principle~(\ref{principle: max}), we need to maximize the contribution from each $\Delta\mW_l$ and ensure $\|\vepsilon_1(L)\|_{\normrms}=\Theta(1)$ at the same time, which naturally requires
\[
\alpha_l \|\Delta\mW_l^{(i)}\|_{\normrms}
\prod_{m\neq i}\|\mW_l^{(m)}\|_{\normrms}
=
\Theta(1/L),
\quad \forall\, l\in[L],\ i\in[k].
\]

\paragraph{Any $j$-order terms.}

Similar to the first-order term, for $j\in [k]$, the $j$-th order feature update term $\vepsilon_j(L)$ admits the explicit form
{
\[
\vepsilon_j(L)
=
\sum_{l=1}^L \alpha_l
\sum_{\substack{S \subseteq [k] \\ |S| = j}}
\widetilde{\mW}_{l,S}^{(k)}\cdots\widetilde{\mW}_{l,S}^{(1)}
(\vh_{l-1}(\vx) + \Delta\vh_{l-1}(\vx)),
\]
where
\[
\widetilde{\mW}_{l,S}^{(i)}=
\begin{cases}
\Delta\mW_l^{(i)}, & i\in S,\\
\mW_l^{(i)}, & i\notin S.
\end{cases}
\]
}
Therefore, by the same subadditivity and submultiplicativity arguments for $\|\vepsilon_1(L)\|_{\normrms}$, the $j$-th order update terms satisfy
\[
\|\vepsilon_j(L)\|_{\normrms}
=
\sum_{\substack{S \subseteq [k] \\ |S| = j}}
\Theta\left(
\sum_{l=1}^L \alpha_l
\prod_{i\in S}\|\Delta\mW_l^{(i)}\|_{\normrms}
\prod_{i\notin S}\|\mW_l^{(i)}\|_{\normrms}
\right).
\]
Principle~(\ref{principle: max}) requires maximizing each summand while ensuring $\|\vepsilon_j(L)\|_{\normrms}=\Theta(1)$. It is therefore sufficient to impose
\begin{align*}
\alpha_l \prod_{i\in S}\|\Delta\mW_l^{(i)}\|_{\normrms}
\prod_{i\notin S}\|\mW_l^{(i)}\|_{\normrms}
=
\Theta(1/L),
\quad \forall\, S \subseteq [k],\ |S|=j,\ j\in [k],\ l\in[L].
\end{align*}

\paragraph{Output layer.}
The same argument as in the two-layer case in Section~\ref{sec: spec condition} yields
$\alpha_{L+1}\|\Delta\mW_{L+1}\|_{\normrms}=\Theta(1)$.

\subsubsection{Derivation for Final Initial Condition}

Multiplying the first-order update conditions for each hidden weight yields
\begin{align*}
\alpha_l^k \prod_{i=1}^k\|\mW_l^{(i)}\|_{\normrms}^{k-1}
\prod_{i=1}^k \|\Delta\mW_l^{(i)}\|_{\normrms}
= \Theta(1/L^k),\quad \forall l\in[L].
\end{align*}
On the other hand, the highest $k$-order update condition is
\(
\alpha_l \prod_{i=1}^k \|\Delta\mW_l^{(i)}\|_{\normrms}
= \Theta(1/L)
\)
for all $l\in[L]$.
Combining the two relations immediately gives
\[
\alpha_l\prod_{i=1}^k \|\mW_l^{(i)}\|_{\normrms}
=
\Theta(1/L),
\quad \forall\, l\in[L],
\]
which refines the preliminary initialization condition.

Finally, as in the two-layer case, we prove that \emph{the refined initial condition and the first-order update condition can derive any $j$-order ($j\ge2$) update condition on hidden weights}. Thus, retaining the refined initial condition and the first-order update condition in Condition~\ref{condition: multi-layer} is sufficient. 

Formally, for any $S \subseteq [k],\ |S|=j,\ j\in [k],\ l\in[L]$, we need to prove that
\[
\alpha_l \prod_{i\in S}\|\Delta\mW_l^{(i)}\|_{\normrms}
\prod_{i\notin S}\|\mW_l^{(i)}\|_{\normrms}
=
\Theta(1/L)
\]
based on the refined initial condition and the first-order update condition. By multiplying the first-order update conditions, we have
\begin{align*}
{\Theta(1/L^j)} &= \prod_{i\in S}\left(\alpha_l\|\Delta\mW_l^{(i)}\|_{\normrms}
\prod_{m\neq i}\|\mW_l^{(m)}\|_{\normrms}\right) \\
&=
\alpha_l^j \left(\prod_{i\in S}\|\Delta\mW_l^{(i)}\|_{\normrms}\right)
\left(\prod_{i\in S} \prod_{m\neq i}\|\mW_l^{(m)}\|_{\normrms}\right) \\
&=\alpha_l^j
\left(\prod_{i\in S}\|\Delta\mW_l^{(i)}\|_{\normrms}\right)
\left(
\prod_{i\notin S}\|\mW_l^{(i)}\|_{\normrms}
\right)
\left( \prod_{i=1}^k\|\mW_l^{(i)}\|_{\normrms}\right)^{j-1} \\
&= \left(\alpha_l\prod_{i\in S}\|\Delta\mW_l^{(i)}\|_{\normrms} \prod_{i\notin S}\|\mW_l^{(i)}\|_{\normrms} \right)
\left( \alpha_l\prod_{i=1}^k\|\mW_l^{(i)}\|_{\normrms}\right)^{j-1} \\
&= \left(\alpha_l \prod_{i\in S}\|\Delta\mW_l^{(i)}\|_{\normrms} \prod_{i\notin S}\|\mW_l^{(i)}\|_{\normrms} \right) \cdot {\Theta(1/L^{j-1})},
\end{align*}
which implies that $\alpha_l\prod_{i\in S}\|\Delta\mW_l^{(i)}\|_{\normrms} \prod_{i\notin S}\|\mW_l^{(i)}\|_{\normrms}=\Theta(1/L)$, which finishes the derivation. Therefore, for any fixed $k\ge2$, the spectral constraints reduce to the same implementation as in the $k=2$ case: hidden residual multiplier $\alpha_l=\Theta(1/L)$ and hidden update norms $\|\Delta \mW_l^{(i)}\|_{\normrms}=\Theta(1)$ under standard width-scaling $\mu$P initialization.

\subsection{Bias Parameters}
\label{app: spectral condition bias}

The purpose of this subsection is not to introduce a new depth-scaling rule, but to show that bias parameters can be assigned order-one initialization and update scales once the matrix-weight parameterization from Condition~\ref{condition: scale-invariant fl} is in place.

\subsubsection{Problem Setup}

As shown in Appendix~\ref{app: spectral multi-layer}, residual blocks with an arbitrary fixed internal depth $k\ge2$ admit spectral scaling conditions that are algorithmically equivalent to the $k=2$ case.
Therefore, to illustrate that bias parameters do not change the main $\mu$P prescription, we analyze the representative two-layer residual block with additive biases.
Specifically, we consider a residual network whose forward propagation is given by
\begin{align*}
\vh_0(\vx) &= \alpha_0 \bigl(\mW_0 \vx + \vb_0\bigr),\\
\vh_l(\vx) &= \vh_{l-1}(\vx)
+ \alpha_l \Bigl(\mW_l^{(2)} \bigl(\mW_l^{(1)} \vh_{l-1}(\vx) + \vb_l^{(1)}\bigr) + \vb_l^{(2)}\Bigr),
\quad \forall\, l \in [L], \\
\vh_{L+1}(\vx) &= \alpha_{L+1}\mW_{L+1} \vh_L(\vx).
\end{align*}
Here, each residual block consists of a two-layer linear transformation with additive biases,
where $\mW_l^{(1)}, \mW_l^{(2)}$ denote the weight matrices and
$\vb_l^{(1)}, \vb_l^{(2)}$ denote the corresponding bias vectors within the $l$-th block.
The scalars $\{\alpha_l\}_{l=0}^{L+1}$ represent block multipliers that control the effective strength of each transformation.
As in the previous sections, $\vh_{L+1}(\vx)$ denotes the network output used to compute the loss.
Our goal is to derive a spectral condition that realizes $\mu$P Principle~\ref{principle: mup} in this setting.

\subsubsection{Spectral Scaling Condition}

We now state the spectral scaling condition for the residual network with biases for realizing the $\mu$P principle under joint width–depth scaling.

\begin{condition}[Spectral condition for $\mu$P under joint width-depth scaling, two-layer residual block with biases]
\label{condition: bias}
To realize $\mu$P Principle~\ref{principle: mup}, the initial parameters and their per-step updates should satisfy:
\begin{itemize}
    \item \textbf{Initial condition.}
    \begin{itemize}
        \item Input parameters:
        $\alpha_0\|\mW_0\|_{\normrms} = \Theta(1)$, $\alpha_0\|\vb_0\|_{\normrms}=\Theta(1)$.
        \item Hidden parameters:
        \begin{itemize}
        \item $\alpha_l \|\mW_l^{(2)}\|_{\normrms}\,\|\mW_l^{(1)}\|_{\normrms}=\Theta(1/L), \quad \forall l\in[L]$.
        \item $\alpha_l \|\mW_l^{(2)}\|_{\normrms} \|\vb_l^{(1)}\|_{\normrms} = \Theta(1/L),\quad \forall l\in[L]$. 
        \item $\alpha_l\|\vb_l^{(2)}\|_{\normrms} = \mathcal{O}(1/\sqrt{L}),\quad \forall l\in[L]$.
        \end{itemize}
        \item Output parameters:
        $\alpha_{L+1}\|\mW_{L+1}\|_{\normrms}=\Theta(1)$.
    \end{itemize}
    \item \textbf{Update condition.}
    \begin{itemize}
        \item Input parameters:
        $\alpha_0\|\Delta \mW_0\|_{\normrms}=\Theta(1),\  \alpha_0\|\Delta \vb_0\|_{\normrms}=\Theta(1)$.
        \item Hidden parameters (first-order):
        \begin{itemize}
        \item $\alpha_l\|\Delta\mW_l^{(2)}\|_{\normrms}\,\|\mW_l^{(1)}\|_{\normrms} =\Theta(1/L),\quad \forall l\in[L]$.
        \item $\alpha_l\|\mW_l^{(2)}\|_{\normrms}\,\|\Delta\mW_l^{(1)}\|_{\normrms}
        =\Theta(1/L),\quad \forall l\in[L]$.
        \item $\alpha_l\|\Delta\mW_l^{(2)}\|_{\normrms}\|\vb_l^{(1)}\|_{\normrms}= \Theta(1/L),\quad \forall l\in[L]$.
        \item $\alpha_l\|\mW_l^{(2)}\|_{\normrms}\|\Delta\vb_l^{(1)}\|_{\normrms}= \Theta(1/L),\quad \forall l\in[L]$.
        \item $\alpha_l\|\Delta\vb_l^{(2)}\|_{\normrms}=\Theta(1/L),\quad \forall l\in[L]$.
        \end{itemize}
        \item Hidden parameters (second-order), automatically satisfied given initial condition and first-order update conditions:
        \begin{itemize}
        \item $\alpha_l\|\Delta\mW_l^{(2)}\|_{\normrms}\,\|\Delta\mW_l^{(1)}\|_{\normrms}=\Theta(1/L),\quad \forall l\in[L]$.
        \item $\alpha_l\|\Delta\mW_l^{(2)}\|_{\normrms}\|\Delta\vb_l^{(1)}\|_{\normrms}= \Theta(1/L),\quad \forall l\in[L]$.
        \end{itemize}
        \item Output parameters:
        $\alpha_{L+1}\|\Delta \mW_{L+1}\|_{\normrms}=\Theta(1)$.
    \end{itemize}
    \item \textbf{Efficient implementation.}
Under the {CompleteP-style $\mu$P parameterization ($k=2$)} of block multipliers $\{\alpha_l\}$ and matrix weights $\{\mW_l\}$ described in Section~\ref{sec: implementation} and Appendix~\ref{app: implementation opts HPs}, all bias-related spectral conditions can be satisfied simultaneously by initializing and training with biases of order $\Theta(1)$.
Concretely, it is sufficient to enforce
\begin{align}
\|\vb_l\|_{\normrms} = \Theta(1),
\quad
\|\Delta \vb_l\|_{\normrms} = \Theta(1),
\quad
\forall\, 0 \le l \le L. \label{eq:bias_spectral_condition}
\end{align}
The initial condition $\|\vb_l\|_{\normrms} = \Theta(1)$ can be satisfied by setting $\sigma_{\vb_l} = \Theta(1)$, and the implementation for update condition $\|\Delta \vb_l\|_{\normrms} = \Theta(1)$ will be derived in Appendix~\ref{app: implementation opts HPs}.
\end{itemize}
\end{condition}

Condition~\ref{condition: bias} shows that, under joint width-depth scaling,
bias parameters can be incorporated without modifying the existing HP parameterization of the weight matrices.
Specifically, once the block multipliers $\{\alpha_l\}$ and weights $\{\mW_l\}$ are implemented as in Section~\ref{sec: implementation}, biases admit additional, simple order-one spectral conditions that guarantee their initialization and updates scale properly.
Thus, biases can be handled by lightweight extensions of our framework, while the $\mu$P formulation for bias-free residual blocks remains unchanged.

\subsubsection{Derivation for Preliminary Initialization Condition}

\paragraph{Input layer.}
By subadditivity and submultiplicativity,
\begin{align*}
\|\vh_0(\vx)\|_{\normrms}
=
\Theta(\alpha_0(
\|\mW_0\|_{\normrms}\|\vx\|_{\normrms}
+
\|\vb_0\|_{\normrms}
))
=
\Theta\left(
\alpha_0\|\mW_0\|_{\normrms}\right)+ \Theta\left(\alpha_0\|\vb_0\|_{\normrms}
\right),
\end{align*}
where we assumed $\|\vx\|_{\normrms}=\Theta(1)$.
Thus, choosing $\alpha_0\|\mW_0\|_{\normrms}, \alpha_0\|\vb_0\|_{\normrms}=\Theta(1)$ ensures
$\|\vh_0(\vx)\|_{\normrms}=\Theta(1)$.

\paragraph{Hidden layers.}
Expanding the residual recursion yields
\begin{align}
\vh_s(\vx)
&=
\vh_0(\vx)
+
\sum_{l=1}^s
\alpha_l\Big(
\mW_l^{(2)}\mW_l^{(1)}\vh_{l-1}(\vx)
+
\mW_l^{(2)}\vb_l^{(1)}
+
\vb_l^{(2)}
\Big).
\label{eqn: hl-bias}
\end{align}
Applying subadditivity, we can estimate their order as
\begin{align*}
\|\vh_s(\vx)\|_{\normrms}
=
\Theta\left(\|\vh_0(\vx)\|_{\normrms}
+
\Big\|
\sum_{l=1}^s
\alpha_l
\mW_l^{(2)}\mW_l^{(1)}\vh_{l-1}(\vx)
+
\sum_{l=1}^s \alpha_l \mW_l^{(2)}\vb_l^{(1)}
+
\sum_{l=1}^s \alpha_l\vb_l^{(2)}
\Big\|_{\normrms}\right).
\end{align*}
Since we have $\|\vh_0(\vx)\|_{\normrms}=\Theta(1)$, it suffices to ensure other terms are
$\mathcal{O}(1)$  for any $s\in[L]$ to preserve $\|\vh_s(\vx)\|_{\normrms}=\Theta(1)$.
Under i.i.d.\ zero-mean Gaussian initialization, the summands are approximately independent zero-mean random vectors~\citep{TP4,TP-6,completep}, so the typical squared RMS norm of their sum scales as the sum of the squared RMS norms (see Theorem~3.3.1 in~\citet{hdp}). Therefore, we can obtain
\begin{align*}
&\Big\|
\sum_{l=1}^s
\alpha_l
\mW_l^{(2)}\mW_l^{(1)}\vh_{l-1}(\vx)
+
\sum_{l=1}^s \alpha_l \mW_l^{(2)}\vb_l^{(1)}
+
\sum_{l=1}^s \alpha_l\vb_l^{(2)}
\Big\|_{\normrms}\\
&=
\sqrt{\Big\|
\sum_{l=1}^s
\alpha_l\mW_l^{(2)}\mW_l^{(1)}\vh_{l-1}(\vx)
\Big\|_{\normrms}^2
+
\Big\|
\sum_{l=1}^s
\alpha_l\mW_l^{(2)}\vb_l^{(1)}
\Big\|_{\normrms}^2
+
\Big\|
\sum_{l=1}^s
\alpha_l\vb_l^{(2)}
\Big\|_{\normrms}^2
}
\end{align*}

Furthermore, using the same probability argument and submultiplicativity inequality as in the derivation without biases (e.g., see Section~\ref{sec: spec condition}), we have
\begin{align*}
\Big\|
\sum_{l=1}^s
\alpha_l\mW_l^{(2)}\mW_l^{(1)}\vh_{l-1}(\vx)
\Big\|_{\normrms}^2
&=
\Theta\left(
{\sum_{l=1}^s
\alpha_l^2
\|\mW_l^{(2)}\|_{\normrms}^2
\|\mW_l^{(1)}\|_{\normrms}^2\|\vh_{l-1}(\vx)\|_{\normrms}^2}
\right),\\
\Big\|
\sum_{l=1}^s
\alpha_l\mW_l^{(2)}\vb_l^{(1)}
\Big\|_{\normrms}^2
&=
\Theta\left(
{\sum_{l=1}^s
\alpha_l^2
\|\mW_l^{(2)}\|_{\normrms}^2
\|\vb_l^{(1)}\|_{\normrms}^2}
\right),\\
\Big\|
\sum_{l=1}^s
\alpha_l\vb_l^{(2)}
\Big\|_{\normrms}^2
&=
\Theta\left(
{\sum_{l=1}^s
\alpha_l^2
\|\vb_l^{(2)}\|_{\normrms}^2}
\right).
\end{align*}
Therefore, starting from $\|\vh_0(\vx)\|_{\normrms}=\Theta(1)$, imposing
\[
\alpha_l
\|\mW_l^{(2)}\|_{\normrms}
\|\mW_l^{(1)}\|_{\normrms}
=
\mathcal{O}(1/\sqrt{L}),
\quad
\alpha_l
\|\mW_l^{(2)}\|_{\normrms}
\|\vb_l^{(1)}\|_{\normrms}
=
\mathcal{O}(1/\sqrt{L}),
\quad
\alpha_l\|\vb_l^{(2)}\|_{\normrms}
=
\mathcal{O}(1/\sqrt{L})
\]
recursively ensures $\|\vh_s(\vx)\|_{\normrms}=\Theta(1)$ for $s\in[L]$.
This yields a preliminary initialization condition, which will be refined after incorporating update constraints.

\paragraph{Output layer.}
{ As in the two-layer residual block case, submultiplicativity gives
\[
\|\vh_{L+1}(\vx)\|_{\normrms}
= \mathcal{O}(\alpha_{L+1}\|\mW_{L+1}\|_{\normrms}\|\vh_L(\vx)\|_{\normrms})
=\mathcal{O}(\alpha_{L+1}\|\mW_{L+1}\|_{\normrms}).
\]
As in the analysis for the two-layer block case, no matching lower bound is assumed at initialization. Hence, $\alpha_{L+1}\|\mW_{L+1}\|_{\normrms}=\Theta(1)$ remains only a sufficient readout condition ensuring $\|\vh_{L+1}(\vx)\|_{\normrms}=\mathcal{O}(1)$.
Nevertheless, this suffices for Principle~(\ref{principle: stable}), which constrains $\vh_l$ only for $l\in[L]$ and does not require $\|\vh_{L+1}(\vx)\|_{\normrms}=\Theta(1)$ at initialization.}

\subsubsection{Derivation for Update Condition}

\paragraph{Input layer.}
Recall that
\[
\vh_0(\vx)=\alpha_0\bigl(\mW_0\vx+\vb_0\bigr).
\]
After one gradient step, the feature update satisfies
\[
\Delta \vh_0(\vx)
=
\alpha_0\bigl(\Delta \mW_0\,\vx+\Delta\vb_0\bigr).
\]
By subadditivity, submultiplicativity, and using $\|\vx\|_{\normrms}=\Theta(1)$ by data assumption, we obtain
\[
\|\Delta \vh_0(\vx)\|_{\normrms}
=
\Theta(\alpha_0
\|\Delta \mW_0\|_{\normrms}
+
\alpha_0\|\Delta \vb_0\|_{\normrms})
.
\]
Therefore, we choose
\[
\alpha_0\|\Delta \mW_0\|_{\normrms}
=
\Theta(1),
\quad
\alpha_0\|\Delta \vb_0\|_{\normrms}
=
\Theta(1)
\]
to realize $\|\Delta \vh_0(\vx)\|_{\normrms}=\Theta(1)$.

\paragraph{Hidden layers.} We next analyze the feature updates $\Delta\vh_s(\vx)$ after one gradient step.
Expanding Equation~(\ref{eqn: hl-bias}) yields
\begin{align*}
\Delta\vh_s(\vx)
&=
\Delta\vh_0(\vx)
+
\Delta\sum_{l=1}^s
\alpha_l\Big(
\mW_l^{(2)}\mW_l^{(1)}\vh_{l-1}(\vx)
+
\mW_l^{(2)}\vb_l^{(1)}
+
\vb_l^{(2)}
\Big) \\
&= \Delta\vh_0(\vx)
+
\Delta\sum_{l=1}^s
\alpha_l
\mW_l^{(2)}\mW_l^{(1)}\vh_{l-1}(\vx)
+
\Delta\sum_{l=1}^s
\alpha_l \mW_l^{(2)}\vb_l^{(1)}
+
\Delta\sum_{l=1}^s
\alpha_l \vb_l^{(2)}.
\end{align*}
By the subadditivity of vector norms, we have
\begin{align*}
&\Vert\Delta\vh_s(\vx)\Vert_\normrms \\
&= \Theta\left(\Vert\Delta\vh_0(\vx)\Vert_\normrms + 
\bigg\Vert\Delta\sum_{l=1}^s
\alpha_l
\mW_l^{(2)}\mW_l^{(1)}\vh_{l-1}(\vx)\bigg\Vert_\normrms
+ 
\bigg\Vert\Delta\sum_{l=1}^s
\alpha_l \mW_l^{(2)}\vb_l^{(1)}\bigg\Vert_\normrms 
+
\bigg\Vert\Delta\sum_{l=1}^s
\alpha_l \vb_l^{(2)}\bigg\Vert_\normrms\right).
\end{align*}
Since $\|\Delta\vh_0(\vx)\|_{\normrms}=\Theta(1)$ by the input-layer update, we have
$\|\Delta\vh_s(\vx)\|_{\normrms}=\Omega(1)$ for all $s\in[L]$.
Moreover, by subadditivity, the remaining terms do not decay with depth, implying
$\|\Delta\vh_s(\vx)\|_{\normrms}=\mathcal{O}(\|\Delta\vh_L(\vx)\|_{\normrms})$ for any $s\in[L]$.
Therefore, to enforce Principle~\ref{principle: mup},
it suffices to require $\|\Delta\vh_L(\vx)\|_{\normrms}=\Theta(1)$ while satisfying
Principle~(\ref{principle: max}). We discuss the components of $\Delta\vh_L(\vx)$ in sequence.

\paragraph{Matrix-weight terms.}
The contributions from
\[
\Delta\vh_0(\vx)
+
\Delta\sum_{l=1}^L
\alpha_l
\mW_l^{(2)}\mW_l^{(1)}\vh_{l-1}(\vx)
\]
have been fully analyzed in the bias-free two-layer case
(see Section~\ref{sec: spec condition}).
Applying the same reasoning yields the first- and second-order update conditions
on hidden matrix weights:
\begin{align*}
\alpha_l\|\Delta\mW_l^{(2)}\|_{\normrms}\,\|\mW_l^{(1)}\|_{\normrms} &= \Theta(1/L), \\
\alpha_l\|\mW_l^{(2)}\|_{\normrms}\,\|\Delta\mW_l^{(1)}\|_{\normrms} &= \Theta(1/L), \\
\alpha_l\|\Delta\mW_l^{(2)}\|_{\normrms}\,\|\Delta\mW_l^{(1)}\|_{\normrms} &= \Theta(1/L),
\qquad \forall\, l\in[L].
\end{align*}
We then need to control the newly introduced bias-related terms.

\paragraph{First-layer bias-related term.}
Consider
$\Delta\sum_{l=1}^L \alpha_l \mW_l^{(2)}\vb_l^{(1)}$.
Expanding the update yields
\begin{align*}
\Delta\sum_{l=1}^L
\alpha_l \mW_l^{(2)}\vb_l^{(1)}
=
\sum_{l=1}^L
\alpha_l\Big(
\Delta\mW_l^{(2)}\vb_l^{(1)}
+
\mW_l^{(2)}\Delta\vb_l^{(1)}
+
\Delta\mW_l^{(2)}\Delta\vb_l^{(1)}
\Big).
\end{align*}
By subadditivity and submultiplicativity of the RMS norm, we have
\begin{align*}
&\bigg\Vert
\Delta\sum_{l=1}^L
\alpha_l \mW_l^{(2)}\vb_l^{(1)}
\bigg\Vert_\normrms \\
&=
\Theta\bigg(\sum_{l=1}^L
\alpha_l
\|\Delta\mW_l^{(2)}\|_{\normrms}
\|\vb_l^{(1)}\|_{\normrms} 
+
\sum_{l=1}^L
\alpha_l
\|\mW_l^{(2)}\|_{\normrms}
\|\Delta\vb_l^{(1)}\|_{\normrms}
+
\sum_{l=1}^L
\alpha_l
\|\Delta\mW_l^{(2)}\|_{\normrms}
\|\Delta\vb_l^{(1)}\|_{\normrms}\bigg).
\end{align*}
According to Principle~(\ref{principle: max}), we require
$\left\Vert
\Delta\sum_{l=1}^L
\alpha_l \mW_l^{(2)}\vb_l^{(1)}
\right\Vert_\normrms
=\Theta(1)$ and maximize the contribution from each summand, leading to
\begin{align*}
\alpha_l
\|\Delta\mW_l^{(2)}\|_{\normrms}
\|\vb_l^{(1)}\|_{\normrms}
&= \Theta(1/L), \\
\alpha_l
\|\mW_l^{(2)}\|_{\normrms}
\|\Delta\vb_l^{(1)}\|_{\normrms}
&= \Theta(1/L), \\
\alpha_l
\|\Delta\mW_l^{(2)}\|_{\normrms}
\|\Delta\vb_l^{(1)}\|_{\normrms}
&= \Theta(1/L),
\qquad \forall\, l\in[L].
\end{align*}

\paragraph{Second-layer bias-related term.}
Finally, for
$\Delta\sum_{l=1}^L \alpha_l \vb_l^{(2)}$,
we have
\begin{align*}
\bigg\Vert
\Delta\sum_{l=1}^L
\alpha_l \vb_l^{(2)}
\bigg\Vert_\normrms
=
\bigg\Vert
\sum_{l=1}^L
\alpha_l \Delta\vb_l^{(2)}
\bigg\Vert_\normrms
=
\Theta\bigg(\sum_{l=1}^L
\alpha_l
\|\Delta\vb_l^{(2)}\|_{\normrms}\bigg).
\end{align*}
To maximally update parameters according to Principle~(\ref{principle: max}), we require this term to remain $\Theta(1)$ and maximize each summand, which yields
\[
\alpha_l
\|\Delta\vb_l^{(2)}\|_{\normrms}
=
\Theta(1/L),
\qquad \forall\, l\in[L].
\]

\paragraph{Output layer.}
The same argument as in the two-layer case in Section~\ref{sec: spec condition} yields
$\alpha_{L+1}\|\Delta\mW_{L+1}\|_{\normrms}=\Theta(1)$.

\subsubsection{Derivation for Final Initial Condition}

We now derive the final initialization conditions by incorporating the update
constraints obtained in the previous subsection.

\paragraph{Hidden matrix weights.}
As already shown in the bias-free setting (Section~\ref{sec: spec condition}), combining the first-order and second-order update conditions immediately yields the initialization constraint
\[
\alpha_l \|\mW_l^{(1)}\|_{\normrms}\,
\|\mW_l^{(2)}\|_{\normrms}
=
\Theta(1/L),
\quad \forall\, l \in [L].
\]
Therefore, the presence of biases does not alter the initialization scaling of hidden matrix weights.

\paragraph{Bias parameters.}
We now derive the initialization conditions for bias terms by combining the
first- and second-order update constraints.
For the first-layer bias $\vb_l^{(1)}$, the first-order update conditions give
\begin{align*}
\alpha_l
\|\Delta\mW_l^{(2)}\|_{\normrms}
\|\vb_l^{(1)}\|_{\normrms}
&= \Theta(1/L), \\
\alpha_l
\|\mW_l^{(2)}\|_{\normrms}
\|\Delta\vb_l^{(1)}\|_{\normrms}
&= \Theta(1/L),
\end{align*}
while the second-order update condition yields
\[
\alpha_l
\|\Delta\mW_l^{(2)}\|_{\normrms}
\|\Delta\vb_l^{(1)}\|_{\normrms}
= \Theta(1/L).
\]
Multiplying the two first-order conditions and dividing by the second-order one,
we obtain
\[
\alpha_l
\|\mW_l^{(2)}\|_{\normrms}
\|\vb_l^{(1)}\|_{\normrms}
= \Theta(1/L),\quad \forall l\in[L].
\]
Similar to the hidden matrix weights, the second-order bias-related condition is automatically satisfied by combining the refined initial condition and the corresponding first-order update condition.

We note that the second-layer bias $\vb_l^{(2)}$ has no multiplicative interaction with another parameter in the forward block, so its initialization condition remains the preliminary upper bound $\alpha_l\|\vb_l^{(2)}\|_{\normrms}=O(1/\sqrt L)$.

\subsubsection{Derivation for Efficient Implementation}

Recall that, based on the {CompleteP-style $\mu$P parameterization ($k=2$)} introduced for matrix weights in Section~\ref{sec: implementation} and Appendix~\ref{app: implementation opts HPs}, we have $\alpha_0=\Theta(1)$ in Equation~(\ref{eq:alpha_inandout}), $\alpha_l=\Theta(1/L)$ for $l\in[L]$ in Equation~(\ref{eq:alpha_l}), $\|{\mW}_l\|_{\normrms}=\Theta(1)$ for $0\leq l\leq L$ in Equation~(\ref{eq:w_norm}) and $\|\Delta{\mW}_l\|_{\normrms}=\Theta(1)$ for $0\leq l\leq L$. Based on these conditions, Condition~\ref{condition: bias} reduces to

\begin{itemize}
    \item \textbf{Initial condition.}
    \begin{itemize}
        \item Input parameters:
        $\|\vb_0\|_{\normrms}=\Theta(1)$.
        \item Hidden parameters:
        \begin{itemize}
        \item $\|\vb_l^{(1)}\|_{\normrms} = \Theta(1),\quad \forall l\in[L]$. 
        \item $\|\vb_l^{(2)}\|_{\normrms} = \mathcal{O}(\sqrt{L}),\quad \forall l\in[L]$.
        \end{itemize}
    \end{itemize}
    \item \textbf{Update condition.}
    \begin{itemize}
        \item Input parameters:
        $\|\Delta \vb_0\|_{\normrms}=\Theta(1)$.
        \item Hidden parameters (first-order):
        \begin{itemize}
        \item $\|\Delta\vb_l^{(1)}\|_{\normrms}= \Theta(1),\quad \forall l\in[L]$.
        \item $\|\Delta\vb_l^{(2)}\|_{\normrms}=\Theta(1),\quad \forall l\in[L]$.
        \end{itemize}
    \end{itemize}
\end{itemize}
Therefore, it is sufficient to enforce the order-one spectral condition for biases: 
\begin{align*}
\|\vb_l\|_{\normrms} = \Theta(1),
\quad
\|\Delta \vb_l\|_{\normrms} = \Theta(1),
\quad
\forall\, 0 \le l \le L.
\end{align*}
Although the preliminary condition permits $\|\vb_l^{(2)}\|_{\normrms}$ as large as $\mathcal{O}(\sqrt{L})$, choosing it to be $\Theta(1)$ is a simpler sufficient choice and keeps all biases on the same scale.

\section{Implementing Condition~\ref{condition: scale-invariant fl} for Optimizers with Weight Decay}
\label{app: implementation opts HPs}

Recall that in Section~\ref{sec:initial_condition} of the main text, we implemented Condition~\ref{condition: scale-invariant fl} for initialization and specified the parameterization of the {block multipliers $\alpha_l$} and the {initialization variances $\sigma_l$}, which is optimizer-agnostic.
In Section~\ref{sec:update_condition_muonkimi}, we further implemented Condition~\ref{condition: scale-invariant fl} for updates and derived the parameterization of the {learning rates $\eta_l$} for the Muon-Kimi~\citep{muon-kimi}.
We now extend this update-condition analysis to a broader class of optimizers with weight decay.
For bias parameters, we also derive the
corresponding bias HPs when the optimizer is commonly applied to biases.

\subsection{Preparation}

\paragraph{Unified update form with weight decay.}
To provide a unified derivation across different optimizers, we begin by expressing their one-step update rules in a common form.
When weight decay is included, a single update step of the weight matrix can be written as
\begin{align*}
    \Delta \mW_l = -\eta_l \bigl( \mA_l + \lambda_l \mW_l \bigr),
\end{align*}
where $\mA_l$ denotes the optimizer-specific update direction before applying
the learning rate and excluding weight decay. For example,
$\mA_l=\nabla_{\mW_l}\mathcal{L}$ for SGD, while for Muon-Kimi
$\mA_l=0.2\sqrt{\max\{\nin,\nout\}}\,\mU_l\mV_l^\top$. The scalar
$\lambda_l$ is the weight-decay coefficient.

The update magnitude $\|\Delta \mW_l\|_{\normrms}
= \eta_l \|\mA_l + \lambda_l \mW_l\|_{\normrms}
$ is required to satisfy the update conditions in Condition~\ref{condition: scale-invariant fl}. 
We analyze this requirement under two complementary regimes.

\textbf{Without weight decay.}
When weight decay is disabled ($\lambda_l = 0$), the update reduces to
$\|\Delta \mW_l\|_{\normrms}= \eta_l \|\mA_l\|_{\normrms}.$ In this case, the learning rate is chosen so that
\begin{align}
    \|\Delta \mW_l\|_{\normrms}
    = \eta_l \|\mA_l\|_{\normrms}
    \ \text{satisfies Condition~\ref{condition: scale-invariant fl}}.
    \tag{$\Delta1$}
    \label{eq:condition_wd_1}
\end{align}

\textbf{With weight decay.}
When weight decay is enabled ($\lambda_l \neq 0$), we choose the weight decay term to be comparable in scale to the optimizer-driven term.
\begin{align}
    \|\lambda_l \mW_l\|_{\normrms}
    = \Theta\left( \|\mA_l\|_{\normrms} \right).
    \tag{$\Delta2$}
    \label{eq:condition_wd_2}
\end{align}
This is a non-degenerate implementation convention: it keeps weight decay active
in the update dynamics without letting it dominate the optimizer-driven term.
Under the usual scale estimate,
$\|\mA_l+\lambda_l\mW_l\|_{\normrms}
=\Theta(\|\mA_l\|_{\normrms})$, so the learning-rate scaling rule derived from
Equation~(\ref{eq:condition_wd_1}) is preserved.

\paragraph{Weights and biases.}
In the following, we derive parameterizations of the learning rate and weight decay coefficient for a range of optimizers.
Since matrix-based optimizers such as Muon and Shampoo are typically not used for bias parameters, we restrict our analysis of bias parameterization to vector-based optimizers such as SGD and AdamW. The same two rules, Equations~(\ref{eq:condition_wd_1}) and~(\ref{eq:condition_wd_2}), are then
applied to the biases with corresponding spectral condition in Equation~\ref{eq:bias_spectral_condition} on vector RMS norms.

\paragraph{Momentum.}
We note that momentum is typically omitted in standard $\mu$P analyses~\citep{TP4,TP4b,TP5} (e.g., by setting $\beta_1=\beta_2=0$ in AdamW), while in practical $\mu$P implementations the momentum coefficients are taken to be $\Theta(1)$.
The main rationale is that the norm of the momentum term and the norm of the current update are expected to be of the same order, and since the spectral condition aims to control the update scale, omitting the momentum is regarded as an acceptable simplification.
Moreover, analyzing the update without momentum can be interpreted as studying the \emph{first update step} after initialization, which has been empirically observed to be reliable for understanding neural network training~\citep{mup-spectral,extending-mup,DBLP:conf/nips/BordelonP22-dmft,DBLP:conf/nips/BordelonCP24-transformer}.
In the subsequent derivations, we adopt this simplification as well.

\paragraph{Low-rank structure of updates.}
Following the common practice used in the $\mu$P spectral
analysis~\citep{mup-spectral,extending-mup}, we also introduce a useful preliminary result for controlling the norm of the update term $\mA_l$.
The key idea is to exploit the effective \emph{low-rank structure} of neural-network updates, which has been widely observed in neural network training~\citep{mup-spectral,DBLP:journals/corr/low-rank,DBLP:conf/icml/Zhao0CWAT24-galore}: only a small number of dominant singular directions carry most of the update.

\begin{assumption}[Low-rank update structure]
\label{ass:low-rank-update}
For vector-based optimizers such as SGD and AdamW, we assume that the effective update term $\mA_l$ has constant rank with respect to width and depth, i.e., $r(\mA_l)=\Theta(1)$.
\end{assumption}

Under this assumption, the spectral norm and Frobenius norm of the update term are of the same order:
\begin{equation}
\|\mA_l\|_{2}
=
\Theta(\|\mA_l\|_{\mathrm{F}}),
\label{eqn: low-rank}
\end{equation}
where the hidden constants are independent of width and depth. Indeed, this is because
\begin{align*}
\|\mA_l\|_{2}
\leq
\|\mA_l\|_{\mathrm{F}}
\leq
\sqrt{r(\mA_l)} \|\mA_l\|_{2}
=
\Theta(\|\mA_l\|_{2}).
\end{align*}
We will use this low-rank update property to estimate the scale of $\Vert\mA_l\Vert_\normrms$ for vector-based optimizers.

\subsection{Overview}
\label{app: implementation opts HPs overview}

Table~\ref{tab: optimizer-family-overview} summarizes the optimizer families covered in this section and points to the corresponding detailed $\mu$P parameterization tables.
All rows use the optimizer-agnostic initialization and block-multiplier rules from Section~\ref{sec:initial_condition}; the table only organizes the optimizer-dependent update rules derived below.

\begin{table}[t]
\renewcommand{\arraystretch}{1.15}
\centering
\caption{\textbf{Overview of the {CompleteP-style $\mu$P parameterization from Condition~\ref{condition: scale-invariant fl} ($k=2$)} for optimizer families with weight decay.}
Optimizers in the same row share the same $\mu$P scaling rules.}
\label{tab: optimizer-family-overview}
\vskip 0.05in
\begin{tabular}{cc}
\toprule
Optimizer family & Detailed parameterization \\
\midrule
Muon-Kimi & Table~\ref{tab: muon-kimi-wd mup} \\
Muon / Shampoo / SOAP & Table~\ref{tab: muon-wd mup} \\
SGD & Table~\ref{tab: sgd-mup} \\
AdamW / Lion / Sophia & Table~\ref{tab: adamw-mup} \\
SSO & Table~\ref{tab: sso mup} \\
\bottomrule
\end{tabular}
\end{table}

\subsection{Muon-Kimi}
\label{app: Muon-Kimi (with Weight Decay)}

\begin{table}[t]

\renewcommand{\arraystretch}{1.3}
\renewcommand{\hl}[1]{\textcolor{purple}{#1}}
\renewcommand{\ll}[1]{\textcolor{gray}{#1}}
\centering

\caption{\textbf{{CompleteP-style $\mu$P parameterization from Condition~\ref{condition: scale-invariant fl} ($k=2$)} for Muon-Kimi~\citep{muon-kimi} with weight decay under width-depth scaling.}
Entries in \hl{purple} indicate differences between $\mu$P and SP, while \ll{gray} shows the corresponding SP choices.
Here, $r_n$ and $r_L$ denote the width and depth scaling ratios relative to the base model. The variance of input weights is $\sigma^2_{\mathrm{base}}$ for language and $\sigma^2_{\mathrm{base}}/d_0$ for image.}
\label{tab: muon-kimi-wd mup}
\vskip 0.05in

\begin{tabular}{cccc}
\toprule
   & Input weights & Hidden weights & Output weights \\
\midrule
Block Multiplier
& $\alpha_{\mathrm{base}}$
& \hl{$\alpha_{\mathrm{base}}/r_L$} \ \ll{($\alpha_{\mathrm{base}}$)}
& \hl{$\alpha_{\mathrm{base}}/r_n$} \ \ll{($\alpha_{\mathrm{base}}$)} \\

Initial Variance
& $\sigma^2_{\mathrm{base}}/d_0$ or $\sigma^2_{\mathrm{base}}$
& \hl{$\sigma^2_{\mathrm{base}}/r_n$} \ \ll{($\sigma^2_{\mathrm{base}}$)}
& $\sigma^2_{\mathrm{base}}$ \\

Learning Rate
& $\eta_{\mathrm{base}}$
& \hl{$\eta_{\mathrm{base}}/\sqrt{r_n}$} \ \ll{($\eta_{\mathrm{base}}$)}
& $\eta_{\mathrm{base}}$ \\

Weight Decay
& $\lambda_{\mathrm{base}}$
& \hl{$\lambda_{\mathrm{base}}\sqrt{r_n}$} \ \ll{($\lambda_{\mathrm{base}}$)}
& $\lambda_{\mathrm{base}}$ \\

\bottomrule
\end{tabular}
\end{table}

For a weight matrix
${\mW}_l \in \R^{\nout\times\nin}$, the update rule of Muon-Kimi~\citep{muon-kimi} with weight decay is
\begin{equation*}
\Delta{\mW}_l
=
-\,\eta_l \bigl(\,\underbrace{0.2 \sqrt{\max\{\nin,\nout\}}
\,\mU_l \mV_l^\top}_{\mA_l} + \lambda_l {\mW}_l\bigl),
\end{equation*}
where $\mU_l,\mV_l$ arise from the compact SVD of the gradient $\nabla_{{\mW}_l}\mathcal{L}=\mU_l \mSigma_l \mV_l^\top$.

Recalling that in Section~\ref{sec:update_condition_muonkimi} in the main text, we have derived the learning rate parameterizations to achieve (\ref{eq:condition_wd_1}) that
\begin{align*}
    \eta_0 = \Theta\left(1\right), \ \eta_l^{(1)} = \Theta\left(\frac{1}{\sqrt{\nin}}\right), \ \eta_l^{(2)} = \Theta\left(\frac{1}{\sqrt{\nin}}\right), \ \eta_{L+1} = \Theta\left(1\right).
\end{align*}

According to the update norm in Equation~(\ref{eq:dw_norm}), we have 
\begin{align*}
    \Vert \mA_l \Vert_{\normrms} 
    = \Theta\left(\sqrt{\nin}\max\left\{1,\sqrt{\frac{\nin}{\nout}}\right\}\right)
    =\left\{
    \begin{array}{ll}
    \Theta(1),   & l = 0,\\
    \Theta(\sqrt{\nin}),  & l \in [L], \\
    \Theta({\nin}),  & l = L+1.
    \end{array} \right.
\end{align*}
Given the magnitude of $\Vert {\mW}_l\Vert_{\normrms}$ in Equation~(\ref{eq:w_norm}), as desired $\|\lambda_l \mW_l\|_{\normrms} = \Theta\left( \|\mA_l\|_{\normrms} \right)$ by (\ref{eq:condition_wd_2}), the parameterizations of $\lambda_l$ need to be set as follows:
\begin{align*}
    \lambda_l 
    =\left\{
    \begin{array}{ll}
    \Theta(1),   & l = 0,\\
    \Theta(\sqrt{\nin}),  & l \in [L], \\
    \Theta(1),  & l = L+1.
    \end{array} \right.
\end{align*}

This completes the implementation of the update condition for
Muon-Kimi with weight decay, as summarized in Table~\ref{tab: muon-kimi-wd mup}.

\subsection{Muon}
\label{app:muon}

\begin{table}[t]

\renewcommand{\arraystretch}{1.3}
\renewcommand{\hl}[1]{\textcolor{purple}{#1}}
\renewcommand{\ll}[1]{\textcolor{gray}{#1}}
\centering
\caption{\textbf{{CompleteP-style $\mu$P parameterization from Condition~\ref{condition: scale-invariant fl} ($k=2$)} for Muon~\citep{jordan6muon}, Shampoo~\citep{gupta2018shampoo}, and SOAP~\citep{vyas2024soap} with weight decay under width-depth scaling.}
Entries in \hl{purple} indicate differences between $\mu$P and SP, while \ll{gray} shows the corresponding SP choices.
Here, $r_n$ and $r_L$ denote the width and depth scaling ratios relative to the base model.
The variance of input weights is $\sigma^2_{\mathrm{base}}$ for language and $\sigma^2_{\mathrm{base}}/d_0$ for image.}
\label{tab: muon-wd mup}
\vskip 0.05in
\begin{tabular}{cccc}
\toprule
 & Input weights & Hidden weights & Output weights \\
\midrule

Block Multiplier
& $\alpha_{\mathrm{base}}$
& \hl{$\alpha_{\mathrm{base}}/r_L$} \ \ll{($\alpha_{\mathrm{base}}$)}
& \hl{$\alpha_{\mathrm{base}}/r_n$} \ \ll{($\alpha_{\mathrm{base}}$)} \\

Initial Variance
& $\sigma^2_{\mathrm{base}}/d_0$ or $\sigma^2_{\mathrm{base}}$
& \hl{$\sigma^2_{\mathrm{base}}/r_n$} \ \ll{($\sigma^2_{\mathrm{base}}$)}
& $\sigma^2_{\mathrm{base}}$ \\

Learning Rate
& \hl{$\eta_{\mathrm{base}}\sqrt{r_n}$} \ \ll{($\eta_{\mathrm{base}}$)}
& $\eta_{\mathrm{base}}$
& \hl{$\eta_{\mathrm{base}}\sqrt{r_n}$} \ \ll{($\eta_{\mathrm{base}}$)} \\

Weight Decay
& \hl{$\lambda_{\mathrm{base}}/\sqrt{r_n}$} \ \ll{($\lambda_{\mathrm{base}}$)}
& $\lambda_{\mathrm{base}}$
& \hl{$\lambda_{\mathrm{base}}/\sqrt{r_n}$} \ \ll{($\lambda_{\mathrm{base}}$)} \\

Shampoo $\varepsilon$
& \hl{$\varepsilon_{\mathrm{base}}/r_n$} \ \ll{($\varepsilon_{\mathrm{base}}$)}
& \hl{$\varepsilon_{\mathrm{base}}/r_L^2$} \ \ll{($\varepsilon_{\mathrm{base}}$)}
& \hl{$\varepsilon_{\mathrm{base}}/r_n$} \ \ll{($\varepsilon_{\mathrm{base}}$)} \\

\bottomrule
\end{tabular}
\end{table}

In this section, we derive the {CompleteP-style $\mu$P parameterization} from
Condition~\ref{condition: scale-invariant fl} for Muon,
which recovers and extends the $\mu$P scaling rules studied in
\citet{muon-cp}.

\subsubsection{Update Rule}

For a weight matrix ${\mW}_l \in \R^{\nout\times\nin}$, the update rule of Muon~\citep{jordan6muon} is
\begin{equation}
\Delta{\mW}_l
=
-\,\eta_l \bigl(\,\underbrace{\mU_l \mV_l^\top}_{\mA_l}
+ \lambda_l {\mW}_l\bigl),
\label{eq:updaterule_muon}
\end{equation}
where $\mU_l,\mV_l$ arise from the compact SVD of the gradient
$\nabla_{{\mW}_l}\mathcal{L}=\mU_l \mSigma_l \mV_l^\top$.
Compared with Muon-Kimi~\citep{muon-kimi}, the only difference lies in the absence of the
$0.2\sqrt{\max\{\nin,\nout\}}$ prefactor.

Considering the dimension assumption in Equation~(\ref{eq:dimensions}), the resulting norm of $\mA_l$ satisfies
\begin{align} 
\Vert \mA_l \Vert_\normrms 
=
\Vert \mU_l \mV_l^\top \Vert_\normrms 
=
\sqrt{\frac{\nin}{\nout}} \Vert \mU_l \mV_l^\top \Vert_2
=
\sqrt{\frac{\nin}{\nout}}
=
\left\{
    \begin{array}{ll}
    \Theta(1/\sqrt{\nout}),   & l = 0,\\
    \Theta(1),   & l \in [L],\\
    \Theta(\sqrt{\nin}),  & l = L+1.
    \end{array} \right. \label{eq:Anorm_muon}
\end{align}

\subsubsection{Derivation of Parameterization}

\paragraph{Input and output layers.} 
When weight decay is disabled ($\lambda_0 = 0$), given the dimension assumption $d_0,d_{L+1}=\Theta(1), n_l=\Theta(n)$ in Equation~(\ref{eq:dimensions}), the multiplier parameterizations $\alpha_0=\Theta(1),\alpha_{L+1}=\Theta(1/\nin)$ in Equation~(\ref{eq:alpha_inandout}), and the scale of $\Vert\mA_l\Vert_\normrms$ in Equation~(\ref{eq:Anorm_muon}), we have
\begin{align*}
    \alpha_l\|\Delta\mW_l\|_{\normrms}
    =\alpha_l\eta_l\|\mA_l\|_{\normrms}
    =\left\{
    \begin{array}{ll}
    \Theta(\eta_0/\sqrt{\nout}),   & l = 0,\\
    \Theta(\eta_{L+1}/\sqrt{\nin}),  & l = L+1.
    \end{array} \right.
\end{align*}
As desired in (\ref{eq:condition_wd_1}), to satisfy (\ref{eq:update_inandout}) that $\alpha_0\|\Delta \mW_0\|_{\normrms}, \alpha_{L+1}\|\Delta \mW_{L+1}\|_{\normrms} = \Theta(1)$, we need to set 
\begin{align*}
    \eta_0=\Theta(\sqrt{\nout}), \quad \eta_{L+1}=\Theta(\sqrt{\nin}).
\end{align*}
When $\lambda_l \neq 0$, given Equation~(\ref{eq:w_norm}) that $\|{\mW}_0\|_{\normrms} = \Theta(1)$ and $\|{\mW}_{L+1}\|_{\normrms} = \Theta(\nin)$, we have
\begin{align*}
    \|\lambda_l\mW_l\|_{\normrms} 
    =\left\{
    \begin{array}{ll}
    \Theta(\lambda_0),   & l = 0,\\
    \Theta(\lambda_{L+1}\nin),  & l = L+1.
    \end{array} \right.
\end{align*}
To satisfy (\ref{eq:condition_wd_2}) that $\|\lambda_l \mW_l\|_{\normrms}
    = \Theta\left( \|\mA_l\|_{\normrms} \right)$, we need to set 
\begin{align*}
    \lambda_0=\Theta(1/\sqrt{\nout}), \quad \lambda_{L+1}=\Theta(1/\sqrt{\nin}),
\end{align*}

\paragraph{Hidden layers (first-order).} 
When weight decay is disabled ($\lambda_l = 0$), given the dimension assumption $d_0,d_{L+1}=\Theta(1), n_l=\Theta(n)$ in Equation~(\ref{eq:dimensions}), the weight norm $\|{\mW}_l\|_{\normrms} = \Theta(1)$ in Equation~(\ref{eq:w_norm}), the multiplier parameterization $\alpha_l=\Theta(1/L)$ in Equation~(\ref{eq:alpha_l}), and the scale of $\Vert\mA_l\Vert_\normrms$ in Equation~(\ref{eq:Anorm_muon}), we have
\begin{align*}
    \alpha_l\|\Delta\mW_l^{(2)}\|_{\normrms}\|\mW_l^{(1)}\|_{\normrms}
    =\Theta(1/L) \cdot \eta_l^{(2)}\|\mA_l^{(2)}\|_{\normrms}\|\mW_l^{(1)}\|_{\normrms}
    =\Theta(\eta_l^{(2)}/L).
\end{align*}
As desired in (\ref{eq:condition_wd_1}), to satisfy the first-order update condition on hidden weights (\ref{eq:update_hidd_1}) that $\alpha_l\|\Delta \mW_l^{(2)}\|_{\normrms}\,\|\mW_l^{(1)}\|_{\normrms}=\Theta(1/L)$, we need to set 
\begin{align*}
    \eta_l^{(2)}=\Theta(1).
\end{align*}
When weight decay is enabled ($\lambda_l \neq 0$), given the weight norm $\|{\mW}_l\|_{\normrms} = \Theta(1)$ in Equation~(\ref{eq:w_norm}), we have
\begin{align*}
    \|\lambda_l^{(2)}\mW_l^{(2)}\|_{\normrms} = \Theta(\lambda_l^{(2)}).
\end{align*}
To satisfy (\ref{eq:condition_wd_2}) that $\|\lambda_l \mW_l\|_{\normrms}
    = \Theta\left( \|\mA_l\|_{\normrms} \right)$ we need to set 
\begin{align*}
    \lambda_l^{(2)}=\Theta(1).
\end{align*}

Symmetrically, we have the same choice for $\mW_l^{(1)}$:
\begin{align*}
    \eta_l^{(1)} = \Theta(1), \quad \lambda_l^{(1)}=\Theta(1).
\end{align*}

\paragraph{Hidden layers (second-order).} 
As discussed in Section~\ref{sec: Final Initial Condition}, the second-order update condition is satisfied automatically once the initial condition and the first-order update condition are met.
We explain here again for clarity: Multiplying two equations in (\ref{eq:update_hidd_1}) gives $\alpha_l^2\|\mW_l^{(2)}\|_{\normrms}\|\mW_l^{(1)}\|_{\normrms}\,\|\Delta \mW_l^{(2)}\|_{\normrms}\|\Delta\mW_l^{(1)}\|_{\normrms}=\Theta(1/L^2).$
Combining this with (\ref{eq:init_hidd}) that $\alpha_l \|\mW_l^{(2)}\|_{\normrms}\,\|\mW_l^{(1)}\|_{\normrms}=\Theta(1/L)$ directly implies the second-order condition (\ref{eq:update_hidd_2}).

This completes the implementation of the update condition for Muon with
weight decay, which is summarized in Table~\ref{tab: muon-wd mup}.

\subsection{SGD}
\label{app:sgd_mup}

\begin{table}[t]
\renewcommand{\arraystretch}{1.3}
\renewcommand{\hl}[1]{\textcolor{purple}{#1}}
\renewcommand{\ll}[1]{\textcolor{gray}{#1}}
\centering
\caption{\textbf{{CompleteP-style $\mu$P parameterization from Condition~\ref{condition: scale-invariant fl} ($k=2$)} for SGD with weight decay under width–depth scaling.}
Entries in \hl{purple} indicate differences between $\mu$P and SP, while \ll{gray} shows the corresponding SP choices.
Here, $r_n$ and $r_L$ denote the width and depth scaling ratios relative to the base model.
The variance of input weights is $\sigma^2_{\mathrm{base}}$ for language and $\sigma^2_{\mathrm{base}}/d_0$ for image. The initial variance of input bias is $\sigma^2_{\mathrm{base}}$.}
\label{tab: sgd-mup}
\vskip 0.05in
\begin{tabular}{ccccc}
\toprule
 & Input weights \& biases & Hidden weights & Output weights & Hidden biases\\
\midrule

Block Multiplier
& $\alpha_{\mathrm{base}}$
& \hl{$\alpha_{\mathrm{base}}/r_L$} \ \ll{($\alpha_{\mathrm{base}}$)}
& \hl{$\alpha_{\mathrm{base}}/r_n$} \ \ll{($\alpha_{\mathrm{base}}$)} 
&\hl{$\alpha_{\mathrm{base}}/r_L$} \ \ll{($\alpha_{\mathrm{base}}$)} \\

Initial Variance
& $\sigma^2_{\mathrm{base}}/d_0$ or $\sigma^2_{\mathrm{base}}$
& \hl{$\sigma^2_{\mathrm{base}}/r_n$} \ \ll{($\sigma^2_{\mathrm{base}}$)}
& $\sigma^2_{\mathrm{base}}$
&$\sigma^2_{\mathrm{base}}$ \\

Learning Rate
& \hl{$\eta_{\mathrm{base}}{r_n}$} \ \ll{($\eta_{\mathrm{base}}$)}
& \hl{$\eta_{\mathrm{base}}r_L$} \ \ll{($\eta_{\mathrm{base}}$)}
& \hl{$\eta_{\mathrm{base}}{r_n}$} \ \ll{($\eta_{\mathrm{base}}$)} &
\hl{$\eta_{\mathrm{base}}r_Lr_n$} \ \ll{($\eta_{\mathrm{base}}$)} \\

Weight Decay
& \hl{$\lambda_{\mathrm{base}}/{r_n}$} \ \ll{($\lambda_{\mathrm{base}}$)}
& \hl{$\lambda_{\mathrm{base}}/r_L$} \ \ll{($\lambda_{\mathrm{base}}$)}
& \hl{$\lambda_{\mathrm{base}}/{r_n}$} \ \ll{($\lambda_{\mathrm{base}}$)}
& \hl{$\lambda_{\mathrm{base}}/(r_Lr_n)$} \ \ll{($\lambda_{\mathrm{base}}$)} \\

\bottomrule
\end{tabular}
\end{table}

In this section, we derive the {CompleteP-style $\mu$P parameterization} from
Condition~\ref{condition: scale-invariant fl} for SGD,
which recovers and extends the $\mu$P scaling rules studied in
\citet{DBLP:conf/nips/BordelonCP24-transformer}.

\subsubsection{Update Rule}

For a weight matrix $\mW_l\in\R^{\nout\times\nin}$, the SGD update rule with weight decay can be written as:
\begin{equation*}
\Delta\mW_l
=
-\eta_l\bigl(\,\underbrace{\nabla_{\mW_l}\mathcal{L}}_{\mA_l}
+\lambda_l \mW_l\bigl).
\end{equation*}
Here, we estimate the scale of raw gradient $\nabla_{\mW_l}\mathcal{L}$ following the spectral argument of~\citet{mup-spectral}. From the derivation of the update condition in Section~\ref{sec: spec condition}, we can observe that gradient updates $\Delta\mW_0,\Delta\mW_{L+1}$ induce a change $\Vert\Delta\vh_{L+1}\Vert_\normrms=\Theta(1)$ in the output, which induces a change $\Delta\mathcal{L} = \Theta(1)$ for common loss functions $\mathcal{L}$. In contrast, each hidden gradient update $\Delta\mW_l$ ($l\in[L]$) induces a change $\Vert\Delta\vh_{L+1}\Vert_\normrms=\Theta(1/L)$ in the output, which induces a change $\Delta\mathcal{L} = \Theta(1/L)$. We use these properties to derive the scale of $\nabla_{\mW_l}\mathcal{L}$ as follows.

For the input weights, we have
\begin{align*}
\Theta(1)=
\Delta_{\mW_0}\mathcal{L} = \Theta(\langle\Delta\mW_0, \nabla_{\mW_0}\mathcal{L}\rangle) = \Theta(\Vert\Delta\mW_0\Vert_\mathrm{F} \Vert\nabla_{\mW_0}\mathcal{L}\Vert_\mathrm{F})=
\Theta(\Vert\Delta\mW_0\Vert_2 \Vert\nabla_{\mW_0}\mathcal{L}\Vert_2),
\end{align*}
where $\langle\cdot,\cdot\rangle$ denotes the trace inner product, and we use the facts that the two arguments of the inner product are proportional to each other and the low-rank property of $\mA_l$ in Equation~(\ref{eqn: low-rank}). Since we finally realize the spectral condition~(\ref{eq:update_inandout}) that $\alpha_0\|\Delta \mW_0\|_{\normrms} = \Theta(1)$ and use $\alpha_0=\Theta(1)$ by initial implementation in Equation~(\ref{eq:alpha_inandout}), we have $\|\Delta\mW_0\|_{\normrms}=\Theta(1)$ so $\|\Delta\mW_0\|_{2}=\Theta(\sqrt{\nout/\nin})$.
Therefore, we obtain $\Vert\nabla_{\mW_0}\mathcal{L}\Vert_2 = \Theta(\sqrt{{\nin}/{\nout}})$, which leads to
\begin{equation*}
\|\mA_0\|_{\normrms}
=
\Vert\nabla_{\mW_0}\mathcal{L}\Vert_\normrms = \sqrt{\frac{\nin}{\nout}} \Vert\nabla_{\mW_0}\mathcal{L}\Vert_2 = \Theta(\frac{\nin}{\nout}) = \Theta(\frac{1}{\nout}).
\end{equation*}

Similarly, for the hidden weight $\mW_l$, we have
\begin{align*}
\Theta(\frac{1}{L})=
\Delta_{\mW_l}\mathcal{L} = \Theta(\langle\Delta\mW_l, \nabla_{\mW_l}\mathcal{L}\rangle) = \Theta(\Vert\Delta\mW_l\Vert_\mathrm{F} \Vert\nabla_{\mW_l}\mathcal{L}\Vert_\mathrm{F})=
\Theta(\Vert\Delta\mW_l\Vert_2 \Vert\nabla_{\mW_l}\mathcal{L}\Vert_2),
\end{align*}
Since we finally set $\alpha_l\|\Delta \mW_l\|_{\normrms}\| \mW_l\|_{\normrms} = \Theta(1/L)$ to satisfy the update condition~(\ref{eq:update_hidd_1}), and use $\alpha_l=\Theta(1/L)$ in Equation~(\ref{eq:alpha_l}), $\| \mW_l\|_{\normrms}=\Theta(1)$ in Equation~(\ref{eq:w_norm}) by initial implementation, we have $\|\Delta\mW_l\|_{\normrms}=\Theta(1)$ so $\|\Delta\mW_l\|_{2}=\Theta(\sqrt{\nout/\nin})$.
Therefore, we obtain $\Vert\nabla_{\mW_l}\mathcal{L}\Vert_2 = \Theta(L^{-1}\sqrt{{\nin}/{\nout}})$, which leads to
\begin{equation*}
\|\mA_l\|_{\normrms}
=
\Vert\nabla_{\mW_l}\mathcal{L}\Vert_\normrms = \sqrt{\frac{\nin}{\nout}} \Vert\nabla_{\mW_l}\mathcal{L}\Vert_2 = \Theta(\frac{\nin}{L\nout}) = \Theta(\frac{1}{L}).
\end{equation*}

Finally, for the output weight $\mW_{L+1}$, we have
\begin{align*}
\Theta(1)=
\Delta_{\mW_{L+1}}\mathcal{L} = \Theta(\langle\Delta\mW_{L+1}, \nabla_{\mW_{L+1}}\mathcal{L}\rangle) = \Theta(\Vert\Delta\mW_{L+1}\Vert_\mathrm{F} \Vert\nabla_{\mW_{L+1}}\mathcal{L}\Vert_\mathrm{F})=
\Theta(\Vert\Delta\mW_{L+1}\Vert_2 \Vert\nabla_{\mW_{L+1}}\mathcal{L}\Vert_2).
\end{align*}
Since we will set $\alpha_{L+1}\|\Delta \mW_{L+1}\|_{\normrms} = \Theta(1)$ to realize the update condition~(\ref{eq:update_inandout}) and use $\alpha_{L+1}=\Theta(1/\nin)$ in initial implementation in Equation~(\ref{eq:alpha_inandout}), we have $\|\Delta\mW_{L+1}\|_{\normrms}=\Theta(\nin)$ so $\|\Delta\mW_{L+1}\|_{2}=\Theta(\sqrt{\nout\nin})$.
Therefore, we obtain $\Vert\nabla_{\mW_{L+1}}\mathcal{L}\Vert_2 = \Theta(1/\sqrt{{\nin}{\nout}})$, which leads to
\begin{equation*}
\|\mA_{L+1}\|_{\normrms}
=
\Vert\nabla_{\mW_{L+1}}\mathcal{L}\Vert_\normrms = \sqrt{\frac{\nin}{\nout}} \Vert\nabla_{\mW_{L+1}}\mathcal{L}\Vert_2 = \Theta(\frac{1}{\nout}) = \Theta(1).
\end{equation*}
To sum up, we have
\begin{align}
\|\mA_l\|_{\normrms}
=
\Vert\nabla_{\mW_l}\mathcal{L}\Vert_\normrms
=
\begin{cases}
\Theta(1/{\nout}), & l=0,\\
\Theta(1/L),             & l\in[\,L\,],\\
\Theta(1),   & l=L+1.
\end{cases}
\label{eqn: sgd gradient matrix}
\end{align}

\subsubsection{Derivation of Parameterization}

\paragraph{Input and output layers.}
When weight decay is disabled ($\lambda_0 = 0$), using the dimension assumptions $d_0,d_{L+1}=\Theta(1)$ and $n_l=\Theta(n)$ in Equation~(\ref{eq:dimensions}), together with the initialization parameterization $\alpha_0=\Theta(1)$ and $\alpha_{L+1}=\Theta(1/\nin)$ in Equation~(\ref{eq:alpha_inandout}), we obtain
\begin{align*}
\alpha_l\|\Delta\mW_l\|_{\normrms}
= \alpha_l\eta_l\|\mA_l\|_{\normrms}
=
\begin{cases}
\Theta\big(\eta_0/{\nout}\big), & l=0,\\
\Theta\big(\eta_{L+1}/{\nin}\big), & l=L+1.
\end{cases}
\end{align*}
To satisfy the input/output update requirement
$\alpha_0\|\Delta\mW_0\|_{\normrms},\ \alpha_{L+1}\|\Delta\mW_{L+1}\|_{\normrms}=\Theta(1)$ in
Condition~(\ref{eq:update_inandout}),
we therefore choose
\begin{align*}
\eta_0 = \Theta({\nout}), \quad
\eta_{L+1} = \Theta({\nin}).
\end{align*}

When weight decay is active ($\lambda_l\neq0$), using $\|\mW_0\|_{\normrms}=\Theta(1)$ and $\|\mW_{L+1}\|_{\normrms}=\Theta(\nin)$ from the initialization implementation in Equation~(\ref{eq:w_norm}), we have
\[
\|\lambda_l\mW_l\|_{\normrms} =
\begin{cases}
\Theta(\lambda_0), & l=0,\\
\Theta(\lambda_{L+1}\nin), & l=L+1.
\end{cases}
\]
Matching this to $\|\mA_l\|_{\normrms}$ as desired by condition~(\ref{eq:condition_wd_2}) yields
\begin{align*}
\lambda_0 = \Theta(1/{\nout}), \quad
\lambda_{L+1} = \Theta(1/{\nin}).
\end{align*}

\paragraph{Hidden layers (first-order).}
For a hidden block we have implemented $\alpha_l=\Theta(1/L)$ in Equation~(\ref{eq:alpha_l}) and $\|\mW_l^{(i)}\|_{\normrms}=\Theta(1)$ in Equation~(\ref{eq:w_norm}).
When $\lambda_l=0$ we obtain
\begin{align*}
\alpha_l \|\Delta\mW_l^{(2)}\|_{\normrms}\,\|\mW_l^{(1)}\|_{\normrms}
&=
\Theta(1/L)\cdot \eta_l^{(2)} \|\mA_l^{(2)}\|_{\normrms}
=
\Theta(\eta_l^{(2)}/L^2).
\end{align*}
Enforcing the first-order hidden update condition~(\ref{eq:update_hidd_1}) that
$\alpha_l\|\Delta\mW_l^{(2)}\|_{\normrms}\,\|\mW_l^{(1)}\|_{\normrms}=\Theta(1/L)$ gives
\begin{align*}
\eta_l^{(2)} = \Theta(L).
\end{align*}
By the same reasoning, the same choice applies to the other learning rate, $\eta_l^{(1)}=\Theta(L)$.

If weight decay is enabled on hidden matrices, using $\|\mW_l^{(i)}\|_{\normrms}=\Theta(1)$ by Equation~(\ref{eq:w_norm}), we obtain
$\|\lambda_l^{(i)}\mW_l^{(i)}\|_{\normrms}=\Theta(\lambda_l^{(i)})$,
so condition~(\ref{eq:condition_wd_2}) that $\|\lambda_l \mW_l\|_{\normrms} = \Theta\left( \|\mA_l\|_{\normrms} \right)=\Theta(1/L)$ implies the natural choice
\begin{align*}
\lambda_l^{(i)}=\Theta(1/L),\qquad i=1,2.
\end{align*}

\paragraph{Hidden layers (second-order).}
As illustrated in Section~\ref{sec: Final Initial Condition} and Appendix~\ref{app:muon}, the second-order update condition is satisfied automatically once the initial condition and the first-order update condition are met.

\paragraph{Biases.}
For biases, we use the spectral bias condition in
Equation~(\ref{eq:bias_spectral_condition}), which sets
$\|\Delta\vb_l\|_{\normrms}=\Theta(1)$ for the input and hidden biases under
Condition~\ref{condition: scale-invariant fl}.
We estimate the corresponding
raw-gradient scales in the same way as for matrix weights.
For the input biases $\vb_0 \in \R^{\nout \times 1}$, we have
\begin{align*}
\Theta(1)=
\Delta_{\vb_0}\mathcal{L} = \Theta(\langle\Delta\vb_0, \nabla_{\vb_0}\mathcal{L}\rangle) = \Theta(\Vert\Delta\vb_0\Vert_2 \Vert\nabla_{\vb_0}\mathcal{L}\Vert_2),
\end{align*}
Since we finally set $\|\Delta \vb_0\|_{\normrms} = \Theta(1)$ to satisfy the update condition in Equation~(\ref{eq:bias_spectral_condition}), we have $\|\Delta\vb_0\|_{2}=\Theta(\sqrt{\nout})$.
Therefore, we obtain $\Vert\nabla_{\vb_0}\mathcal{L}\Vert_2 = \Theta(1/\sqrt{{\nout}})$, which leads to
\begin{equation*}
\Vert\nabla_{\vb_0}\mathcal{L}\Vert_\normrms = \sqrt{\frac{1}{\nout}} \Vert\nabla_{\vb_0}\mathcal{L}\Vert_2 = \Theta(\frac{1}{\nout}).
\end{equation*}

Similarly, for the hidden biases $\vb_l \in \R^{\nout \times 1}$, we have
\begin{align*}
\Theta(1/L)=
\Delta_{\vb_l}\mathcal{L} = \Theta(\langle\Delta\vb_l, \nabla_{\vb_l}\mathcal{L}\rangle) = \Theta(\Vert\Delta\vb_l\Vert_2 \Vert\nabla_{\vb_l}\mathcal{L}\Vert_2),
\end{align*}
Since we finally set $\|\Delta \vb_l\|_{\normrms} = \Theta(1)$ to satisfy the update condition in Equation~(\ref{eq:bias_spectral_condition}), we have $\|\Delta\vb_l\|_{2}=\Theta(\sqrt{\nout})$.
Therefore, we obtain $\Vert\nabla_{\vb_l}\mathcal{L}\Vert_2 = \Theta(1/(L\sqrt{{\nout}}))$, which leads to
\begin{equation*}
\Vert\nabla_{\vb_l}\mathcal{L}\Vert_\normrms = \sqrt{\frac{1}{\nout}} \Vert\nabla_{\vb_l}\mathcal{L}\Vert_2 = \Theta(\frac{1}{L\nout}).
\end{equation*}

To sum up, we can estimate the scale of $\Vert\nabla_{\vb_l}\mathcal{L}\Vert_\normrms$ as
\begin{align*}
\Vert\nabla_{\vb_l}\mathcal{L}\Vert_\normrms
=
\begin{cases}
\Theta(1/{\nout}), & l=0,\\
\Theta(1/(L\nout)),       & l\in[L].
\end{cases}
\end{align*}
Requiring $\|\Delta\vb_l\|_\normrms = \eta_{\vb_l} \Vert\nabla_{\vb_l}\mathcal{L}\Vert_\normrms = \Theta(1)$ according to update condition in Equation~(\ref{eq:bias_spectral_condition}) leads to the learning rate as
\begin{align*}
    \eta_{\vb_l}
    =
    \begin{cases}
    \Theta(\nout), & l=0,\\
    \Theta(L\nout),       & l\in[L].
    \end{cases}
\end{align*}
For the weight decays, we need to satisfy $\lambda_{\vb_l} \|\vb_l\|_\normrms = \Vert\nabla_{\vb_l}\mathcal{L}\Vert_\normrms$ and given $\|\vb_l\|_\normrms = \Theta(1)$ by initialization implementation in Condition~\ref{condition: bias}, we have
\begin{align*}
    \lambda_{\vb_l}=
    \begin{cases}
    \Theta(1/\nout), & l=0,\\
    \Theta\left(1/(L\nout)\right),       & l\in[L].
    \end{cases}
\end{align*}

This completes the implementation of the update condition for SGD with
weight decay, which is summarized in Table~\ref{tab: sgd-mup}.

\subsection{AdamW}
\label{app:admw}

\begin{table}[t]
\renewcommand{\arraystretch}{1.3}
\setlength{\tabcolsep}{2pt}
\renewcommand{\hl}[1]{\textcolor{purple}{#1}}
\renewcommand{\ll}[1]{\textcolor{gray}{#1}}
\centering
\caption{\textbf{{CompleteP-style $\mu$P parameterization from Condition~\ref{condition: scale-invariant fl} ($k=2$)} for AdamW~\citep{adamw}, Lion~\citep{DBLP:conf/nips/ChenLHRW0DLHLL23-lion}, and Sophia~\citep{DBLP:conf/iclr/Liu0HL024-sophia} with weight decay under width–depth scaling.}
Entries in \hl{purple} indicate differences between $\mu$P and SP, while \ll{gray} shows the corresponding SP choices.
Here, $r_n$ and $r_L$ denote the width and depth scaling ratios relative to the base model.
The variance of input weights is $\sigma^2_{\mathrm{base}}$ for language and $\sigma^2_{\mathrm{base}}/d_0$ for image. The initial variance of input bias is $\sigma^2_{\mathrm{base}}$.}
\label{tab: adamw-mup}
\vskip 0.05in
\begin{tabular}{ccccc}
\toprule
 & Input weights \& biases & Hidden weights & Output weights & Hidden biases\\
\midrule

Block Multiplier
& $\alpha_{\mathrm{base}}$
& \hl{$\alpha_{\mathrm{base}}/r_L$} \ \ll{($\alpha_{\mathrm{base}}$)}
& \hl{$\alpha_{\mathrm{base}}/r_n$} \ \ll{($\alpha_{\mathrm{base}}$)} 
&\hl{$\alpha_{\mathrm{base}}/r_L$} \ \ll{($\alpha_{\mathrm{base}}$)} \\

Initial Variance
& $\sigma^2_{\mathrm{base}}/d_0$ or $\sigma^2_{\mathrm{base}}$
& \hl{$\sigma^2_{\mathrm{base}}/r_n$} \ \ll{($\sigma^2_{\mathrm{base}}$)}
& $\sigma^2_{\mathrm{base}}$
&$\sigma^2_{\mathrm{base}}$ \\

Learning Rate
& $\eta_{\mathrm{base}}$
& \hl{$\eta_{\mathrm{base}}/r_n$} \ \ll{($\eta_{\mathrm{base}}$)}
& $\eta_{\mathrm{base}}$
& $\eta_{\mathrm{base}}$ \\

Weight Decay
& $\lambda_{\mathrm{base}}$
& \hl{$\lambda_{\mathrm{base}}r_n$} \ \ll{($\lambda_{\mathrm{base}}$)}
& $\lambda_{\mathrm{base}}$
& $\lambda_{\mathrm{base}}$ \\

AdamW $\varepsilon$
& \hl{$\varepsilon_{\mathrm{base}}/r_n$} \ \ll{($\varepsilon_{\mathrm{base}}$)}
& \hl{$\varepsilon_{\mathrm{base}}/(r_Lr_n)$} \ \ll{($\varepsilon_{\mathrm{base}}$)}
& \hl{$\varepsilon_{\mathrm{base}}/r_n$} \ \ll{($\varepsilon_{\mathrm{base}}$)}
& \hl{$\varepsilon_{\mathrm{base}}/(r_Lr_n)$} \ \ll{($\varepsilon_{\mathrm{base}}$)}\\
\bottomrule
\end{tabular}
\end{table}

In this section, we derive the {CompleteP-style $\mu$P parameterization} from
Condition~\ref{condition: scale-invariant fl} for AdamW,
which recovers and extends the $\mu$P scaling rules studied in
\citet{completep}.

\subsubsection{Update Rule}
\label{app:adam_update_rule}

First, we present the full update rule of AdamW~\citep{adamw}.
To distinguish iteration steps, we append a superscript $t \in [T]$, which might be omitted later when no confusion arises.
\begin{align*}
    \mW_l^{(t)} = \mW_l^{(t-1)} - \eta_l^{(t)} \left(\mathrm{AdamW}\left({\nabla_{\mW_l^{(t-1)}}}\mathcal{L}\right) + \lambda_l {\mW_l^{(t-1)}}\right),
\end{align*}
where
\begin{equation}
\begin{aligned}
    &\mathrm{AdamW}\left({\nabla_{\mW_l^{(t-1)}}}\mathcal{L}\right) = \frac{\hat{\vm}_l^{(t)}}{\sqrt{\hat{\vv}_l^{(t)}}+\varepsilon_l}, \\
    &\text{where} \quad \left\{
    \begin{array}{ll}
    \hat{\vm}_l^{(t)} = \frac{{\vm}_l^{(t)}}{1-\beta_1^t}, \quad {\vm}_l^{(t)} = \beta_1 {\vm}_l^{(t-1)} + (1-\beta_1){\nabla_{\mW_l^{(t-1)}}}\mathcal{L},\\
    \hat{\vv}_l^{(t)} = \frac{{\vv}_l^{(t)}}{1-\beta_2^t}, \quad {\vv}_l^{(t)} = \beta_2 {\vv}_l^{(t-1)} + (1-\beta_2)\left({\nabla_{\mW_l^{(t-1)}}}\mathcal{L}\right)^2.
    \end{array} \right.
\end{aligned}
\label{eq:Adam_as_function}
\end{equation}
We simplify the full update rule by omitting the momentum and the stabilization term, i.e., setting $\beta_1 = 0$, $\beta_2 = 0$, and $\varepsilon_l = 0$.
As discussed at the beginning of the Appendix~\ref{app: implementation opts HPs}, omitting momentum does not affect the scaling analysis.
The stabilization term $\varepsilon_l$ must, in fact, be scaled consistently with $\sqrt{\hat{\vv}_l^{(t)}}$, and since it does not alter the resulting parameterization of learning rate, we defer its discussion to the end of this section.
Now, the AdamW is reduced to sign gradient descent as:
\begin{align*}
    \mW_l^{(t)} = \mW_l^{(t-1)} - \eta_l^{(t)} \left(\mathrm{sign}\left({\nabla_{\mW_l^{(t-1)}}}\mathcal{L}\right) + \lambda_l {\mW_l^{(t-1)}}\right).
\end{align*}
Here, the superscript of the iteration step can be left out, and we write this simplified update rule as:
\begin{align}
    \Delta{\mW}_l
    =-\eta_l \bigl(\,\underbrace{\mathrm{sign}\left(\nabla_{\mW_l}\mathcal{L}\right)}_{\mA_l} 
    + \lambda_l {\mW}_l\bigl).
\label{eq:updaterule_signgd}
\end{align}
Given the dimension assumption $d_0,d_{L+1}=\Theta(1), n_l=\Theta(n)$ in Equation~(\ref{eq:dimensions}), the norm of $\mA_l$ satisfies
\begin{align} 
\Vert \mA_l \Vert_\normrms 
&= \left\Vert \mathrm{sign}\left(\nabla_{\mW_l}\mathcal{L}\right) \right\Vert_\normrms 
= \sqrt{\frac{\nin}{\nout}}\left\Vert \mathrm{sign}\left(\nabla_{\mW_l}\mathcal{L}\right) \right\Vert_2 \nonumber \\
&= \sqrt{\frac{\nin}{\nout}} \cdot \Theta\left(\left\Vert \mathrm{sign}\left(\nabla_{\mW_l}\mathcal{L}\right) \right\Vert_F\right) \tag{low-rank property of $\mA_l$ in Equation~(\ref{eqn: low-rank})} \\
&= \Theta\left(\sqrt{\frac{\nin}{\nout}}\sqrt{\nin\nout}\right) \nonumber \\
&= \Theta(\nin)
=\left\{
    \begin{array}{ll}
    \Theta(1),   & l = 0,\\
    \Theta(\nin),   & l \in [L],\\
    \Theta(\nin),  & l = L+1.
    \end{array} \right. \label{eq:Anorm_adam}
\end{align}

\subsubsection{Derivation of Parameterization}

\paragraph{Input and output layers.} 
When weight decay is disabled ($\lambda_0 = 0$), given the dimension assumption $d_0,d_{L+1}=\Theta(1), n_l=\Theta(n)$ in Equation~(\ref{eq:dimensions}), the multiplier parameterizations $\alpha_0=\Theta(1),\alpha_{L+1}=\Theta(1/\nin)$ in Equation~(\ref{eq:alpha_inandout}), and the scale of $\Vert\mA_l\Vert_\normrms$ in Equation~(\ref{eq:Anorm_adam}), we have
\begin{align*}
    \alpha_l\|\Delta\mW_l\|_{\normrms}
    =\alpha_l\eta_l\|\mA_l\|_{\normrms}
    =\left\{
    \begin{array}{ll}
    \Theta(\eta_0),   & l = 0,\\
    \Theta(\eta_{L+1}),  & l = L+1.
    \end{array} \right.
\end{align*}
As desired in (\ref{eq:condition_wd_1}), to satisfy (\ref{eq:update_inandout}) that $\alpha_0\|\Delta \mW_0\|_{\normrms}, \alpha_{L+1}\|\Delta \mW_{L+1}\|_{\normrms} = \Theta(1)$, we need to set 
\begin{align*}
    \eta_0=\Theta(1), \quad \eta_{L+1}=\Theta(1).
\end{align*}
When $\lambda_l \neq 0$, given Equation~(\ref{eq:w_norm}) that $\|{\mW}_0\|_{\normrms} = \Theta(1)$ and $\|{\mW}_{L+1}\|_{\normrms} = \Theta(\nin)$, we have
\begin{align*}
    \|\lambda_l\mW_l\|_{\normrms} 
    =\left\{
    \begin{array}{ll}
    \Theta(\lambda_0),   & l = 0,\\
    \Theta(\lambda_{L+1}\nin),  & l = L+1.
    \end{array} \right.
\end{align*}
To satisfy (\ref{eq:condition_wd_2}) that $\|\lambda_l \mW_l\|_{\normrms}
    = \Theta\left( \|\mA_l\|_{\normrms} \right)$, we need to set 
\begin{align*}
    \lambda_0=\Theta(1), \quad \lambda_{L+1}=\Theta(1).
\end{align*}

\paragraph{Hidden layers (first-order).} 
When weight decay is disabled (${\lambda_l = 0}$), given the dimension assumption $d_0,d_{L+1}=\Theta(1), n_l=\Theta(n)$ in Equation~(\ref{eq:dimensions}), the weight norm $\|{\mW}_l\|_{\normrms} = \Theta(1)$ in Equation~(\ref{eq:w_norm}), the multiplier parameterization $\alpha_l=\Theta(1/L)$ in Equation~(\ref{eq:alpha_l}), and the scale of $\Vert\mA_l\Vert_\normrms$ in Equation~(\ref{eq:Anorm_adam}), we have
\begin{align*}
    \alpha_l\|\Delta\mW_l^{(2)}\|_{\normrms}\|\mW_l^{(1)}\|_{\normrms}
    =\Theta(1/L) \cdot \eta_l^{(2)}\|\mA_l^{(2)}\|_{\normrms}\|\mW_l^{(1)}\|_{\normrms}
    =\Theta(\eta_l^{(2)}\nin/L).
\end{align*}
As desired in (\ref{eq:condition_wd_1}), to satisfy the first-order update condition on hidden weights (\ref{eq:update_hidd_1}) that $\alpha_l\|\Delta \mW_l^{(2)}\|_{\normrms}\,\|\mW_l^{(1)}\|_{\normrms}=\Theta(1/L)$, we need to set 
\begin{align*}
    \eta_l^{(2)}=\Theta(1/\nin).
\end{align*}
When $\lambda_l \neq 0$, given the weight norm $\|{\mW}_l\|_{\normrms} = \Theta(1)$ in Equation~(\ref{eq:w_norm}), we have
\begin{align*}
    \|\lambda_l^{(2)}\mW_l^{(2)}\|_{\normrms} = \Theta(\lambda_l^{(2)}).
\end{align*}
To satisfy (\ref{eq:condition_wd_2}) that $\|\lambda_l \mW_l\|_{\normrms}
    = \Theta\left( \|\mA_l\|_{\normrms} \right)$ we need to set 
\begin{align*}
    \lambda_l^{(2)}=\Theta(\nin).
\end{align*}

Symmetrically, we have the same choice for $\mW_l^{(1)}$:
\begin{align*}
    \eta_l^{(1)} = \Theta(1/\nin), \quad \lambda_l^{(1)}=\Theta(\nin).
\end{align*}

\paragraph{Hidden layers (second-order).} 
As illustrated in Section~\ref{sec: Final Initial Condition} or in Appendix~\ref{app:muon} for Muon, the second-order update condition is satisfied automatically once the initial condition and the first-order update condition are met.

\paragraph{Biases.} For bias parameters $\vb_l \in \R^{\nout\times 1}$, by the definition we have
\begin{align*}
\|\mathrm{sign}\left(\nabla_{\vb_l}\mathcal{L}\right)\|_\normrms = \Theta(1),  \quad 0 \leq l \leq L.
\end{align*}

To satisfy the condition $
\|\Delta \vb_l\|_{\normrms} = \Theta(1)$ for $\forall\ 0 \le l \le L$ in Equation~(\ref{eq:bias_spectral_condition}), we set
\begin{align*}
    \eta_{\vb_l}=\Theta(1), \quad 0 \leq l \leq L,
\end{align*}
and the corresponding weight decay
\begin{align*}
    \lambda_{\vb_l}=\Theta(1), \quad 0 \leq l \leq L.
\end{align*}

\paragraph{Parameterization of $\varepsilon_l$.}

To make the stabilization term $\varepsilon_l$ effective and not dominate the gradient, we desire it to be of the same scale as $\sqrt{\hat{\vv}_l^{(t)}}$. 
When omitting the momentum, we have ${\sqrt{\hat{\vv}_l} = \left|\nabla_{\mW_l}\mathcal{L}\right|}$ {elementwise}.
Therefore, we need to ensure $\varepsilon_l = \Theta(\|\nabla_{\mathrm{Vec}(\mW_l)}\mathcal{L}\|_\normrms)$, the latter can be estimated based on the derivation for $\Vert\nabla_{\mW_l}\mathcal{L}\Vert_\normrms$ in Equation~(\ref{eqn: sgd gradient matrix}) of Appendix~\ref{app:sgd_mup}. 

For the input weights $\mW_0$, we have
\begin{align*}
\|\nabla_{\mathrm{Vec}(\mW_0)}\mathcal{L}\|_\normrms 
= \frac{1}{\sqrt{\nin\nout}}\|\nabla_{\mW_0}\mathcal{L}\|_\mathrm{F}
= \Theta\left(\frac{1}{\sqrt{\nin\nout}} \|\nabla_{\mW_0}\mathcal{L}\|_\mathrm{2}\right)
= \Theta\left(\frac{1}{\sqrt{\nin\nout}} \sqrt{\frac{\nin}{\nout}}\right)
= \Theta(\frac{1}{\nout}).
\end{align*}
Therefore, we set 
$$\varepsilon_0 = \Theta(\frac{1}{\nout}).$$

For the hidden weights $\mW_l, l\in[L]$, we have
\begin{align*}
\|\nabla_{\mathrm{Vec}(\mW_l)}\mathcal{L}\|_\normrms 
= \Theta\left(\frac{1}{\sqrt{\nin\nout}} \|\nabla_{\mW_l}\mathcal{L}\|_\mathrm{2}\right)
= \Theta\left(\frac{1}{\sqrt{\nin\nout}} \frac{1}{L} \sqrt{\frac{\nin}{\nout}}\right)
= \Theta(\frac{1}{L\nout}).
\end{align*}
Therefore, we set 
$$\varepsilon_l = \Theta(\frac{1}{L\nout}), \quad l \in [L].$$

For the output weights $\mW_{L+1}$, we have
\begin{align*}
\|\nabla_{\mathrm{Vec}(\mW_{L+1})}\mathcal{L}\|_\normrms 
= \Theta\left(\frac{1}{\sqrt{\nin\nout}} \|\nabla_{\mW_{L+1}}\mathcal{L}\|_\mathrm{2}\right)
= \Theta\left(\frac{1}{\sqrt{\nin\nout}}  \sqrt{\frac{1}{\nin\nout}}\right)
= \Theta(\frac{1}{\nin}).
\end{align*}
Therefore, we set 
$$\varepsilon_{L+1} = \Theta(\frac{1}{\nin}).$$

Similarly, for the biases we have derived in Appendix~\ref{app:sgd_mup} that
\begin{align*}
\Vert\nabla_{\vb_l}\mathcal{L}\Vert_\normrms
=
\begin{cases}
\Theta(1/{\nout}), & l=0,\\
\Theta(1/(L\nout)),       & l\in[L].
\end{cases}
\end{align*}
Therefore, we set the stabilization term as
\begin{align*}
\varepsilon_{\vb_l}
=
\begin{cases}
\Theta(1/{\nout}), & l=0,\\
\Theta(1/(L\nout)),       & l\in[L].
\end{cases}
\end{align*}

This completes the implementation of the update condition for AdamW, which is summarized in Table~\ref{tab: adamw-mup}.

\subsection{Lion}
\label{app:lion}

In this section, we derive the {CompleteP-style $\mu$P parameterization} from
Condition~\ref{condition: scale-invariant fl} for Lion.

The full update rule of Lion~\citep{DBLP:conf/nips/ChenLHRW0DLHLL23-lion} is
\begin{align*}
    \mW_l^{(t)} = \mW_l^{(t-1)} - \eta^{(t)} \left(\vu_l^{(t)} + \lambda_l {\mW_l^{(t-1)}}\right),
\end{align*}
where
\begin{align*}
    &\vu_l^{(t)} = \mathrm{sign}\left(\beta_1 {\vm}_l^{(t-1)} + (1-\beta_1){\nabla_{\mW_l^{(t-1)}}}\mathcal{L}\right),\\
    & {\vm}_l^{(t)} = \beta_2 {\vm}_l^{(t-1)} + (1-\beta_2) {\nabla_{\mW_l^{(t-1)}}}\mathcal{L}.
\end{align*}

If the momentum terms are omitted, i.e., setting $\beta_1 = 0$ and $\beta_2 = 0$, Lion is reduced to sign gradient descent as in the simplified AdamW update rule in Equation~(\ref{eq:updaterule_signgd}).
Therefore, we reuse AdamW's parameterizations in Table~\ref{tab: adamw-mup} for Lion.

\subsection{Sophia}
\label{app:sophia}

In this section, we derive the {CompleteP-style $\mu$P parameterization} from
Condition~\ref{condition: scale-invariant fl} for Sophia.

The full update rule of Sophia~\citep{DBLP:conf/iclr/Liu0HL024-sophia} is
\begin{align*}
    \mW_l^{(t)} = \mW_l^{(t-1)} - \eta^{(t)} \left(\mathrm{clip}\left(\frac{\vm_l^{(t)}}{\max\{\rho \vh_l^{(t)}, \varepsilon\}}, 1 \right) + \lambda_l {\mW_l^{(t-1)}}\right),
\end{align*}
where {$\mG_l^{(t)}=\nabla_{\mW_l^{(t-1)}}\mathcal{L}$ is the gradient at the start of step $t$}, and
\begin{align*}
    \vm_l^{(t)} = \beta_1 \vm_l^{(t-1)} + (1-\beta_1) \mG_l^{(t)},
\end{align*}
and $\vh_l^{(t)}$ is updated every $k$ iterations as
\begin{align*}
    \vh_l^{(t)} =
    \left\{
    \begin{array}{ll}
    \beta_2 \vh_l^{(t-1)} + (1-\beta_2) \hat{\vh}_l^{(t)}, & {(t-1) \ \mathrm{mod} \ k = 0},\\
    \vh_l^{(t-1)} ,& {(t-1) \ \mathrm{mod} \ k \neq 0}.
    \end{array}\right.
\end{align*}
where
the elements of $\hat{\vh}_l^{(t)}$ are the diagonal second-order derivatives with respect to ${\mW_l^{(t-1)}}$, i.e.,
$\left(\hat{\vh}_l^{(t)}\right)_{ij} = \frac{\partial^2\mathcal{L}}{\partial {({\mW_l^{(t-1)}})}_{ij}^2}$.

{
Let $\mA_l^{(t)} = \mathrm{clip}\left(\frac{\vm_l^{(t)}}{\max\{\rho \vh_l^{(t)}, \varepsilon\}}, 1 \right)$. We introduce the following non-degeneracy assumption.
\begin{assumption}[Non-degenerate clipping for Sophia]
\label{ass:sophia-nondegeneracy}
For each layer $l$ and training step $t$, at least a size-independent constant fraction of the entries of $\mA_l^{(t)}$ are clipped and hence have magnitude $1$.
\end{assumption}

This assumption is consistent with the recommended usage of Sophia. The original paper~\citep{DBLP:conf/iclr/Liu0HL024-sophia} recommends a clipping proportion between $50\%$ and $90\%$, and the clipping proportions observed in our experiments fall in a similar range.

Since every entry of $\mA_l^{(t)}$ has magnitude at most $1$, while a constant fraction have magnitude $1$, Assumption~\ref{ass:sophia-nondegeneracy} gives
\begin{align*}
    \|\mA_l^{(t)}\|_{\mathrm F}
    = \Theta\left(\sqrt{\nin\nout}\right).
\end{align*}
Together with the low-rank update structure in Assumption~\ref{ass:low-rank-update}, this yields
\begin{align*}
    \|\mA_l^{(t)}\|_2
    &= \Theta\left(\|\mA_l^{(t)}\|_{\mathrm F}\right)
    = \Theta\left(\sqrt{\nin\nout}\right),\\
    \|\mA_l^{(t)}\|_{\normrms}
    &= \sqrt{\frac{\nin}{\nout}}\|\mA_l^{(t)}\|_2
    = \Theta(\nin).
\end{align*}
It shows that Sophia has the same update-norm scaling as AdamW in Equation~(\ref{eq:Anorm_adam}), so Sophia shares AdamW's parameterizations in Table~\ref{tab: adamw-mup}.
}

\subsection{Shampoo}
\label{app:shampoo}

In this section, we derive the {CompleteP-style $\mu$P parameterization} from
Condition~\ref{condition: scale-invariant fl} for Shampoo,
which recovers and extends the $\mu$P scaling rules studied in
\citet{muon-cp}.

\subsubsection{Update Rule}

Denote $\mG_l^{(t)} = {\nabla_{\mW_l^{(t-1)}}}\mathcal{L}$.
Following Shampoo~\citep{gupta2018shampoo}, the gradient is preconditioned by Kronecker-factored left and right preconditioners:
\begin{align}
    &\mM_l^{(t)} = \beta_1\mM_l^{(t-1)} + (1-\beta_1)\mG_l^{(t)}, \nonumber\\
    &\mL_l^{(t)} = \beta_2\mL_l^{(t-1)} + (1-\beta_2)\mG_l^{(t)}{\mG_l^{(t)}}^\top,\nonumber\\
    &\mR_l^{(t)} = \beta_2\mR_l^{(t-1)} + (1-\beta_2){\mG_l^{(t)}}^\top\mG_l^{(t)},\nonumber\\
    &\widehat{\mL}_l^{(t)} = \frac{\mL_l^{(t)}}{1-\beta_2^t}, \qquad
    \widehat{\mR}_l^{(t)} = \frac{\mR_l^{(t)}}{1-\beta_2^t},\nonumber\\
    &\mW_l^{(t)}
    =
    \mW_l^{(t-1)}
    -
    \eta^{(t)}
    \left(
    \left(\widehat{\mL}_l^{(t)} + \varepsilon_l \mI\right)^{-\frac{1}{4}}
    \mM_l^{(t)}
    \left(\widehat{\mR}_l^{(t)} + \varepsilon_l \mI\right)^{-\frac{1}{4}}
    + \lambda_l{\mW_l^{(t-1)}}
    \right),
\label{eq:full_shampoo}
\end{align}

We simplify the full update rule by omitting the momentum and the stabilization term, i.e., setting $\beta_1 = 0$, $\beta_2 = 0$, and $\varepsilon_l = 0$.
As discussed at the beginning of the Appendix~\ref{app: implementation opts HPs}, omitting momentum does not affect the scaling analysis.
The stabilization term $\varepsilon_l$ must, in fact, be scaled consistently with $\widehat{\mL}_l^{(t)}$ and $\widehat{\mR}_l^{(t)}$, and since it does not alter the resulting parameterization of learning rate, we defer its discussion to the end of this section.

Now, we have $\widehat{\mL}_l^{(t)}=\mG_l^{(t)}{\mG_l^{(t)}}^\top$, and
$\widehat{\mR}_l^{(t)}={\mG_l^{(t)}}^\top\mG_l^{(t)}$.
Therefore, we obtain
\begin{align*}
    \mW_l^{(t)}
    =
    \mW_l^{(t-1)}
    -
    \eta^{(t)}
    \left(
    \left(\mG_l^{(t)}{\mG_l^{(t)}}^\top\right)^{-\frac{1}{4}}
    \mG_l^{(t)}
    \left({\mG_l^{(t)}}^\top\mG_l^{(t)}\right)^{-\frac{1}{4}}
    + \lambda_l{\mW_l^{(t-1)}}
    \right).
\end{align*}

\subsubsection{Derivation of Parameterization}

\paragraph{Learning rate and weight decay.}
Applying compact SVD to $\mG_l^{(t)}$ as in Muon, and interpreting the inverse powers on the support of the gradient, we have
\begin{align*}
    \mG_l^{(t)}=\mU_l^{(t)} \mSigma_l^{(t)} {\mV_l^{(t)}}^\top,
\end{align*}
and thus
\begin{align*}
    \mG_l^{(t)}{\mG_l^{(t)}}^\top
    = \mU_l^{(t)} {\mSigma_l^{(t)}}^2 {\mU_l^{(t)}}^\top,
    \quad
    {\mG_l^{(t)}}^\top\mG_l^{(t)}
    =
    \mV_l^{(t)} {\mSigma_l^{(t)}}^2 {\mV_l^{(t)}}^\top.
\end{align*}
The preconditioned direction is therefore
\begin{align*}
&\left(\mG_l^{(t)}{\mG_l^{(t)}}^\top\right)^{-\frac{1}{4}}
\mG_l^{(t)}
\left({\mG_l^{(t)}}^\top\mG_l^{(t)}\right)^{-\frac{1}{4}} \\
&\quad =
\mU_l^{(t)}{\mSigma_l^{(t)}}^{-\frac{1}{2}}{\mU_l^{(t)}}^\top
\mU_l^{(t)}\mSigma_l^{(t)}{\mV_l^{(t)}}^\top
\mV_l^{(t)}{\mSigma_l^{(t)}}^{-\frac{1}{2}}{\mV_l^{(t)}}^\top
=
\mU_l^{(t)}{\mV_l^{(t)}}^\top.
\end{align*}
Consequently, under this simplification, Shampoo reduces to
\begin{align*}
    \mW_l^{(t)}
    =
    \mW_l^{(t-1)}
    -
    \eta^{(t)}
    \left(\mU_l^{(t)}{\mV_l^{(t)}}^\top + \lambda_l {\mW_l^{(t-1)}}\right),
\end{align*}
which matches the update rule of Muon in Equation~(\ref{eq:updaterule_muon}).
Therefore, the learning rate and weight decay of Shampoo share Muon's parameterizations in Table~\ref{tab: muon-wd mup}.

Note that the hidden layer learning rate parameterization derived in~\citet[Table 1]{muon-cp} is $\frac{(\nout/\nin)^{1-(e_L+e_R)}}{L^{2(e_L+e_R)-1}n_{\mathrm{blk}}^{e_L+e_R}}$, where the $e_L$ and $e_R$ are the exponents of $\mL_l^{(t)}$ and $\mR_l^{(t)}$ which equal $\frac{1}{4}$ in the standard Shampoo as Equation~(\ref{eq:full_shampoo}), and $n_{\mathrm{blk}}$ is the number of blocks which equals $1$ when blocking is not used. In this case, $\frac{(\nout/\nin)^{1-(e_L+e_R)}}{L^{2(e_L+e_R)-1}n_{\mathrm{blk}}^{e_L+e_R}} = \sqrt{\nout/\nin} = \Theta(1)$, consistent with our result for Shampoo in Table~\ref{tab: muon-wd mup}.

\paragraph{Parameterization of $\varepsilon_l$.}

To make the stabilization term $\varepsilon_l$ effective and not dominate the update, we desire it to be of the same scale as the {spectral norms} of $\widehat{\mL}_l^{(t)}$ and $\widehat{\mR}_l^{(t)}$.
When omitting the momentum, we have
\begin{align*}
\|\widehat{\mL}_l^{(t)}\|_2 = \|\widehat{\mR}_l^{(t)}\|_2 =
\|\mG_l^{(t)}\|^2_2
=\|{\nabla_{\mW_l^{(t-1)}}}\mathcal{L}\|^2_2 =
\left\{
    \begin{array}{ll}
    \Theta({\frac{\nin}{\nout}}) = \Theta({\frac{1}{\nout}}),   & l = 0,\\
    \Theta(\frac{1}{L^2}{\frac{\nin}{\nout}}) = \Theta(\frac{1}{L^2}),   & l \in [L],\\
    \Theta({\frac{1}{\nin\nout}}) = \Theta({\frac{1}{\nin}}),  & l = L+1,
    \end{array} \right.
\end{align*}
where the estimation of $\|{\nabla_{\mW_l^{(t-1)}}}\mathcal{L}\|_2$ can be found in the derivation for $\Vert\nabla_{\mW_l}\mathcal{L}\Vert_\normrms$ in Equation~(\ref{eqn: sgd gradient matrix}) at Appendix~\ref{app:sgd_mup}.
Therefore, we need to set the stabilization term as
\begin{align*}
\varepsilon_l =
\left\{
    \begin{array}{ll}
    \Theta({\frac{1}{\nout}}),   & l = 0,\\
     \Theta(\frac{1}{L^2}),   & l \in [L],\\
    \Theta({\frac{1}{\nin}}),  & l = L+1.
    \end{array} \right.
\end{align*}

Note that the hidden layer $\varepsilon_l$ parameterization derived in~\citet[Table 1]{muon-cp} is $\frac{\nin}{L^2\nout n_{\mathrm{blk}}}$, where the $n_{\mathrm{blk}}$ is the number of blocks which equals $1$ when blocking is not used. In this case,$\frac{\nin}{L^2\nout n_{\mathrm{blk}}} = \frac{\nin}{L^2\nout} = \Theta(\frac{1}{L^2})$, consistent with our result for Shampoo in Table~\ref{tab: muon-wd mup}.

\subsection{SOAP}
\label{app:soap}

In this section, we derive the {CompleteP-style $\mu$P parameterization} from
Condition~\ref{condition: scale-invariant fl} for SOAP,
which recovers and extends the $\mu$P scaling rules studied in
\citet{muon-cp}.

Denote the weight gradient as $\mG_l^{(t)} = {\nabla_{\mW_l^{(t-1)}}}\mathcal{L}\in \R^{\nout\times\nin}$ and its rank $r = \mathrm{rank}(\mG_l^{(t)})$.
SOAP adopts a similar precondition of Shampoo's for $\mL_l^{(t)}$ and $\mR_l^{(t)}$:
\begin{align*}
    \mL_l^{(t)} = \beta_3\mL_l^{(t-1)} + (1-\beta_3)\mG_l^{(t)}{\mG_l^{(t)}}^\top, \quad \mR_l^{(t)} = \beta_3\mR_l^{(t-1)} + (1-\beta_3){\mG_l^{(t)}}^\top\mG_l^{(t)}.
\end{align*}
By applying eigendecomposition to matrices $\mL_l^{(t)} \in \R^{\nout\times\nout}$ and $\mR_l^{(t)}\in \R^{\nin\times\nin}$, we get two orthogonal matrices $\mQ_{\mL_l}^{(t)}\in \R^{\nout\times\nout}$ and $\mQ_{\mR_l}^{(t)}\in \R^{\nin\times\nin}$ as:
\begin{align}
    \mL_l^{(t)} = \mQ_{\mL_l}^{(t)}\boldsymbol{\Lambda}_{\mL_l}^{(t)}{\mQ_{\mL_l}^{(t)}}^\top, \quad \mR_l^{(t)} = \mQ_{\mR_l}^{(t)}\boldsymbol{\Lambda}_{\mR_l}^{(t)}{\mQ_{\mR_l}^{(t)}}^\top.
\label{eq:LR_eigendecomp_soap}
\end{align}
This induces a rotated gradient:
\begin{align*}
    {\mG_l^\prime}^{(t)} = {\mQ_{\mL_l}^{(t)}}^\top \mG_l^{(t)} \mQ_{\mR_l}^{(t)}.
\end{align*}

The full update rule of SOAP is 
\begin{align*}
    \mW_l^{(t)} = \mW_l^{(t-1)} - \eta^{(t)} \left(\mQ_{\mL_l}^{(t)}\mathrm{AdamW}\left({\mG_l^\prime}^{(t)}\right){\mQ_{\mR_l}^{(t)}}^\top + \lambda_l {\mW_l^{(t-1)}}\right),
\end{align*}
where $\mathrm{AdamW}\left(\cdot\right)$ is defined as in Equation~(\ref{eq:Adam_as_function}).

First, omit the momentum and the stabilization term in the AdamW operator. Then, $\mathrm{AdamW}\left(\cdot\right)$ is reduced to $\mathrm{sign}\left(\cdot\right)$ as discussed in Appendix~\ref{app:adam_update_rule}.

Then, omit the momentum term in $\mL_l^{(t)}$ and $\mR_l^{(t)}$, i.e., set $\beta_3 = 0$. 
Applying compact SVD to the weight matrix $\mG_l^{(t)}=\mU_l^{(t)} \mSigma_l^{(t)} {\mV_l^{(t)}}^\top$, where $\mU_l^{(t)}\in\R^{\nout \times r}$, $\mSigma_l^{(t)}\in\R^{r\times r}$, and $\mV_l^{(t)}\in\R^{\nin \times r}$, we have 
\begin{align*}
    \mL_l^{(t)} = \mG_l^{(t)}{\mG_l^{(t)}}^\top = \mU_l^{(t)} {\mSigma_l^{(t)}}^2 {\mU_l^{(t)}}^\top, \quad \mR_l^{(t)} = {\mG_l^{(t)}}^\top\mG_l^{(t)} = \mV_l^{(t)} {\mSigma_l^{(t)}}^2 {\mV_l^{(t)}}^\top.
\end{align*}

According to Equation~(\ref{eq:LR_eigendecomp_soap}), the eigenvectors corresponding to the non-zero eigenvalues match the singular vectors, i.e., 
\begin{align*}
     {\mQ_{\mL_l}^{(t)}}_{[:,:r]} = \mU_l^{(t)}, \quad {\mQ_{\mR_l}^{(t)}}_{[:,:r]} = \mV_l^{(t)}.
\end{align*}
We can partition the orthogonal matrices as $\mQ_{\mL_l}^{(t)} = [\mU_l^{(t)} \ \mU_{\perp}^{(t)}]$ and $\mQ_{\mR_l}^{(t)} = [\mV_l^{(t)} \ \mV_{\perp}^{(t)}]$. Substituting the eigendecomposition of $\mG_l^{(t)}$, the rotated gradient becomes:
\begin{align*}
    {\mG_l^\prime}^{(t)} &= {\mQ_{\mL_l}^{(t)}}^\top \mG_l^{(t)} \mQ_{\mR_l}^{(t)} \nonumber = \begin{bmatrix} {\mU_l^{(t)}}^\top \\ {\mU_{\perp}^{(t)}}^\top \end{bmatrix}  \mU_l^{(t)} \mSigma_l^{(t)} {\mV_l^{(t)}}^\top  \begin{bmatrix} \mV_l^{(t)} & \mV_{\perp}^{(t)} \end{bmatrix} \nonumber = \begin{bmatrix} \mSigma_l^{(t)} & \mathbf{0} \\ \mathbf{0} & \mathbf{0} \end{bmatrix}.
\end{align*}
Then, we find that SOAP is reduced to:
\begin{align*}
    \mW_l^{(t)} 
    &= \mW_l^{(t-1)} - \eta^{(t)} \left(\mQ_{\mL_l}^{(t)}\mathrm{sign}\left(\begin{bmatrix} \mSigma_l^{(t)} & \mathbf{0} \\ \mathbf{0} & \mathbf{0} \end{bmatrix}\right){\mQ_{\mR_l}^{(t)}}^\top + \lambda_l {\mW_l^{(t-1)}}\right)\\
    &= \mW_l^{(t-1)} - \eta^{(t)} \left(\mQ_{\mL_l}^{(t)}\begin{bmatrix} \mI_{r \times r} & \mathbf{0} \\ \mathbf{0} & \mathbf{0} \end{bmatrix}{\mQ_{\mR_l}^{(t)}}^\top + \lambda_l {\mW_l^{(t-1)}}\right)\\
    &=\mW_l^{(t-1)} - \eta^{(t)} \left(\mU_l^{(t)}{\mV_l^{(t)}}^\top + \lambda_l {\mW_l^{(t-1)}}\right),
\end{align*}
which, again, matches exactly the update rule of Muon in Equation~(\ref{eq:updaterule_muon}).
Therefore, SOAP shares Muon's parameterizations in Table~\ref{tab: muon-wd mup}.

Note that the hidden layer learning rate parameterization derived in~\citet[Table 1]{muon-cp} is $\frac{\nout^{e_L/2}\nin^{e_R/2}}{\nin}$, where the $e_L$ and $e_R$ are the indicators for left- and right-side preconditioners, which equals $1$ for standard SOAP. In this case, $\frac{\nout^{e_L/2}\nin^{e_R/2}}{\nin} = \frac{\sqrt{\nout\nin}}{\nin} = \Theta(1)$, consistent with our result for SOAP in Table~\ref{tab: muon-wd mup}.

\subsection{Spectral Sphere Optimizer (SSO)}
\label{app:sso}

\begin{table}[t]

\renewcommand{\arraystretch}{1.3}
\renewcommand{\hl}[1]{\textcolor{purple}{#1}}
\renewcommand{\ll}[1]{\textcolor{gray}{#1}}
\centering

\caption{\textbf{{CompleteP-style $\mu$P parameterization from Condition~\ref{condition: scale-invariant fl} ($k=2$)} for SSO~\citep{xie2026controlled-sso} with weight decay under width-depth scaling.}
Entries in \hl{purple} indicate differences between $\mu$P and SP, while \ll{gray} shows the corresponding SP choices.
Here, $r_n$ and $r_L$ denote the width and depth scaling ratios relative to the base model.
The variance of input weights is $\sigma^2_{\mathrm{base}}$ for language and $\sigma^2_{\mathrm{base}}/d_0$ for image.}
\label{tab: sso mup}
\vskip 0.05in
\begin{tabular}{cccc}
\toprule
 & Input weights & Hidden weights & Output weights \\
\midrule

Block Multiplier
& $\alpha_{\mathrm{base}}$
& \hl{$\alpha_{\mathrm{base}}/r_L$} \ \ll{($\alpha_{\mathrm{base}}$)}
& \hl{$\alpha_{\mathrm{base}}/r_n$} \ \ll{($\alpha_{\mathrm{base}}$)} \\

Initial Variance
& $\sigma^2_{\mathrm{base}}/d_0$ or $\sigma^2_{\mathrm{base}}$
& \hl{$\sigma^2_{\mathrm{base}}/r_n$} \ \ll{($\sigma^2_{\mathrm{base}}$)}
& $\sigma^2_{\mathrm{base}}$ \\

Learning Rate
& $\eta_{\mathrm{base}}$
& $\eta_{\mathrm{base}}$
& \hl{$\eta_{\mathrm{base}}{r_n}$} \ \ll{($\eta_{\mathrm{base}}$)} \\

Weight Decay
& $\lambda_{\mathrm{base}}$
& $\lambda_{\mathrm{base}}$
& \hl{$\lambda_{\mathrm{base}}/{r_n}$} \ \ll{($\lambda_{\mathrm{base}}$)} \\

\bottomrule
\end{tabular}
\end{table}

In this section, we derive the {CompleteP-style $\mu$P parameterization} from
Condition~\ref{condition: scale-invariant fl} for SSO.

\subsubsection{Update Rule}
SSO~\citep{xie2026controlled-sso} aims to perform steepest descent on the spectral sphere (see Section 3.1 in the original paper), where the update follows:
\begin{equation*}
\begin{aligned}
    \Delta{\mW}_l = -\eta_l\bigl(\,\underbrace{R\boldsymbol{\Phi}_l}_{\mA_l}  + \lambda_l {\mW}_l\bigl),
\end{aligned}
\end{equation*}
with
\begin{align*}
    R = \Theta\left(\sqrt{\frac{\nout}{\nin}}\right),\quad \text{and} \quad \boldsymbol{\Phi}_l = \arg\max_{\boldsymbol{\Phi}} \langle \nabla_{\mW_l}\mathcal{L}, \boldsymbol{\Phi}\rangle 
    \ \mathrm{s.t.} \ 
    \Vert \boldsymbol{\Phi} \Vert_2 = 1, \ \Vert \mW_l - \eta_l \boldsymbol{\Phi} \Vert_2 = \Vert \mW_l\Vert_2 = R.
\end{align*}

Thus we have
\begin{align}
    \Vert \mA_l \Vert_{\normrms} 
    = \sqrt{\frac{\nin}{\nout}} \Vert \mA_l \Vert_{2} = \sqrt{\frac{\nin}{\nout}}R \Vert \boldsymbol{\Phi}_l \Vert_{2} = \sqrt{\frac{\nin}{\nout}}\Theta\left(\sqrt{\frac{\nout}{\nin}}\right) = {\Theta(1)}.
\label{eq:Anorm_sso}
\end{align}

\subsubsection{Derivation of Parameterization}

\paragraph{Input and output layers.} 
When weight decay is disabled ($\lambda_0 = 0$), given the dimension assumption $d_0,d_{L+1}=\Theta(1), n_l=\Theta(n)$ in Equation~(\ref{eq:dimensions}), the multiplier parameterizations $\alpha_0=\Theta(1),\alpha_{L+1}=\Theta(1/\nin)$ in Equation~(\ref{eq:alpha_inandout}), and the scale of $\Vert\mA_l\Vert_\normrms$ in Equation~(\ref{eq:Anorm_sso}), we have
\begin{align*}
    \alpha_l\|\Delta\mW_l\|_{\normrms}
    =\alpha_l\eta_l\|\mA_l\|_{\normrms}
    =\left\{
    \begin{array}{ll}
    \Theta(\eta_0),   & l = 0,\\
    \Theta(\eta_{L+1}/{\nin}),  & l = L+1.
    \end{array} \right.
\end{align*}
As desired in (\ref{eq:condition_wd_1}), to satisfy (\ref{eq:update_inandout}) that $\alpha_0\|\Delta \mW_0\|_{\normrms}, \alpha_{L+1}\|\Delta \mW_{L+1}\|_{\normrms} = \Theta(1)$, we need to set 
\begin{align*}
    \eta_0=\Theta(1), \quad \eta_{L+1}=\Theta({\nin}).
\end{align*}
When $\lambda_l \neq 0$, given Equation~(\ref{eq:w_norm}) that $\|{\mW}_0\|_{\normrms} = \Theta(1)$ and $\|{\mW}_{L+1}\|_{\normrms} = \Theta(\nin)$, we have
\begin{align*}
    \|\lambda_l\mW_l\|_{\normrms} 
    =\left\{
    \begin{array}{ll}
    \Theta(\lambda_0),   & l = 0,\\
    \Theta(\lambda_{L+1}\nin),  & l = L+1.
    \end{array} \right.
\end{align*}
To satisfy (\ref{eq:condition_wd_2}) that $\|\lambda_l \mW_l\|_{\normrms}
    = \Theta\left( \|\mA_l\|_{\normrms} \right)$, we need to set 
\begin{align*}
    \lambda_0=\Theta(1), \quad \lambda_{L+1}=\Theta(1/{\nin}),
\end{align*}

\paragraph{Hidden layers (first-order).} 
When weight decay is disabled ($\lambda_l = 0$), given the dimension assumption $d_0,d_{L+1}=\Theta(1), n_l=\Theta(n)$ in Equation~(\ref{eq:dimensions}), the weight norm $\|{\mW}_l\|_{\normrms} = \Theta(1)$ in Equation~(\ref{eq:w_norm}), the multiplier parameterization $\alpha_l=\Theta(1/L)$ in Equation~(\ref{eq:alpha_l}), and the scale of $\Vert\mA_l\Vert_\normrms$ in Equation~(\ref{eq:Anorm_sso}), we have
\begin{align*}
\alpha_l\|\Delta\mW_l^{(2)}\|_{\normrms}\|\mW_l^{(1)}\|_{\normrms}
    =\Theta(1/L) \cdot \eta_l^{(2)}\|\mA_l^{(2)}\|_{\normrms}\|\mW_l^{(1)}\|_{\normrms}
    =\Theta(\eta_l^{(2)}/L).
\end{align*}
As desired in (\ref{eq:condition_wd_1}), to satisfy the first-order update condition on hidden weights (\ref{eq:update_hidd_1}) that $\alpha_l\|\Delta \mW_l^{(2)}\|_{\normrms}\,\|\mW_l^{(1)}\|_{\normrms}=\Theta(1/L)$, we need to set 
\begin{align*}
    \eta_l^{(2)}=\Theta(1).
\end{align*}
When weight decay is enabled ($\lambda_l \neq 0$), given the weight norm $\|{\mW}_l\|_{\normrms} = \Theta(1)$ in Equation~(\ref{eq:w_norm}), we have
\begin{align*}
    \|\lambda_l^{(2)}\mW_l^{(2)}\|_{\normrms} = \Theta(\lambda_l^{(2)}).
\end{align*}
To satisfy (\ref{eq:condition_wd_2}) that $\|\lambda_l \mW_l\|_{\normrms}
    = \Theta\left( \|\mA_l\|_{\normrms} \right)$ we need to set 
\begin{align*}
    \lambda_l^{(2)}=\Theta(1).
\end{align*}

Symmetrically, we have the same choice for $\mW_l^{(1)}$:
\begin{align*}
    \eta_l^{(1)} = \Theta(1), \quad \lambda_l^{(1)}=\Theta(1).
\end{align*}

\paragraph{Hidden layers (second-order).} 
As illustrated in Section~\ref{sec: Final Initial Condition} or in Appendix~\ref{app:muon} for Muon, the second-order update condition is satisfied automatically once the initial condition and the first-order update condition are met.

This completes the implementation of the update condition for SSO with weight decay, which is summarized in Table~\ref{tab: sso mup}.

\section{Implementing Condition~\ref{condition: one-layer} for Optimizers with Weight Decay}
\label{app: implementation one-layer}

We implement Condition~\ref{condition: one-layer} using the same
width-scaling $\mu$P initialization convention as in
Section~\ref{sec:initial_condition}.
Under this convention, hidden matrix weights satisfy
$\|\mW_l\|_{\normrms}=\Theta(1)$ at initialization, so the hidden initial
condition in Condition~\ref{condition: one-layer} permits
$\alpha_l=\mathcal{O}(1/\sqrt{L})$.
In this section, we focus on the Depth-$\mu$P-style choice
$\alpha_l=\Theta(1/\sqrt{L})$, which corresponds to maximizing the
zero-order feature-update contribution discussed in
Appendix~\ref{app: spec one-layer}.
The input and output layer multipliers remain the same as in Section~\ref{sec:initial_condition}, that $\alpha_0=\Theta(1)$ and
$\alpha_{L+1}=\Theta(1/\nin)$.
We now start update-condition analysis for optimizers with weight decay.

To provide a unified derivation across different optimizers, we begin by expressing their update rules in a general form. 
When weight decay is included, a single update step of the weight matrix can be written as
\begin{align*}
    \Delta \mW_l = -\eta_l \bigl( \mA_l + \lambda_l \mW_l \bigr),
\end{align*}
where $\mA_l$ denotes an optimizer-specific update for $\mW_l$, and $\lambda_l$ is the weight decay coefficient.

The update magnitude $\|\Delta \mW_l\|_{\normrms}
= \eta_l \|\mA_l + \lambda_l \mW_l\|_{\normrms}
$ is required to satisfy the update conditions in Condition~\ref{condition: one-layer}. 
We analyze this requirement under two complementary regimes.

\textbf{Without weight decay.}
When weight decay is disabled ($\lambda_l = 0$), the update reduces to
$\|\Delta \mW_l\|_{\normrms}= \eta_l \|\mA_l\|_{\normrms}.$ In this case, we require:
\begin{align}
    \|\Delta \mW_l\|_{\normrms}
    = \eta_l \|\mA_l\|_{\normrms}
    \ \text{satisfies Condition~\ref{condition: one-layer}}.
    \tag{$\Upsilon1$}
    \label{eq:condition_wd_1_onelayer}
\end{align}

\textbf{With weight decay.}
When weight decay is enabled ($\lambda_l \neq 0$), we choose the weight decay term to be comparable in scale to the optimizer-driven term:
\begin{align}
    \|\lambda_l \mW_l\|_{\normrms}
    = \Theta\left( \|\mA_l\|_{\normrms} \right).
    \tag{$\Upsilon2$}
    \label{eq:condition_wd_2_onelayer}
\end{align}
As discussed in Appendix~\ref{app: implementation opts HPs}, this is a
non-degenerate implementation convention: it keeps weight decay active in the
update dynamics without letting it dominate the optimizer-driven term.
Under the usual scale estimate,
$\|\mA_l+\lambda_l\mW_l\|_{\normrms}
=\Theta(\|\mA_l\|_{\normrms})$, so the learning-rate scaling rule derived from
Equation~(\ref{eq:condition_wd_1_onelayer}) is preserved.

\paragraph{Weights and biases.}
In the following, we derive parameterizations of the learning rate and weight
decay coefficient for a range of optimizers.
As in Appendix~\ref{app: implementation opts HPs}, matrix-based optimizers such
as Muon and Shampoo are typically not applied to bias parameters, so we restrict
the analysis of bias parameterization to vector-based optimizers such as SGD and
AdamW.
The same two rules, Equations~(\ref{eq:condition_wd_1_onelayer})
and~(\ref{eq:condition_wd_2_onelayer}), are then applied to the biases with the
corresponding one-layer specialization of spectral Condition~\ref{condition: bias}.

\paragraph{Momentum.}
We use the same momentum simplification as in
Appendix~\ref{app: implementation opts HPs}: the derivations omit momentum,
while practical implementations take momentum coefficients to be $\Theta(1)$.
This can also be interpreted as analyzing the first update step after
initialization.

\subsection{Overview}
\label{app: implementation one-layer overview}

Table~\ref{tab: optimizer-family-overview-one-layer} summarizes the optimizer
families covered in this section and points to the corresponding detailed
Depth-$\mu$P-style parameterization tables.
All rows use the initialization convention above, with hidden residual
multiplier $\alpha_l=\Theta(1/\sqrt{L})$; the table only organizes the
optimizer-dependent update rules derived below.

\begin{table}[t]
\renewcommand{\arraystretch}{1.15}
\centering
\caption{\textbf{Overview of the {Depth-$\mu$P-style parameterization from Condition~\ref{condition: one-layer} ($k=1$)} for optimizer families with weight decay.}
Optimizers in the same row share the same $\mu$P scaling rules.}
\label{tab: optimizer-family-overview-one-layer}
\vskip 0.05in
\begin{tabular}{cc}
\toprule
Optimizer family & Detailed parameterization \\
\midrule
Muon-Kimi & Table~\ref{tab: muon-kimi-wd mup one-layer} \\
Muon / Shampoo / SOAP & Table~\ref{tab: muon-wd mup one-layer} \\
SGD & Table~\ref{tab: sgd-mup one-layer} \\
AdamW / Lion / Sophia & Table~\ref{tab: adamw-mup one-layer} \\
SSO & Table~\ref{tab: sso-mup one-layer} \\
\bottomrule
\end{tabular}
\end{table}

\subsection{Muon-Kimi}
\label{app:muon-kimi onelayer}

\begin{table}[t]

\renewcommand{\arraystretch}{1.3}
\renewcommand{\hl}[1]{\textcolor{purple}{#1}}
\renewcommand{\ll}[1]{\textcolor{gray}{#1}}
\centering

\caption{\textbf{{Depth-$\mu$P-style parameterization from Condition~\ref{condition: one-layer} ($k=1$)} for Muon-Kimi~\citep{muon-kimi} with weight decay under width-depth scaling.}
Entries in \hl{purple} indicate differences between $\mu$P and SP, while \ll{gray} shows the corresponding SP choices.
Here, $r_n$ and $r_L$ denote the width and depth scaling ratios relative to the base model. The variance of input weights is $\sigma^2_{\mathrm{base}}$ for language and $\sigma^2_{\mathrm{base}}/d_0$ for image.}
\label{tab: muon-kimi-wd mup one-layer}
\vskip 0.05in
\begin{tabular}{cccc}
\toprule
   & Input weights & Hidden weights & Output weights \\
\midrule
Block Multiplier
& $\alpha_{\mathrm{base}}$
& \hl{$\alpha_{\mathrm{base}}/\sqrt{r_L}$} \ \ll{($\alpha_{\mathrm{base}}$)}
& \hl{$\alpha_{\mathrm{base}}/r_n$} \ \ll{($\alpha_{\mathrm{base}}$)} \\

Initial Variance
& $\sigma^2_{\mathrm{base}}/d_0$ or $\sigma^2_{\mathrm{base}}$
& \hl{$\sigma^2_{\mathrm{base}}/r_n$} \ \ll{($\sigma^2_{\mathrm{base}}$)}
& $\sigma^2_{\mathrm{base}}$ \\

Learning Rate
& $\eta_{\mathrm{base}}$
& \hl{$\eta_{\mathrm{base}}/\sqrt{r_n r_L}$} \ \ll{($\eta_{\mathrm{base}}$)}
& $\eta_{\mathrm{base}}$ \\

Weight Decay
& $\lambda_{\mathrm{base}}$
& \hl{$\lambda_{\mathrm{base}}\sqrt{r_n}$} \ \ll{($\lambda_{\mathrm{base}}$)}
& $\lambda_{\mathrm{base}}$ \\

\bottomrule
\end{tabular}
\end{table}

In this section, we derive the {Depth-$\mu$P-style parameterization} from
Condition~\ref{condition: one-layer} for Muon-Kimi.

For a weight matrix
${\mW}_l \in \R^{\nout\times\nin}$, the update rule of Muon-Kimi~\citep{muon-kimi} with weight decay is
\begin{equation*}
\Delta{\mW}_l
=
-\,\eta_l \bigl(\,\underbrace{0.2 \sqrt{\max\{\nin,\nout\}}
\,\mU_l \mV_l^\top}_{\mA_l} + \lambda_l {\mW}_l\bigl),
\end{equation*}
where $\mU_l,\mV_l$ arise from the compact SVD of the gradient
$\nabla_{{\mW}_l}\mathcal{L}=\mU_l \mSigma_l \mV_l^\top$.
Using the norm computation in Equation~(\ref{eq:dw_norm}), we have 
\begin{align*}
    \Vert \mA_l \Vert_{\normrms} 
    = \Theta\left(\sqrt{\nin}\max\left\{1,\sqrt{\frac{\nin}{\nout}}\right\}\right)
    =\left\{
    \begin{array}{ll}
    \Theta(1),   & l = 0,\\
    \Theta(\sqrt{\nin}),  & l \in [L], \\
    \Theta({\nin}),  & l = L+1.
    \end{array} \right.
\end{align*}

\subsubsection{Derivation of Parameterization}

\paragraph{Input and output layers.} 
The input and output layer conditions in Condition~\ref{condition: one-layer}
are identical to those in Condition~\ref{condition: scale-invariant fl}.
Therefore, as in Appendix~\ref{app: Muon-Kimi (with Weight Decay)}, their
Muon-Kimi learning-rate and weight-decay parameterizations are
\begin{align*}
    \eta_0 = \Theta\left(1\right),\quad
    \eta_{L+1} = \Theta\left(1\right),\quad
    \lambda_0=\Theta(1), \quad
    \lambda_{L+1}=\Theta(1).
\end{align*}

\paragraph{Hidden layers (first-order).} 
The only change from the $k=2$ Muon-Kimi implementation is the hidden
multiplier: {we adopt the
Depth-$\mu$P-style choice $\alpha_l=\Theta(1/\sqrt{L})$ rather than
$\Theta(1/L)$}.
When $\lambda_l = 0$, given the dimension assumption
$d_0,d_{L+1}=\Theta(1), n_l=\Theta(n)$ in Equation~(\ref{eq:dimensions}) and
the scale of $\Vert\mA_l\Vert_\normrms$ above, we have
\begin{align*}
    \alpha_l\|\Delta\mW_l\|_{\normrms}
    =\Theta(1/\sqrt{L}) \cdot \eta_l\|\mA_l\|_{\normrms}
    =\Theta(\eta_l\sqrt{\nin}/\sqrt{L}).
\end{align*}
As desired in Equation~(\ref{eq:condition_wd_1_onelayer}), to satisfy the
first-order update condition on hidden weights,
$\alpha_l\|\Delta \mW_l\|_{\normrms}=\Theta(1/L)$, we need to set 
\begin{align*}
    \eta_l=\Theta(\frac{1}{\sqrt{\nin L}}).
\end{align*}
When $\lambda_l \neq 0$, the weight-decay parameterization is unchanged from
the $k=2$ Muon-Kimi implementation because it matches the scale of $\mA_l$ and
does not depend on the residual multiplier.
Given the weight norm $\|{\mW}_l\|_{\normrms} = \Theta(1)$ in
Equation~(\ref{eq:w_norm}), we have
\begin{align*}
    \|\lambda_l\mW_l\|_{\normrms} = \Theta(\lambda_l).
\end{align*}
To satisfy Equation~(\ref{eq:condition_wd_2_onelayer}) that
$\|\lambda_l \mW_l\|_{\normrms}
    = \Theta\left( \|\mA_l\|_{\normrms} \right)$, we need to set 
\begin{align*}
    \lambda_l=\Theta(\sqrt{\nin}).
\end{align*}

This completes the implementation of Condition~\ref{condition: one-layer} for
Muon-Kimi with weight decay, as summarized in
Table~\ref{tab: muon-kimi-wd mup one-layer}.

\subsection{Muon, Shampoo and SOAP}
\label{app:muon onelayer}

\begin{table}[t]

\renewcommand{\arraystretch}{1.3}
\renewcommand{\hl}[1]{\textcolor{purple}{#1}}
\renewcommand{\ll}[1]{\textcolor{gray}{#1}}
\centering
\caption{\textbf{{Depth-$\mu$P-style parameterization from Condition~\ref{condition: one-layer} ($k=1$)} for Muon~\citep{jordan6muon}, Shampoo~\citep{gupta2018shampoo}, and SOAP~\citep{vyas2024soap} with weight decay under width-depth scaling.}
Entries in \hl{purple} indicate differences between $\mu$P and SP, while \ll{gray} shows the corresponding SP choices.
Here, $r_n$ and $r_L$ denote the width and depth scaling ratios relative to the base model.
The variance of input weights is $\sigma^2_{\mathrm{base}}$ for language and $\sigma^2_{\mathrm{base}}/d_0$ for image.}
\label{tab: muon-wd mup one-layer}
\vskip 0.05in
\begin{tabular}{cccc}
\toprule
 & Input weights & Hidden weights & Output weights \\
\midrule

Block Multiplier
& $\alpha_{\mathrm{base}}$
& \hl{$\alpha_{\mathrm{base}}/\sqrt{r_L}$} \ \ll{($\alpha_{\mathrm{base}}$)}
& \hl{$\alpha_{\mathrm{base}}/r_n$} \ \ll{($\alpha_{\mathrm{base}}$)} \\

Initial Variance
& $\sigma^2_{\mathrm{base}}/d_0$ or $\sigma^2_{\mathrm{base}}$
& \hl{$\sigma^2_{\mathrm{base}}/r_n$} \ \ll{($\sigma^2_{\mathrm{base}}$)}
& $\sigma^2_{\mathrm{base}}$ \\

Learning Rate
& \hl{$\eta_{\mathrm{base}}\sqrt{r_n}$} \ \ll{($\eta_{\mathrm{base}}$)}
& \hl{$\eta_{\mathrm{base}}/\sqrt{r_L}$} \ \ll{($\eta_{\mathrm{base}}$)}
& \hl{$\eta_{\mathrm{base}}\sqrt{r_n}$} \ \ll{($\eta_{\mathrm{base}}$)} \\

Weight Decay
& \hl{$\lambda_{\mathrm{base}}/\sqrt{r_n}$} \ \ll{($\lambda_{\mathrm{base}}$)}
& $\lambda_{\mathrm{base}}$
& \hl{$\lambda_{\mathrm{base}}/\sqrt{r_n}$} \ \ll{($\lambda_{\mathrm{base}}$)} \\

Shampoo $\varepsilon$
& \hl{$\varepsilon_{\mathrm{base}}/r_n$} \ \ll{($\varepsilon_{\mathrm{base}}$)}
& \hl{$\varepsilon_{\mathrm{base}}/r_L$} \ \ll{($\varepsilon_{\mathrm{base}}$)}
& \hl{$\varepsilon_{\mathrm{base}}/r_n$} \ \ll{($\varepsilon_{\mathrm{base}}$)} \\

\bottomrule
\end{tabular}
\end{table}

In this section, we derive the {Depth-$\mu$P-style parameterization} from
Condition~\ref{condition: one-layer} for Muon, Shampoo, and SOAP.

For a weight matrix ${\mW}_l \in \R^{\nout\times\nin}$, the update rule of
Muon~\citep{jordan6muon} with weight decay is
\begin{equation*}
\Delta{\mW}_l
=
-\,\eta_l \bigl(\,\underbrace{\mU_l \mV_l^\top}_{\mA_l}
+ \lambda_l {\mW}_l\bigl),
\end{equation*}
where $\mU_l,\mV_l$ arise from the compact SVD of the gradient
$\nabla_{{\mW}_l}\mathcal{L}=\mU_l \mSigma_l \mV_l^\top$.
Using the norm computation in Equation~(\ref{eq:Anorm_muon}), we have
\begin{align*} 
\Vert \mA_l \Vert_\normrms 
=
\sqrt{\frac{\nin}{\nout}}
=
\left\{
    \begin{array}{ll}
    \Theta(1/\sqrt{\nout}),   & l = 0,\\
    \Theta(1),   & l \in [L],\\
    \Theta(\sqrt{\nin}),  & l = L+1.
    \end{array} \right.
\end{align*}
As discussed in Appendices~\ref{app:shampoo} and~\ref{app:soap}, Shampoo and
SOAP reduce to the same Muon update direction under the simplifications used in
this section, so they share the same parameterization.

\subsubsection{Derivation of Parameterization}

\paragraph{Input and output layers.} 
The input and output layer conditions in Condition~\ref{condition: one-layer}
are identical to those in Condition~\ref{condition: scale-invariant fl}.
Therefore, as in Appendix~\ref{app:muon}, their Muon learning-rate and
weight-decay parameterizations are
\begin{align*}
    \eta_0=\Theta(\sqrt{\nout}),\quad
    \eta_{L+1}=\Theta(\sqrt{\nin}),\quad
    \lambda_0=\Theta(1/\sqrt{\nout}), \quad
    \lambda_{L+1}=\Theta(1/\sqrt{\nin}).
\end{align*}

\paragraph{Hidden layers (first-order).} 
The only change from the $k=2$ Muon implementation is the hidden multiplier:
{we adopt the Depth-$\mu$P-style choice
$\alpha_l=\Theta(1/\sqrt{L})$ rather than $\Theta(1/L)$}.
When $\lambda_l = 0$, given the dimension assumption
$d_0,d_{L+1}=\Theta(1), n_l=\Theta(n)$ in Equation~(\ref{eq:dimensions}) and
the scale of $\Vert\mA_l\Vert_\normrms$ above, we have
\begin{align*}
    \alpha_l\|\Delta\mW_l\|_{\normrms}
    =\Theta(1/\sqrt{L}) \cdot \eta_l\|\mA_l\|_{\normrms}
    =\Theta(\eta_l/\sqrt{L}).
\end{align*}
As desired in Equation~(\ref{eq:condition_wd_1_onelayer}), to satisfy the
first-order update condition on hidden weights,
$\alpha_l\|\Delta \mW_l\|_{\normrms}=\Theta(1/L)$, we need to set 
\begin{align*}
    \eta_l=\Theta(1/\sqrt{L}).
\end{align*}
When $\lambda_l \neq 0$, the weight-decay parameterization is unchanged from
the $k=2$ Muon implementation because it matches the scale of $\mA_l$ and does
not depend on the residual multiplier.
Given the weight norm $\|{\mW}_l\|_{\normrms} = \Theta(1)$ in
Equation~(\ref{eq:w_norm}), we have
\begin{align*}
    \|\lambda_l\mW_l\|_{\normrms} = \Theta(\lambda_l).
\end{align*}
To satisfy Equation~(\ref{eq:condition_wd_2_onelayer}) that
$\|\lambda_l \mW_l\|_{\normrms}
    = \Theta\left( \|\mA_l\|_{\normrms} \right)$, we need to set 
\begin{align*}
    \lambda_l=\Theta(1).
\end{align*}

\paragraph{Parameterization of Shampoo $\varepsilon_l$.}
As in Appendix~\ref{app:shampoo}, Shampoo's stabilization term is added to
the left and right preconditioners
$\widehat{\mL}_l^{(t)}$ and $\widehat{\mR}_l^{(t)}$, so it should match the
scale of these preconditioners rather than the elementwise gradient scale.
When omitting the momentum, this gives
\begin{align*}
\|\widehat{\mL}_l^{(t)}\|_2 = \|\widehat{\mR}_l^{(t)}\|_2 =
\|\mG_l^{(t)}\|^2_2
=\|\nabla_{\mW_l^{(t)}}\mathcal{L}\|^2_2.
\end{align*}
The input and output gradient estimates are unchanged from
Appendix~\ref{app:shampoo}. For hidden weights, {we adopt the
Depth-$\mu$P-style choice $\alpha_l=\Theta(1/\sqrt{L})$}, and the raw-gradient estimate in
Equation~(\ref{eqn: sgd gradient matrix one-layer}) gives
$\|\nabla_{\mW_l}\mathcal{L}\|_2
=\Theta(\sqrt{\nin/(\nout L)})$.
Therefore,
\begin{align*}
\|\widehat{\mL}_l^{(t)}\|_2 = \|\widehat{\mR}_l^{(t)}\|_2
=
\left\{
    \begin{array}{ll}
    \Theta({\frac{\nin}{\nout}}) = \Theta({\frac{1}{\nout}}),   & l = 0,\\
    \Theta(\frac{1}{L}{\frac{\nin}{\nout}}) = \Theta(\frac{1}{L}),   & l \in [L],\\
    \Theta({\frac{1}{\nin\nout}}) = \Theta({\frac{1}{\nin}}),  & l = L+1.
    \end{array} \right.
\end{align*}
Thus, the stabilization term should be parameterized as
\begin{align*}
\varepsilon_l =
\left\{
    \begin{array}{ll}
    \Theta({\frac{1}{\nout}}),   & l = 0,\\
    \Theta(\frac{1}{L}),   & l \in [L],\\
    \Theta({\frac{1}{\nin}}),  & l = L+1.
    \end{array} \right.
\end{align*}
This corresponds to the Shampoo $\varepsilon$ row in
Table~\ref{tab: muon-wd mup one-layer}.

This completes the implementation of Condition~\ref{condition: one-layer} for
Muon, Shampoo, and SOAP with weight decay, as summarized in
Table~\ref{tab: muon-wd mup one-layer}.

\subsection{SGD}
\label{app:sgd dmup}

\begin{table}[t]
\renewcommand{\arraystretch}{1.3}
\setlength{\tabcolsep}{2pt}
\renewcommand{\hl}[1]{\textcolor{purple}{#1}}
\renewcommand{\ll}[1]{\textcolor{gray}{#1}}
\centering
\caption{\textbf{{Depth-$\mu$P-style parameterization from Condition~\ref{condition: one-layer} ($k=1$)} for SGD with weight decay under width–depth scaling.}
Entries in \hl{purple} indicate differences between $\mu$P and SP, while \ll{gray} shows the corresponding SP choices.
Here, $r_n$ and $r_L$ denote the width and depth scaling ratios relative to the base model.
The variance of input weights is $\sigma^2_{\mathrm{base}}$ for language and $\sigma^2_{\mathrm{base}}/d_0$ for image. The initial variance of input bias is $\sigma^2_{\mathrm{base}}$.}
\label{tab: sgd-mup one-layer}
\vskip 0.05in
\begin{tabular}{ccccc}
\toprule
 & Input weights \& biases & Hidden weights & Output weights & Hidden biases\\
\midrule

Block Multiplier
& $\alpha_{\mathrm{base}}$
& \hl{$\alpha_{\mathrm{base}}/\sqrt{r_L}$} \ \ll{($\alpha_{\mathrm{base}}$)}
& \hl{$\alpha_{\mathrm{base}}/r_n$} \ \ll{($\alpha_{\mathrm{base}}$)} 
& \hl{$\alpha_{\mathrm{base}}/\sqrt{r_L}$} \ \ll{($\alpha_{\mathrm{base}}$)} \\

Initial Variance
& $\sigma^2_{\mathrm{base}}/d_0$ or $\sigma^2_{\mathrm{base}}$
& \hl{$\sigma^2_{\mathrm{base}}/r_n$} \ \ll{($\sigma^2_{\mathrm{base}}$)}
& $\sigma^2_{\mathrm{base}}$
&$\sigma^2_{\mathrm{base}}$ \\

Learning Rate
& \hl{$\eta_{\mathrm{base}}{r_n}$} \ \ll{($\eta_{\mathrm{base}}$)}
& {$\eta_{\mathrm{base}}$}
& \hl{$\eta_{\mathrm{base}}{r_n}$} \ \ll{($\eta_{\mathrm{base}}$)} &
\hl{$\eta_{\mathrm{base}} r_n$} \ \ll{($\eta_{\mathrm{base}}$)} \\

Weight Decay
& \hl{$\lambda_{\mathrm{base}}/{r_n}$} \ \ll{($\lambda_{\mathrm{base}}$)}
& \hl{$\lambda_{\mathrm{base}}/\sqrt{r_L}$} \ \ll{($\lambda_{\mathrm{base}}$)}
& \hl{$\lambda_{\mathrm{base}}/{r_n}$} \ \ll{($\lambda_{\mathrm{base}}$)}
& \hl{$\lambda_{\mathrm{base}}/(r_n\sqrt{r_L})$} \ \ll{($\lambda_{\mathrm{base}}$)} \\

\bottomrule
\end{tabular}
\end{table}

In this section, we derive the {Depth-$\mu$P-style parameterization} from
Condition~\ref{condition: one-layer} for SGD,
which recovers and extends the $\mu$P scaling rules studied in
\citet{TP-6,DBLP:conf/iclr/BordelonNLHP24-dmft-depth}.

For a weight matrix $\mW_l\in\R^{\nout\times\nin}$, the SGD update rule with weight decay can be written as:
\begin{equation*}
\Delta\mW_l
=
-\,\eta_l\bigl(\,\underbrace{\nabla_{\mW_l}\mathcal{L}}_{\mA_l}
+\lambda_l \mW_l\bigl).
\end{equation*}
For SGD, $\mA_l$ is the raw gradient, so the scale of $\mA_l$ depends on the
residual multiplier used in the spectral update condition.
The input and output layer estimates are the same as in
Appendix~\ref{app:sgd_mup}:
$\Vert\mA_0\Vert_\normrms = \Theta(\nin/\nout) = \Theta(1/\nout)$ and
$\Vert\mA_{L+1}\Vert_\normrms = \Theta(1/\nout) = \Theta(1)$.
The hidden-layer estimate changes because {we adopt the
Depth-$\mu$P-style choice $\alpha_l=\Theta(1/\sqrt{L})$}.

For a hidden weight $\mW_l$, we have
\begin{align*}
\Theta(\frac{1}{L})=
\Delta_{\mW_l}\mathcal{L} = \Theta(\langle\Delta\mW_l, \nabla_{\mW_l}\mathcal{L}\rangle) = \Theta(\Vert\Delta\mW_l\Vert_\mathrm{F} \Vert\nabla_{\mW_l}\mathcal{L}\Vert_\mathrm{F})=
\Theta(\Vert\Delta\mW_l\Vert_2 \Vert\nabla_{\mW_l}\mathcal{L}\Vert_2).
\end{align*}
Since we set
$\alpha_l \| \Delta\mW_l\|_{\normrms} = \Theta(1/L)$ to satisfy the hidden
update condition in Condition~\ref{condition: one-layer}, and use
$\alpha_l=\Theta(1/\sqrt{L})$, we have
$\|\Delta\mW_l\|_{\normrms}=\Theta(1/\sqrt{L})$ and thus
$\|\Delta\mW_l\|_{2}=\Theta(\sqrt{\nout/(\nin L)})$.
Therefore,
$\Vert\nabla_{\mW_l}\mathcal{L}\Vert_2
= \Theta(\sqrt{{\nin}/{(\nout L)}})$, which leads to
\begin{equation*}
\|\mA_l\|_{\normrms}
=
\Vert\nabla_{\mW_l}\mathcal{L}\Vert_\normrms = \sqrt{\frac{\nin}{\nout}} \Vert\nabla_{\mW_l}\mathcal{L}\Vert_2 = \Theta(\frac{\nin}{\nout\sqrt{L}}) = \Theta(\frac{1}{\sqrt{L}}).
\end{equation*}

To sum up, we have
\begin{align}
\|\mA_l\|_{\normrms}
=
\Vert\nabla_{\mW_l}\mathcal{L}\Vert_\normrms
=
\begin{cases}
\Theta(1/{\nout}), & l=0,\\
\Theta(1/\sqrt{L}),             & l\in[\,L\,],\\
\Theta(1),   & l=L+1.
\end{cases}
\label{eqn: sgd gradient matrix one-layer}
\end{align}

\subsubsection{Derivation of Parameterization}

\paragraph{Input and output layers.}
The input and output layer conditions in Condition~\ref{condition: one-layer}
are identical to those in Condition~\ref{condition: scale-invariant fl}.
Therefore, as in Appendix~\ref{app:sgd_mup}, their SGD learning-rate and
weight-decay parameterizations are
\begin{align*}
\eta_0 = \Theta({\nout}),\quad
\eta_{L+1} = \Theta({\nin}),\quad
\lambda_0 = \Theta(1/{\nout}), \quad
\lambda_{L+1} = \Theta(1/{\nin}).
\end{align*}

\paragraph{Hidden layers (first-order).}
Unlike normalized or preconditioned optimizers, SGD uses the raw gradient, so
changing the hidden multiplier also changes the hidden raw-gradient scale.
When $\lambda_l=0$, using
$\alpha_l=\Theta(1/\sqrt{L})$ and
$\|\mA_l\|_{\normrms}=\Theta(1/\sqrt{L})$, we obtain
\begin{align*}
\alpha_l \|\Delta\mW_l\|_{\normrms}
=
\Theta(1/\sqrt{L}) \cdot \eta_l \|\mA_l\|_{\normrms}
=
\Theta(\eta_l/L).
\end{align*}
Enforcing the hidden update condition in Condition~\ref{condition: one-layer} that
$\alpha_l\|\Delta\mW_l\|_{\normrms}=\Theta(1/L)$ gives
\begin{align*}
\eta_l = \Theta(1).
\end{align*}

If weight decay is enabled on hidden matrices, using
$\|\mW_l\|_{\normrms}=\Theta(1)$ by Equation~(\ref{eq:w_norm}), we obtain
$\|\lambda_l \mW_l\|_{\normrms}=\Theta(\lambda_l)$,
so Equation~(\ref{eq:condition_wd_2_onelayer}) implies
$\|\lambda_l \mW_l\|_{\normrms}
=\Theta(\|\mA_l\|_{\normrms})=\Theta(1/\sqrt{L})$ and therefore
\begin{align*}
\lambda_l =\Theta(1/\sqrt{L}).
\end{align*}

\paragraph{Biases.} 
As in Appendix~\ref{app:sgd_mup}, the raw-gradient scale of input biases is
$\Vert\nabla_{\vb_0}\mathcal{L}\Vert_\normrms = \Theta(1/\nout)$.
For hidden biases $\vb_l \in \R^{\nout \times 1}$, we similarly have
\begin{align*}
\Theta(1/L)=
\Delta_{\vb_l}\mathcal{L} = \Theta(\langle\Delta\vb_l, \nabla_{\vb_l}\mathcal{L}\rangle) = \Theta(\Vert\Delta\vb_l\Vert_2 \Vert\nabla_{\vb_l}\mathcal{L}\Vert_2),
\end{align*}
Since we set
$\alpha_l\|\Delta \vb_l\|_{\normrms}
= \Theta(1/\sqrt{L}) \cdot \|\Delta \vb_l\|_{\normrms}
=\Theta(1/L)$ to satisfy the one-layer bias update condition in Condition~\ref{condition: bias}, we have
$\|\Delta\vb_l\|_\normrms=\Theta(1/\sqrt{L})$ and thus
$\|\Delta\vb_l\|_{2}=\Theta(\sqrt{\nout/L})$.
Therefore,
$\Vert\nabla_{\vb_l}\mathcal{L}\Vert_2
= \Theta(1/\sqrt{\nout L})$, which leads to
\begin{equation*}
\Vert\nabla_{\vb_l}\mathcal{L}\Vert_\normrms = \sqrt{\frac{1}{\nout}} \Vert\nabla_{\vb_l}\mathcal{L}\Vert_2 = \Theta(\frac{1}{\nout\sqrt{L}}).
\end{equation*}

To sum up, we can estimate the scale of $\Vert\nabla_{\vb_l}\mathcal{L}\Vert_\normrms$ as
\begin{align*}
\Vert\nabla_{\vb_l}\mathcal{L}\Vert_\normrms
=
\begin{cases}
\Theta(1/{\nout}), & l=0,\\
\Theta\left(1/(\nout\sqrt{L})\right),       & l\in[L].
\end{cases}
\end{align*}
Requiring
$\alpha_0 \|\Delta\vb_0\|_\normrms
= \alpha_0 \eta_{\vb_0} \Vert\nabla_{\vb_0}\mathcal{L}\Vert_\normrms
= \Theta(1)$ and
$\alpha_l \|\Delta\vb_l\|_\normrms
= \alpha_l \eta_{\vb_l} \Vert\nabla_{\vb_l}\mathcal{L}\Vert_\normrms
= \Theta(1/L)$ for $l \in [L]$ gives
\begin{align*}
    \eta_{\vb_l}
    =
    \Theta(\nout), \quad l \leq L.
\end{align*}
For the weight decays, we match
$\lambda_{\vb_l} \|\vb_l\|_\normrms$ to
$\Vert\nabla_{\vb_l}\mathcal{L}\Vert_\normrms$.
Given $\|\vb_l\|_\normrms = \Theta(1)$ by the initialization implementation in
Condition~\ref{condition: bias}, we have
\begin{align*}
    \lambda_{\vb_l}=
    \begin{cases}
    \Theta(1/\nout), & l=0,\\
    \Theta\left(1/(\nout\sqrt{L})\right),       & l\in[L].
    \end{cases}
\end{align*}

This completes the implementation of Condition~\ref{condition: one-layer} for
SGD with weight decay, as summarized in Table~\ref{tab: sgd-mup one-layer}.

{
\paragraph{The alternative $\Theta(1/L)$ choice and its relation to MLU regime in~\citet{chizat2025hidden}.}
Condition~\ref{condition: one-layer} also permits a smaller residual multiplier than the Depth-$\mu$P boundary.
To make this explicit, denote the global scale in~\citet{chizat2025hidden} by $a$.
Their generic residual recursion and GD update are
\begin{align*}
&\vh_l
=\vh_{l-1}+\frac{a}{LM}\sum_{j=1}^{M}\phi(\vh_{l-1},\vz_l^{(j)}),\\
&\vz_l^{(j,t+1)}
=\vz_l^{(j,t)}-\frac{LM\eta_{\mathrm{base}}}{a^2}
\nabla_{\vz_l^{(j)}}\mathcal{L}.
\end{align*}
Setting $M=1$ and $\phi(\vh,\mW)=\mW\vh$ gives the residual recursion $\vh_l=\vh_{l-1}+(a/L)\mW_l\vh_{l-1}$, and the corresponding HPs
\begin{align}
\alpha_l=\frac{a}{L}=\Theta(1/L),
\quad
\eta_l=\frac{L\eta_{\mathrm{base}}}{a^2}=\Theta(L).
\label{eq:chizat-k1-mapping}
\end{align}
Under the standard width-$\mu$P initialization $\|\mW_l\|_{\normrms}=\Theta(1)$, the initial condition in Condition~\ref{condition: one-layer} is satisfied because
$\alpha_l\|\mW_l\|_{\normrms}=\Theta(1/L)=\mathcal{O}(1/\sqrt{L})$.
Besides, repeating the hidden raw-gradient calculation of Appendix~\ref{app:sgd_mup} with $\alpha_l=\Theta(1/L)$ gives
$\|\nabla_{\mW_l}\mathcal{L}\|_{\normrms}=\Theta(1/L)$.
Together with Equation~(\ref{eq:chizat-k1-mapping}), this yields
\begin{align*}
&\|\Delta\mW_l\|_{\normrms}
=\eta_l\|\nabla_{\mW_l}\mathcal{L}\|_{\normrms}=\Theta(1),\\
&\alpha_l\|\Delta\mW_l\|_{\normrms}
=\Theta(1/L),
\end{align*}
which verifies the update condition in Condition~\ref{condition: one-layer}.
Thus, the $\Theta(1/L)$ parameterization is a feasible $k=1$ choice corresponding to the maximal local-update regime, whereas the Depth-$\mu$P choice $\alpha_l=\Theta(1/\sqrt{L})$ adopted in this appendix selects the boundary of the same feasible range by maximizing the zero-order update contribution.
}

\subsection{AdamW, Sophia and Lion}
\label{app:admw one-layer}

\begin{table}[t]
\renewcommand{\arraystretch}{1.3}
\setlength{\tabcolsep}{2pt} 
\renewcommand{\hl}[1]{\textcolor{purple}{#1}}
\renewcommand{\ll}[1]{\textcolor{gray}{#1}}
\centering
\caption{\textbf{{Depth-$\mu$P-style parameterization from Condition~\ref{condition: one-layer} ($k=1$)} for AdamW~\citep{adamw}, Lion~\citep{DBLP:conf/nips/ChenLHRW0DLHLL23-lion}, and Sophia~\citep{DBLP:conf/iclr/Liu0HL024-sophia} with weight decay under width–depth scaling.}
Entries in \hl{purple} indicate differences between $\mu$P and SP, while \ll{gray} shows the corresponding SP choices.
Here, $r_n$ and $r_L$ denote the width and depth scaling ratios relative to the base model.
The variance of input weights is $\sigma^2_{\mathrm{base}}$ for language and $\sigma^2_{\mathrm{base}}/d_0$ for image. The initial variance of input bias is $\sigma^2_{\mathrm{base}}$.}
\label{tab: adamw-mup one-layer}
\vskip 0.05in
\begin{tabular}{ccccc}
\toprule
 & Input weights \& biases & Hidden weights & Output weights & Hidden biases\\
\midrule

Block Multiplier
& $\alpha_{\mathrm{base}}$
& \hl{$\alpha_{\mathrm{base}}/\sqrt{r_L}$} \ \ll{($\alpha_{\mathrm{base}}$)}
& \hl{$\alpha_{\mathrm{base}}/r_n$} \ \ll{($\alpha_{\mathrm{base}}$)} 
&\hl{$\alpha_{\mathrm{base}}/\sqrt{r_L}$} \ \ll{($\alpha_{\mathrm{base}}$)} \\

Initial Variance
& $\sigma^2_{\mathrm{base}}/d_0$ or $\sigma^2_{\mathrm{base}}$
& \hl{$\sigma^2_{\mathrm{base}}/r_n$} \ \ll{($\sigma^2_{\mathrm{base}}$)}
& $\sigma^2_{\mathrm{base}}$
&$\sigma^2_{\mathrm{base}}$ \\

Learning Rate
& $\eta_{\mathrm{base}}$
& \hl{$\eta_{\mathrm{base}}/(r_n\sqrt{r_L})$} \ \ll{($\eta_{\mathrm{base}}$)}
& $\eta_{\mathrm{base}}$
& \hl{$\eta_{\mathrm{base}}/\sqrt{r_L}$} \ \ll{($\eta_{\mathrm{base}}$)} \\

Weight Decay
& $\lambda_{\mathrm{base}}$
& \hl{$\lambda_{\mathrm{base}}r_n$} \ \ll{($\lambda_{\mathrm{base}}$)}
& $\lambda_{\mathrm{base}}$
& $\lambda_{\mathrm{base}}$ \\

AdamW $\varepsilon$
& \hl{$\varepsilon_{\mathrm{base}}/r_n$} \ \ll{($\varepsilon_{\mathrm{base}}$)}
& \hl{$\varepsilon_{\mathrm{base}}/(r_n\sqrt{r_L})$} \ \ll{($\varepsilon_{\mathrm{base}}$)}
& \hl{$\varepsilon_{\mathrm{base}}/r_n$} \ \ll{($\varepsilon_{\mathrm{base}}$)}
& \hl{$\varepsilon_{\mathrm{base}}/(r_n\sqrt{r_L})$} \ \ll{($\varepsilon_{\mathrm{base}}$)}\\
\bottomrule
\end{tabular}
\end{table}

In this section, we derive the {Depth-$\mu$P-style parameterization} from
Condition~\ref{condition: one-layer} for AdamW (same for Sophia and Lion),
which recovers and extends the $\mu$P scaling rules studied in~\citet{TP-6}.

As in Appendix~\ref{app:admw}, we reduce AdamW to sign gradient descent by setting $\beta_1 = 0$,
$\beta_2 = 0$, and $\varepsilon_l = 0$:
\begin{align*}
    \Delta{\mW}_l
    =-\eta_l \bigl(\,\underbrace{\mathrm{sign}\left(\nabla_{\mW_l}\mathcal{L}\right)}_{\mA_l} 
    + \lambda_l {\mW}_l\bigl).
\end{align*}
Using the norm estimate in Equation~(\ref{eq:Anorm_adam}), we have
\begin{align*} 
\Vert \mA_l \Vert_\normrms 
= \left\Vert \mathrm{sign}\left(\nabla_{\mW_l}\mathcal{L}\right) \right\Vert_\normrms 
= \Theta(\nin)
=\left\{
    \begin{array}{ll}
    \Theta(1),   & l = 0,\\
    \Theta(\nin),   & l \in [L],\\
    \Theta(\nin),  & l = L+1.
    \end{array} \right.
\end{align*}
As discussed in Appendices~\ref{app:sophia} and~\ref{app:lion}, Sophia and
Lion share the same parameterization as AdamW under the simplifications used in
this section.

\subsubsection{Derivation of Parameterization}

\paragraph{Input and output layers.} 
The input and output layer conditions in Condition~\ref{condition: one-layer}
are identical to those in Condition~\ref{condition: scale-invariant fl}.
Therefore, as in Appendix~\ref{app:admw}, their AdamW learning-rate and
weight-decay parameterizations are
\begin{align*}
    \eta_0=\Theta(1), \quad
    \eta_{L+1}=\Theta(1),\quad
    \lambda_0=\Theta(1), \quad \lambda_{L+1}=\Theta(1).
\end{align*}

\paragraph{Hidden layers (first-order).} 
Unlike SGD, the sign-style update direction has norm determined by dimension
rather than by the raw-gradient magnitude.
Thus, the only change from the $k=2$ AdamW implementation is the hidden
multiplier: {we adopt the
Depth-$\mu$P-style choice $\alpha_l=\Theta(1/\sqrt{L})$ rather than
$\Theta(1/L)$}.
When $\lambda_l = 0$, given the scale of $\Vert\mA_l\Vert_\normrms$ above, we have
\begin{align*}
    \alpha_l\|\Delta\mW_l\|_{\normrms}
    =\Theta(1/\sqrt{L}) \cdot \eta_l\|\mA_l\|_{\normrms}
    =\Theta(\eta_l\nin/\sqrt{L}).
\end{align*}
As desired in Equation~(\ref{eq:condition_wd_1_onelayer}), to satisfy the
first-order update condition on hidden weights in
Condition~\ref{condition: one-layer},
$\alpha_l\|\Delta \mW_l\|_{\normrms}=\Theta(1/L)$, we need to set 
\begin{align*}
    \eta_l=\Theta(\frac{1}{\nin \sqrt{L}}).
\end{align*}
When $\lambda_l \neq 0$, the weight-decay parameterization is unchanged from
the $k=2$ AdamW implementation because it matches the scale of $\mA_l$ and
does not depend on the residual multiplier.
Given the weight norm $\|{\mW}_l\|_{\normrms} = \Theta(1)$ in
Equation~(\ref{eq:w_norm}), we have
\begin{align*}
    \|\lambda_l\mW_l\|_{\normrms} = \Theta(\lambda_l).
\end{align*}
To satisfy Equation~(\ref{eq:condition_wd_2_onelayer}) that
$\|\lambda_l \mW_l\|_{\normrms}
    = \Theta\left( \|\mA_l\|_{\normrms} \right)$, we need to set 
\begin{align*}
    \lambda_l=\Theta(\nin).
\end{align*}

\paragraph{Biases.}
For bias parameters $\vb_l \in \R^{\nout\times 1}$, by the definition we have
\begin{align*}
\|\mathrm{sign}\left(\nabla_{\vb_l}\mathcal{L}\right)\|_\normrms = \Theta(1),  \quad 0 \leq l \leq L.
\end{align*}
To satisfy the one-layer bias update condition,
$\alpha_0 \|\Delta \vb_0\|_{\normrms} = \Theta(\eta_{\vb_0}) = \Theta(1)$ and
$\alpha_l \|\Delta \vb_l\|_{\normrms}
= \Theta(1/\sqrt{L} \cdot \eta_{\vb_l}) = \Theta(1/L)$ for $l \in [L]$, we need to set
\begin{align*}
    \eta_{\vb_0}=\Theta(1), \quad \eta_{\vb_l}=\Theta(1/\sqrt{L}), \quad l \in [L].
\end{align*}
For the weight decays, we match
$\lambda_{\vb_l} \|\vb_l\|_\normrms$ to
$\|\mathrm{sign}\left(\nabla_{\vb_l}\mathcal{L}\right)\|_\normrms=\Theta(1)$.
Given $\|\vb_l\|_\normrms = \Theta(1)$ by the initialization implementation in
Condition~\ref{condition: bias}, we have
\begin{align*}
    \lambda_{\vb_l}=\Theta(1), \quad 0 \leq l \leq L.
\end{align*}

\paragraph{Parameterization of $\varepsilon_l$.}

To make the stabilization term $\varepsilon_l$ effective and not dominate the gradient, we desire it to be of the same scale as $\sqrt{\hat{\vv}_l^{(t)}}$. 
When omitting the momentum, we have ${\sqrt{\hat{\vv}_l} = \left|\nabla_{\mW_l}\mathcal{L}\right|}$ {elementwise}.
Therefore, we need to ensure
$\varepsilon_l = \Theta(\|\nabla_{\mathrm{Vec}(\mW_l)}\mathcal{L}\|_\normrms)$.
The latter can be estimated from the raw-gradient derivation for
$\Vert\nabla_{\mW_l}\mathcal{L}\Vert_\normrms$ in
Equation~(\ref{eqn: sgd gradient matrix one-layer}) of Appendix~\ref{app:sgd dmup}. 

For the input weights $\mW_0$, we have
\begin{align*}
\|\nabla_{\mathrm{Vec}(\mW_0)}\mathcal{L}\|_\normrms 
= \frac{1}{\sqrt{\nin\nout}}\|\nabla_{\mW_0}\mathcal{L}\|_\mathrm{F}
= \Theta\left(\frac{1}{\sqrt{\nin\nout}} \|\nabla_{\mW_0}\mathcal{L}\|_\mathrm{2}\right)
= \Theta\left(\frac{1}{\sqrt{\nin\nout}} \sqrt{\frac{\nin}{\nout}}\right)
= \Theta(\frac{1}{\nout}).
\end{align*}
Therefore, we set 
\[
\varepsilon_0 = \Theta(\frac{1}{\nout}).
\]

For the hidden weights $\mW_l$ where $l\in[L]$, we have
\begin{align*}
\|\nabla_{\mathrm{Vec}(\mW_l)}\mathcal{L}\|_\normrms 
= \Theta\left(\frac{1}{\sqrt{\nin\nout}} \|\nabla_{\mW_l}\mathcal{L}\|_\mathrm{2}\right)
= \Theta\left(\frac{1}{\sqrt{\nin\nout}} \sqrt{\frac{\nin}{\nout L}}\right)
= \Theta(\frac{1}{\nout\sqrt{L}}).
\end{align*}
Therefore, we set 
\[
\varepsilon_l = \Theta(\frac{1}{\nout\sqrt{L}}), \quad l \in [L].
\]

For the output weights $\mW_{L+1}$, we have
\begin{align*}
\|\nabla_{\mathrm{Vec}(\mW_{L+1})}\mathcal{L}\|_\normrms 
= \Theta\left(\frac{1}{\sqrt{\nin\nout}} \|\nabla_{\mW_{L+1}}\mathcal{L}\|_\mathrm{2}\right)
= \Theta\left(\frac{1}{\sqrt{\nin\nout}}  \sqrt{\frac{1}{\nin\nout}}\right)
= \Theta(\frac{1}{\nin}).
\end{align*}
Therefore, we set 
\[
\varepsilon_{L+1} = \Theta(\frac{1}{\nin}).
\]

Similarly, for the biases we have derived in Appendix~\ref{app:sgd dmup} that
\begin{align*}
\Vert\nabla_{\vb_l}\mathcal{L}\Vert_\normrms
=
\begin{cases}
\Theta(1/{\nout}), & l=0,\\
\Theta(1/(\nout \sqrt{L})),       & l\in[L].
\end{cases}
\end{align*}
Therefore, we set the stabilization term as
\begin{align*}
\varepsilon_{\vb_l}
=
\begin{cases}
\Theta(1/{\nout}), & l=0,\\
\Theta(1/(\nout\sqrt{L})),       & l\in[L].
\end{cases}
\end{align*}

This completes the implementation of Condition~\ref{condition: one-layer} for
AdamW, Lion, and Sophia with weight decay, as summarized in
Table~\ref{tab: adamw-mup one-layer}.

\subsection{Spectral Sphere Optimizer (SSO)}
\label{app:sso one-layer}

\begin{table}[t]

\renewcommand{\arraystretch}{1.3}
\renewcommand{\hl}[1]{\textcolor{purple}{#1}}
\renewcommand{\ll}[1]{\textcolor{gray}{#1}}
\centering

\caption{\textbf{{Depth-$\mu$P-style parameterization from Condition~\ref{condition: one-layer} ($k=1$)} for SSO~\citep{xie2026controlled-sso} with weight decay under width-depth scaling.}
Entries in \hl{purple} indicate differences between $\mu$P and SP, while \ll{gray} shows the corresponding SP choices.
Here, $r_n$ and $r_L$ denote the width and depth scaling ratios relative to the base model.
The variance of input weights is $\sigma^2_{\mathrm{base}}$ for language and $\sigma^2_{\mathrm{base}}/d_0$ for image.}
\label{tab: sso-mup one-layer}
\vskip 0.05in
\begin{tabular}{cccc}
\toprule
 & Input weights & Hidden weights & Output weights \\
\midrule

Block Multiplier
& $\alpha_{\mathrm{base}}$
& \hl{$\alpha_{\mathrm{base}}/\sqrt{r_L}$} \ \ll{($\alpha_{\mathrm{base}}$)}
& \hl{$\alpha_{\mathrm{base}}/r_n$} \ \ll{($\alpha_{\mathrm{base}}$)} \\

Initial Variance
& $\sigma^2_{\mathrm{base}}/d_0$ or $\sigma^2_{\mathrm{base}}$
& \hl{$\sigma^2_{\mathrm{base}}/r_n$} \ \ll{($\sigma^2_{\mathrm{base}}$)}
& $\sigma^2_{\mathrm{base}}$ \\

Learning Rate
& $\eta_{\mathrm{base}}$
& \hl{$\eta_{\mathrm{base}}/\sqrt{r_L}$} \ \ll{($\eta_{\mathrm{base}}$)}
& \hl{$\eta_{\mathrm{base}}{r_n}$} \ \ll{($\eta_{\mathrm{base}}$)} \\

Weight Decay
& $\lambda_{\mathrm{base}}$
& $\lambda_{\mathrm{base}}$
& \hl{$\lambda_{\mathrm{base}}/{r_n}$} \ \ll{($\lambda_{\mathrm{base}}$)} \\

\bottomrule
\end{tabular}
\end{table}

In this section, we derive the {Depth-$\mu$P-style parameterization} from
Condition~\ref{condition: one-layer} for SSO.

As in Appendix~\ref{app:sso}, SSO~\citep{xie2026controlled-sso} uses the update
\begin{equation*}
\begin{aligned}
    \Delta{\mW}_l = -\eta_l\bigl(\,\underbrace{R\boldsymbol{\Phi}_l}_{\mA_l}  + \lambda_l {\mW}_l\bigl),
\end{aligned}
\end{equation*}
where
\begin{align*}
    R = \Theta\left(\sqrt{\frac{\nout}{\nin}}\right),\quad \text{and} \quad \boldsymbol{\Phi}_l = \arg\max_{\boldsymbol{\Phi}} \langle \nabla_{\mW_l}\mathcal{L}, \boldsymbol{\Phi}\rangle 
    \ \mathrm{s.t.} \ 
    \Vert \boldsymbol{\Phi} \Vert_2 = 1, \ \Vert \mW_l - \eta_l \boldsymbol{\Phi} \Vert_2 = \Vert \mW_l\Vert_2 = R.
\end{align*}
Using the norm computation in Equation~(\ref{eq:Anorm_sso}), we have
\begin{align*}
    \Vert \mA_l \Vert_{\normrms} = \Theta(1).
\end{align*}

\subsubsection{Derivation of Parameterization}

\paragraph{Input and output layers.} 
The input and output layer conditions in Condition~\ref{condition: one-layer}
are identical to those in Condition~\ref{condition: scale-invariant fl}.
Therefore, as in Appendix~\ref{app:sso}, their SSO learning-rate and
weight-decay parameterizations are
\begin{align*}
    \eta_0=\Theta(1), \quad
    \eta_{L+1}=\Theta({\nin}),\quad
    \lambda_0=\Theta(1), \quad
    \lambda_{L+1}=\Theta(1/{\nin}).
\end{align*}

\paragraph{Hidden layers (first-order).} 
The only change from the $k=2$ SSO implementation is the hidden multiplier:
{we adopt the Depth-$\mu$P-style choice
$\alpha_l=\Theta(1/\sqrt{L})$ rather than $\Theta(1/L)$}.
When $\lambda_l = 0$, given the scale of $\Vert\mA_l\Vert_\normrms$ above, we have
\begin{align*}
    \alpha_l\|\Delta\mW_l\|_{\normrms}
    =\Theta(1/\sqrt{L}) \cdot \eta_l\|\mA_l\|_{\normrms}
    =\Theta(\eta_l/\sqrt{L}).
\end{align*}
As desired in Equation~(\ref{eq:condition_wd_1_onelayer}), to satisfy the
first-order update condition on hidden weights in
Condition~\ref{condition: one-layer},
$\alpha_l\|\Delta \mW_l\|_{\normrms}=\Theta(1/L)$, we need to set 
\begin{align*}
    \eta_l=\Theta(1/\sqrt{L}).
\end{align*}
When $\lambda_l \neq 0$, the weight-decay parameterization is unchanged from
the $k=2$ SSO implementation because it matches the scale of $\mA_l$ and does
not depend on the residual multiplier.
Given the weight norm $\|{\mW}_l\|_{\normrms} = \Theta(1)$ in
Equation~(\ref{eq:w_norm}), we have
\begin{align*}
    \|\lambda_l\mW_l\|_{\normrms} = \Theta(\lambda_l).
\end{align*}
To satisfy Equation~(\ref{eq:condition_wd_2_onelayer}) that
$\|\lambda_l \mW_l\|_{\normrms}
    = \Theta\left( \|\mA_l\|_{\normrms} \right)$, we need to set 
\begin{align*}
    \lambda_l=\Theta(1).
\end{align*}

This completes the implementation of Condition~\ref{condition: one-layer} for
SSO with weight decay, as summarized in Table~\ref{tab: sso-mup one-layer}.

\section{Additional Details and Results of GPT-2 Experiments}
\label{app: Additional Experimental Details}

This section provides the experimental details and complete numerical results for the GPT-2 style language-model experiments in Section~\ref{sec: experiment}. We organize the results around three comparisons. First, we compare SP with the {CompleteP-style $\mu$P parameterization from Condition~\ref{condition: scale-invariant fl} ($k\ge2$)}, which is predicted for residual branches with multiple transformations. Second, we compare this parameterization with the {Depth-$\mu$P-style parameterization from Condition~\ref{condition: one-layer} ($k=1$)}, which applies to one-layer residual branches. Third, we report additional weight-decay transfer results.

\subsection{Assets and Licenses}
\label{app: Assets and Licenses}

All used assets (datasets and codes) and their licenses are listed in Table~\ref{tab: license}.

\begin{table}[ht]
\renewcommand{\arraystretch}{1.3}
\centering
\caption{\textbf{Used assets and their licenses.}}
\setlength{\tabcolsep}{1pt}
\begin{tabular}{cc} 
\toprule
URL and citation                                    & License                                                    \\ 
\midrule
https://github.com/EleutherAI/nanoGPT-mup/tree/completep~\citep{completep}      & MIT    \\
https://github.com/karpathy/nanoGPT~\citep{nanogpt} &  MIT \\
https://skylion007.github.io/OpenWebTextCorpus/~\citep{Gokaslan2019OpenWeb} &  Creative Commons CC0 \\
\bottomrule
\end{tabular}
\label{tab: license}
\end{table}

\subsection{Additional Details of Feature Learning Experiments}
\label{app: Additional Details of Feature Learning Experiments}

{This subsection reports the coordinate-check experiments used to evaluate feature-scale stability under width, depth, and simultaneous width-depth scaling. For the simultaneous-scaling experiment, we set $(n,L)=(256s,4s)$ with joint scale $s\in\{1,2,4,8\}$.} We measure the RMS norm at the output of the final Transformer block, $\|\vh_L\|_{\normrms}$, after short training runs. Results are averaged over three independent runs with random seeds $1$, $2$, and $3$. The base initialization variance for matrix weights and biases is set to $0.02^2$ and $0$, respectively. All models are trained with a constant learning rate schedule, a batch size of $8$, a gradient clipping of $1.0$, and no weight decay for $10$ training steps.
The HP scaling rules can be found in Table~\ref{tab: optimizer-family-overview} for the {CompleteP-style formulation ($k\ge2$)}.

{
\paragraph{Effect of gradient clipping.}
In all experiments, gradient clipping is applied before optimizer processing, multiplying every raw gradient at step $t$ by the same positive scalar $c_t=\min\{1,C/\|g_t\|_2\}$. All modern optimizers evaluated in our experiments (e.g., Muon) include normalization in their update processing, so this subsequent normalization therefore absorbs the rescaling by $c_t$ and introduces no additional width- or depth-dependent factor. For example, Muon satisfies
\begin{align*}
    \operatorname{Muon}(c_t\mG)=\operatorname{Muon}(\mG), \quad \forall c_t>0,
\end{align*}
because multiplying $\mG$ by $c_t$ changes its singular values but not its singular-vector update $\mU\mV^\top$. Thus, even if $c_t$ varies with model size, fixing $C=1.0$ does not alter the width-depth scaling of the optimizer-processed updates and requires no additional size-dependent adjustment. This treatment is discussed in Footnote~42 of~\citet{TP4b} and consistent with existing $\mu$P practice~\citep{TP5,completep}.
}

In addition to common HPs, Optimizer-specific HPs are listed below.

\paragraph{Muon-Kimi-AdamW.}
Following common practice~\citep{muon-kimi}, hidden matrix parameters are optimized by Muon-Kimi with a Nesterov-style momentum~\citep{nesterov1983method} of 0.95. All other parameters (e.g., embedding layer, LM head, all biases) are updated by AdamW with $\beta_1=0.9$, $\beta_2=0.95$, $\epsilon_{\mathrm{base}}=10^{-16}$. The learning rate is $2^{-7}$ for both Muon-Kimi and AdamW. {The results are presented in Figure~\ref{figures: mup vs sp}(a,c,e).}

\paragraph{Muon-AdamW.}
Following common practice~\citep{jordan6muon}, hidden matrix parameters are optimized by Muon with a learning rate of $0.02$ and a Nesterov-style momentum of 0.95. All other parameters (e.g., embedding layer, LM head, all biases) are updated by AdamW with a learning rate of $0.001$, $\beta_1=0.9$, $\beta_2=0.95$, and $\epsilon_{\mathrm{base}}=10^{-16}$. The results are presented in Figure~\ref{figures: mup vs sp muon}(a,b).

\paragraph{Shampoo-AdamW.}
Following common practice~\citep{modded_nanogpt_2024}, hidden matrix parameters are optimized by Shampoo with a learning rate of $0.001$, $\beta_1=0.95$, $\beta_2=0.95$, $\epsilon_{\mathrm{base}}=10^{-16}$, a shampoo precondition frequency of $1$, and a maximal precondition dimension of $20000$. All other parameters (e.g., embedding layer, LM head, all biases) are updated by AdamW with a learning rate of $0.002$, $\beta_1=0.9$, $\beta_2=0.95$, and $\epsilon_{\mathrm{base}}=10^{-16}$. The results are presented in Figure~\ref{figures: mup vs sp shampoo}(a,b).

\paragraph{Sophia.}
Following common practice~\citep{DBLP:conf/iclr/Liu0HL024-sophia}, parameters are updated by Sophia with a learning rate of $0.001$, $\beta_1=0.965$, $\beta_2=0.99$, $\rho=0.05$, and a Hessian update frequency of $1$. The results are presented in Figure~\ref{figures: mup vs sp sophia}(a,b).

\paragraph{Overview of results.}
{Across these optimizer settings, the separate width- and depth-scaling coordinate checks show the same qualitative pattern as in the main text: SP exhibits feature-scale growth, whereas the CompleteP-style $\mu$P parameterization from Condition~\ref{condition: scale-invariant fl} ($k\ge2$) keeps the feature norm approximately scale-invariant. For Muon-Kimi-AdamW, Figure~\ref{figures: mup vs sp}(e) confirms the same result when width and depth are scaled simultaneously.}

\subsection{Additional Details of HP Transfer Experiments}
\label{app: Additional Details of HP Transfer Experiments}

This subsection provides the detailed setup and complete numerical results for the HP-transfer experiments in Section~\ref{sec: experiment}.
We organize the results by optimizer. 
For each optimizer, we first compare SP with the {CompleteP-style $\mu$P parameterization from Condition~\ref{condition: scale-invariant fl} ($k\ge2$)} under width and depth scaling.
{For Muon-Kimi-AdamW, we additionally evaluate simultaneous width-depth scaling.}
When applicable, we then compare this parameterization with the {Depth-$\mu$P-style parameterization from Condition~\ref{condition: one-layer} ($k=1$)} to test the role of residual-block depth.
{We do not include the alternative MLU-style $k=1$ choice in these comparisons because Depth-$\mu$P is more commonly adopted in practical width-depth $\mu$P studies~\citep{TP-6,completep}.}
For Muon-Kimi-AdamW, we additionally report weight-decay transfer and the no-LayerNorm depth-scaling diagnostic discussed in the main text.

\subsubsection{Basic Experimental Setup}

Unless otherwise stated, all HP-transfer experiments use GPT-2 style models trained on OpenWebText with sequence length $1024$ and the GPT-2 tokenizer, following the setup in Section~\ref{sec: experiment}. 
The base model has width $n_{\mathrm{base}}=256$ and depth $L_{\mathrm{base}}=4$.
For width-scaling experiments, we vary $n$ while keeping $L=4$; for depth-scaling experiments, we vary $L$ while keeping $n=256$.
{For simultaneous width-depth scaling, we use $(n,L)\in\{(256,4),(512,8),(1024,16),(2048,32)\}$, corresponding to joint scales $s=n/256=L/4\in\{1,2,4,8\}$.}
The base initialization variance for matrix weights is set to $0.02^2$, and biases are initialized to zero.

All models are trained with batch size $240$ for $1221$ iterations, corresponding to about $300$M tokens, using $120$ warmup iterations followed by cosine learning-rate decay to a minimum learning rate of $3\times10^{-5}$, gradient clipping of $1.0$, and the optimizer-specific settings described in the following sections.

The HPs shown in the tables are base HPs, such as $\eta_{\mathrm{base}}$ or $\lambda_{\mathrm{base}}$; the actual model HPs are obtained by applying the corresponding SP or $\mu$P scaling rules.
The {CompleteP-style ($k\ge2$) and Depth-$\mu$P-style ($k=1$) formulations} can be found in Tables~\ref{tab: optimizer-family-overview} and~\ref{tab: optimizer-family-overview-one-layer}, respectively.
We report the final validation loss.

\subsubsection{Additional Details of Muon-Kimi-AdamW}
\label{app: Additional Details of Muon-Kimi-AdamW}

\begin{figure*}[t]
\centering
\includegraphics[height=0.25\textwidth]{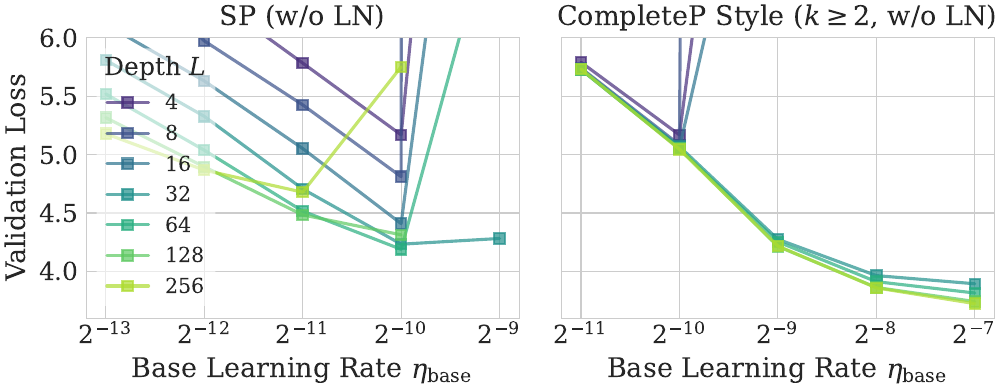}
\vskip 0.05in
\caption{
\textbf{Depth-wise HP transfer of Muon-Kimi-AdamW under SP and the {CompleteP-style $\mu$P parameterization from Condition~\ref{condition: scale-invariant fl} ($k\ge2$)} without LayerNorm.}
First, in terms of training stability, SP becomes increasingly prone to loss divergence as depth increases in the absence of LayerNorm, whereas $\mu$P enables stable training. 
Second, unlike SP, $\mu$P preserves HP transferability at large depths without LayerNorm.
}
\label{figures: mup vs sp-hp-nonLN}

\end{figure*}

\begin{figure*}[t]
\centering

\includegraphics[height=0.24\textwidth]{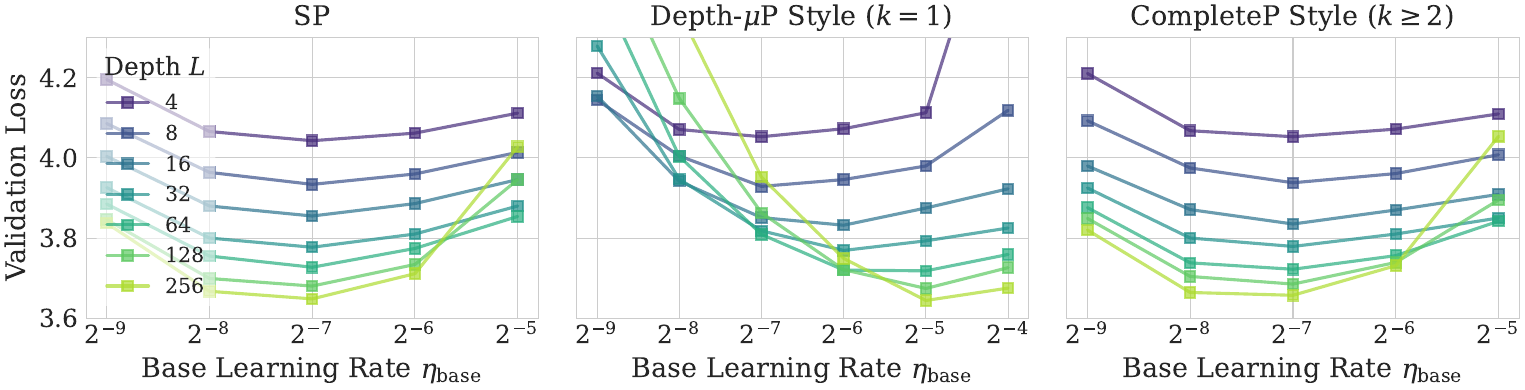}

\vskip 0.05in
\caption{{
\textbf{Controlled $k=3$ validation of depth-wise HP transfer for Muon-Kimi-AdamW.}
SP and the CompleteP-style parameterization ($k\ge 2$) show stable HP transfer and comparable loss, whereas the Depth-$\mu$P-style parameterization ($k=1$) exhibits a clear rightward shift in the optimal HP across depths.
The corresponding numerical values are reported in Appendix~\ref{app: Additional Details of Muon-Kimi-AdamW}.
}}

\label{figures: sp, cp, dmup Muon-kimi k3}
\end{figure*}

\begin{figure*}[t]
\centering

\subfloat{
\includegraphics[height=0.185\textwidth]{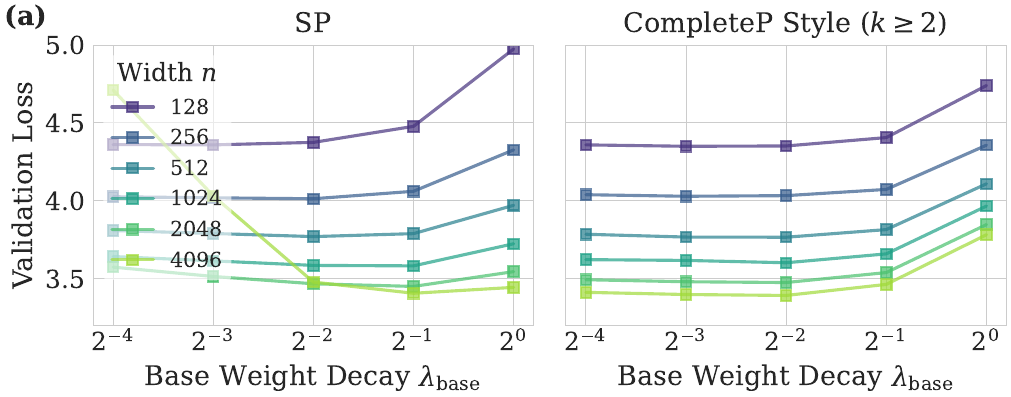}
\label{fig: sp_width_transfer_wd}
}%
\hskip 1ex
\subfloat{
\includegraphics[height=0.185\textwidth]{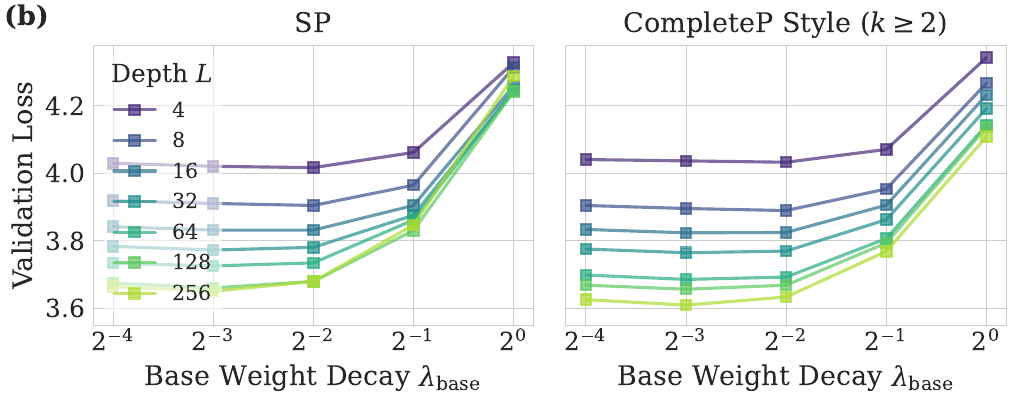}
\label{fig: sp_depth_transfer_wd}
}%

\vskip 0.05in
\caption{
\textbf{Weight decay transfer of Muon-Kimi-AdamW under SP and the {CompleteP-style $\mu$P parameterization from Condition~\ref{condition: scale-invariant fl} ($k\ge2$)}.}
We train GPT-2 style models with Muon-Kimi and AdamW using SP and the {CompleteP-style parameterization} (see Tables~\ref{tab: muon-kimi-wd mup} and~\ref{tab: adamw-mup}). The {CompleteP-style parameterization} enables robust HP transfer across both width and depth scaling, while consistently achieving lower loss than SP as the model depth increases. The corresponding numerical values are reported in Appendix~\ref{app: Additional Details of Muon-Kimi-AdamW}.
}

\label{figures: mup vs sp wd}
\end{figure*}

This subsection provides the complete Muon-Kimi-AdamW results underlying Figures~\ref{figures: mup vs sp},~\ref{figures: cp vs dmup muon style},~\ref{figures: mup vs sp-hp-nonLN}, and~\ref{figures: mup vs sp wd}{, together with the controlled $k=3$ results in Figure~\ref{figures: sp, cp, dmup Muon-kimi k3}}.
{We first report learning-rate transfer under width, depth, and simultaneous width-depth scaling, including the comparison between CompleteP-style formulation ($k\ge2$) and Depth-$\mu$P-style formulation ($k=1$) for depth scaling and the controlled $k=3$ comparison.}
We then report the no-LayerNorm diagnostic and the weight-decay transfer experiments.

\paragraph{Experimental setup.}
Following common practice~\citep{muon-kimi}, hidden matrix parameters are optimized by Muon-Kimi with Nesterov-style momentum $0.95$, while all other parameters (e.g., embeddings, LM head, and biases) are optimized by AdamW with $\beta_1=0.9$, $\beta_2=0.95$, and $\epsilon_\mathrm{base}=10^{-16}$.
For learning-rate transfer experiments, both Muon-Kimi and AdamW use the same base learning rate $\eta_{\mathrm{base}}$, and weight decay is disabled.
{For the controlled $k=3$ experiment in Figure~\ref{figures: sp, cp, dmup Muon-kimi k3}, we use the same training and optimizer protocol as the other depth-scaling experiments.
We add one linear layer at the end of both the attention and FFN modules so that each residual branch instantiates $k=3$, while leaving all other architectural settings unchanged.
}
For weight-decay transfer experiments, we fix the base learning rate to $2^{-7}$ and sweep the base weight decay $\lambda_{\mathrm{base}}$.

\paragraph{Additional results of width-wise learning rate transfer.}

Complete numerical results of base learning rate transferability across different widths are presented in Table~\ref{tab: width sp} and Table~\ref{tab: width cp} for SP and the {CompleteP-style $\mu$P parameterization from Condition~\ref{condition: scale-invariant fl} ($k\ge2$)}, respectively.

\begin{table}[t]
\centering
\caption{\textbf{For Muon-Kimi-AdamW, SP fails to transfer the optimal base learning rate across widths.}
This table reports the numerical results corresponding to Figure~\ref{figures: mup vs sp}, where the best validation loss for each width is highlighted in \textbf{bold}.}
\renewcommand{\arraystretch}{1.3}
\vskip 0.05in
\begin{tabular}{ccccccc} 
\toprule
$n/\log_2(\eta_{\mathrm{base}})$ & -10   & -9    & -8    & -7    & -6    & -5     \\ 
\midrule
128         & 5.127 & 4.685 & 4.45  & 4.373 & \textbf{4.364} & 4.372  \\
256         & 4.552 & 4.219 & 4.081 & \textbf{4.053} & 4.062 & 4.091  \\
512         & 4.093 & 3.886 & \textbf{3.819} & 3.833 & 3.837 & 4.062  \\
1024        & 3.817 & 3.699 & \textbf{3.672} & 3.68  & 3.952 & 5.603  \\
2048        & 3.654 & 3.571 & \textbf{3.555} & 3.798 & 5.472 & 6.438  \\
4096        & 3.56  & \textbf{3.516} & 3.747 & 5.557 & 6.159 & 6.552  \\
\bottomrule
\end{tabular}
\label{tab: width sp}
\end{table}

\begin{table}[t]
\centering
\caption{\textbf{For Muon-Kimi-AdamW, the {CompleteP-style $\mu$P parameterization from Condition~\ref{condition: scale-invariant fl} ($k\ge2$)} approximately transfers the optimal base learning rate across widths, and achieves lower loss than SP as the width increases.}
This table reports the numerical results corresponding to Figure~\ref{figures: mup vs sp}, where the best validation loss for each width is highlighted in \textbf{bold}.}
\renewcommand{\arraystretch}{1.3}
\vskip 0.05in
\begin{tabular}{ccccccc} 
\toprule
$n/\log_2(\eta_{\mathrm{base}})$ & -10   & -9    & -8    & -7    & -6    & -5     \\ 
\midrule
128         & 4.875 & 4.53  & 4.42  & \textbf{4.374} & 4.383 & 4.397  \\
256         & 4.561 & 4.227 & 4.081 & \textbf{4.059} & 4.079 & 4.104  \\
512         & 4.305 & 3.974 & 3.83  & \textbf{3.811} & 3.828 & 3.873  \\
1024        & 4.125 & 3.798 & 3.654 & \textbf{3.646} & 3.676 & 3.726  \\
2048        & 3.957 & 3.636 & 3.516 & \textbf{3.515} & 3.552 & 3.689  \\
4096        & 3.882 & 3.531 & \textbf{3.446} & 3.461 & 3.523 & 3.752  \\
\bottomrule
\end{tabular}
\label{tab: width cp}
\end{table}

\paragraph{Additional results of depth-wise learning rate transfer with LayerNorm.}

With LayerNorm, complete numerical results of base learning rate transferability across different depths are presented in Table~\ref{tab: depth sp}, Table~\ref{tab: depth cp}, and Table~\ref{tab: depth dmup} for SP, the {CompleteP-style $\mu$P parameterization from Condition~\ref{condition: scale-invariant fl} ($k\ge2$)}, and the {Depth-$\mu$P-style parameterization from Condition~\ref{condition: one-layer} ($k=1$)}, respectively.
As discussed in Section~\ref{sec: experiment}, this apparent depth-wise transfer under SP should be interpreted cautiously, because LayerNorm and the tested depth range can partially mask feature-scale instability.

\begin{table}[t]
\centering
\caption{\textbf{For Muon-Kimi-AdamW, with LayerNorm, SP transfers the optimal base learning rate across the tested depths.}
This table reports the numerical results corresponding to Figure~\ref{figures: mup vs sp}, where the best validation loss for each depth is highlighted in \textbf{bold}.}
\renewcommand{\arraystretch}{1.3}
\vskip 0.05in
\begin{tabular}{cccccc} 
\toprule
$L/\log_2(\eta_{\mathrm{base}})$ & -9    & -8    & -7    & -6    & -5     \\ 
\hline
4     & 4.219 & 4.081 & \textbf{4.056} & 4.067 & 4.09   \\
8     & 4.109 & 3.985 & \textbf{3.952} & 3.973 & 4.013  \\
16    & 4.016 & 3.893 & \textbf{3.864} & 3.889 & 3.929  \\
32    & 3.949 & 3.824 & \textbf{3.799} & 3.82  & 3.885  \\
64    & 3.916 & 3.777 & \textbf{3.747} & 3.777 & 3.91   \\
128   & 3.898 & 3.75  & \textbf{3.723} & 3.772 & 4.031  \\
256   & 3.883 & 3.719 & \textbf{3.688} & 3.753 & 4.174  \\
\bottomrule
\end{tabular}
\label{tab: depth sp}
\end{table}

\begin{table}[t]
\centering
\caption{\textbf{For Muon-Kimi-AdamW, with LayerNorm, the {CompleteP-style $\mu$P parameterization from Condition~\ref{condition: scale-invariant fl} ($k\ge2$)} transfers the optimal base learning rate across depths, and achieves lower loss than SP as the depth increases.}
This table reports the numerical results corresponding to Figure~\ref{figures: mup vs sp}, where the best validation loss for each depth is highlighted in \textbf{bold}.}
\renewcommand{\arraystretch}{1.3}
\vskip 0.05in
\begin{tabular}{cccccc} 
\toprule
$L/\log_2(\eta_{\mathrm{base}})$ & -9    & -8    & -7    & -6    & -5     \\ 
\hline
4     & 4.228 & 4.081 & \textbf{4.06}  & 4.075 & 4.098  \\
8     & 4.089 & 3.972 & \textbf{3.938} & 3.957 & 3.988  \\
16    & 4.01  & 3.886 & \textbf{3.85}  & 3.874 & 3.907  \\
32    & 3.96  & 3.826 & \textbf{3.8}   & 3.828 & 3.879  \\
64    & 3.917 & 3.771 & \textbf{3.747} & 3.796 & 3.942  \\
128   & 3.878 & 3.715 & \textbf{3.694} & 3.754 & 4.002  \\
256   & 3.878 & 3.697 & \textbf{3.678} & 3.761 & 3.964  \\
\bottomrule
\end{tabular}
\label{tab: depth cp}
\end{table}

\begin{table}[t]
\centering
\caption{\textbf{For Muon-Kimi-AdamW, with LayerNorm, the {Depth-$\mu$P-style parameterization from Condition~\ref{condition: one-layer} ($k=1$)} fails to transfer the optimal base learning rate across depths.}
This table reports the numerical results corresponding to Figure~\ref{figures: cp vs dmup muon style}, where the best validation loss for each depth is highlighted in \textbf{bold}.
}
\renewcommand{\arraystretch}{1.3}
\vskip 0.05in
\begin{tabular}{ccccccc} 
\toprule
$L/\log_2(\eta_{\mathrm{base}})$  & -9    & -8    & -7    & -6    & -5    & -4     \\
\hline
4         & 4.225 & 4.081 & \textbf{4.06}  & 4.067 & 4.095 & 4.971  \\
8         & 4.149 & 3.988 & \textbf{3.923} & 3.934 & 3.951 & 4.03   \\
16        & 4.198 & 3.963 & 3.871 & \textbf{3.857} & 3.88  & 3.922  \\
32        & 4.325 & 3.987 & 3.843 & \textbf{3.796} & 3.814 & 3.855  \\
64        & 4.517 & 4.052 & 3.825 & 3.728 & \textbf{3.725} & 3.768  \\
128       & 4.666 & 4.242 & 3.912 & 3.745 & \textbf{3.7}   & 3.746  \\
256       & 4.774 & 4.454 & 4.025 & 3.788 & \textbf{3.667} & 3.681   \\
\bottomrule
\end{tabular}
\label{tab: depth dmup}
\end{table}

\paragraph{{Additional results of simultaneous width-depth learning-rate transfer with LayerNorm.}}

{Figure~\ref{figures: mup vs sp}(f) reports the joint-scaling results for $(n,L)\in\{(256,4),(512,8),(1024,16),(2048,32)\}$, with complete numerical values in Tables~\ref{tab: joint lr sp} and~\ref{tab: joint lr cp}. As width and depth increase together, the optimal base learning rate under SP shifts toward smaller values. In contrast, CompleteP-style $\mu$P formulation keeps the optimal base learning rate approximately invariant and generally achieves lower validation loss at the larger joint scales.}

\begin{table}[t]

\centering
\caption{{\textbf{For Muon-Kimi-AdamW, SP fails to transfer the optimal base learning rate under simultaneous width-depth scaling.}
This table reports the numerical results corresponding to Figure~\ref{figures: mup vs sp}(f), where the best validation loss for each $(n,L)$ is highlighted in \textbf{bold}.}}
\renewcommand{\arraystretch}{1.3}
\vskip 0.05in
\begin{tabular}{cccccc}
\toprule
$(n,L)/\log_2(\eta_{\mathrm{base}})$ & -10 & -9 & -8 & -7 & -6 \\
\midrule
$(256,4)$   & 4.56  & 4.213 & 4.07  & \textbf{4.047} & 4.053 \\
$(512,8)$   & 3.985 & 3.794 & \textbf{3.72}  & 3.733 & 3.76  \\
$(1024,16)$ & 3.648 & 3.539 & \textbf{3.508} & 3.568 & 4.089 \\
$(2048,32)$ & 3.437 & \textbf{3.362} & 3.417 & 4.3   & 6.483 \\
\bottomrule
\end{tabular}
\label{tab: joint lr sp}
\end{table}

\begin{table}[t]

\centering
\caption{{\textbf{For Muon-Kimi-AdamW, the CompleteP-style $\mu$P parameterization from Condition~\ref{condition: scale-invariant fl} ($k\ge2$) approximately transfers the optimal base learning rate under simultaneous width-depth scaling and generally achieves lower loss than SP as the joint scale increases.}
This table reports the numerical results corresponding to Figure~\ref{figures: mup vs sp}(f), where the best validation loss for each $(n,L)$ is highlighted in \textbf{bold}.}}
\renewcommand{\arraystretch}{1.3}
\vskip 0.05in
\begin{tabular}{cccccc}
\toprule
$(n,L)/\log_2(\eta_{\mathrm{base}})$ & -10 & -9 & -8 & -7 & -6 \\
\midrule
$(256,4)$   & 4.552 & 4.223 & 4.079 & \textbf{4.06}  & 4.068 \\
$(512,8)$   & 4.147 & 3.853 & 3.734 & \textbf{3.707} & 3.736 \\
$(1024,16)$ & 3.885 & 3.59  & 3.483 & \textbf{3.46}  & 3.512 \\
$(2048,32)$ & 3.747 & 3.441 & \textbf{3.329} & 3.331 & 3.424 \\
\bottomrule
\end{tabular}
\label{tab: joint lr cp}
\end{table}

\paragraph{Additional results of depth-wise learning rate transfer without LayerNorm.}

Without LayerNorm, depth-wise base learning-rate transfer results for SP and the {CompleteP-style $\mu$P parameterization from Condition~\ref{condition: scale-invariant fl} ($k\ge2$)} are shown in Figure~\ref{figures: mup vs sp-hp-nonLN}, with complete numerical results in Tables~\ref{tab: depth sp noln} and~\ref{tab: depth cp noln}.

\begin{table}[t]
\centering
\caption{\textbf{For Muon-Kimi-AdamW, without LayerNorm, SP fails to preserve stable training.}
This table reports the numerical results corresponding to Figure~\ref{figures: mup vs sp-hp-nonLN}.
NaN data points indicate training instability, where the loss explodes.}
\renewcommand{\arraystretch}{1.3}
\vskip 0.05in
\begin{tabular}{cccccc} 
\toprule
$L/\log_2(\eta_{\mathrm{base}})$ & -13   & -12   & -11   & -10   & -9      \\ 
\midrule
4     & 7.318 & 6.394 & 5.784 & {5.169} & 13.77   \\
8     & 6.775 & 5.974 & 5.426 & {4.811} & NaN     \\
16    & 6.115 & 5.631 & 5.052 & {4.409} & 10.814  \\
32    & 5.809 & 5.328 & 4.706 & {4.233} & 4.282   \\
64    & 5.519 & 5.038 & 4.516 & {4.189} & 7.251   \\
128   & 5.316 & 4.896 & 4.484 & {4.313} & NaN     \\
256   & 5.179 & 4.867 & {4.678} & 5.752 & NaN     \\
\bottomrule
\end{tabular}
\label{tab: depth sp noln}
\end{table}

\begin{table}[t]
\centering
\caption{\textbf{For Muon-Kimi-AdamW, without LayerNorm, the {CompleteP-style $\mu$P parameterization from Condition~\ref{condition: scale-invariant fl} ($k\ge2$)} has stable runs and approximately transfers the optimal base learning rate at large depth $L\ge 32$.}
This table reports the numerical results corresponding to Figure~\ref{figures: mup vs sp-hp-nonLN}.
NaN data points indicate training instability, where the loss explodes.
The best validation loss for each depth is highlighted in \textbf{bold}.}
\renewcommand{\arraystretch}{1.3}
\vskip 0.05in
\begin{tabular}{cccccc} 
\toprule
$L/\log_2(\eta_{\mathrm{base}})$   & -11   & -10   & -9     & -8     & -7      \\ 
\hline
4   & 5.791 & \textbf{5.169} & 11.6   & NaN    & 345.85  \\
8   & 5.741 & \textbf{5.084} & 131.43 & NaN    & NaN     \\
16  & 5.73  & \textbf{5.059} & 8.732  & 246.45 & 122.99  \\
32  & 5.734 & 5.069 & 4.275  & 3.964  & \textbf{3.894}   \\
64  & 5.73  & 5.051 & 4.253  & 3.912  & \textbf{3.815}   \\
128 & 5.728 & 5.052 & 4.214  & 3.862  & \textbf{3.742}   \\
256 & 5.733 & 5.045 & 4.217  & 3.859  & \textbf{3.724}   \\
\bottomrule
\end{tabular}
\label{tab: depth cp noln}
\end{table}

\paragraph{{Controlled $k=3$ depth-wise learning-rate transfer.}}
{
The results are shown in Figure~\ref{figures: sp, cp, dmup Muon-kimi k3}.
Complete numerical results for the controlled $k=3$ experiment are reported in Tables~\ref{tab: depth sp k3},~\ref{tab: depth cp k3}, and~\ref{tab: depth dmup k3} for SP, CompleteP-style $\mu$P parameterization from Condition~\ref{condition: scale-invariant fl} ($k\ge2$), and the {Depth-$\mu$P-style parameterization from Condition~\ref{condition: one-layer} ($k=1$)}, respectively.}

\begin{table}[t]

\centering
\caption{\textbf{For Muon-Kimi-AdamW, with $k=3$ residual branches and LayerNorm, SP transfers the optimal base learning rate across the tested depths.}
This table reports the numerical results corresponding to Figure~\ref{figures: sp, cp, dmup Muon-kimi k3}, where the best validation loss for each depth is highlighted in \textbf{bold}.}
\renewcommand{\arraystretch}{1.3}
\vskip 0.05in
\begin{tabular}{cccccc}
\toprule
$L/\log_2(\eta_{\mathrm{base}})$ & -9 & -8 & -7 & -6 & -5 \\
\midrule
4   & 4.196 & 4.066 & \textbf{4.043} & 4.062 & 4.112 \\
8   & 4.086 & 3.964 & \textbf{3.934} & 3.960 & 4.014 \\
16  & 4.004 & 3.880 & \textbf{3.855} & 3.886 & 3.946 \\
32  & 3.926 & 3.800 & \textbf{3.777} & 3.810 & 3.880 \\
64  & 3.885 & 3.755 & \textbf{3.727} & 3.774 & 3.854 \\
128 & 3.846 & 3.699 & \textbf{3.680} & 3.734 & 3.946 \\
256 & 3.837 & 3.667 & \textbf{3.648} & 3.711 & 4.028 \\
\bottomrule
\end{tabular}
\label{tab: depth sp k3}
\end{table}

\begin{table}[t]

\centering
\caption{\textbf{For Muon-Kimi-AdamW with $k=3$ residual branches and LayerNorm, the CompleteP-style parameterization transfers the optimal base learning rate across the tested depths and achieves validation losses comparable to SP.}
This table reports the numerical results corresponding to Figure~\ref{figures: sp, cp, dmup Muon-kimi k3}, where the best validation loss for each depth is highlighted in \textbf{bold}.}
\renewcommand{\arraystretch}{1.3}
\vskip 0.05in
\begin{tabular}{cccccc}
\toprule
$L/\log_2(\eta_{\mathrm{base}})$ & -9 & -8 & -7 & -6 & -5 \\
\midrule
4   & 4.211 & 4.068 & \textbf{4.053} & 4.072 & 4.110 \\
8   & 4.093 & 3.975 & \textbf{3.938} & 3.961 & 4.008 \\
16  & 3.980 & 3.871 & \textbf{3.835} & 3.870 & 3.909 \\
32  & 3.925 & 3.800 & \textbf{3.779} & 3.810 & 3.850 \\
64  & 3.876 & 3.738 & \textbf{3.722} & 3.756 & 3.842 \\
128 & 3.849 & 3.704 & \textbf{3.685} & 3.739 & 3.895 \\
256 & 3.820 & 3.664 & \textbf{3.657} & 3.731 & 4.054 \\
\bottomrule
\end{tabular}
\label{tab: depth cp k3}
\end{table}

\begin{table}[t]

\centering
\caption{\textbf{For Muon-Kimi-AdamW, with $k=3$ residual branches and LayerNorm, the Depth-$\mu$P-style parameterization fails to transfer the optimal base learning rate across depths.}
This table reports the numerical results corresponding to Figure~\ref{figures: sp, cp, dmup Muon-kimi k3}, where the best validation loss for each depth is highlighted in \textbf{bold}.}
\renewcommand{\arraystretch}{1.3}
\vskip 0.05in
\begin{tabular}{ccccccc}
\toprule
$L/\log_2(\eta_{\mathrm{base}})$ & -9 & -8 & -7 & -6 & -5 & -4 \\
\midrule
4   & 4.212 & 4.071 & \textbf{4.053} & 4.073 & 4.113 & 4.759 \\
8   & 4.145 & 4.005 & \textbf{3.929} & 3.946 & 3.980 & 4.119 \\
16  & 4.155 & 3.942 & 3.851 & \textbf{3.832} & 3.875 & 3.923 \\
32  & 4.280 & 3.947 & 3.818 & \textbf{3.769} & 3.793 & 3.825 \\
64  & 4.445 & 4.004 & 3.809 & 3.720 & \textbf{3.718} & 3.759 \\
128 & 4.594 & 4.149 & 3.863 & 3.721 & \textbf{3.674} & 3.726 \\
256 & 4.720 & 4.366 & 3.952 & 3.749 & \textbf{3.644} & 3.675 \\
\bottomrule
\end{tabular}
\label{tab: depth dmup k3}
\end{table}

\paragraph{Additional results of width-wise weight decay transfer.}

Width-wise base weight-decay transfer results are shown in Figure~\ref{figures: mup vs sp wd}(a). 
Complete numerical results are presented in Table~\ref{tab: width sp wd} and Table~\ref{tab: width cp wd} for SP and the {CompleteP-style $\mu$P parameterization from Condition~\ref{condition: scale-invariant fl} ($k\ge2$)}, respectively.

\begin{table}[t]
\centering
\caption{\textbf{For Muon-Kimi-AdamW, SP fails to transfer the optimal base weight decay across widths.}
This table reports the numerical results corresponding to Figure~\ref{figures: mup vs sp wd}, where the best validation loss for each width is highlighted in \textbf{bold}.}
\renewcommand{\arraystretch}{1.3}
\vskip 0.05in
\begin{tabular}{cccccc} 
\toprule
$n/\log_2(\lambda_{\mathrm{base}})$ & -4    & -3    & -2    & -1    & 0      \\
\midrule
128            & 4.361 & \textbf{4.359} & 4.375 & 4.478 & 4.975  \\
256            & 4.026 & 4.019 & \textbf{4.013} & 4.061 & 4.327  \\
512            & 3.809 & 3.79  & \textbf{3.77}  & 3.789 & 3.971  \\
1024           & 3.642 & 3.614 & 3.585 & \textbf{3.582} & 3.724  \\
2048           & 3.574 & 3.514 & 3.467 & \textbf{3.45}  & 3.545  \\
4096           & 4.711 & 4.032 & 3.478 & \textbf{3.406} & 3.444   \\
\bottomrule
\end{tabular}
\label{tab: width sp wd}
\end{table}

\begin{table}[t]
\centering
\caption{\textbf{For Muon-Kimi-AdamW, the {CompleteP-style $\mu$P parameterization from Condition~\ref{condition: scale-invariant fl} ($k\ge2$)} approximately transfers the optimal base weight decay across widths.}
This table reports the numerical results corresponding to Figure~\ref{figures: mup vs sp wd}, where the best validation loss for each width is highlighted in \textbf{bold}.}
\renewcommand{\arraystretch}{1.3}
\vskip 0.05in
\begin{tabular}{cccccc} 
\toprule
$n/\log_2(\lambda_{\mathrm{base}})$ & -4    & -3    & -2    & -1    & 0      \\
\midrule
128            & 4.359 & \textbf{4.35}  & 4.352 & 4.406 & 4.74   \\
256            & 4.039 & \textbf{4.029} & 4.033 & 4.073 & 4.357  \\
512            & 3.785 & 3.767 & \textbf{3.766} & 3.815 & 4.109  \\
1024           & 3.622 & 3.617 & \textbf{3.602} & 3.659 & 3.966  \\
2048           & 3.493 & 3.48  & \textbf{3.475} & 3.539 & 3.847  \\
4096           & 3.412 & 3.398 & \textbf{3.392} & 3.463 & 3.782  \\
\bottomrule
\end{tabular}
\label{tab: width cp wd}
\end{table}

\paragraph{Additional results of depth-wise weight decay transfer with LayerNorm.}

With LayerNorm, depth-wise base weight-decay transfer results are shown in Figure~\ref{figures: mup vs sp wd}(b). 
Complete numerical results are presented in Table~\ref{tab: depth sp wd} and Table~\ref{tab: depth cp wd} for SP and the {CompleteP-style $\mu$P parameterization from Condition~\ref{condition: scale-invariant fl} ($k\ge2$)}, respectively.

\begin{table}[t]
\centering
\caption{
This table reports the numerical results corresponding to \textbf{depth scaling of Muon-Kimi-AdamW under SP} in Figure~\ref{figures: mup vs sp wd}, where the best validation loss for each depth is highlighted in \textbf{bold}.}
\renewcommand{\arraystretch}{1.3}
\vskip 0.05in
\begin{tabular}{cccccc} 
\toprule
$L/\log_2(\lambda_{\mathrm{base}})$ & -4    & -3    & -2    & -1    & 0      \\
\hline
4         & 4.03  & 4.021 & \textbf{4.017} & 4.062 & 4.328  \\
8         & 3.919 & 3.911 & \textbf{3.905} & 3.965 & 4.314  \\
16        & 3.842 & 3.832 & \textbf{3.832} & 3.905 & 4.259  \\
32        & 3.784 & \textbf{3.773} & 3.781 & 3.876 & 4.246  \\
64        & 3.734 & \textbf{3.726} & 3.735 & 3.862 & 4.244  \\
128       & 3.675 & \textbf{3.661} & 3.68  & 3.831 & 4.242  \\
256       & 3.665 & \textbf{3.651} & 3.68  & 3.847 & 4.289  \\
\bottomrule
\end{tabular}
\label{tab: depth sp wd}
\end{table}

\begin{table}[t]
\centering
\caption{\textbf{For Muon-Kimi-AdamW, the {CompleteP-style $\mu$P parameterization from Condition~\ref{condition: scale-invariant fl} ($k\ge2$)} approximately transfers the optimal base weight decay across depths, and achieves lower loss than SP as the depth increases.}
This table reports the numerical results corresponding to Figure~\ref{figures: mup vs sp wd}, where the best validation loss for each depth is highlighted in \textbf{bold}.}
\renewcommand{\arraystretch}{1.3}
\vskip 0.05in
\begin{tabular}{cccccc} 
\toprule
$L/\log_2(\lambda_{\mathrm{base}})$ & -4    & -3    & -2    & -1    & 0      \\
\hline
4         & 4.041 & 4.037 & \textbf{4.033} & 4.071 & 4.343  \\
8         & 3.905 & 3.896 & \textbf{3.89}  & 3.954 & 4.267  \\
16        & 3.834 & \textbf{3.824} & 3.825 & 3.906 & 4.232  \\
32        & 3.776 & \textbf{3.765} & 3.77  & 3.863 & 4.192  \\
64        & 3.699 & \textbf{3.686} & 3.693 & 3.807 & 4.143  \\
128       & 3.669 & \textbf{3.657} & 3.669 & 3.794 & 4.137  \\
256       & 3.626 & \textbf{3.61}  & 3.634 & 3.769 & 4.109  \\
\bottomrule
\end{tabular}
\label{tab: depth cp wd}
\end{table}

\clearpage
\newpage

\subsubsection{Additional Details of Muon-AdamW}
\label{app: Additional Details of Muon-AdamW}

\begin{figure*}[t]
\centering

\subfloat{
\includegraphics[height=0.185\textwidth]{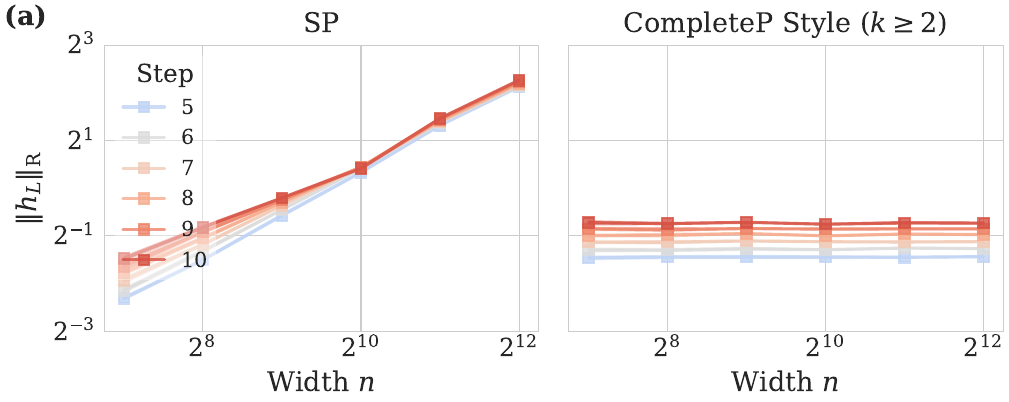}
\label{fig: sp_width_fl muon}
}%
\hskip 1ex
\subfloat{
\includegraphics[height=0.185\textwidth]{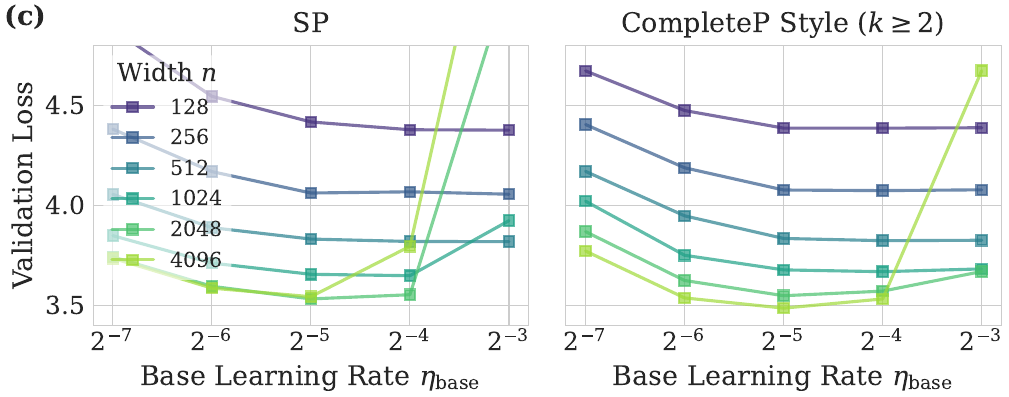}
\label{fig: sp_width_transfer muon}
}%

\vskip 1ex

\subfloat{
\includegraphics[height=0.185\textwidth]{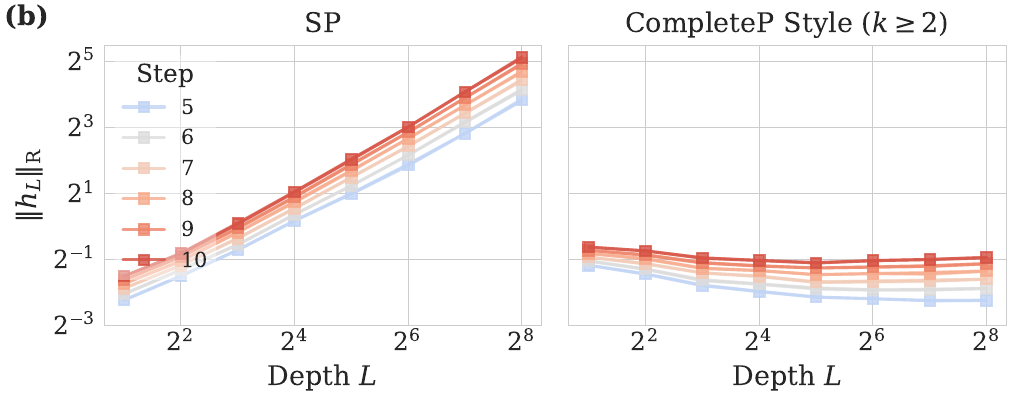}
\label{fig: sp_depth_fl muon}
}%
\hskip 1ex
\subfloat{
\includegraphics[height=0.185\textwidth]{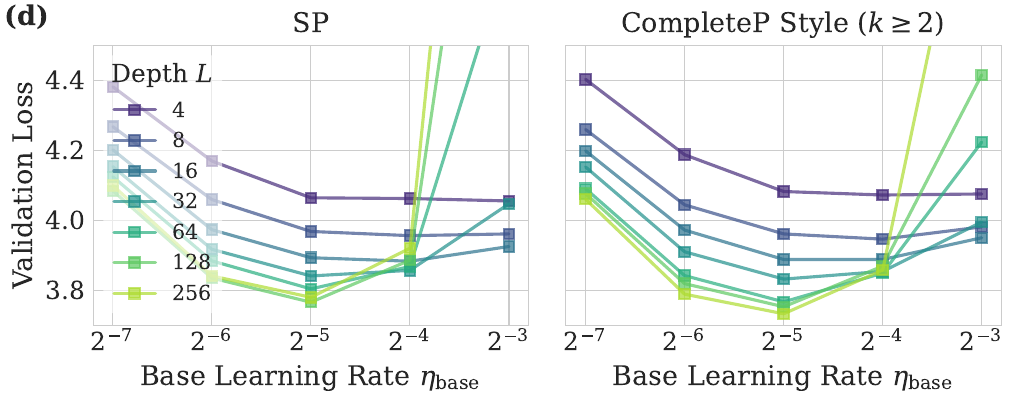}
\label{fig: sp_depth_transfer muon}
}%

\vskip 0.1in
\caption{
\textbf{Feature learning and HP transfer of Muon-AdamW under SP and the {CompleteP-style $\mu$P parameterization from Condition~\ref{condition: scale-invariant fl} ($k\ge2$)}.}
We train GPT-2 style models with Muon-AdamW using SP and the {CompleteP-style parameterization} (see Tables~\ref{tab: muon-wd mup} and~\ref{tab: adamw-mup}). The {CompleteP-style parameterization} maintains stable feature norms and enables robust HP transfer across both width and depth scaling, while generally achieving lower loss than SP as the model size increases. The detailed numerical values are provided in Appendix~\ref{app: Additional Details of Muon-AdamW}.
}

\label{figures: mup vs sp muon}
\end{figure*}

This subsection provides the complete Muon-AdamW results underlying Figures~\ref{figures: mup vs sp muon} and~\ref{figures: cp vs dmup muon style}. 
We report learning-rate transfer under width and depth scaling, including the comparison between SP, Condition~\ref{condition: scale-invariant fl} ($k\ge2$) and Condition~\ref{condition: one-layer} ($k=1$). 

\paragraph{Experimental setup.}
Following common practice~\citep{jordan6muon,modded_nanogpt_2024}, hidden matrix parameters are optimized by Muon with a base learning rate of $\eta_{\mathrm{base}}$, and a Nesterov-style momentum of 0.95, while all other parameters (e.g., embedding layer, LM head, all biases) are updated by AdamW with a base learning rate of $\eta_{\mathrm{base}}/10$, $\beta_1=0.9$, $\beta_2=0.95$, $\epsilon_{\mathrm{base}}=10^{-16}$. We do not use weight decay in all learning rate transfer experiments.

\paragraph{Additional results of width-wise learning rate transfer.}
Complete numerical results of base learning rate transferability across different widths are presented in Table~\ref{tab: width sp muon} and Table~\ref{tab: width cp muon} for SP and the {CompleteP-style $\mu$P parameterization from Condition~\ref{condition: scale-invariant fl} ($k\ge2$)}, respectively.

\begin{table}[t]
\centering
\caption{\textbf{For Muon-AdamW, SP fails to transfer the optimal base learning rate across widths.}
This table reports the numerical results corresponding to Figure~\ref{figures: mup vs sp muon}, where the best validation loss for each width is highlighted in \textbf{bold}.}
\renewcommand{\arraystretch}{1.3}
\vskip 0.05in
\begin{tabular}{ccccccc} 
\toprule
$n/\log_2(\eta_{\mathrm{base}})$ & -7   & -6    & -5    & -4    & -3    \\ 
\midrule
128  & 4.88  & 4.544 & 4.416 & 4.377 & \textbf{4.375}  \\
256  & 4.38  & 4.168 & 4.061 & 4.067 & \textbf{4.055}  \\
512  & 4.054 & 3.888 & 3.831 & 3.819 & \textbf{3.818}  \\
1024 & 3.848 & 3.711 & 3.655 & \textbf{3.648} & 3.923  \\
2048 & 3.742 & 3.595 & \textbf{3.532} & 3.553 & 5.137  \\
4096 & 3.735 & 3.586 & \textbf{3.543} & 3.795 & 5.992  \\
\bottomrule
\end{tabular}
\label{tab: width sp muon}
\end{table}

\begin{table}[t]
\centering
\caption{\textbf{For Muon-AdamW, the {CompleteP-style $\mu$P parameterization from Condition~\ref{condition: scale-invariant fl} ($k\ge2$)} approximately transfers the optimal base learning rate across widths, and achieves lower loss than SP as the width increases.}
This table reports the numerical results corresponding to Figure~\ref{figures: mup vs sp muon}, where the best validation loss for each width is highlighted in \textbf{bold}.}
\renewcommand{\arraystretch}{1.3}
\vskip 0.05in
\begin{tabular}{ccccccc} 
\toprule
$n/\log_2(\eta_{\mathrm{base}})$ & -7   & -6    & -5    & -4    & -3    \\ 
\midrule
128  & 4.671 & 4.473 & \textbf{4.385} & 4.385 & 4.387  \\
256  & 4.403 & 4.187 & 4.076 & \textbf{4.073} & 4.077  \\
512  & 4.169 & 3.947 & 3.834 & \textbf{3.823} & 3.825  \\
1024 & 4.02  & 3.75  & 3.677 & \textbf{3.668} & 3.682  \\
2048 & 3.868 & 3.624 & \textbf{3.548} & 3.571 & 3.669  \\
4096 & 3.771 & 3.537 & \textbf{3.486} & 3.532 & 4.672 \\
\bottomrule
\end{tabular}
\label{tab: width cp muon}
\end{table}

\paragraph{Additional results of depth-wise learning rate transfer.}
Complete numerical results of base learning rate transferability across different depths are presented in Table~\ref{tab: depth sp muon}, Table~\ref{tab: depth cp muon}, and Table~\ref{tab: depth dmup muon} for SP, the {CompleteP-style $\mu$P parameterization from Condition~\ref{condition: scale-invariant fl} ($k\ge2$)}, and the {Depth-$\mu$P-style parameterization from Condition~\ref{condition: one-layer} ($k = 1$)}, respectively.

\begin{table}[t]
\centering
\caption{\textbf{For Muon-AdamW, SP shifts the optimal base learning rate across depths.}
This table reports the numerical results corresponding to Figure~\ref{figures: mup vs sp muon}, where the best validation loss for each depth is highlighted in \textbf{bold}.}
\renewcommand{\arraystretch}{1.3}
\vskip 0.05in
\begin{tabular}{cccccc} 
\toprule
$L/\log_2(\eta_{\mathrm{base}})$ & -7    & -6    & -5    & -4    & -3     \\ 
\hline
4         & 4.381 & 4.169 & 4.064 & 4.062 & \textbf{4.055}  \\
8         & 4.267 & 4.059 & 3.968 & \textbf{3.956} & 3.961  \\
16        & 4.201 & 3.973 & 3.893 & \textbf{3.883} & 3.925  \\
32        & 4.153 & 3.917 & \textbf{3.841} & 3.857 & 4.047  \\
64        & 4.124 & 3.884 & \textbf{3.804} & 3.868 & 4.772  \\
128       & 4.084 & 3.835 & \textbf{3.766} & 3.886 & 5.798  \\
256       & 4.099 & 3.841 & \textbf{3.78}  & 3.92  & 6.705   \\
\bottomrule
\end{tabular}
\label{tab: depth sp muon}
\end{table}

\begin{table}[t]
\centering
\caption{\textbf{For Muon-AdamW, the {CompleteP-style $\mu$P parameterization from Condition~\ref{condition: scale-invariant fl} ($k\ge2$)} approximately transfers the optimal base learning rate across depths, and achieves lower loss than SP as the depth increases.}
This table reports the numerical results corresponding to Figure~\ref{figures: mup vs sp muon}, where the best validation loss for each depth is highlighted in \textbf{bold}.}
\renewcommand{\arraystretch}{1.3}
\vskip 0.05in
\begin{tabular}{cccccc} 
\toprule
$L/\log_2(\eta_{\mathrm{base}})$ & -7   & -6    & -5    & -4    & -3    \\ 
\hline
4         & 4.402 & 4.187 & 4.082 & \textbf{4.072} & 4.075  \\
8         & 4.259 & 4.044 & 3.961 & \textbf{3.946} & 3.981  \\
16        & 4.198 & 3.972 & \textbf{3.888} & 3.888 & 3.95   \\
32        & 4.152 & 3.91  & \textbf{3.832} & 3.854 & 3.995  \\
64        & 4.09  & 3.842 & \textbf{3.767} & 3.849 & 4.223  \\
128       & 4.076 & 3.819 & \textbf{3.753} & 3.869 & 4.415  \\
256       & 4.06  & 3.789 & \textbf{3.733} & 3.858 & 5.201  \\
\bottomrule
\end{tabular}
\label{tab: depth cp muon}
\end{table}

\begin{table}[t]
\centering
\caption{\textbf{For Muon-AdamW, the {Depth-$\mu$P-style parameterization from Condition~\ref{condition: one-layer} ($k=1$)} fails to transfer the optimal base learning rate across depths.}
This table reports the numerical results corresponding to Figure~\ref{figures: cp vs dmup muon style}, where the best validation loss for each depth is highlighted in \textbf{bold}.
}
\renewcommand{\arraystretch}{1.3}
\vskip 0.05in
\begin{tabular}{ccccccc} 
\toprule
$L/\log_2(\eta_{\mathrm{base}})$  & -7    & -6    & -5    & -4    & -3    & -2     \\
\hline
4         & 4.402 & 4.189 & 4.078 & 4.07  & 4.073 & \textbf{4.068}  \\
8         & 4.323 & 4.067 & 3.97  & \textbf{3.937} & 3.96  & 3.983  \\
16        & 4.368 & 4.052 & 3.909 & \textbf{3.862} & 3.881 & 3.947  \\
32        & 4.477 & 4.073 & 3.884 & \textbf{3.809} & 3.823 & 3.902  \\
64        & 4.649 & 4.12  & 3.874 & 3.754 & \textbf{3.735} & 3.834  \\
128       & 4.793 & 4.262 & 3.946 & 3.78  & \textbf{3.722} & 3.795  \\
256       & 4.887 & 4.447 & 4.031 & 3.815 & \textbf{3.707} & 3.72   \\
\bottomrule
\end{tabular}
\label{tab: depth dmup muon}
\end{table}

\clearpage
\newpage

\subsubsection{Additional Details of Shampoo-AdamW}
\label{app: Additional Details of Shampoo-AdamW}

\begin{figure*}[t]
\centering

\subfloat{
\includegraphics[height=0.185\textwidth]{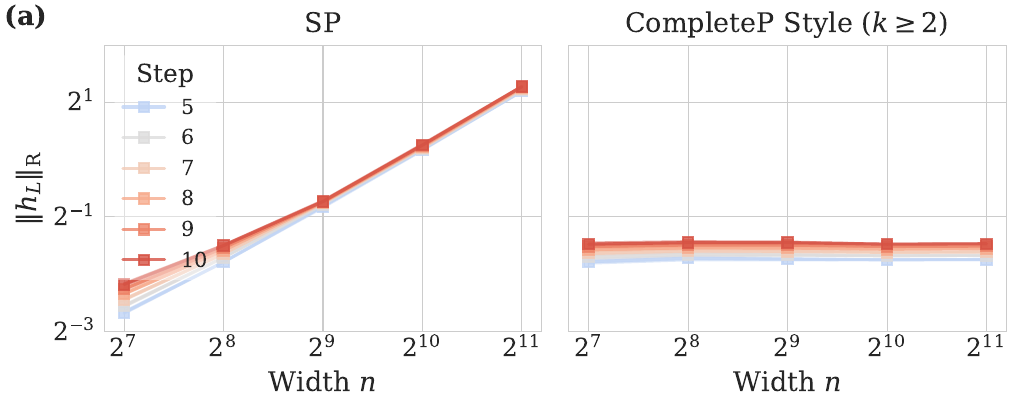}
\label{fig: sp_width_fl shampoo}
}%
\hskip 1ex
\subfloat{
\includegraphics[height=0.185\textwidth]{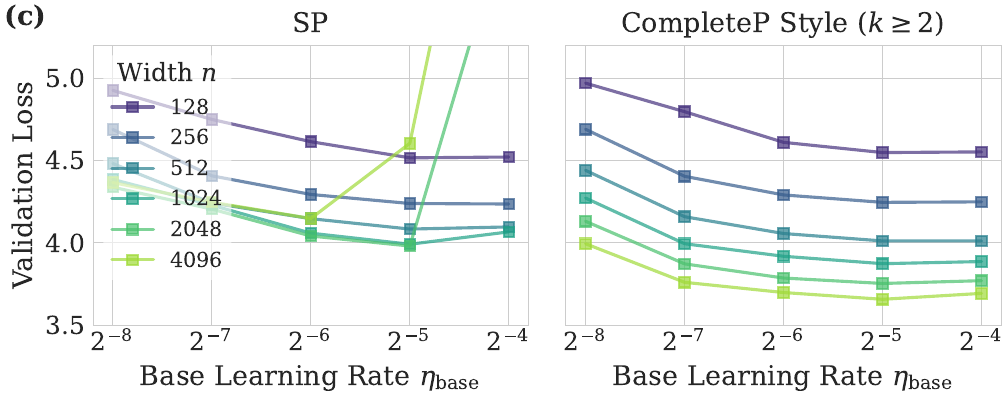}
\label{fig: sp_width_transfer shampoo}
}%

\vskip 1ex

\subfloat{
\includegraphics[height=0.185\textwidth]{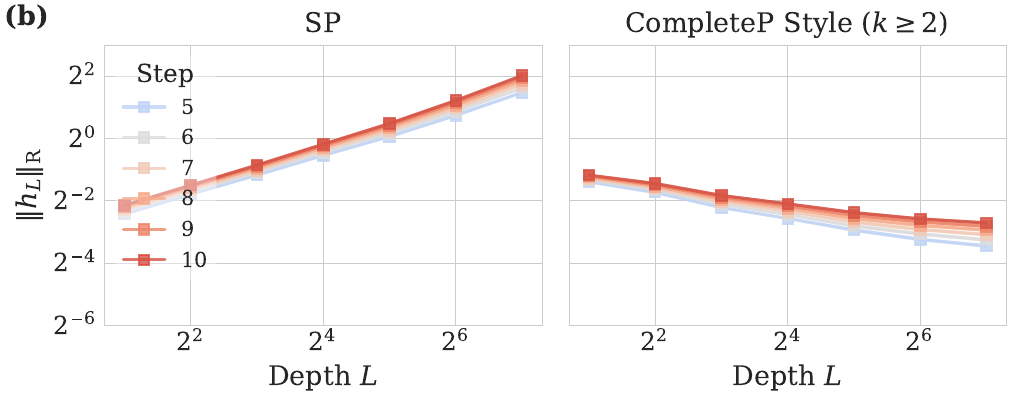}
\label{fig: sp_depth_fl shampoo}
}%
\hskip 1ex
\subfloat{
\includegraphics[height=0.185\textwidth]{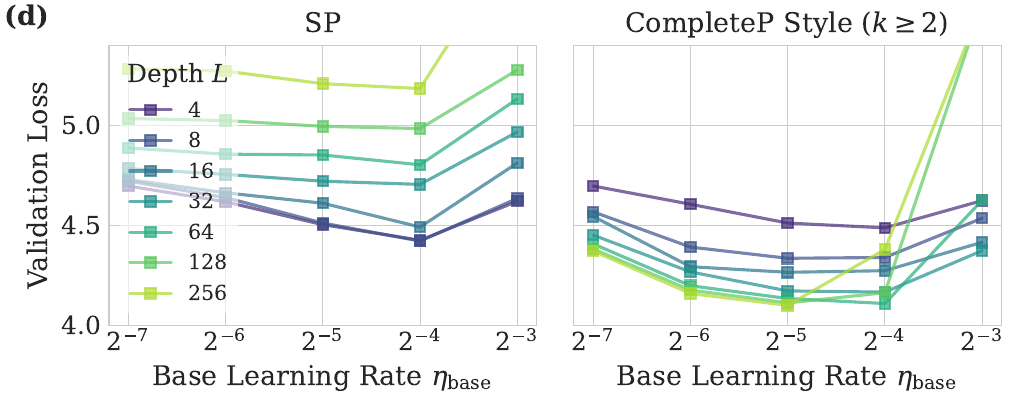}
\label{fig: sp_depth_transfer shampoo}
}%

\vskip 0.1in
\caption{
\textbf{Feature learning and HP transfer of Shampoo-AdamW under SP and the {CompleteP-style $\mu$P parameterization from Condition~\ref{condition: scale-invariant fl} ($k\ge2$)}.}
We train GPT-2 style models with Shampoo-AdamW using SP and the {CompleteP-style parameterization} (see Tables~\ref{tab: muon-wd mup} and~\ref{tab: adamw-mup}). The {CompleteP-style parameterization} maintains stable feature norms and enables robust HP transfer across both width and depth scaling, while generally achieving lower loss than SP as the model size increases. The detailed numerical values are provided in Appendix~\ref{app: Additional Details of Shampoo-AdamW}.
}

\label{figures: mup vs sp shampoo}
\end{figure*}

\begin{figure*}[t]
\centering

\includegraphics[height=0.25\textwidth]{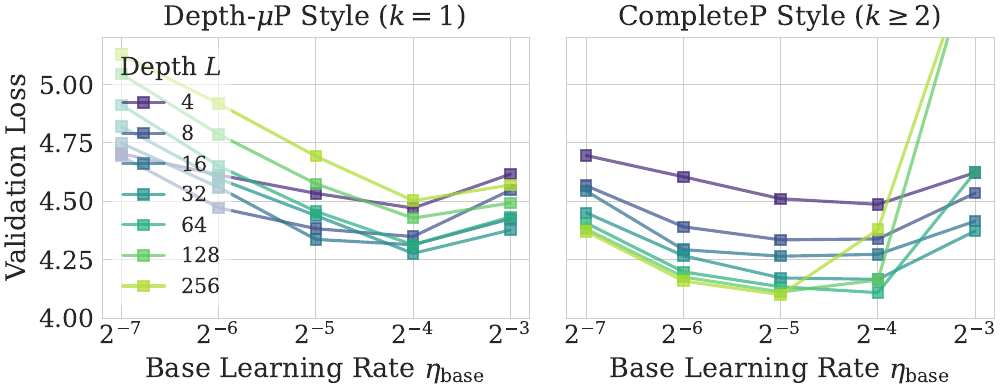}

\vskip 0.05in
\caption{
\textbf{Validating the role of residual block depth $k$.}
We compare two {$\mu$P parameterizations} for GPT-2-style models trained with Shampoo-AdamW: {the Depth-$\mu$P-style formulation from Condition~\ref{condition: one-layer} ($k=1$) and the CompleteP-style formulation from Condition~\ref{condition: scale-invariant fl} ($k\ge 2$)}.
Condition~\ref{condition: scale-invariant fl} yields more stable HP transfer and lower validation loss, empirically supporting it as the appropriate {depth-scaling condition} for architectures with multi-transformation residual branches.
The numerical values are provided in Appendix~\ref{app: Additional Details of Shampoo-AdamW}.
}

\label{figures: cp vs dmup shampoo}
\end{figure*}

This subsection provides the complete Shampoo-AdamW results underlying Figures~\ref{figures: mup vs sp shampoo} and~\ref{figures: cp vs dmup shampoo}. 
We report learning-rate transfer under width and depth scaling, including the comparison between SP, Condition~\ref{condition: scale-invariant fl} ($k\ge2$) and Condition~\ref{condition: one-layer} ($k=1$).

\paragraph{Experimental setup.}
Following common practice~\citep{modded_nanogpt_2024}, hidden matrix parameters are optimized by Shampoo with a base learning rate of $\eta_{\mathrm{base}}$, $\beta_1=0.95$, $\beta_2=0.95$, a shampoo precondition frequency of $1$, a maximal precondition dimension of $20000$, $\epsilon_{\mathrm{base}}=10^{-8}$ for width scaling, and $\epsilon_{\mathrm{base}}=10^{-5}$ for depth scaling, while all other parameters (e.g., embedding layer, LM head, all biases) are updated by AdamW with a base learning rate of $2\eta_{\mathrm{base}}$, $\beta_1=0.9$, $\beta_2=0.95$, and $\epsilon_{\mathrm{base}}=10^{-16}$. We do not use weight decay in all learning rate transfer experiments.

\paragraph{Additional results of width-wise learning rate transfer.}
Complete numerical results of base learning rate transferability across different widths are presented in Table~\ref{tab: width sp shampoo} and Table~\ref{tab: width cp shampoo} for SP and the {CompleteP-style $\mu$P parameterization from Condition~\ref{condition: scale-invariant fl} ($k\ge2$)}, respectively.

\begin{table}[t]
\centering
\caption{\textbf{For Shampoo-AdamW, SP fails to transfer the optimal base learning rate across widths.}
This table reports the numerical results corresponding to Figure~\ref{figures: mup vs sp shampoo}, where the best validation loss for each width is highlighted in \textbf{bold}.}
\renewcommand{\arraystretch}{1.3}
\vskip 0.05in
\begin{tabular}{ccccccc} 
\toprule
$n/\log_2(\eta_{\mathrm{base}})$ & -8   & -7    & -6    & -5    & -4    \\ 
\midrule
128            & 4.925 & 4.749 & 4.614 & \textbf{4.516} & 4.521  \\
256            & 4.688 & 4.407 & 4.294 & 4.239 & \textbf{4.236}  \\
512            & 4.481 & 4.242 & 4.147 & \textbf{4.084} & 4.097  \\
1024           & 4.385 & 4.229 & 4.059 & \textbf{3.992} & 4.067  \\
2048           & 4.338 & 4.205 & 4.042 & \textbf{3.981} & 6.002  \\
4096           & 4.366 & 4.247 & \textbf{4.148} & 4.604 & 7.476 \\
\bottomrule
\end{tabular}
\label{tab: width sp shampoo}
\end{table}

\begin{table}[t]
\centering
\caption{\textbf{For Shampoo-AdamW, the {CompleteP-style $\mu$P parameterization from Condition~\ref{condition: scale-invariant fl} ($k\ge2$)} transfers the optimal base learning rate across widths, and achieves lower loss than SP as the width increases.}
This table reports the numerical results corresponding to Figure~\ref{figures: mup vs sp shampoo}, where the best validation loss for each width is highlighted in \textbf{bold}.}
\renewcommand{\arraystretch}{1.3}
\vskip 0.05in
\begin{tabular}{ccccccc} 
\toprule
$n/\log_2(\eta_{\mathrm{base}})$ & -8   & -7    & -6    & -5    & -4    \\ 
\midrule
128            & 4.969 & 4.796 & 4.61  & \textbf{4.548} & 4.552  \\
256            & 4.689 & 4.404 & 4.291 & \textbf{4.246} & 4.249  \\
512            & 4.44  & 4.159 & 4.057 & \textbf{4.012} & 4.012  \\
1024           & 4.272 & 3.995 & 3.919 & \textbf{3.874} & 3.887  \\
2048           & 4.131 & 3.873 & 3.787 & \textbf{3.754} & 3.771  \\
4096           & 3.995 & 3.761 & 3.699 & \textbf{3.658} & 3.694 \\
\bottomrule
\end{tabular}
\label{tab: width cp shampoo}
\end{table}

\paragraph{Additional results of depth-wise learning rate transfer.}
Complete numerical results of base learning rate transferability across different depths are presented in Table~\ref{tab: depth sp shampoo}, Table~\ref{tab: depth cp shampoo}, and Table~\ref{tab: depth dmup shampoo} for SP, the {CompleteP-style $\mu$P parameterization from Condition~\ref{condition: scale-invariant fl} ($k\ge2$)}, and the {Depth-$\mu$P-style parameterization from Condition~\ref{condition: one-layer} ($k = 1$)}, respectively.

\begin{table}[t]
\centering
\caption{
\textbf{For Shampoo-AdamW, SP yields increasing validation loss under depth scaling.}
This table reports the numerical results corresponding to Figure~\ref{figures: mup vs sp shampoo}.}
\renewcommand{\arraystretch}{1.3}
\vskip 0.05in
\begin{tabular}{cccccc} 
\toprule
$L/\log_2(\eta_{\mathrm{base}})$ & -7    & -6    & -5    & -4    & -3     \\ 
\hline
4         & 4.696 & 4.617 & 4.503 & 4.425 & 4.62   \\
8         & 4.721 & 4.638 & 4.511 & 4.421 & 4.634  \\
16        & 4.728 & 4.662 & 4.61  & 4.491 & 4.811  \\
32        & 4.783 & 4.754 & 4.72  & 4.704 & 4.966  \\
64        & 4.886 & 4.855 & 4.851 & 4.802 & 5.131  \\
128       & 5.033 & 5.023 & 4.994 & 4.983 & 5.276  \\
256       & 5.281 & 5.27  & 5.207 & 5.183 & 5.945   \\
\bottomrule
\end{tabular}
\label{tab: depth sp shampoo}
\end{table}

\begin{table}[t]
\centering
\caption{\textbf{For Shampoo-AdamW, the {CompleteP-style $\mu$P parameterization from Condition~\ref{condition: scale-invariant fl} ($k\ge2$)} approximately transfers the optimal base learning rate across depths, and achieves lower loss than SP as the depth increases.}
This table reports the numerical results corresponding to Figure~\ref{figures: mup vs sp shampoo}, where the best validation loss for each depth is highlighted in \textbf{bold}.}
\renewcommand{\arraystretch}{1.3}
\vskip 0.05in
\begin{tabular}{cccccc} 
\toprule
$L/\log_2(\eta_{\mathrm{base}})$ & -7   & -6    & -5    & -4    & -3    \\ 
\hline
4         & 4.696 & 4.605 & 4.511 & \textbf{4.487} & 4.623  \\
8         & 4.567 & 4.391 & \textbf{4.335} & 4.339 & 4.535  \\
16        & 4.544 & 4.293 & \textbf{4.265} & 4.273 & 4.415  \\
32        & 4.451 & 4.267 & 4.172 & \textbf{4.166} & 4.373  \\
64        & 4.407 & 4.198 & 4.134 & \textbf{4.109} & 4.625  \\
128       & 4.381 & 4.176 & \textbf{4.112} & 4.163 & 5.603  \\
256       & 4.371 & 4.159 & \textbf{4.1}   & 4.379 & 5.601  \\
\bottomrule
\end{tabular}
\label{tab: depth cp shampoo}
\end{table}

\begin{table}[t]
\centering
\caption{
\textbf{For Shampoo-AdamW, the {Depth-$\mu$P-style parameterization from Condition~\ref{condition: one-layer} ($k=1$)} gives weaker depth scaling and higher validation loss than the {CompleteP-style $\mu$P parameterization from Condition~\ref{condition: scale-invariant fl} ($k\ge2$)}.}
This table reports the numerical results corresponding to Figure~\ref{figures: cp vs dmup shampoo}.}
\renewcommand{\arraystretch}{1.3}
\vskip 0.05in
\begin{tabular}{cccccc} 
\toprule
$L/\log_2(\eta_{\mathrm{base}})$ & -7    & -6    & -5    & -4    & -3     \\ 
\hline
4         & 4.704 & 4.613 & 4.534 & 4.471 & 4.616  \\
8         & 4.693 & 4.472 & 4.383 & 4.349 & 4.548  \\
16        & 4.749 & 4.56  & 4.337 & 4.314 & 4.425  \\
32        & 4.82  & 4.597 & 4.439 & 4.277 & 4.377  \\
64        & 4.913 & 4.651 & 4.457 & 4.314 & 4.433  \\
128       & 5.045 & 4.787 & 4.575 & 4.428 & 4.493  \\
256       & 5.128 & 4.918 & 4.693 & 4.503 & 4.57   \\
\bottomrule
\end{tabular}
\label{tab: depth dmup shampoo}
\end{table}

\clearpage
\newpage

\subsubsection{Additional Details of Sophia}
\label{app: Additional Details of Sophia}

\begin{figure*}[t]
\centering

\subfloat{
\includegraphics[height=0.185\textwidth]{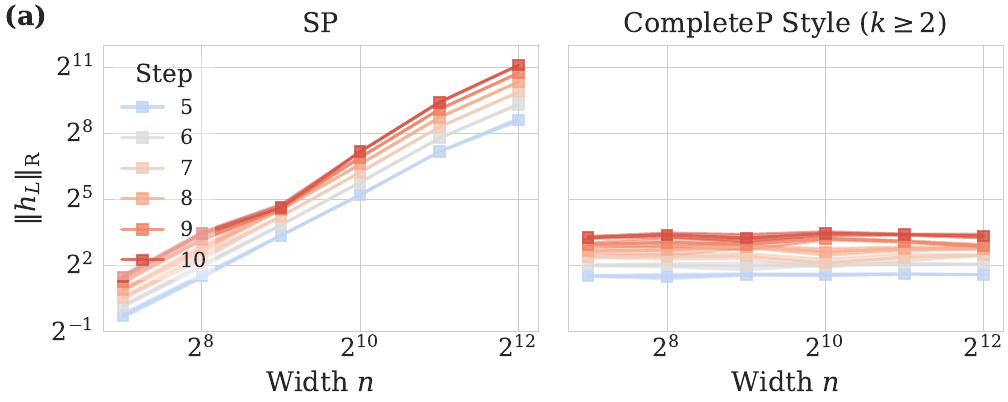}
\label{fig: sp_width_fl sophia}
}%
\hskip 1ex
\subfloat{
\includegraphics[height=0.185\textwidth]{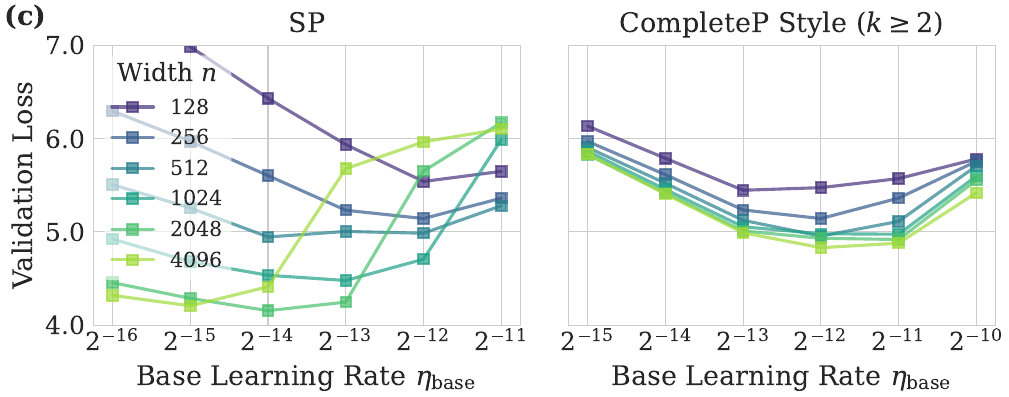}
\label{fig: sp_width_transfer sophia}
}%

\vskip 1ex

\subfloat{
\includegraphics[height=0.185\textwidth]{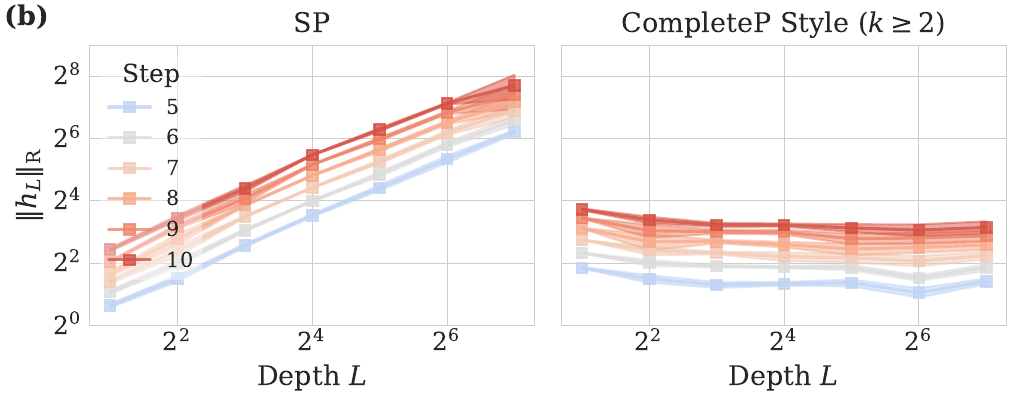}
\label{fig: sp_depth_fl sophia}
}%
\hskip 1ex
\subfloat{
\includegraphics[height=0.185\textwidth]{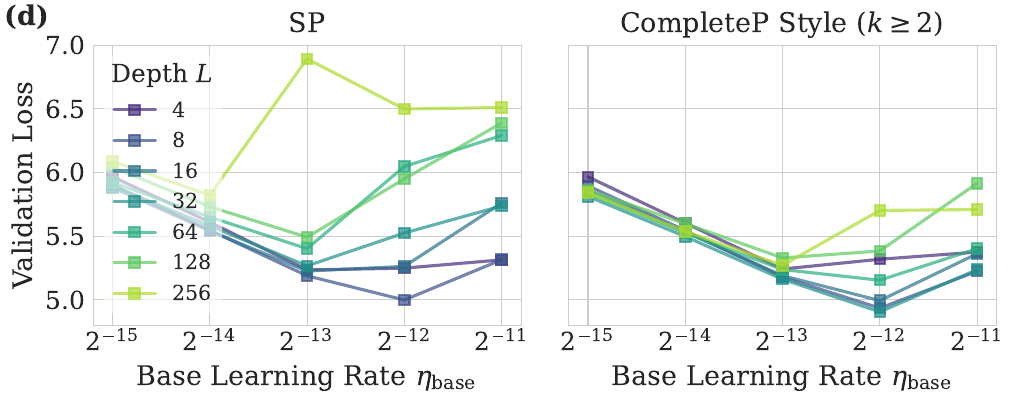}
\label{fig: sp_depth_transfer sophia}
}%

\vskip 0.1in
\caption{
\textbf{Feature learning and HP transfer of Sophia under SP and the {CompleteP-style $\mu$P parameterization from Condition~\ref{condition: scale-invariant fl} ($k\ge2$)}.}
We train GPT-2 style models with Sophia using SP and the {CompleteP-style parameterization} (see Table~\ref{tab: adamw-mup}). The {CompleteP-style parameterization} maintains stable feature norms and enables robust HP transfer across both width and depth scaling. The detailed numerical values are provided in Appendix~\ref{app: Additional Details of Sophia}.
}

\label{figures: mup vs sp sophia}
\end{figure*}

\begin{figure*}[t]
\centering

\includegraphics[height=0.25\textwidth]{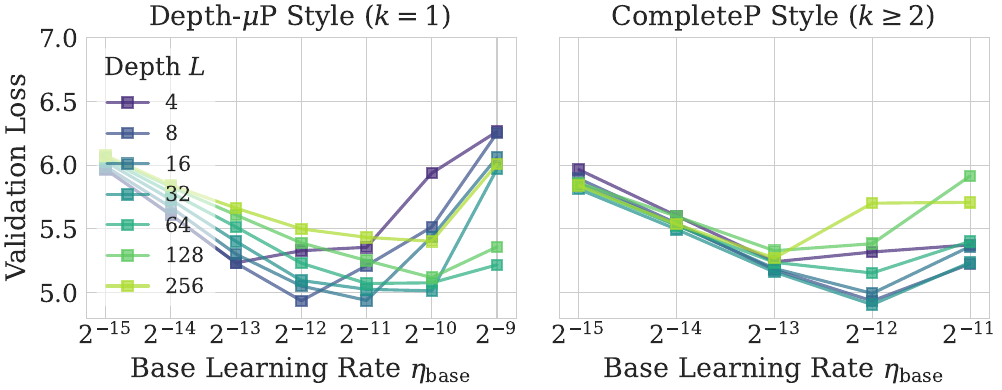}

\vskip 0.05in
\caption{
\textbf{Validating the role of residual block depth $k$.}
We compare two {$\mu$P parameterizations} for GPT-2-style models trained with Sophia: {the Depth-$\mu$P-style formulation from Condition~\ref{condition: one-layer} ($k=1$) and the CompleteP-style formulation from Condition~\ref{condition: scale-invariant fl} ($k\ge 2$)}.
Condition~\ref{condition: scale-invariant fl} yields more stable HP transfer, empirically supporting it as the appropriate {depth-scaling condition} for architectures with multi-transformation residual branches.
The numerical values are provided in Appendix~\ref{app: Additional Details of Sophia}.
}

\label{figures: cp vs dmup sophia}
\end{figure*}

This subsection provides the complete Sophia results underlying Figures~\ref{figures: mup vs sp sophia} and~\ref{figures: cp vs dmup sophia}. 
We report learning-rate transfer under width and depth scaling, including the comparison between SP, Condition~\ref{condition: scale-invariant fl} ($k\ge2$) and Condition~\ref{condition: one-layer} ($k=1$). 

\paragraph{Experimental setup.}
Following common practice~\citep{DBLP:conf/iclr/Liu0HL024-sophia}, parameters are updated by Sophia with a base learning rate of $\eta_{\mathrm{base}}$, $\beta_1=0.965$, $\beta_2=0.99$, $\rho=0.05$, and a Hessian update frequency of $10$.

\paragraph{Additional results of width-wise learning rate transfer.}
{Complete} numerical results of base learning rate transferability across different widths are presented in Table~\ref{tab: width sp sophia} and Table~\ref{tab: width cp sophia} for SP and the {CompleteP-style $\mu$P parameterization from Condition~\ref{condition: scale-invariant fl} ($k\ge2$)}, respectively.

\begin{table}[t]
\centering
\caption{\textbf{For Sophia, SP fails to transfer the optimal base learning rate across widths.}
This table reports the numerical results corresponding to Figure~\ref{figures: mup vs sp sophia}, where the best validation loss for each width is highlighted in \textbf{bold}.}
\renewcommand{\arraystretch}{1.3}
\vskip 0.05in
\begin{tabular}{ccccccc} 
\toprule
$n/\log_2(\eta_{\mathrm{base}})$ & -16   & -15    & -14    & -13    & -12  & -11   \\ 
\midrule
128         & 7.485 & 6.982 & 6.43  & 5.935 & \textbf{5.54}  & 5.648  \\
256         & 6.293 & 5.97  & 5.603 & 5.232 & \textbf{5.143} & 5.361  \\
512         & 5.506 & 5.257 & {\textbf{4.946}} & 5.005 & {4.986} & 5.28   \\
1024        & 4.924 & 4.682 & 4.536 & \textbf{4.48}  & 4.708 & 5.988  \\
2048        & 4.457 & 4.289 & \textbf{4.157} & 4.249 & 5.651 & 6.169  \\
4096        & 4.321 & \textbf{4.209} & 4.413 & 5.677 & 5.964 & 6.101 \\
\bottomrule
\end{tabular}
\label{tab: width sp sophia}
\end{table}

\begin{table}[t]
\centering
\caption{\textbf{For Sophia, the {CompleteP-style $\mu$P parameterization from Condition~\ref{condition: scale-invariant fl} ($k\ge2$)} approximately transfers the optimal base learning rate across widths larger than $256$.}
This table reports the numerical results corresponding to Figure~\ref{figures: mup vs sp sophia}, where the best validation loss for each width is highlighted in \textbf{bold}.}
\renewcommand{\arraystretch}{1.3}
\vskip 0.05in
\begin{tabular}{ccccccc} 
\toprule
$n/\log_2(\eta_{\mathrm{base}})$ & -15   & -14    & -13    & -12    & -11 & -10   \\ 
\midrule
128         & 6.133 & 5.788 & \textbf{5.446} & 5.474 & 5.571 & 5.781  \\
256         & 5.969 & 5.613 & 5.235 & \textbf{5.143} & 5.362 & 5.756  \\
512         & 5.903 & 5.521 & 5.124 & \textbf{4.955} & 5.113 & 5.707  \\
1024        & 5.863 & 5.455 & 5.055 & 4.98  & \textbf{4.976} & 5.598  \\
2048        & 5.827 & 5.419 & 5.01  & 4.932 & \textbf{4.919} & 5.555  \\
4096        & 5.833 & 5.406 & 4.991 & \textbf{4.831} & 4.88  & 5.419 \\
\bottomrule
\end{tabular}
\label{tab: width cp sophia}
\end{table}

\paragraph{Additional results of depth-wise learning rate transfer.}
Complete numerical results of base learning rate transferability across different depths are presented in Table~\ref{tab: depth sp sophia}, Table~\ref{tab: depth cp sophia}, and Table~\ref{tab: depth dmup sophia} for SP, the {CompleteP-style $\mu$P parameterization from Condition~\ref{condition: scale-invariant fl} ($k\ge2$)}, and the {Depth-$\mu$P-style parameterization from Condition~\ref{condition: one-layer} ($k = 1$)}, respectively.

\begin{table}[t]
\centering
\caption{
\textbf{For Sophia, SP yields increasing validation loss under depth scaling.}
This table reports the numerical results corresponding to Figure~\ref{figures: mup vs sp sophia}.}
\renewcommand{\arraystretch}{1.3}
\vskip 0.05in
\begin{tabular}{cccccc} 
\toprule
$L/\log_2(\eta_{\mathrm{base}})$ & -15    & -14    & -13    & -12    & -11     \\ 
\hline
4         & 5.967 & 5.61  & 5.235 & 5.249 & 5.314  \\
8         & 5.902 & 5.547 & 5.189 & 4.999 & 5.315  \\
16        & 5.881 & 5.541 & 5.226 & 5.263 & 5.76   \\
32        & 5.894 & 5.58  & 5.267 & 5.525 & 5.737  \\
64        & 5.922 & 5.65  & 5.402 & 6.047 & 6.29   \\
128       & 6.041 & 5.731 & 5.492 & 5.95  & 6.389  \\
256       & 6.087 & 5.822 & 6.892 & 6.498 & 6.51   \\
\bottomrule
\end{tabular}
\label{tab: depth sp sophia}
\end{table}

\begin{table}[t]
\centering
\caption{\textbf{For Sophia, the {CompleteP-style $\mu$P parameterization from Condition~\ref{condition: scale-invariant fl} ($k\ge2$)} approximately transfers the optimal base learning rate across depths, and achieves lower loss than SP as the depth increases.}
This table reports the numerical results corresponding to Figure~\ref{figures: mup vs sp sophia}, where the best validation loss for each depth is highlighted in \textbf{bold}.}
\renewcommand{\arraystretch}{1.3}
\vskip 0.05in
\begin{tabular}{cccccc} 
\toprule
$L/\log_2(\eta_{\mathrm{base}})$ & -15   & -14    & -13    & -12    & -11    \\ 
\hline
4         & 5.966 & 5.603 & \textbf{5.242} & 5.319 & 5.375  \\
8         & 5.895 & 5.544 & 5.178 & \textbf{4.935} & 5.225  \\
16        & 5.861 & 5.532 & 5.19  & \textbf{4.994} & 5.362  \\
32        & 5.812 & 5.497 & 5.162 & \textbf{4.906} & 5.237  \\
64        & 5.837 & 5.527 & 5.238 & \textbf{5.154} & 5.402  \\
128       & 5.861 & 5.599 & \textbf{5.329} & 5.384 & 5.915  \\
256       & 5.841 & 5.537 & \textbf{5.271} & 5.702 & 5.709  \\
\bottomrule
\end{tabular}
\label{tab: depth cp sophia}
\end{table}

\begin{table}[t]
\centering
\caption{\textbf{For Sophia, the {Depth-$\mu$P-style parameterization from Condition~\ref{condition: one-layer} ($k=1$)} fails to transfer the optimal base learning rate across depths.}
This table reports the numerical results corresponding to Figure~\ref{figures: cp vs dmup sophia}, where the best validation loss for each depth is highlighted in \textbf{bold}.
}
\renewcommand{\arraystretch}{1.3}
\vskip 0.05in
\begin{tabular}{cccccccc} 
\toprule
$L/\log_2(\eta_{\mathrm{base}})$ & -15   & -14   & -13   & -12   & -11   & -10   & -9      \\ 
\hline
4         & 5.969 & 5.609 & \textbf{5.234} & 5.329 & 5.356 & 5.938 & 6.264  \\
8         & 5.966 & 5.612 & 5.231 & \textbf{4.936} & 5.212 & 5.515 & 6.252  \\
16        & 5.99  & 5.67  & 5.302 & 5.052 & \textbf{4.939} & 5.438 & 6.061  \\
32        & 6.028 & 5.73  & 5.402 & 5.095 & 5.026 & \textbf{5.014} & 5.97   \\
64        & 6.051 & 5.785 & 5.515 & 5.231 & \textbf{5.072} & 5.077 & 5.218  \\
128       & 6.066 & 5.83  & 5.609 & 5.39  & 5.252 & \textbf{5.118} & 5.359  \\
256       & 6.078 & 5.845 & 5.663 & 5.5   & 5.434 & \textbf{5.403} & 6.009   \\
\bottomrule
\end{tabular}
\label{tab: depth dmup sophia}
\end{table}

\clearpage
\newpage

\section{Justification of Upper Bound Estimation}
\label{app: lower bound}

In the derivation in Section~\ref{sec: spec condition}, we rely on the assumption that the subadditivity and submultiplicativity inequalities used throughout the analysis are tight under standard neural network initialization and training dynamics.
Under this assumption, controlling the upper bounds of $\|\vh_l(\vx)\|_{\normrms}$ and $\|\Delta\vh_l(\vx)\|_{\normrms}$ is sufficient to characterize the typical scaling behavior of $\|\vh_l(\vx)\|_{\normrms}$ and $\|\Delta\vh_l(\vx)\|_{\normrms}$ themselves, up to constant factors.
In this section, we provide a more concrete justification for the validity of this assumption.

\subsection{Subadditivity Inequalities}
\label{app: Subadditivity Inequalities}

{
Subadditivity inequalities are used in the derivation of the update conditions to control the norm of the accumulated feature update. To state the corresponding $\Theta$ estimates, we make the required coherent-accumulation condition explicit.

\begin{assumption}[Update Coherent Accumulation]
\label{ass:update_coherent_accumulation}
For each sum of layer-wise or update-order contributions $\{\vu_i\}_{i\in I}$ appearing in the one-step feature update $\Vert\Delta\vh_l(\vx)\Vert_{\normrms}$ derivation, such as the hidden-layer expansion in Equation~(\ref{eq: hidden update}), we assume
\begin{align*}
\Big\|\sum_{i\in I}\vu_i\Big\|_{\normrms}
=\Theta\Big(\sum_{i\in I}\|\vu_i\|_{\normrms}\Big),
\end{align*}
with constants independent of width $n$ and depth $L$. In other words, the feature-update contributions accumulate without leading-order cancellation.
\end{assumption}
}

{
We justify the rationality of Assumption~\ref{ass:update_coherent_accumulation} as follows.
}
For instance, by decomposing $\Delta \vh_s(\vx)$ into several layer-wise contributions, we obtain
\begin{align}
\Delta\vh_s(\vx) 
&= \Delta\vh_0(\vx) + \underbrace{\sum_{l=1}^s \alpha_l \mW_l^{(2)}\mW_l^{(1)}\Delta\vh_{l-1}(\vx)}_{\vepsilon_0(s)} + \underbrace{\sum_{l=1}^s \alpha_l \mW_l^{(2)}\Delta\mW_l^{(1)}(\vh_{l-1}(\vx)+\Delta\vh_{l-1}(\vx))}_{\vepsilon_1^{(1)}(s)} \nonumber \\
&+ \underbrace{\sum_{l=1}^s \alpha_l \Delta\mW_l^{(2)}\mW_l^{(1)}(\vh_{l-1}(\vx)+\Delta\vh_{l-1}(\vx))}_{\vepsilon_1^{(2)}(s)} + \underbrace{\sum_{l=1}^s \alpha_l \Delta\mW_l^{(2)}\Delta\mW_l^{(1)}(\vh_{l-1}(\vx)+\Delta\vh_{l-1}(\vx))}_{\vepsilon_2(s)}, \label{eq: hidden update}
\end{align}
which leads to the following upper bound:
\begin{align*}
\Vert\Delta\vh_L(\vx)\Vert_{\normrms}
\leq \Vert\Delta\vh_0(\vx)\Vert_{\normrms}
+ \Vert\vepsilon_0(L)\Vert_{\normrms}
+ \Vert\vepsilon_1^{(1)}(L)\Vert_{\normrms}
+ \Vert\vepsilon_1^{(2)}(L)\Vert_{\normrms}
+ \Vert\vepsilon_2(L)\Vert_{\normrms}.
\end{align*}
A similar subadditivity argument is further applied to each term, e.g.,
\begin{align*}
\|\vepsilon_1^{(1)}(L)\|_{\normrms}
\leq
\sum_{l=1}^L \alpha_l
\big\|
\mW_l^{(2)}\Delta\mW_l^{(1)}\big(\vh_{l-1}(\vx)+\Delta\vh_{l-1}(\vx)\big)
\big\|_{\normrms}.
\end{align*}

In principle, such subadditivity bounds may be loose when the summands point in largely different or canceling directions.
However, due to the \emph{chain rule in backpropagation, the parameter updates $\{\Delta \mW_l\}_{l=1}^L$ across different layers are strongly correlated} (e.g., see discussion in~\citet{completep,TP-6}).
More precisely, each $\Delta \mW_l$ is proportional to the product of a forward feature $\vh_{l-1}(\vx)$ and a backpropagated error signal, which itself is obtained by repeatedly multiplying upstream Jacobians.
As a result, the layerwise update contributions to $\Delta\vh_L(\vx)$ share similar directions in feature space rather than behaving as independent or adversarial vectors.

Consequently, the terms appearing in the sums defining $\vepsilon_0(L)$, $\vepsilon_1^{(1)}(L)$, and $\vepsilon_1^{(2)}(L)$ tend to be positively aligned, and cancellations between different layers are atypical~\citep{completep,TP-6}.
In this regime, the norm of the sum scales proportionally to the sum of the norms, implying that {{Assumption~\ref{ass:update_coherent_accumulation} is reasonable}} and the subadditivity inequality provides an accurate characterization of the magnitude of $\Delta\vh_L(\vx)$ up to constant factors.
Therefore, under standard training dynamics, controlling the subadditive upper bounds suffices to capture the typical scaling behavior of the feature updates.

\subsection{Submultiplicativity Inequalities}
\label{app: Submultiplicativity Inequalities}

Submultiplicativity inequalities are extensively used in the analysis of both the initial condition and the update condition.
{We discuss the two sources of matching lower bounds separately. At initialization, they follow from independence and concentration for the input and comparable-width hidden layers. During training, update-feature alignment supplies the product lower bounds, with one representative gradient-descent term verified explicitly below. The exception is the output layer at initialization, where we retain only an upper bound. Our reasoning is closely aligned with that of~\citep{mup-spectral}, which employs a similar perspective in deriving spectral conditions for width scaling.}

\subsubsection{{Initialization} Condition}

In the derivation of the initialization conditions, submultiplicativity inequalities are applied to the input, hidden, and output layers.
{For the input layer, $\mW_0$ is independent of the input $\vx$; conditioning on $\vx$ and applying Gaussian concentration therefore supplies the matching lower bound for $\Vert\vh_0(\vx)\Vert_\normrms
\leq \alpha_0\Vert\mW_0\Vert_\normrms \Vert\vx\Vert_\normrms$, as in Claim~1 of~\citet{mup-spectral}. For a hidden block, condition on the weights and features generated by preceding blocks. The randomly initialized $\mW_l^{(1)}$ is then independent of $\vh_{l-1}(\vx)$, and $\mW_l^{(2)}$ is independent of both; successive concentration gives the lower bound formalized in Claim~\ref{claim:init}. The comparable-dimension condition is essential because each residual block contains two or more stacked transformations} (e.g., we used $\|\alpha_l \mW_l^{(2)} \mW_l^{(1)} \vh_{l-1}(\vx)\|_{\normrms} \leq \alpha_l
\|\mW_l^{(2)}\|_{\normrms}
\|\mW_l^{(1)}\|_{\normrms}
\|\vh_{l-1}(\vx)\|_{\normrms}$).
In what follows, we therefore focus on establishing the tightness of the submultiplicativity bounds for these multi-layer residual blocks.

\begin{claim}[{Tightness under comparable hidden-layer dimensions}]
\label{claim:init}
Fix a feature vector $\vh_{l-1}(\vx)\in\R^{n}$.
Recall that $\mW_l^{(1)} \in \R^{n_l\times n}, \mW_l^{(2)} \in \R^{n \times n_l}$ are initialized with 
$\left({\mW}_l^{(1)}\right)_{ij}, \left({\mW}_l^{(2)}\right)_{ij} \overset{\mathrm{i.i.d.}}{\sim} \mathcal{N}(0,\sigma_l^2)$. 
Provided that $n_l = \Theta(n)$, then with high probability:
\begin{equation*}
    \| \mW_l^{(2)} \mW_l^{(1)} \vh_{l-1}(\vx)\|_\normrms = \Theta \Big(\|\mW_l^{(2)}\|_\normrms \|\mW_l^{(1)}\|_\normrms  \|\vh_{l-1}(\vx)\|_\normrms \Big),
\end{equation*}
{which means that the submultiplicativity inequality used for the hidden residual blocks at initialization is tight.}
\end{claim}

\begin{proof}
We first consider the intermediate feature $\vz_l := \mW_l^{(1)} \vh_{l-1}(\vx) \in \R^{n_l}$. 
Since $\mW_l^{(1)}$ has i.i.d. Gaussian entries with zero mean and variance $\sigma_l^2$, by the law of large numbers, we have
\begin{align*}
\|\vz_l\|_\normrms = \|\mW_l^{(1)} \vh_{l-1}(\vx)\|_\normrms \approx \sigma_l \sqrt{n} \|\vh_{l-1}(\vx)\|_\normrms.
\end{align*}

In the meanwhile, by the standard concentration inequalities for random matrices~\citep{hdp}, we have, with high probability that $\|\mW_l^{(1)}\|_\normrms = \sqrt{n/n_l} \cdot \Theta(\sigma_l (\sqrt{n} + \sqrt{n_l})) = \Theta(\sigma_l \sqrt{n})$. Therefore, we obtain
\begin{align}
\|\vz_l\|_\normrms = \|\mW_l^{(1)} \vh_{l-1}(\vx)\|_\normrms = \Theta(\|\mW_l^{(1)} \|_\normrms \|\vh_{l-1}(\vx)\|_\normrms).
\label{eq:w1_align}
\end{align}

Next, we apply $\mW_l^{(2)}$ to $\vz_l$. Again, $\mW_l^{(2)}$ is an i.i.d. Gaussian matrix with variance $\sigma_l^2$, so by the law of large numbers, we have
\[
\|\mW_l^{(2)} \vz_l\|_\normrms \approx \sigma_l \sqrt{n_l} \|\vz_l\|_\normrms.
\]
{Likewise, standard concentration inequalities for random matrices~\citep{hdp} give, with high probability,
$\|\mW_l^{(2)}\|_\normrms = \sqrt{n_l/n} \cdot \Theta(\sigma_l (\sqrt{n} + \sqrt{n_l})) = \Theta(\sigma_l \sqrt{n_l})$.} Therefore, we obtain
\begin{align*}
\| \mW_l^{(2)} \mW_l^{(1)} \vh_{l-1}(\vx)\|_\normrms
=
\|\mW_l^{(2)} \vz_l\|_\normrms 
=
\Theta(\|\mW_l^{(2)}\|_\normrms \|\vz_l\|_\normrms)
= \Theta(\|\mW_l^{(2)} \|_\normrms \|\mW_l^{(1)} \|_\normrms \|\vh_{l-1}(\vx)\|_\normrms),
\end{align*}
which shows that the submultiplicativity $\|\alpha_l \mW_l^{(2)} \mW_l^{(1)} \vh_{l-1}(\vx)\|_{\normrms} \leq \alpha_l
\|\mW_l^{(2)}\|_{\normrms}
\|\mW_l^{(1)}\|_{\normrms}
\|\vh_{l-1}(\vx)\|_{\normrms}$ used to derive the initial condition is tight.
\end{proof}

{
The tightness result in Claim~\ref{claim:init} relies on the comparable-dimension condition $n_l=\Theta(n)$. This condition fails for the output layer because $\mW_{L+1}\in\R^{d_{L+1}\times n}$ with $d_{L+1}=\Theta(1)$ while $n\to\infty$. Under the Gaussian initialization in Section~\ref{sec: two-layer setup}, $\mW_{L+1}$ is independent of $\vh_L(\vx)$, and the typical scales satisfy
\begin{align*}
& \|\mW_{L+1}\|_{\normrms} = \sqrt{n/d_{L+1}} \cdot \Theta(\sigma_{L+1} (\sqrt{n} + \sqrt{d_{L+1}}))
= \Theta(\sigma_{L+1}n),\\
& \|\mW_{L+1}\vh_L(\vx)\|_{\normrms}
= \Theta(\sigma_{L+1}\sqrt{n}\,\|\vh_L(\vx)\|_{\normrms}) = \Theta\left(n^{-1/2}\|\mW_{L+1}\|_{\normrms}\|\vh_L(\vx)\|_{\normrms}\right).
\end{align*}
Thus, the submultiplicative operator-norm bound is loose by a factor of order $\sqrt{n}$ at output-layer initialization, and we retain only
$\|\vh_{L+1}(\vx)\|_{\normrms}=\mathcal{O}(\alpha_{L+1}\|\mW_{L+1}\|_{\normrms}\|\vh_L(\vx)\|_{\normrms})$.
Nevertheless, it does not affect the fact that $\alpha_{L+1}\|W_{L+1}\|_{\mathrm R}=\Theta(1)$ is a sufficient condition for realizing the $\mu$P principle, since~(\ref{principle: stable}) constrains the features only for $l\in[L]$ and does not require the initial output $h_{L+1}$ to be $\Theta(1)$.}

\subsubsection{Update Condition}

We now justify the use of submultiplicativity inequalities in the update regime and argue that the resulting upper bounds on $\|\Delta\vh_L(\vx)\|$ are tight in terms of scaling.
{
The matching lower bounds require update-feature alignment. This alignment reflects the structure of gradient-based neural-network training: by the chain rule, a raw gradient matrix is formed from a forward feature and a backpropagated error signal, rather than being oriented independently of the feature on which it acts.}

For the input and output layers, the analysis is identical to that in the width-scaling regime.
As a consequence, the tightness of the corresponding bounds follows directly from Claim~2 in~\citet{mup-spectral}.
In contrast, the hidden layers in our setting require additional justification, as each residual block consists of multiple stacked linear transformations and gives rise to a more involved update structure due to the presence of residual connections.
{
The following claim verifies the same structure for hidden layers.
}

\begin{claim}[{Update-feature alignment}]
\label{claim:update}
{For any $l\in[L]$, if $\Delta\mW_l^{(2)}$ is given by gradient descent with batch size 1, then}
    \begin{equation*}
        \| \Delta\mW_{l}^{(2)} \mW_{l}^{(1)} \vh_{l-1}(\vx)\|_\normrms = \Theta\left(\| \Delta\mW_{l}^{(2)}\|_\normrms \|\mW_{l}^{(1)}\|_\normrms \|{\vh_{l-1}(\vx)}\|_\normrms\right).
    \end{equation*}
\end{claim}
\begin{proof}
By the chain rule, we can write $\Delta\mW_{l}^{(2)}$ as
\begin{align*}
\Delta\mW_{l}^{(2)} = -\eta_l^{(2)}{\alpha_l} \nabla_{\vh_l(\vx)}\mathcal{L} \cdot (\mW_{l}^{(1)}\vh_{l-1}(\vx))^\top,
\end{align*}
which is rank-one and aligns with the incoming feature. Therefore, we have
\begin{align*}
\| \Delta\mW_{l}^{(2)} \mW_{l}^{(1)} \vh_{l-1}(\vx)\|_\normrms &= 
\| \eta_l^{(2)}{\alpha_l} \nabla_{\vh_l(\vx)}\mathcal{L} \cdot (\mW_{l}^{(1)}\vh_{l-1}(\vx))^\top \mW_{l}^{(1)} \vh_{l-1}(\vx)\|_\normrms \\
&=\eta_l^{(2)}{\alpha_l} \|\nabla_{\vh_l(\vx)}\mathcal{L}\|_\normrms
\|\mW_{l}^{(1)}\vh_{l-1}(\vx)\|_2^2 \\
&=\sqrt{\frac{n_l}{n}} \cdot \eta_l^{(2)}{\alpha_l} \sqrt{n}\|\nabla_{\vh_l(\vx)}\mathcal{L}\|_\normrms \|\mW_{l}^{(1)}\vh_{l-1}(\vx)\|_2 \cdot
\|\mW_{l}^{(1)}\vh_{l-1}(\vx)\|_\normrms \\
&=\sqrt{\frac{n_l}{n}} \cdot \eta_l^{(2)}{\alpha_l} \|\nabla_{\vh_l(\vx)}\mathcal{L}\|_2 \|\mW_{l}^{(1)}\vh_{l-1}(\vx)\|_2 \cdot
\|\mW_{l}^{(1)}\vh_{l-1}(\vx)\|_\normrms \\
&=\sqrt{\frac{n_l}{n}} \cdot \|\Delta\mW_{l}^{(2)}\|_2 \cdot
\|\mW_{l}^{(1)}\vh_{l-1}(\vx)\|_\normrms \\
&=\|\Delta\mW_{l}^{(2)}\|_\normrms \cdot
\|\mW_{l}^{(1)}\vh_{l-1}(\vx)\|_\normrms.
\end{align*}
Furthermore, by the initial alignment $\|\mW_l^{(1)} \vh_{l-1}(\vx)\|_\normrms = \Theta(\|\mW_l^{(1)} \|_\normrms \|\vh_{l-1}(\vx)\|_\normrms)$ in Equation~(\ref{eq:w1_align}), we obtain
\begin{equation*}
        \| \Delta\mW_{l}^{(2)} \mW_{l}^{(1)} \vh_{l-1}(\vx)\|_\normrms = \Theta\left(\| \Delta\mW_{l}^{(2)}\|_\normrms \|\mW_{l}^{(1)}\|_\normrms \|{\vh_{l-1}(\vx)}\|_\normrms\right),
\end{equation*}
which completes the proof.

\end{proof}

Based on Claim~\ref{claim:update} and Assumption~\ref{ass:update_coherent_accumulation}, we demonstrate how the tightness of such submultiplicativity inequality directly leads to a tight upper bound on $\|\Delta\vh_L(\vx)\|_\normrms$ in terms of scaling.
In particular, the claim ensures that the norms of the layerwise update contributions are accurately captured by their submultiplicative estimates, so that summing these bounds yields an upper bound that faithfully reflects the true magnitude of the accumulated feature update. We can rewrite the expression of the hidden layer update in Equation~(\ref{eq: hidden update}) as
\begin{align*}
\Delta\vh_L(\vx) =
\sum_{l=1}^L \alpha_l \Delta\mW_l^{(2)}\mW_l^{(1)}\vh_{l-1}(\vx) + \cdots
\end{align*}
{
Using Assumption~\ref{ass:update_coherent_accumulation}, we have
\begin{align*}
\|\Delta\vh_L(\vx)\|_\normrms 
&= \Omega\left(\|\sum_{l=1}^L \alpha_l \Delta\mW_l^{(2)}\mW_l^{(1)}\vh_{l-1}(\vx)\|_\normrms\right) \\
&=
\Omega\left(\sum_{l=1}^L \alpha_l  \| \Delta\mW_l^{(2)}\mW_l^{(1)}\vh_{l-1}(\vx)\|_\normrms\right) \\
&= \Omega\left(\sum_{l=1}^L \alpha_l  \| \Delta\mW_{l}^{(2)}\|_\normrms \|\mW_{l}^{(1)}\|_\normrms \|{\vh_{l-1}(\vx)}\|_\normrms \right) \\
&= \Omega(1),
\end{align*}
where the first and second equality use Assumption~\ref{ass:update_coherent_accumulation}, and the third equality uses update-feature alignment. The subadditive and submultiplicative bounds give the matching $\mathcal{O}(1)$ scale. Thus the $\|\Delta\vh_L(\vx)\|_\normrms=\Theta(1)$ update estimate used in Section~\ref{sec: spec condition} is valid under the stated assumption.
}

\section{Extension to General Training Settings}
\label{app: Extension to General Training Settings}

As derived in the main text, our theoretical framework primarily investigates a simplified scenario: a \emph{one-step} update of a \emph{linear} residual MLP on a \emph{single} datapoint. In this section, we discuss its extension to the general practical setting: \emph{finite multi-step} updates of a \emph{non-linear} residual MLP on a \emph{finite batch} of datapoints.

This extension relies on three key assumptions, as those justified in the width-scaling literature~\citep{mup-spectral}. Below, we formally restate these assumptions and empirically verify their validity in the width-depth scaling context using the experimental setup detailed in Appendix~\ref{app:entension_experiment}.

\subsection{Assumptions for Extensions}
\label{app: Assumptions for Extensions}

\textbf{Multi-step Training.} 
In the main text, we derived a parameterization that ensures the weight matrices $\mW_l$ and their first-step updates $\Delta \mW_l$ (also, $\Vert\vh_l\Vert_\normrms$ and $\Vert\Delta\vh_l\Vert_\normrms$) scale correctly with the model size to achieve feature learning. To ensure these properties hold throughout \emph{finite multi-step} training, the updated parameters must maintain the same scaling order as the first step. This is formalized in Assumption~\ref{ass:extension_1_update}.

\begin{assumption}[{Non-cancellation update}]
\label{ass:extension_1_update}
During finite-step training, we assume the updated weights and feature vectors for any layer $l$ satisfy:
\begin{align*}
    &\Vert\mW_l + \Delta\mW_l\Vert_\normrms = \Theta\left( \Vert\mW_l\Vert_\normrms + \Vert\Delta\mW_l\Vert_\normrms \right), \\
    &\Vert\vh_l(\vx) + \Delta\vh_l(\vx)\Vert_\normrms = \Theta\left( \Vert\vh_l(\vx)\Vert_\normrms + \Vert\Delta\vh_l(\vx)\Vert_\normrms \right).
\end{align*}
\end{assumption}

{ In Assumption~\ref{ass:extension_1_update}, the upper bound of the orders of the left-hand side by the right-hand side (or say, $O(\cdot)$) is guaranteed by the subadditivity. The substantive requirement is the lower bound of the orders (or say $\Omega(\cdot)$), which rules out destructive cancellation between $\mW_l$ and $\Delta\mW_l$, and analogously between $\vh_l$ and $\Delta\vh_l$. We note that this assumption concerns cancellation in the updated weights and features, rather than whether the updates themselves have non-vanishing magnitude. As discussed in~\citet{mup-spectral}, such exact cancellation is extremely rare in practical neural network training.}
We empirically verify this {non-cancellation} assumption in Figures~\ref{figures:extension_1_weight_update} and \ref{figures:extension_1_hidden_update}, where the norm ratios remain constant across varying depths.

\begin{figure*}[t]
    \centering
    \includegraphics[width=0.95\textwidth]{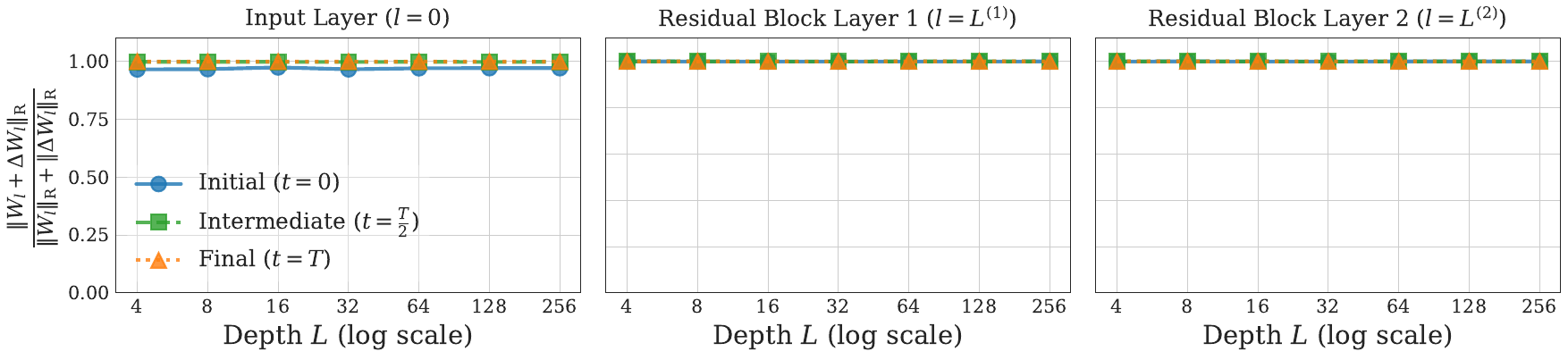}
    \caption{
    \textbf{Validation of Assumption~\ref{ass:extension_1_update} (weight update).} {During training, the ratio $\frac{\Vert\mW_l + \Delta\mW_l\Vert_\normrms}{ \Vert\mW_l\Vert_\normrms + \Vert\Delta\mW_l\Vert_\normrms}$ remains constant near 1 across depth for the input layer and residual block layers, indicating that the updates do not destructively cancel the existing weights during multi-step training.}}
    \label{figures:extension_1_weight_update}
\end{figure*}

\begin{figure*}[t]
    \centering
    \includegraphics[width=0.95\textwidth]{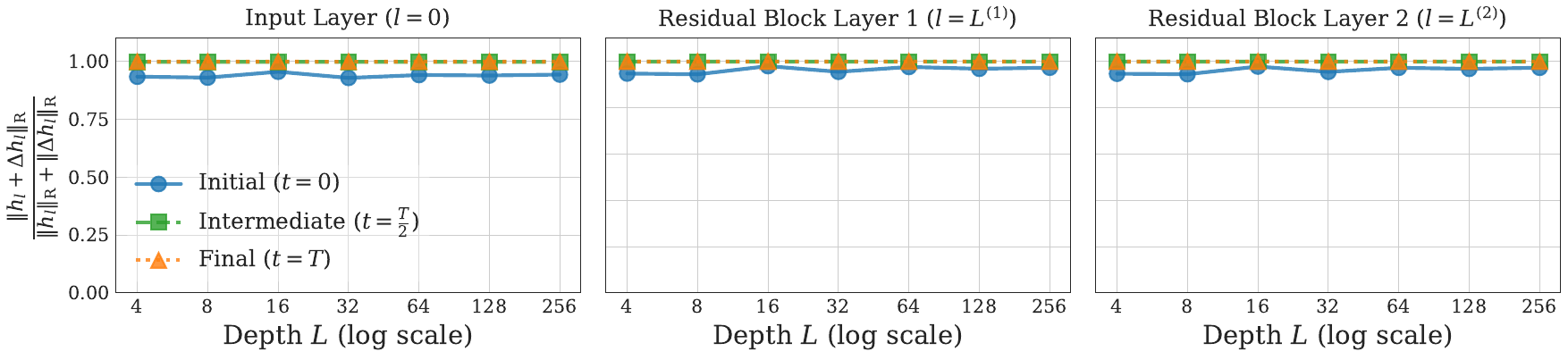}
    \caption{\textbf{Validation of Assumption~\ref{ass:extension_1_update} (feature update).} {During training, the ratio $\frac{\Vert\vh_l + \Delta\vh_l\Vert_\normrms}{  \Vert\vh_l\Vert_\normrms + \Vert\Delta\vh_l\Vert_\normrms}$ remains constant near 1 across varying depths, indicating that the feature updates do not destructively cancel the existing features during multi-step training.}}
    \label{figures:extension_1_hidden_update}
\end{figure*}

\textbf{Non-linearity.}
To extend the analysis to non-linear activations, we substitute the linear transformation $\mW_l\vh_{l-1}(\vx)$ with $\phi\left(\mW_l\vh_{l-1}(\vx)\right)$, where $\phi(\cdot)$ is an activation function (e.g., ReLU). The resulting architecture is as in Equation~(\ref{eq:network_nonlinear}) versus Equation~(\ref{eqn: resnet}) in the linear case.
We assume that the activation function preserves the asymptotic order of the feature norms, ensuring that the scaling properties derived for the pre-activations remain valid for the post-activations.

\begin{assumption}[Stable activation]
\label{ass:extension_2_stable_act}
During the training, we assume that the features before and after the nonlinear activation are of the same scale:
\begin{align*}
    & \Vert \phi(\mW_l\vh_{l-1}(\vx))\Vert_\normrms = \Theta(\Vert \mW_l\vh_{l-1}(\vx)\Vert_\normrms).
\end{align*}
\end{assumption}

Figure~\ref{figures:extension_2_activation} empirically verifies this assumption for the ReLU activation, showing that the ratio of post-activation to pre-activation norms is stable across depth.
\begin{figure*}[t]
    \centering
    \includegraphics[width=0.95\textwidth]{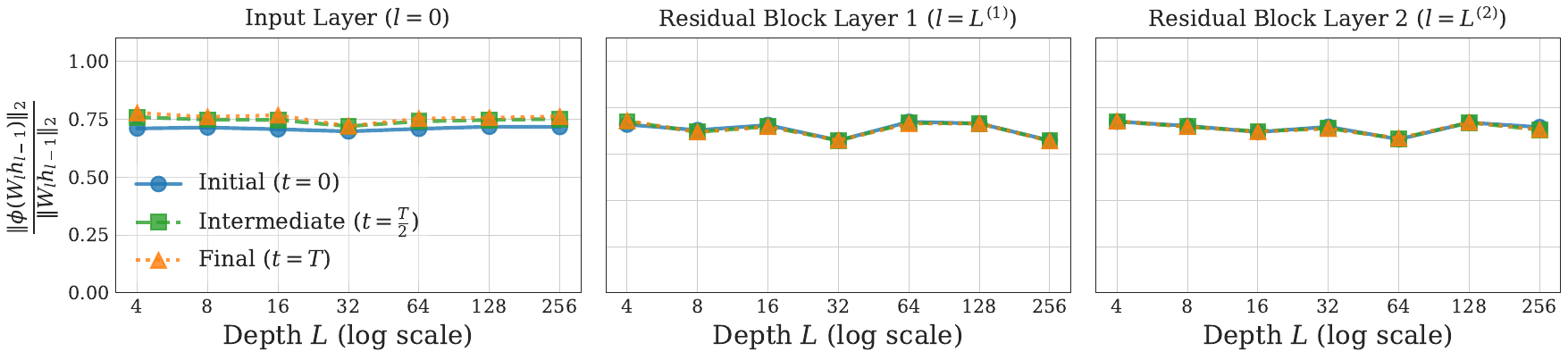}
    \caption{\textbf{Validation of Assumption~\ref{ass:extension_2_stable_act} (stable activation).} During the training, the ratio of post-activation to pre-activation norms $\frac{\Vert \phi(\mW_l\vh_{l-1})\Vert_\normrms}{\Vert \mW_l\vh_{l-1}\Vert_\normrms}$ remains stable across varying depths, confirming that the ReLU activation does not collapse the norm in non-linear networks.}
    \label{figures:extension_2_activation}
\end{figure*}

\textbf{Training with Mini-batch.}
Finally, to extend beyond the single-sample setting, we consider updates computed on a batch of data $\{\vx^{(i)}, \vy^{(i)}\}_{i=1}^B$. Let $\Delta \mW_l^{(i)}$ denote the update contribution from the $i$-th datapoint (e.g., $\Delta \mW_l^{(i)} = -\eta_l \nabla_{\mW_l}\mathcal{L}(\vx^{(i)}, \vy^{(i)})$ for SGD), such that the total batch update is $\Delta \mW_l = \frac{1}{B}\sum_{i=1}^B \Delta \mW_l^{(i)}$.
We expect that the gradient contributions from different samples do not destructively cancel out. This is formalized in Assumption~\ref{ass:extension_3_alignment}.

\begin{assumption}[Per-sample update alignment]
\label{ass:extension_3_alignment}
We assume that the batch size satisfies $B=\Theta(1)$, and the batch update norm scales consistently with the per-sample update norm during the training:
\begin{align*}
    & \Vert \Delta \mW_l\vh_{l-1}(\vx^{(i)}) \Vert_\normrms = \Theta\left( \frac{1}{B} \Vert \Delta \mW_l^{(i)}\vh_{l-1}(\vx^{(i)})\Vert_\normrms \right).
\end{align*}
\end{assumption}

We verify this in Figure~\ref{figures:extension_3_sample_alignment}, where the alignment ratio remains $\Theta(1)$, indicating that batch-averaged updates preserve the scaling properties of single-sample updates.

\begin{figure*}[t]
    \centering
    \includegraphics[width=0.95\textwidth]{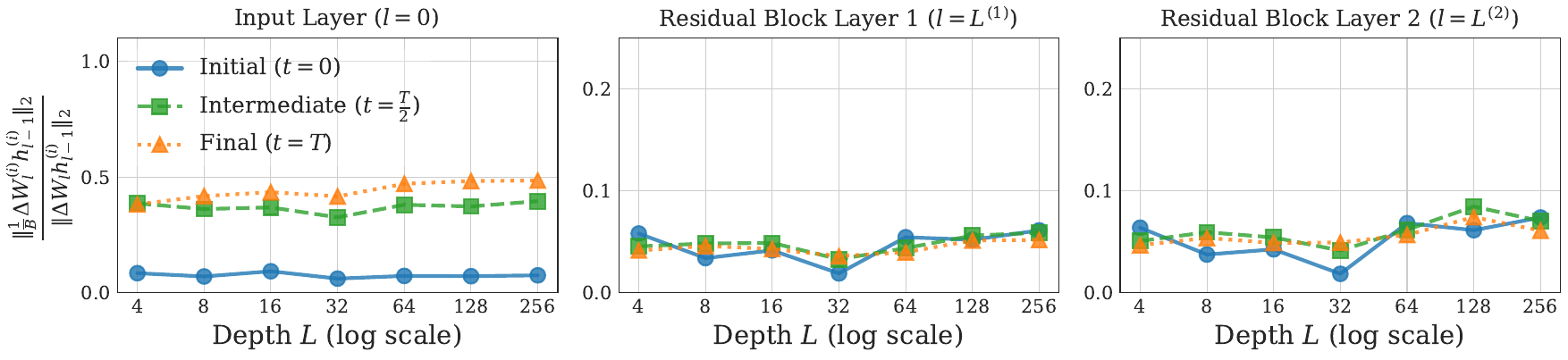}
    \caption{\textbf{Validation of Assumption~\ref{ass:extension_3_alignment} (per-sample update alignment).} We report the averaged ratio of $\frac{\Vert \frac{1}{B}\Delta \mW_l^{(i)} \vh_{l-1}^{(i)} \Vert_\normrms}{\Vert \Delta \mW_l \vh_{l-1}^{(i)} \Vert_\normrms}$ across the batch of data. The values remain stable across varying depths, suggesting that the batch update does not alter the depth-wise scaling of the single-sample update.}
    \label{figures:extension_3_sample_alignment}
\end{figure*}

\subsection{Experimental Details}
\label{app:entension_experiment}

We conduct simulations to empirically {assess} Assumptions~\ref{ass:extension_1_update}-\ref{ass:extension_3_alignment}.
Our experimental setup largely follows~\citet{mup-spectral}, with a different emphasis on depth scaling instead of width scaling.
Details are provided below.

\textbf{Dataset.}
We construct a binary classification dataset using a subset of CIFAR-10, selecting 100 samples each from the ``airplane'' and ``automobile'' classes. The inputs are flattened image vectors in $\mathbb{R}^{3072}$ associated with binary labels in $\{0, 1\}$. 

\textbf{Architecture and Training.}
The architecture is a deep residual MLP with ReLU activations, consisting of an input layer, $L$ residual blocks, and a final linear output layer:
\begin{equation}
\begin{aligned}
& \vh_0(\vx) = \alpha_0\phi\left(\mW_0\vx\right), \\
& \vh_l(\vx) = \vh_{l-1}(\vx) + \alpha_l\phi\left(\mW_l^{(2)}\phi\left(\mW_l^{(1)}\vh_{l-1}(\vx)\right)\right), \quad l \in [L], \\
&\vh_{L+1}(\vx) = \alpha_{L+1}\mW_{L+1}\vh_L(\vx),
\end{aligned}
\label{eq:network_nonlinear}
\end{equation}
where $\phi$ denotes the ReLU activation. The dimensions are set as: input dimension $d_0 = 3072$, model width $n = 256$, residual block width $n_l = n$, and output dimension $d_{L+1} = 1$.
This aligns well with the simplified setup discussed in the main text (Section~\ref{sec: two-layer setup}).
Models are trained to minimize the binary cross-entropy loss using full-batch Gradient Descent (GD) for $T=200$ steps.

\textbf{Parameterization.}
We implement the width-depth $\mu$P parametrization for SGD derived in Table~\ref{tab: sgd-mup} of Appendix~\ref{app:sgd_mup} as follows:
\begin{align*}
    \alpha_0 &= \alpha_{\mathrm{base}}, \quad \alpha_l = \frac{\alpha_{\mathrm{base}}}{L}, \quad \alpha_{L+1} = \frac{\alpha_{\mathrm{base}}}{n}, \\
    \sigma^2_0 &= \frac{\sigma^2_{\mathrm{base}}}{d_0}, \quad \sigma^2_l = \frac{\sigma^2_{\mathrm{base}}}{n}, \quad \sigma^2_{L+1} = \sigma^2_{\mathrm{base}},\\
    \eta_0 &= \eta_{\mathrm{base}}n, \quad \eta_l = \eta_{\mathrm{base}}{L}, \quad \eta_{L+1} = \eta_{\mathrm{base}}n,
\end{align*}
with base constants set to:
\begin{align*}
    \alpha_{\mathrm{base}} = 1, \quad \sigma^2_{\mathrm{base}} = 2, \quad \eta_{\mathrm{base}} = 0.001.
\end{align*}

\textbf{Verification of Assumptions.}
We perform a depth scaling analysis by training networks with depths $L \in \{4, 8, 16, 32, 64, 128, 256\}$. We track the metrics corresponding to the assumptions above at three distinct training phases: initialization ($t=0$), intermediate training ($t=T/2$), and the end of training ($t=T$), and different layers: the input layer ($l=0$) and the internal layers of the final residual block (here, we denote them by $l=L^{(1)}$ and $l=L^{(2)}$), as representatives.

\end{appendices}

\end{document}